\documentclass[11pt,a4paper,british,english,australian]{book}
\usepackage[T1]{fontenc}
\usepackage{float}
\usepackage{calc}
\usepackage{mathtools}
\usepackage{url}
\usepackage{amsthm}
\usepackage{amsmath}
\usepackage{amssymb}
\usepackage{makeidx}
\makeindex
\usepackage[pdftex]{graphicx}
\usepackage{esint}
\usepackage[authoryear]{natbib}

\makeatletter

\pdfpageheight\paperheight
\pdfpagewidth\paperwidth

\providecommand{\tabularnewline}{\\}
\floatstyle{ruled}
\newfloat{algorithm}{tbp}{loa}[chapter]
\providecommand{\algorithmname}{Algorithm}
\floatname{algorithm}{\protect\algorithmname}

\theoremstyle{plain}
\ifx\thechapter\undefined
\newtheorem{thm}{\protect\theoremname}
\else
\newtheorem{thm}{\protect\theoremname}[chapter]
\fi
\theoremstyle{definition}
\newtheorem{defn}[thm]{\protect\definitionname}
\ifx\proof\undefined
\newenvironment{proof}[1][\protect\proofname]{\par
\normalfont\topsep6\p@\@plus6\p@\relax
\trivlist
\itemindent\parindent
\item[\hskip\labelsep\scshape #1]\ignorespaces
}{%
\endtrivlist\@endpefalse
}
\providecommand{\proofname}{Proof}
\fi
\theoremstyle{remark}
\newtheorem{rem}[thm]{\protect\remarkname}
\theoremstyle{plain}
\newtheorem{lem}[thm]{\protect\lemmaname}
\theoremstyle{plain}
\newtheorem{cor}[thm]{\protect\corollaryname}
\theoremstyle{plain}
\newtheorem{prop}[thm]{\protect\propositionname}

\usepackage{svn-multi}
\svnid{$Id$}
\usepackage[hyperindex=true,
			bookmarks=true,
            pdftitle={New Optimization Methods for Machine Learning}, 
            pdfauthor={Aaron Defazio},
            colorlinks=false,
            pdfborder=0,
            pagebackref=false,
            citecolor=blue,
            plainpages=false,
            pdfpagelabels,
            pagebackref=true,
            hyperfootnotes=false]{hyperref}
\usepackage[all]{hypcap}
\usepackage[palatino]{anuthesis}
\usepackage{afterpage}
\usepackage{graphicx}
\usepackage{thesis}
\usepackage[normalem]{ulem}
\usepackage[table]{xcolor}
\usepackage{makeidx}
\usepackage{cleveref}
\usepackage[centerlast]{caption}
\usepackage{float}
\usepackage[T1]{fontenc}
\usepackage[scaled]{beramono}
\usepackage{pifont}
\usepackage{rotating}
\usepackage{algorithmic}
\usepackage{thmtools}
\usepackage{thm-restate}
\usepackage{adjustbox}

\usepackage{multirow}

\usepackage{booktabs}
\usepackage{relsize}
\usepackage{xspace}
\usepackage{subfig}
\usepackage{listings}

\renewcommand*{\backref}[1]{}
\renewcommand*{\backrefalt}[4]{
  \ifcase #1 %
  \or
    (cited on page #2)%
  \else
    (cited on pages #2)%
  \fi
}

\newcommand{\cmark}{\ding{51}}
\newcommand{\xmark}{\ding{55}}

\title{New Optimisation Methods for Machine Learning}
\author{Aaron Defazio}
\date{November 2014}

\renewcommand{\thepage}{\roman{page}}

\@ifundefined{showcaptionsetup}{}{%
 \PassOptionsToPackage{caption=false}{subfig}}
\usepackage{subfig}
\makeatother

\usepackage{babel}
\addto\captionsaustralian{\renewcommand{\corollaryname}{Corollary}}
\addto\captionsaustralian{\renewcommand{\definitionname}{Definition}}
\addto\captionsaustralian{\renewcommand{\lemmaname}{Lemma}}
\addto\captionsaustralian{\renewcommand{\propositionname}{Proposition}}
\addto\captionsaustralian{\renewcommand{\remarkname}{Remark}}
\addto\captionsaustralian{\renewcommand{\theoremname}{Theorem}}
\addto\captionsbritish{\renewcommand{\algorithmname}{Algorithm}}
\addto\captionsbritish{\renewcommand{\corollaryname}{Corollary}}
\addto\captionsbritish{\renewcommand{\definitionname}{Definition}}
\addto\captionsbritish{\renewcommand{\lemmaname}{Lemma}}
\addto\captionsbritish{\renewcommand{\propositionname}{Proposition}}
\addto\captionsbritish{\renewcommand{\remarkname}{Remark}}
\addto\captionsbritish{\renewcommand{\theoremname}{Theorem}}
\addto\captionsenglish{\renewcommand{\algorithmname}{Algorithm}}
\addto\captionsenglish{\renewcommand{\corollaryname}{Corollary}}
\addto\captionsenglish{\renewcommand{\definitionname}{Definition}}
\addto\captionsenglish{\renewcommand{\lemmaname}{Lemma}}
\addto\captionsenglish{\renewcommand{\propositionname}{Proposition}}
\addto\captionsenglish{\renewcommand{\remarkname}{Remark}}
\addto\captionsenglish{\renewcommand{\theoremname}{Theorem}}
\providecommand{\corollaryname}{Corollary}
\providecommand{\definitionname}{Definition}
\providecommand{\lemmaname}{Lemma}
\providecommand{\propositionname}{Proposition}
\providecommand{\remarkname}{Remark}
\providecommand{\theoremname}{Theorem}

\begin{document}

\begin{titlepage}
  \enlargethispage{2cm}
  \begin{center}
    \makeatletter
    \Huge\textbf{\@title} \\[.4cm]
    \Huge\textbf{\thesisqualifier} \\[2.5cm]
    \huge\textbf{\@author} \\[9cm]
    \makeatother
    \LARGE A thesis submitted for the degree of \\
    Doctor of Philosophy\\
    of The Australian National University \\[2cm]
    November 2014
  \end{center}
\end{titlepage}
\selectlanguage{english}%
\global\long\def\n#1{\left\Vert #1\right\Vert }

\global\long\def\ns#1{\left\Vert #1\right\Vert ^{2}}

\global\long\def\r{\mathbb{R}}

\global\long\def\e{\mathbb{E}}

\global\long\def\ip#1#2{\left\langle #1,#2\right\rangle }
\selectlanguage{australian}%

\include{frontmatter}
\chapter*{Acknowledgements}

I would like to thank several NICTA researchers for conversations
and brainstorming sessions during the course of my PhD, particularly
Scott Sanner and my supervisor Tiberio Caetano.

I would like to thank Justin Domke for many discussions about the
Finito algorithm, and his assistance with developing and checking
the proof. Likewise, for the SAGA algorithm I would like to thank
Francis Bach and Simon Lacoste-Julien for discussion and assistance
with the proofs. The SAGA algorithm was discovered in collaboration
with them while visiting the INRIA lab, with some financial support
from INRIA.

I would also like to thank my family for all their support during
the course of my PhD. Particularly my mother for giving me a place
to stay for part of the duration of the PhD as well as food, love
and support. I do not thank her often enough.

I also would like to thank NICTA for their scholarship during the
course of the PhD. NICTA is funded by the Australian Government through
the Department of Communications and the Australian Research Council
through the ICT Centre of Excellence Program.

\chapter*{Abstract}

In this work we introduce several new optimisation methods for problems
in machine learning. Our algorithms broadly fall into two categories:
optimisation of finite sums and of graph structured objectives. The
finite sum problem is simply the minimisation of objective functions
that are naturally expressed as a summation over a large number of
terms, where each term has a similar or identical weight. Such objectives
most often appear in machine learning in the empirical risk minimisation
framework in the non-online learning setting. The second category,
that of graph structured objectives, consists of objectives that result
from applying maximum likelihood to Markov random field models. Unlike
the finite sum case, all the non-linearity is contained within a \emph{partition
function} term, which does not readily decompose into a summation.

For the finite sum problem, we introduce the Finito and SAGA algorithms,
as well as variants of each. The Finito algorithm is best suited to
strongly convex problems where the number of terms is of the same
order as the condition number of the problem. We prove the fast convergence
rate of Finito for strongly convex problems and demonstrate its state-of-the-art
empirical performance on 5 datasets. 

The SAGA algorithm we introduce is complementary to the Finito algorithm.
It is more generally applicable, as it can be applied to problems
without strong convexity, and to problems that have a non-differentiable
regularisation term. In both cases we establish strong convergence
rate proofs. It is also better suited to sparser problems than Finito.
The SAGA method has a broader and simpler theory than any existing
fast method for the problem class of finite sums, in particular it
is the first such method that can provably be applied to non-strongly
convex problems with non-differentiable regularisers without introduction
of additional regularisation. 

For graph-structured problems, we take three complementary approaches.
We look at learning the parameters for a fixed structure, learning
the structure independently, and learning both simultaneously. Specifically,
for the combined approach, we introduce a new method for encouraging
graph structures with the ``scale-free'' property. For the structure
learning problem, we establish SHORTCUT, a $O(n^{2.5})$ expected
time approximate structure learning method for Gaussian graphical
models. For problems where the structure is known but the parameters
unknown, we introduce an approximate maximum likelihood learning algorithm
that is capable of learning a useful subclass of Gaussian graphical
models.

Our thesis as a whole introduces a new suit of techniques for machine
learning practitioners that increases the size and type of problems
that can be efficiently solved. Our work is backed by extensive theory,
including proofs of convergence for each method discussed.

\tableofcontents{}

\mainmatter

\chapter{Introduction and Overview}

\label{chap:introduction}

Numerical optimisation is in many ways the core problem in modern
machine learning. Virtually all learning problems can be tackled by
formulating a real valued objective function expressing some notation
of loss or suboptimality which can be optimised over. Indeed approaches
that don't have well founded objective functions are rare, perhaps
contrastive divergence \citep{hinton2002training} and some sampling
schemes being notable examples. 

Many methods that started as heuristics were able to be significantly
improved once well-founded objectives were discovered and exploited.
Often a convex variant can be developed. A prime example is belief
propagation, the relation to the Bethe approximation \citep{bethe},
and the later development of tree weighted variants \citep{TRW} which
allowed a convex formulation. The core of this thesis is the development
of several new numerical optimisation schemes, primarily focusing
on convex objectives, which either address limitations of existing
approaches, or improve on the performance of state-of-the-art algorithms.
These methods increase the breadth and depth of machine learning problems
that are tractable on modern computers.

\section{Convex Machine Learning Problems}

In this work we particularly focus on problems that have convex objectives.
This is a major restriction, and one at the core of much of modern
optimisation theory, but one that nevertheless requires justification.
The primary reasons for targeting convex problems is their ubiquitousness
in applications and their relative ease of solving them. Logistic
regression, least-squares, support vector machines, conditional random
fields and tree-weighted belief propagation all involve convex models.
All of these techniques have seen real world application, although
their use has been overshadowed in recent years by non-convex models
such as neural networks. 

Convex optimisation is still of interest when addressing non-convex
problems though. Many algorithms that were developed for convex problems,
motivated by their provably fast convergence have later been applied
to non-convex problems with good empirical results. Additionally,
often the best approach to solving a non-convex problem is through
the repeated solution of convex sub-problems, or by replacing the
problem entirely with a close convex surrogate.

The class of convex numerical problems is sometimes considered synonymous
with that of computationally tractable problems. This is not strictly
true in the usual compute science sense of tractability as some convex
problems on complicated but convex sets can still be NP-hard \citep{np-hard-convex}.

On the other hand, we can sometimes approximately solve non-convex
problems of massive scale using modern approaches (e.g. \citealp{google-deep}).
Instead, convex problems can be better thought of as the \emph{reliably
solvable} problems. For convex problems we can almost always establish
theoretical results giving a practical bound on the amount of computation
time required to solve a given convex problem \citep{nes-interior-point}.
For non-convex problems we can rarely do better than finding a locally
optimal solution. 

Together with the small or no tuning required by convex optimisation
algorithms, they can be used as building blocks within larger programs;
details of the problem can be abstracted away from the users. This
is rarely the case for non-convex problems, where the most commonly
used methods require substantial hand tuning.

When tasked with solving a convex problem, we have at our disposal
powerful and flexible algorithms such as interior point methods and
in particular Newton's method. While Newton's method is strikingly
successful on small problems, its approximately cubic running time
per iteration resulting from the need to do a linear solve means that
it scales extremely poorly to problems with large numbers of variables.
It is also unable to directly handle non-differentiable problems common
in machine learning. Both of these shortcomings have been addressed
to some degree \citep{lbfgs,lbfgs-large,lbfgs-l1}, by the use of
low-rank approximations and tricks for specific non-differentiable
structures, although problems remain.

An additional complication is a divergence in between the numerical
optimisation and machine learning communities. Numerical convex optimisation
researchers in the 80s and 90s largely focused on solving problems
with large numbers of complex constraints, particularly Quadratic
Programming (QP) and Linear Programming (LP) problems. These advances
were applicable to the kernel methods of the early 2000s, but at odds
with many of the more modern machine learning problems which are characterised
by large numbers of potentially non-differentiable terms. The core
examples would be linear support vector machines, other max-margin
methods and neural networks with non-differentiable activation functions.
The problem we address in Chapter \ref{chap:submodular} also fits
into this class.

In this thesis we will focus on smooth optimisation problems, allowing
only a controlled level of non-smooth structure in the form of certain
non-differentiable regularisation terms (detailed in Section \ref{sec:problem-structure-intro}).
The notion of smoothness we use is that of Lipschitz smoothness. A
function $f$ is Lipschitz smooth with constant $L$ if its gradients
are Lipschitz continuous. That is, for all $x,y\in\mathbb{R}^{d}$:
\[
\left\Vert f^{\prime}(x)-f^{\prime}(y)\right\Vert \leq L\left\Vert x-y\right\Vert .
\]

Lipschitz smooth functions are differentiable, and if their Hessian
matrix exists it is bounded in spectral norm. The other assumption
we will sometimes make is that of strong convexity. A function $f$
is strongly convex with constant $\mu$ if for all $x,y\in\mathbb{R}^{d}$
and $\alpha\in[0,1]$:
\[
f\left(\alpha x+(1-\alpha)y\right)\leq\alpha f(x)+(1-\alpha)f(y)-\alpha\left(1-\alpha\right)\frac{\mu}{2}\left\Vert x-y\right\Vert ^{2}.
\]

Essentially rather than the usual convexity interpolation bound $f\left(\alpha x+(1-\alpha)y\right)\leq\alpha f(x)+(1-\alpha)f(y)$,
we have it strengthened by a quadratic term.

\section{Problem Structure and Black Box Methods}

\label{sec:problem-structure-intro}

The last few years have seen a resurgence in convex optimisation centred
around the technique of exploiting problem structure, an approach
we take as well. When no structure is assumed by the optimisation
method about the problem other than the degree of convexity, very
strong results are known about the best possible convergence rates
obtainable. These results date back to the seminal work of \citet{nem-yudin}
and \citet[earlier work in Russian]{nes-book}. These results have
contributed to the widely held attitude that convex optimisation is
a solved problem. 

But when the problem has some sort of additional structure these worst-case
theoretical results are no longer applicable. Indeed, a series of
recent results suggest that practically all problems of interest have
such structure, allowing advances in theoretical, not just practical
convergence. For example, non-differentiable problems under reasonable
Lipschitz smoothness assumptions can be solved with an error reduction
of $O(\sqrt{t})$ times after $t$ iterations, for standard measures
of convergence rate, at best \citep[ Theorem 3.2.1]{nes-book}. In
practice, virtually all non-differentiable problems can be treated
by a smoothing transformation, giving a $O(t)$ reduction in error
after t iterations when an optimal algorithm is used \citep{nes-smoothing}.

Many problems of interest have a structure where most terms in the
objective involve only a small number of variables. This is the case
for example in inference problems on graphical models. In such cases
block coordinate descent methods can give better theoretical and practical
results \citep{random-coordinate}.

Another exploitable structure involves a sum of two terms $F(x)=f(x)+h(x)$,
where the first term $f(x)$ is structurally nice, say smooth and
differentiable, but potentially complex to evaluate, and where the
second term $h(x)$ is non-differentiable. As long as $h(x)$ is simple
in the sense that its \emph{proximal operator} is easy to evaluate,
then algorithms exist with the same theoretical convergence rate as
if $h(x)$ was not part of the objective at all ($F(x)=f(x)$) \citep{beck2009fast}.
The proximal operator is a key construction in this work, and indeed
in modern optimisation theory. It is defined for a function $h$ and
constant $\gamma$ as:
\[
\text{prox}_{\gamma}^{h}(v)=\arg\min_{x}\left\{ h(x)+\frac{\gamma}{2}\left\Vert x-v\right\Vert ^{2}\right\} .
\]

Some definitions of the proximal operator use the weighting $\frac{1}{2\gamma}$
instead of $\frac{\gamma}{2}$; we use this form throughout this work.
The proximal operator is itself an optimisation problem, and so in
general it is only useful when the function $h$ is simple. In many
cases of interest the proximal operator has a closed form solution.

The first four chapters of this work focus on quite possibly the simplest
problem structure, that of a finite summation. This occurs when there
is a large number of terms with similar structure added together or
averaged in the objective. Recent results have shown that for strongly
convex problems better convergence rates are possible under such summation
structures than is possible for black box problems \citep{sag,sdca}.
We provide three new algorithms for this problem structure, discussed
in Chapters \ref{chap:finito} and \ref{chap:saga}. We also discuss
properties of problems in the finite sum class extensively in Chapter
\ref{chap:inc-discus}.

\section{Early \& Late Stage Convergence}

\begin{figure}
\includegraphics[width=1\textwidth]{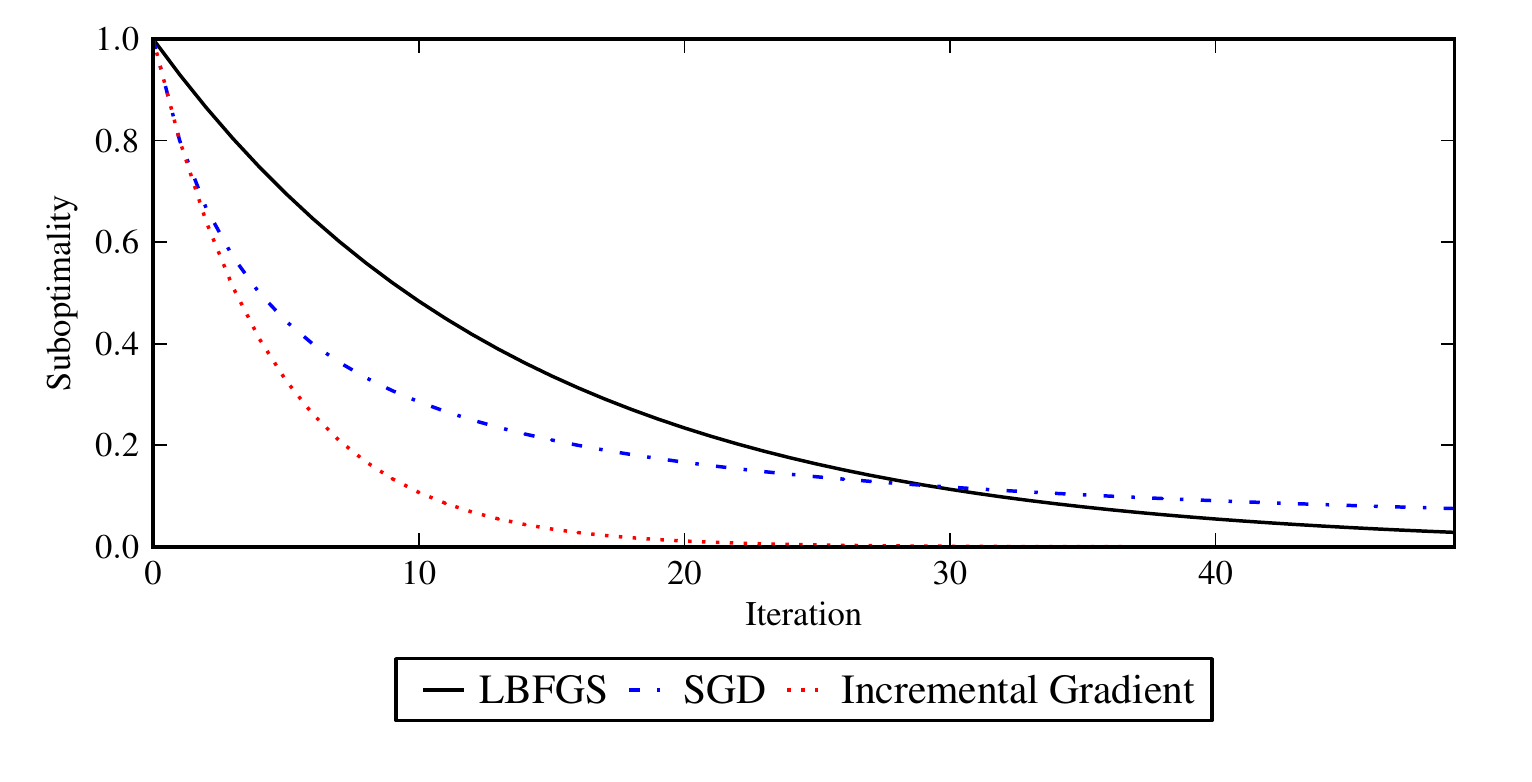}\protect\caption{\label{fig:conv-schematic}Schematic illustration of convergence rates}
\end{figure}

When dealing with problems with a finite sum structure, practitioners
have traditionally had to make a key trade-off between stochastic
methods which access the objective one term at a time, and batch methods
which work directly with the full objective. Stochastic methods such
as SGD (stochastic gradient descent, \citealp{robbins1951}) exhibit
rapid convergence during early stages of optimisation, yielding a
good approximate solution quickly, but this convergence slows down
over time; getting a high accuracy solution is nearly impossible with
SGD. Fortunately, in machine learning it is often the case that a
low accuracy solution gives just as a good a result as a high accuracy
solution for minimising the test loss on held out data. A high accuracy
solution can effectively over-fit to the training data. Running SGD
for a small number of epochs is common in practice.

Batch methods on the other hand are slowly converging but steady;
if run for long enough they yield a high accuracy solution. For strongly
convex problems, the difference in convergence is between a $O(1/t)$
error after $t$ iterations for SGD versus a $O(\rho^{t})$ error
($\rho<1$) for LBFGS\footnote{Quasi-newton methods are often cited as having local super-linear
convergence. This is only true if the dimensionality of the underlying
parameter space is comparable to the number of iterations used. In
machine learning the parameter space is usually much larger in effective
dimension than the number of iterations.}, the most popular batch method \citep{lbfgs}. We have illustrated
the difference schematically in Figure \ref{fig:conv-schematic}.
The SGD and LBFGS lines here are typical of simple prediction problems,
where SGD gives acceptable solutions after 5-10 epochs (passes over
the data), where LBFGS eventually gives a better solution, taking
30-100 iterations to do so. LBFGS is particularly well suited to use
in a distributed computing setting, and it is sometimes the case LBFGS
will give better results ultimately on the test loss, particularly
for poorly conditioned (high-curvature) problems.

Figure \ref{fig:conv-schematic} also illustrates the kind of convergence
that the recently developed class of incremental gradient methods
potentially offers. Incremental gradient methods have the same linear
$O(\rho^{t})$ error after $t$ epochs as a batch method, but with
a coefficient $\rho$ dramatically better. The difference being in
theory thousands of times faster convergence, and in practice usually
10-20 times better. With favorable problem structure incremental gradient
have the potential to offer the best of both worlds, having rapid
initial convergence without the later stage slow-down of SGD.

Another traditional advantage of batch methods over stochastic methods
is their ease of use. Methods such as LBFGS require no hand tuning
to be applied to virtually any smooth problem. Some tuning of the
memory constant that holds the number of past gradients to remember
at each step can give faster convergence, but bad choices of this
constant still result in convergence. SGD and other traditional stochastic
methods on the other hand require a step size parameter and a parameter
annealing schedule to be set. SGD is sensitive to these choices, and
will diverge for poor choices. 

Incremental gradient methods offer a solution to the tuning problem
as well. Most incremental gradient algorithms have only a single step
size parameter that needs to be set. Fortunately the convergence rate
is fairly robust to the value of this parameter. The SDCA algorithm
(\citealt{sdca}) reduces this to 0 parameters, but at the expense
of being limited to problems with efficient to compute proximal operators.

\section{Approximations}

The exploitation of problem structure is not always directly possible
with the objectives we encounter in machine learning. A case we focus
on in this work is the learning of weight parameters in a Gaussian
graphical model structure. This is an undirected graph structure with
weights associated with both edges and nodes. These weights are the
entries of the precision matrix (inverse covariance matrix) of a Gaussian
distribution. Absent edges effectively have a weight of zero (Figure
\ref{fig:intro-gaussian-illus}). A formal definition is given in
Chapter \ref{chap:background-gaussian}. A key approach to such problems
is the use of approximations that introduce additional structure in
the objective which we can exploit. 

\begin{figure}
\begin{centering}
\includegraphics[scale=1.4]{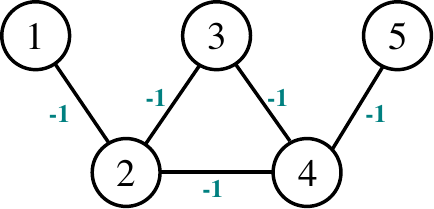}
\par\end{centering}

\smallskip{}

\begin{centering}
$C=\frac{1}{100}\left[\begin{array}{ccccc}
10 & 1 & 0.1 & 0.1 & 0.01\\
1 & 10 & 1 & 1 & 0.1\\
0.1 & 1 & 10 & 1 & 0.0\\
0.1 & 1 & 1 & 10 & 1\\
0.01 & 0.1 & 0.1 & 1 & 10
\end{array}\right]$$\qquad P=C^{-1}=\left[\begin{array}{ccccc}
10 & -1 & 0 & 0 & 0\\
-1 & 10 & -1 & -1 & 0\\
0 & -1 & 10 & -1 & 0\\
0 & -1 & -1 & 10 & -1\\
0 & 0 & 0 & -1 & 10
\end{array}\right]$
\par\end{centering}

\protect\caption{\label{fig:intro-gaussian-illus}Gaussian graphical model defined
by the precision matrix $P$, together with the non-sparse covariance
matrix $C$ it induces with rounding to 1 significant figure. Correlations
are indicated by negative edge weights in a Gaussian model.}
\end{figure}

The regularised maximum likelihood objective for fitting a Gaussian
graphical model can require time $O(n^{3})$ to evaluate\footnote{Theoretically it takes time equivalent to the big-O cost of a fast
matrix multiplication such as Strassen's algorithm ($\approx O(n^{2.8})$),
but in practice simpler $O(n^{3})$ techniques are used.}. This is prohibitively long on many problems of interest. Instead,
approximations can be introduced that decompose the objective, allowing
more efficient techniques to be used. In Chapter \ref{chap:colab}
we show how the \textit{Bethe approximation} may be applied for learning
the edge weights on restricted classes of Gaussian graphical models.
This approximation allows for the use of an efficient dual decomposition
optimisation method, and has direct practical applicability in the
domain of recommendation systems.

Besides parameter learning, the other primary task involving graphs
is directly learning the structure. Structure learning for Gaussian
graphical models is problem that has seen a lot of interest in machine
learning. The structure can be used in a machine learning pipeline
as the precursor to parameter learning, or it can be used for its
own sake as indicator of correlation structure in a dataset. The use
of approximations in structure learning is more widespread than in
parameter learning, and we give an overview of approaches in Chapter
\ref{chap:background-gaussian}. We improve on an existing technique
in Chapter \ref{chap:approx-covsel}, where we show that an existing
approximation can be \emph{further approximated}, giving a substantial
practical and theoretical speed-up by a factor of $O(\sqrt{n})$.

\section{Non-differentiability in Machine Learning}

As mentioned, machine learning problems tend to have substantial non-differentiable
structure compared to the constraint structures more commonly addressed
in numerical optimisation. These two forms of structure are in a sense
two sides of the same coin, as for convex problems the transformation
to the dual problem can often convert from one to the other. The primary
example being support vector machines, where non-differentiability
in the primal hinge loss is converted to a constraint set when the
dual is considered. 

Recent progress in optimisation has seen the use of proximal methods
as the tool of choice for handling both structures in machine learning
problems. When using a regularised loss objective of the form $F(x)=f(x)+h(x)$
as mentioned above in Section \ref{sec:problem-structure-intro},
the non-differentiability can be in the regulariser $h(x)$ or the
loss term $f(x)$. We introduce methods addressing both cases in this
work. The SAGA algorithm of Chapter \ref{chap:saga} is a new primal
method, the first primal incremental gradient method able to be used
on non-strongly convex problems with non-differentiable regularisers
directly. It makes use of the proximal operator of the regulariser.
It can also be used on problems with constraints, where the function
$h(x)$ is the indicator function of the constraint set, and proximal
operator is projection onto the constraint set. 

In this work we also introduce a new non-differentiable regulariser
for the above mentioned graph structure learning problem, which can
also be attacked using proximal methods. Its non-differentiable structure
is atypically complex compared to other regularisers used in machine
learning, requiring a special optimisation procedure to be used just
to evaluate the proximal operator.

For non-differentiable losses, we introduce the Prox-Finto algorithm
(Section \ref{sec:prox-finito}). This incremental gradient algorithm
uses the proximal operator of the single datapoint loss. It provides
a bridge between the Finito algorithm (Section \ref{sec:finito})
and the SDCA algorithm \citep{sdca}, having properties of both methods.

\section{Publications Related to This Thesis}

The majority of the content in this thesis has been published as conference
articles. For the work on incremental gradient methods, the Finito
method has been published as \citet{adefazio-icml2014}, and the SAGA
method as \citet{adefazio-nips2014}. Chapters \ref{chap:finito}
\& \ref{chap:saga} contain much more detailed theory than has been
previously published. Some of the discussion in Chapter \ref{chap:inc-discus}
appears in \citet{adefazio-icml2014} also. For the portion of this
thesis on Gaussian graphical models, Chapter \ref{chap:submodular}
largely follows the publication \citet{adefazio-nips2012}. Chapter
\ref{chap:colab} is based on the work in \citet{adefazio-icml2012},
although heavily revised.

\chapter{Incremental Gradient Methods}

\label{chap:background-incremental}

In this chapter we give an introduction to the class of incremental
gradient (IG) methods. Incremental gradient methods are simply a class
of methods that can take advantage of known summation structure in
an optimisation objective by accessing the objective one term at a
time. Objectives that are decomposable as a sum of a number of terms
come up often in applied mathematics and scientific computing, but
are particularly prevalent in machine learning applications. Research
in the last two decades on optimisation problems with a summation
structure has focused more on the stochastic approximation setting,
where the summation is assumed to be over an infinite set of terms.
The \emph{finite} sum case that incremental gradient methods cover
has seen a resurgence in recent years after the discovery that there
exist \emph{fast} incremental gradient methods whose convergence rates
are better than any possible black box method for finite sums with
particular (common) structures. We provide an extensive overview of
all known fast incremental gradient methods in the later parts of
this chapter. Building on the described methods, in Chapters \ref{chap:finito}
\& \ref{chap:saga} we introduce three novel fast incremental gradient
methods. Depending on the problem structure, each of these methods
can have state-of-the-art performance.

\section{Problem Setup}

We are interested in minimising functions of the form

\[
f(x)=\frac{1}{n}\sum_{i=1}^{n}f_{i}(x),
\]

where $x\in\mathbb{R}^{d}$ and each $f_{i}$ is convex and Lipschitz
smooth with constant $L$. We will also consider the case where each
$f_{i}$ is additionally strongly convex with constant $\mu$. See
Appendix \ref{sec:appendix-defs} for defintions of Lipschitz smoothness
and strong convexity. Incremental gradient methods are algorithms
that at each step evaluate the gradient and function value of only
a single $f_{i}$. 

We will measure convergence rates in terms of the number of $(f_{i}(x),f_{i}^{\prime}(x))$
evaluations, normally these are much cheaper computationally than
evaluations of the whole function gradient $f^{\prime}$, such as
performed by the \emph{gradient descent} algorithm. We use the notation
$x^{*}$ to denote a \foreignlanguage{british}{minimiser} of $f$.
For strongly convex problems this is the unique \foreignlanguage{british}{minimiser}.

This setup differs from the traditional black box smooth convex optimisation
problem only in that we are assuming that our function is decomposable
into a \emph{finite sum} structure. This finite sum structure is widespread
in machine learning applications. For example, the standard framework
of Empirical Risk Minimisation (ERM) takes this form, where for a
loss function $L\,:\,\mathbb{R}^{d}\times\mathbb{R}\rightarrow\mathbb{R}$
and data label tuples $(x_{i},y_{i})$, we have:
\[
R_{\text{emp}}(h)=\frac{1}{n}\sum_{i}^{n}L(h(x_{i}),y_{i}),
\]

where $h$ is the hypothesis function that we intend to \foreignlanguage{british}{optimise}
over. The most common case of ERM is minimisation of the negative
log-likelihood, for instance the classical logistic regression problem
(See standard texts such as \citealt{bishop}). Often ERM is an approximation
to an underlying stochastic programming problem, where the summation
is replaced with an expectation over a population of data tuples from
which we can sample from.

\subsection{Exploiting problem structure}

\label{sub:assumptions}

Given the very general nature of the finite sum structure, we can
not expect to get faster convergence than we would by accessing the
whole gradient without additional assumptions. For example, suppose
the summation only has one term, or alternatively each $f_{i}$ is
the zero function except one of the $n$. 

Notice that the Lipschitz smoothness and strong convexity assumptions
we made are on each $f_{i}$ rather than on $f$. This is a key point.
If the directions of maximum curvature of each term are aligned and
of similar magnitude, then we can expect the term Lipschitz smoothness
to be similar to the smoothness of the whole function. However, it
is easy to construct problems for which this is not the case, in fact
the Lipschitz smoothness of $f$ may be $n$ times smaller than that
of each $f_{i}$. In that case the incremental gradient methods will
give no improvement over black box optimisation methods.

For machine learning problems, and particularly the empirical risk
minimisation problem, this worst case behavior is not common. The
curvature and hence the Lipschitz constants are defined largely by
the loss function, which is shared between the terms, rather than
the data point. Common data preprocessing methods such as data whitening
can improve this even further.

The requirement that the magnitude of the Lipschitz constants be approximately
balanced can be relaxed in some cases. It is possible to formulate
IG methods where the convergence is stated in terms of the average
of the Lipschitz constants of the $f_{i}$ instead of the maximum.
This is the case for the Prox-Finito algorithm described in Section
\ref{sec:prox-finito}. All known methods that make use of the average
Lipschitz constant require knowledge of the ratios of the Lipschitz
constants of the $f_{i}$ terms, which limits their practicality unfortunately.

Regardless of the condition number of the problem, if we have a summation
with enough terms optimisation becomes easy. This in made precise
in the definition that follows.
\begin{defn}
\label{enu:big-data-condition}\textbf{The big data condition:} For
some known constant $\beta\geq1$, 
\[
n\geq\beta\frac{L}{\mu}.
\]
This condition obviously requires strong convexity and Lipschitz smoothness
so that $L/\mu$ is well defined. It is a very strong assumption for
small $n$, as the condition number $L/\mu$ in typical machine learning
problems is at least in the thousands. For applications of this assumption,
$\beta$ is typically between $1$ and $8$. Several of the methods
we describe below have a fixed and very fast convergence rate independent
of the condition number when this big-data condition holds.
\end{defn}

\subsection{Randomness and expected convergence rates}

This thesis works extensively with optimisation methods that make
random decisions during the course of the algorithm. Unlike the stochastic
approximation setting, we are dealing with deterministic, known optimisation
problems; the stochasticity is introduced by our optimisation methods,
it is not inherent in the problem. We introduce randomness because
it allows us to get convergence rates faster than that of any currently
known deterministic methods. The caveat is that these convergence
rates are in \emph{expectation}, so they don't always hold precisely.
This is not as bad as it first seems though. Determining that the
expectation of a general random variable converges is normally quite
a weak result, as its value may vary around the expectation substantially
in practice, potentially by far more than it converges by. The reason
why this is not an issue for the optimisation methods we consider
is that all the random variables we bound are \emph{non-negative}.
A non-negative random variable $X$ with a very small expectation,
say:
\[
E[X]=1\times10^{-5},
\]
is with high probability close to its expectation. This is a fundamental
result implied by \foreignlanguage{british}{Markov's} inequality.
For example, suppose $E[X]=1\times10^{-5}$ and we want to bound the
probability that $X$ is greater than $1\times10^{-3}$, i.e. a factor
of 100 worse than its expectation. Then Markov's inequality tells
us that:
\[
P(X\geq1\times10^{-3})\leq\frac{1}{100}.
\]

So there is only a 1\% chance of $X$ being larger than 100 times
its expected value here. We will largely focus on methods with linear
convergence in the following chapters, so in order to increase the
probability of the value $X$ holding by a factor $\rho$, only a
logarithmic number of additional iterations in $\rho$ is required
($O(\log\rho)$).

We would also like to note that Markov's inequality can be quite conservative.
Our experiments in later chapters show little in the way of random
noise attributable to the optimisation procedure, particularly when
the amount of data is large.

\subsection{Data access order}

\label{subsec:access-orders}

The source of randomness in all the methods considered in this chapter
is the order of accessing the $f_{i}$ terms. By \emph{access} we
mean the evaluation of $f_{i}(x)$ and $f_{i}^{\prime}(x)$ at an
$x$ of our choice. This is more formally known as an oracle evaluation
(see Section \ref{sec:oracles}), and typically constitutes the most
computationally expensive part of the main loop of each algorithm
we consider. The access order is defined on a per\emph{-epoch} basis,
where an epoch is $n$ evaluations. Only three different access orders
are considered in this work:
\begin{description}
\item [{Cyclic}] Each step with $j=1+(k\mod n)$. Effectively we access
$f_{i}$ in the order they appear, then loop to the beginning and
the end of every epoch.
\item [{Permuted}] Each epoch with $j$ is sampled without replacement
from the set of indices not accessed yet in that epoch. This is equivalent
to permuting the $f_{i}$ at the beginning of each epoch, then using
the cyclic order within the epoch.
\item [{Randomised}] The value of $j$ is sampled uniformly at random with
replacement from $1,\dots,n$. 
\end{description}
The ``permuted'' terminology is our nomenclature, whereas the \foreignlanguage{british}{other}
two terms are standard.

\section{Early Incremental Gradient Methods}

The classical incremental gradient (IG) method is simply a step of
the form:
\[
x^{k+1}=x^{k}-\gamma_{k}f_{j}^{\prime}(x^{k}),
\]

where at step $k$ we use cyclic access, taking $j=1+(k\mod n)$.
This is similar to the more well known stochastic gradient descent,
but with a cyclic order of access of the data instead of a random
order. We have introduced here a superscript notation $x^{k}$ for
the variable $x$ at step $k$. We use this notation throughout this
work. 

It turns out to be much easier to analyse such methods under a random
access ordering. For the random order IG method (i.e. SGD) on smooth
strongly convex problems, the following rate holds for an appropriately
chosen step sizes:
\[
\mathbb{E}\left[f(x^{k})-f(x^{*})\right]\leq\frac{L}{2k}\left\Vert x^{0}-x^{*}\right\Vert ^{2}.
\]

The step size scheme required is of the form $\gamma_{k}=\frac{\theta}{k}$,
where $\theta$ is a constant that depends on the gradient norm bound
$R$ as well as the degree of strong convexity $\mu$. It may be required
to be quite small in some cases. This is what's known as a \emph{sublinear}
rate of convergence, as the dependence on $k$ is of the form $O(\frac{L}{2k})$,
which is slower than the linear rate $O((1-\alpha)^{k})$ for any
$\alpha\in(0,1)$ asymptotically. 

Incremental gradient methods for strongly convex smooth problems were
of less practical utility in machine learning up until the development
of fast variants (discussed below), as the sublinear rates for the
previously known methods did not compare favourably to the (super-)linear
rate of quasi-Newton methods. For non-strongly convex problems, or
strongly convex but non-smooth problems, the story is quite different.
In those cases, the theoretical and practical rates are hard to beat
with full (sub-)gradient methods. The non-convex case is of particular
interest in machine learning. SGD has been the de facto standard optimisation
method for neural networks for example since the 1980s \citep{backprop}.

Such incremental gradient methods have a long history, having been
applied to specific problems as far back as the 1960s \citep{ig-old-radio}.
An up-to-date survey can be found in \citet{ig-bert-survey}.

\section{Stochastic Dual Coordinate Descent (SDCA)}

The stochastic dual coordinate descent method \citep{sdca} is based
on the principle that for problems with explicit quadratic regularisers,
the dual takes a particularly easy to work with form. Recall the finite
sum structure $f(x)=\frac{1}{n}\sum_{i=1}^{n}f_{i}(x)$ defined earlier.
Instead of assuming that each $f_{i}$ is strongly convex, we instead
need to consider the regularised objective:
\[
f(x)=\frac{1}{n}\sum_{i=1}^{n}f_{i}(x)+\frac{\mu}{2}\left\Vert x\right\Vert ^{2}.
\]
 For any strongly convex $f_{i}$, we may transform our function to
this form by replacing each $f_{i}$ with $f_{i}-\frac{\mu}{2}\left\Vert x\right\Vert ^{2}$,
then including a separate regulariser. This changes the Lipschitz
smoothness constant for each $f_{i}$ to $L-\mu$, and preserves convexity.
We are now ready to consider the dual transformation. We apply the
technique of \textbf{dual decomposition}, where we decouple the terms
in our objective as follows:
\[
\min_{x,x_{1},\dots x_{i},\dots,x_{n}}f(x)=\frac{1}{n}\sum_{i=1}^{n}f_{i}(x_{i})+\frac{\mu}{2}\left\Vert x\right\Vert ^{2},
\]
\[
s.t.\;x=x_{i}\;i=1\dots n.
\]

This reformulation initially achieves nothing, but the key idea is
that we now have a constrained optimisation problem, and so we may
apply Lagrangian duality (Section \ref{sec:duality}). The Lagrangian
function is:
\begin{eqnarray}
L(x,x_{1},\dots\alpha_{1},\dots) & = & \frac{1}{n}\sum_{i=1}^{n}f_{i}(x_{i})+\frac{\mu}{2}\left\Vert x\right\Vert ^{2}+\frac{1}{n}\sum_{i}^{n}\left\langle \alpha_{i},x-x_{i}\right\rangle \nonumber \\
 & = & \frac{1}{n}\sum_{i=1}^{n}\left(f_{i}(x_{i})-\left\langle \alpha_{i},x_{i}\right\rangle \right)+\frac{\mu}{2}\left\Vert x\right\Vert ^{2}+\left\langle \frac{1}{n}\sum_{i}^{n}\alpha_{i},x\right\rangle ,\label{eq:sdca-lagrangian}
\end{eqnarray}

where $\alpha_{i}\in\mathbb{R}^{d}$ are the introduced dual variables.
The Lagrangian dual function is formed by taking the minimum of $L$
with respect to each $x_{i}$, leaving $\alpha$, the set of $\alpha_{i}$
$i=1\dots n$ free:

\begin{equation}
D(\alpha)=\frac{1}{n}\sum_{i=1}^{n}\min_{x_{i}}\left\{ f_{i}(x_{i})-\left\langle \alpha_{i},x_{i}\right\rangle \right\} +\min_{x}\left\{ \frac{\mu}{2}\left\Vert x\right\Vert ^{2}+\left\langle \frac{1}{n}\sum_{i}\alpha_{i},x\right\rangle \right\} ,\label{eq:sdca-partial-dual}
\end{equation}
Now recall that the definition of the convex conjugate (Section \ref{sec:convex-conj})
says that: 
\[
\min\left\{ f(x)-\left\langle a,x\right\rangle \right\} =-\text{sup}_{x}\left\{ \left\langle a,x\right\rangle -f(x)\right\} =-f^{*}(\alpha).
\]
Clearly we can plug this in for each $f_{i}$ to get:
\[
D(\alpha)=-\frac{1}{n}\sum_{i=1}^{n}f_{i}^{*}(\alpha_{i})+\min_{x}\left[\frac{\mu}{2}\left\Vert x\right\Vert ^{2}+\left\langle \frac{1}{n}\sum_{i}\alpha_{i},x\right\rangle \right].
\]

We still need to simplify the remaining $\min$ term, which is also
in the form of a convex conjugate. We know that squared norms are
self-conjugate, and scaling a function by a positive constant $\beta$
transforms its conjugate from $f^{*}(a)$ to $\beta f^{*}(a/\beta)$,
so we in fact have:

\[
D(\alpha)=-\frac{1}{n}\sum_{i=1}^{n}f_{i}^{*}(\alpha_{i})-\frac{\mu}{2}\left\Vert \frac{1}{\mu n}\sum_{i}\alpha_{i}\right\Vert ^{2}.
\]

This is the objective directly maximised by SDCA. As the name implies,
SDCA is randomised (block) coordinate ascent on this objective, where
only one $\alpha_{i}$ is changed each step. 

In coordinate descent we have the option of performing a gradient
step in a coordinate direction, or an exact minimisation. For the
exact coordinate minimisation, the update is easy to derive:
\begin{eqnarray}
\alpha_{j}^{k+1} & = & \arg\min_{\alpha_{j}}\left[\frac{1}{n}\sum_{i=1}^{n}f^{*}(\alpha_{i})+\frac{\mu}{2}\left\Vert \frac{1}{\mu n}\sum_{i}^{n}\alpha_{i}\right\Vert ^{2}\right]\nonumber \\
 & = & \arg\min_{\alpha_{j}}\left[f_{j}^{*}(\alpha_{j})+\frac{\mu n}{2}\left\Vert \frac{1}{\mu n}\sum_{i}^{n}\alpha_{i}\right\Vert ^{2}\right].\label{eq:coord-exact}
\end{eqnarray}

The primal point $x^{k}$ corresponding to the dual variables $\alpha_{i}^{k}$
at step $k$ is the minimiser of the conjugate problem $x^{k}=\arg\min_{x}\left[\frac{\mu}{2}\left\Vert x\right\Vert ^{2}+\left\langle \frac{1}{n}\sum_{i}^{n}\alpha_{i}^{k},x\right\rangle \right]$,
which in closed form is simply $x^{k}=-\frac{1}{\mu n}\sum_{i}^{n}\alpha_{i}^{k}$.
This can be used to further simplify Equation \ref{eq:coord-exact}.
The full method is Algorithm \ref{alg:sdca-exact}.

\begin{algorithm}
Initialise $x^{0}$ and $\alpha_{i}^{0}$ as the zero vector, for
all $i$.

Step $k+1$:
\begin{enumerate}
\item Pick an index $j$ uniformly at random.
\item Update $\alpha_{j}$, leaving the other $\alpha_{i}$ unchanged:
\[
\alpha_{j}^{k+1}=\arg\min_{y}\left[f_{j}^{*}(y)+\frac{\mu n}{2}\left\Vert x^{k}-\frac{1}{\mu n}\left(y-\alpha_{j}^{k}\right)\right\Vert ^{2}\right].
\]

\item Update $x^{k+1}=x^{k}-\frac{1}{\mu n}\left(\alpha_{j}^{k+1}-\alpha_{j}^{k}\right).$
\end{enumerate}
At completion, for smooth $f_{i}$ return $x^{k}$. For non-smooth,
return a tail average of the $x^{k}$ sequence.

\protect\caption{\label{alg:sdca-exact}SDCA (exact coordinate descent)}
\end{algorithm}

The SDCA method has a geometric convergence rate in the dual objective
$D$ of the form:
\[
\mathbb{E}\left[D(\alpha^{k})-D(\alpha^{*})\right]\leq\left(1-\frac{\mu}{L+\mu n}\right)^{k}\left[D(\alpha^{0})-D(\alpha^{*})\right].
\]

This is easily extended to a statement about the duality gap $f(x^{k})-D(\alpha^{k})$
and hence the suboptimality $f(x^{k})-f(x^{*})$ by using the relation:
\[
f(x^{k})-D(\alpha^{k})\leq\frac{L+\mu n}{\mu}\left(D(\alpha^{k})-D(\alpha^{*})\right).
\]

\subsection{Alternative steps}

The full coordinate minimisation step discussed in the previous section
is not always practical. If we are treating each element $f_{i}$
in the summation $\frac{1}{n}\sum_{i}^{n}f_{i}(x)$ as a single data
point loss, then even for the simple binary logistic loss there is
not a closed form solution for the exact coordinate step. We can use
a black-box 1D optimisation method to find the coordinate minimiser,
but this will generally require 20-30 exponential function evaluations,
together with one vector dot product. 

For multiclass logistic loss, the subproblem solve is not fast enough
to yield a usable algorithm. In the case of non-differentiable losses,
the situation is better. Most non-differentiable functions we use
in machine learning, such as the hinge loss, yield closed form solutions.

For performance reasons we often want to treat each $f_{i}$ as a
minibatch loss, in which case we virtually never have a closed form
solution for the subproblem, even in the non-differentiable case.

\citet{sdca-accel} describe a number of other possible steps which
lead to the same theoretical convergence rate as the exact minimisation
step, but which are more usable in practice:
\begin{description}
\item [{Interval Line search:}] It turns out that it is sufficient to
perform the minimisation in Equation \ref{eq:coord-exact} along the
interval between the current dual variable $\alpha_{j}^{k}$ and the
point $u=f_{j}^{\prime}(x^{k})$. The update takes the form:
\[
s=\arg\min_{s\in[0,1]}\left[f_{j}^{*}\left(\alpha_{j}^{k}+s(u-\alpha_{j}^{k})\right)+\frac{\mu n}{2}\left\Vert x^{k}+\frac{s}{\mu n}\left(u-\alpha_{j}^{k}\right)\right\Vert ^{2}\right],
\]
\[
\alpha_{j}^{k+1}=\alpha_{j}^{k}+s(u-\alpha_{j}^{k}).
\]

\item [{Constant step:}] If the value of the Lipschitz smoothness constant
$L$ is known, we can calculate a conservative value for the parameter
$s$ instead of optimising over it with an interval line search. This
gives an update of the form:
\[
\alpha_{j}^{k+1}=\alpha_{j}^{k}+s(u-\alpha_{j}^{k})
\]
\[
\text{where }\;s=\frac{\mu n}{\mu n+L}.
\]
This method is much slower in practice than performing a line-search,
just as a $\frac{1}{L}$ step size with gradient descent is much slower
than performing a line search.
\end{description}

\subsection{Reducing storage requirements}

We have presented the SDCA algorithm in full generality above. This
results in dual variables of dimension $d$, for which the total storage
$d\times n$ can be prohibitive. In practice, the dual variables often
lie on a low-dimensional subspace. This is the case with linear classifiers
and regressors, where a $r$ class problem has gradients on a $r-1$
dimensional subspace. 

A linear classifier takes the form $f_{i}(x)=\phi\left(X_{i}^{T}x\right)$,
for a fixed loss $\phi:\mathbb{R}^{r}\rightarrow\mathbb{R}$ and data
instance matrix $X_{i}:d\times r$. In the simplest case $X_{i}$
is just the data point duplicated as $r$ rows. Then the dual variables
are $r$ dimensional, and the $x^{k}$ updates change to:
\[
x^{k}=-\frac{1}{\mu n}\sum_{i}^{n}X_{i}\alpha_{i}.
\]

\[
\alpha_{j}^{k+1}=\arg\min_{\alpha}\left[\phi_{j}^{*}(\alpha)+\frac{\mu n}{2}\left\Vert x^{k}+\frac{1}{\mu n}X_{i}\left(\alpha-\alpha_{j}^{k}\right)\right\Vert ^{2}\right].
\]

This is the form of SDCA presented by \citet{sdca-accel}, although
with the negation of our dual variables.

\subsection{Accelerated SDCA}

\label{sub:asdca}

The SDCA method is also currently the only fast incremental gradient
method to have a known accelerated variant. By acceleration, we refer
to the modification of an optimisation method to improve the convergence
rate by an amount greater than any constant factor. This terminology
is common in optimisation although a precise definition is not normally
given.

The \textbf{A}SDCA method \citep{sdca-accel} works by utilising the
regular SDCA method as a sub-procedure. It has an outer loop, which
at each step invokes SDCA on a modified problem $x^{k+1}=\min_{x}f(x)+\frac{\lambda}{2}\left\Vert x-y\right\Vert ^{2},$
where $y$ is chosen as an over-relaxed step of the form:
\[
y=x^{k}+\beta(x^{k}-x^{k-1}),
\]

for some known constant $\beta$. The constant $\lambda$ is likewise
computed from the Lipschitz smoothness and strong convexity constants.
These regularised sub-problems $f(x)+\frac{\lambda}{2}\left\Vert x-y\right\Vert ^{2}$
have a greater degree of strong convexity than $f(x)$, and so individually
are much faster to solve. By a careful choice of the accuracy at which
they are computed to, the total number of steps made between all the
subproblem solves is much smaller than would be required if regular
SDCA is applied directly to $f(x)$ to reach the same accuracy.

In particular, they state that to reach an accuracy of $\epsilon$
in expectation for the function value, they need $k$ iterations,
where:
\[
k=\tilde{O}\left(dn+\min\left\{ \frac{dL}{\mu},d\sqrt{\frac{nL}{\mu}}\right\} \right)\log(1/\epsilon).
\]

The $\tilde{O}$ notation hides constant factors. This rate is not
of the same precise form as the other convergence rates we will discuss
in this chapter. We can make some general statements though. When
$n$ is in the range of the big-data condition, this rate is no better
than regular SDCA's rate, and probably worse in practice due to overheads
hidden by the $\tilde{O}$ notation. When $n$ is much smaller than
$\frac{L}{\mu}$, then potentially it can be much faster than regular
SDCA.

Unfortunately, the ASDCA procedure has significant computational overheads
that make it not necessarily the best choice in practice. Probably
the biggest issue however is a sensitivity to the Lipschitz smoothness
and strong convexity constants. It assumes these are known, and if
the used values differ from the true values, it may be significantly
slower than regular SDCA. In contrast, regular SDCA requires no knowledge
of the Lipschitz smoothness constants (for the prox variant at least),
just the strong convexity (regularisation) constant.

\section{Stochastic Average Gradient (SAG)}

\label{sec:sag-background}

The SAG algorithm \citep{sag} is the closest in form to the classical
SGD algorithm among the fast incremental gradient methods. Instead
of storing dual variables $\alpha_{i}$ like SDCA above, we store
a table of past gradients $y_{i}$, which has the same storage cost
in general, $n\times d$. The SAG method is given in Algorithm \ref{alg:sag}.
The key equation for SAG is the step:
\[
x^{k+1}=x^{k}-\frac{1}{\gamma n}\sum_{i}^{n}y_{i}^{k}.
\]

Essentially we move in the direction of the average of the past gradients.
Note that this average contains one past gradient for each term, and
they are equally weighted. This can be contrasted to the SGD method
with momentum, which uses a geometrically decaying weighted sum of
all past gradient evaluations. SGD with momentum however is not a
linearly convergent method. It is surprising that using equal weights
like this actually yields a much faster converging algorithm, even
though some of the gradients in the summation can be extremely out
of date.

SAG is an evolution of the earlier incremental averaged gradient method
(IAG, \citealp{sag-blatt}) which has the same update with a different
constant factor, and with cyclic access used instead of randomised.
IAG has a more limited convergence theory covering quadratic or bounded
gradient problems, and a much slower rate of convergence.

\begin{algorithm}
Initialise $x^{0}$ as the zero vector, and $y_{i}=f_{i}^{\prime}(x^{0})$
for each $i$.

Step $k+1$:
\begin{enumerate}
\item Pick an index $j$ uniformly at random.
\item Update $x$ using step length constant $\gamma$:
\[
x^{k+1}=x^{k}-\frac{1}{\gamma n}\sum_{i}^{n}y_{i}^{k}.
\]

\item Set $y_{j}^{k+1}=f_{j}^{\prime}(x^{k+1})$. Leave $y_{i}^{k+1}=y_{i}^{k}$
for $i\neq j$.
\end{enumerate}
\protect\caption{\label{alg:sag}SAG}
\end{algorithm}

The convergence rate of SAG for strongly convex problems is of the
same order as SDCA, although the constants are not quite as good.
In particular, we have an expected convergence rate in terms of function
value suboptimality of:
\[
\mathbb{E}[f(x^{k})-f(x^{*})]\leq\left(1-\min\left\{ \frac{1}{8n},\frac{\mu}{16L}\right\} \right)^{k}L_{0},
\]

Where $L_{0}$ is a complex expression involving $f(x^{0}+\frac{1}{\gamma n}\sum_{i}^{n}y_{i}^{0})$
and a quadratic form of $x^{0}$ and each $y_{i}^{0}$. This theoretical
convergence rate is between 8 and 16 times worse than SDCA. In practice
SAG is often faster than SDCA though, suggesting that the SAG theory
is not tight. A nice feature of SAG is that unlike SDCA, it can be
directly applied to non-strongly convex problems. Differentiability
is still required though. The convergence rate is then in terms of
the average iterate $\bar{x}^{k}=\frac{1}{k}\sum_{l}^{k}x^{l}$: 
\[
\mathbb{E}[f(\bar{x}^{k})-f(x^{*})]\leq\frac{32n}{k}L_{0}.
\]

The SAG algorithm has great practical performance, but it is surprisingly
difficult to analyse theoretically. The above rates are likely conservative
by a factor of between $4$ and $8$. Due to the difficulty of analysis,
the proximal version for composite losses has not yet had its theoretical
convergence established.

\section{Stochastic Variance Reduced Gradient (SVRG)}

\label{sec:svrg-background}

The SVRG method \citep{svrg} is a recently developed fast incremental
gradient method. It was developed to address the potentially high
storage costs of SDCA and SAG, by trading off storage against computation.
The SVRG method is given in Algorithm \ref{alg:svrg}. Unlike the
other methods discussed, there is a tunable parameter $m$, which
specifies the number of iterations to complete before the current
gradient approximation is ``recalibrated'' by computing a full gradient
$f^{\prime}(\tilde{x})$ at the last iterate before the recalibration,
$\tilde{x}\coloneqq x^{k}$. Essentially, instead of maintaining a
table of past gradients $y_{i}$ for each $i$ like SAG does, the
algorithm just stores the location $\tilde{x}$ at which those gradients
should be evaluated, then re-evaluates them when needed by just computing
$f_{j}^{\prime}(\tilde{x}).$

Like the SAG algorithm, at each step we need to know the updated term
gradient $f_{j}^{\prime}(x^{k})$, the old term gradient $f_{j}^{\prime}(\tilde{x})$
and the average of the old gradients $f^{\prime}(\tilde{x})$. Since
we are not storing the old term gradient, just its average, we need
to calculate two term gradients instead of the one term gradient calculated
by SAG at each step.

\begin{algorithm}
Initialise $x^{0}$ as the zero vector, $g^{k}=\frac{1}{n}\sum_{i}f_{i}^{\prime}(x^{0})$
and $\tilde{x}^{0}=x^{0}$.

Step $k+1$:
\begin{enumerate}
\item Pick $j$ uniformly at random.
\item Update $x$:
\[
x^{k+1}=x^{k}-\frac{1}{\eta}f_{j}^{\prime}(x^{k})+\frac{1}{\eta}\left[f_{j}^{\prime}(\tilde{x}^{k})-g^{k}\right].
\]

\item \label{enu:xtilde}Every $m$ iterations, set $\tilde{x}$ and recalculate
the full gradient at that point: 
\[
\tilde{x}^{k+1}=x^{k+1}.
\]
\[
g^{k}=\frac{1}{n}\sum_{i}f_{i}^{\prime}(\tilde{x}^{k+1}).
\]
Otherwise leave $\tilde{x}^{k+1}=\tilde{x}^{k}$ and $g^{k+1}=g^{k}$.
\end{enumerate}
At completion return $\tilde{x}$.

\protect\caption{\label{alg:svrg}SVRG}
\end{algorithm}

The S2GD method \citep{semi} was concurrently developed with SVRG.
It has the same update as SVRG, just differing in that the theoretical
choice of $\tilde{x}$ discussed in the next paragraph. We use SVRG
henceforth to refer to both methods.

The update $\tilde{x}^{k+1}=x^{k+1}$ in step \ref{enu:xtilde} above
is technically not supported by the theory. Instead, one of the following
two updates are used:
\begin{enumerate}
\item $\tilde{x}$ is the average of the $x$ values from the last $m$
iterations. This is the variant suggested by \citet{svrg}.
\item $\tilde{x}$ is a randomly sampled $x$ from the last $m$ iterations.
This is used in the S2GD variant \citep{semi}.
\end{enumerate}
These alternative updates are required theoretically as the convergence
between recalibrations is expressed in terms of the average of function
values of the last $m$ points, 
\[
\frac{1}{m}\sum_{r=k-m}^{k}\left[f(x^{r})-f(x^{*})\right],
\]
 instead of in terms of $f(x^{k})-f(x^{*})$ directly. Variant 1 avoids
this issue by using Jensen's inequality to pull the summation inside:
\[
\frac{1}{m}\sum_{r=k-m}^{k}\left[f(x^{r})-f(x^{*})\right]\geq f(\frac{1}{m}\sum_{r=k-m}^{k}x^{r})-f(x^{*}).
\]

Variant 2 uses a sampled $x$, which in expectation will also have
the required value. In practice, there is a very high probability
that $f(x^{k})-f(x^{*})$ is less than the last-$m$ average, so just
taking $\tilde{x}=x^{k}$ works.

The SVRG method has the following convergence rate if $k$ is a multiple
of $m$:
\[
\mathbb{E}[f(\tilde{x}^{k})-f(x^{*})]\leq\rho^{k/m}\left[f(\tilde{x}^{0})-f(x^{*})\right],
\]
\[
\text{where }\rho=\frac{\eta}{\mu(1-4L/\eta)m}+\frac{4L(m+1)}{\eta(1-4L/\eta)m}.
\]

Note also that each step requires two term gradients, so the rate
must be halved when comparing against the other methods described
in this chapter. There is also the cost of the recalibration pass,
which (depending on $m$) can further increase the run time to three
times that of SAG per step. This convergence rate has quite a different
form from that of the other methods considered in this section, making
direct comparison difficult. However, for most parameter values this
theoretical rate is worse than that of the other fast incremental
gradient methods. In Section \ref{sec:svrg-understanding} we give
an analysis of SVRG that requires additional assumptions, but gives
a rate that is directly comparable to the other fast incremental gradient
methods.

\chapter{New Dual Incremental Gradient Methods}

\label{chap:finito}

In this chapter we introduce a novel fast incremental gradient method
for strongly convex problems that we call \emph{Finito}. Like SDCA,
SVRG and SAG, Finito is a stochastic method that is able to achieve
linear convergence rates for strongly convex problems. Although the
Finito algorithm only uses primal quantities directly, the proof of
its convergence rate uses lower bounds extensively, so it can be considered
a dual method, like SDCA. Similar to SDCA, its theory does not support
its use on non-strongly convex problems, although there are no practical
issues with its application.

In Section \ref{sec:finito-theory} we prove the convergence rate
of the Finito method under the big-data condition described in the
previous chapter. This theoretical rate is better than the SAG and
SVRG rates but not quite as good as the SDCA rate. In Section \ref{sec:finito-experiments}
we compare Finito empirically against SAG and SDCA, showing that it
converges faster. This difference is most pronounced when using a
permuted access order, which unfortunately is not covered by current
convergence theory.

The relationship between Finito and SDCA allows a kind of midpoint
algorithm to be constructed, which has favourable properties of both
methods. We call this midpoint Prox-Finito. It is described in Section
\ref{sec:prox-finito}.

An earlier version of the work in this chapter has been published
as \citet{adefazio-icml2014}.

\section{The Finito Algorithm}

\label{sec:finito}

As discussed in Chapter \ref{chap:background-incremental}, we are
interested in convex functions of the form
\[
f(w)=\frac{1}{n}\sum_{i=1}^{n}f_{i}(w).
\]

We assume that each $f_{i}$ is Lipschitz smooth with constant $L$
and is strongly convex with constant $\mu$. We will focus on the
\emph{big data} setting:
\[
n\geq\beta\frac{L}{\mu}
\]
 with $\beta=2$, as described in Section \ref{sub:assumptions}.

\subsection{Additional notation}

We omit the $n$ superscript on summations throughout, and subscript
$i$ with the implication that indexing starts at $1$. When we use
separate arguments for each $f_{i}$, we denote them $\phi_{i}$.
Let $\bar{\phi}^{k}$ denote the average $\bar{\phi}^{k}=\frac{1}{n}\sum_{i=1}^{n}\phi_{i}^{k}$.
Our step length constant, which depends on $\ensuremath{\beta}$ (the
big-data constant), is denoted $\alpha$. Note that $\alpha$ is an
inverse step length, in the sense that large $\alpha$ results in
smaller steps. We use angle bracket notation for dot products $\langle\cdot,\cdot\rangle$.

\subsection{Method}

We start with a table of known $\phi_{i}^{0}$ values, and a table
of known gradients $f_{i}^{\prime}(\phi_{i}^{0})$, for each $i$.
We will update these two tables during the course of the algorithm.
The Finito method is described in Algorithm \ref{alg:finito-algorithm}.

\begin{algorithm}[H]
Initialise $\phi_{i}^{0}=\phi^{0}$ for some $\phi^{0}$ and all $i$,
and calculate and store each $f_{i}^{\prime}(\phi_{i}^{0})$.

The step for iteration $k$, is:
\begin{enumerate}
\item Update $w$ using the step:
\begin{equation}
w^{k}=\bar{\phi}^{k}-\frac{1}{\alpha\mu n}\sum_{i}f_{i}^{\prime}(\phi_{i}^{k}).\label{eq:main-finito-update}
\end{equation}

\item Pick an index $j$ uniformly at random, or using without-replacement
sampling as discussed in Section \ref{sec:randomness}.
\item Set $\phi_{j}^{k+1}=w^{k}$ in the table and leave the other variables
the same ($\phi_{i}^{k+1}=\phi_{i}^{k}$ for $i\neq j$).
\item Calculate and store $f_{j}^{\prime}(\phi_{j}^{k+1})$ in the table.
\end{enumerate}
\protect\caption{\label{alg:finito-algorithm}Finito Algorithm}

\end{algorithm}

We have established the theoretical convergence rate of Finito under
the big-data condition:
\begin{thm}
When the big-data condition holds with $\beta=2$, $\alpha=2$ may
be used. In that setting, the expected convergence rate is: 
\[
\mathbb{E}\left[f(\bar{\phi}^{k})-f(w^{*})\right]\leq\left(1-\frac{1}{2n}\right)^{k}\frac{3}{4\mu}\left\Vert f^{\prime}(\phi^{0})\right\Vert ^{2}.
\]

\end{thm}
See Section \ref{sec:finito-theory} for the proof. This can be compared
to the SAG method, which achieves a $1-\frac{1}{8n}$ geometric rate
when $\beta=2$. The SDCA method has a $1-\frac{2L/n}{2L+L}=1-\frac{2}{3n}$
rate, which is very slightly better rate ($0.666\dots$ v.s. $0.5$)
than Finito.

Note that on a per epoch basis, the Finito rate is $\left(1-\frac{1}{2n}\right)^{n}\approx\exp(-1/2)=0.606$.
Lemma \ref{lem:upper-bernoulli-bound} discusses this $\exp(\dots)$
approximation in more detail. To put that rate into context, 10 epochs
will see the error bound reduced by more than 148x.

One notable feature of our method is the fixed step size. In typical
machine learning problems the strong convexity constant is given by
the strength constant of the quadratic regulariser used. Since this
is a known quantity, as long as the big-data condition holds $\alpha=2$
may be used without any tuning or adjustment of Finito required. This
lack of tuning is a major feature of Finito.

Our theory is not as complete as for SDCA and SAG. In cases where
the big-data condition does not hold, we conjecture that the step
size must be reduced proportionally to the violation of the big-data
condition. In practice, the most effective step size can be found
by testing a number of step sizes, as is usually done with other stochastic
optimisation methods. We do not have any convergence theory for when
the big-data condition doesn't hold, except for when very small step
sizes are used, such as by setting the inverse step size constant
$\alpha$ to $\alpha=\frac{L}{\mu}$ (Section \ref{sec:miso-cyclic}).

\subsection{Storage costs}

Another difference compared to the SAG method is that we store both
gradients $f_{i}^{\prime}(\phi_{i})$ and points $\phi_{i}$. We do
not actually need twice as much memory however, as they can be stored
summed together. In particular we can store the quantities $p_{i}=f_{i}^{\prime}(\phi_{i})-\alpha\mu\phi_{i}$,
and use the update rule $w=-\frac{1}{\alpha\mu n}\sum_{i}p_{i}$.
This trick does not work when step lengths are adjusted during optimisation
unfortunately.

When using this trick it would on the surface appear that storage
of $\phi_{i}$ is a disadvantage when the gradients $f_{i}^{\prime}(\phi_{i})$
are sparse but the $\phi_{i}$ values are not. However, when is $f_{i}$
is strongly convex (which is one of our assumptions) this can not
occur. 

There is an additional possible method for avoiding the storage of
the $\phi_{i}$ values. If we can easily evaluate the gradient of
the convex conjugate of each $f_{i}$, then we can use the relation
(Section \ref{sec:convex-conj}):
\[
\phi_{i}=f_{i}^{*\prime}\left(f_{i}^{\prime}(\phi_{i})\right).
\]

This is possible because for strongly convex functions there is an
isomorphism between the dual space that the gradients live in and
the primal space.

\section{Permutation \& the Importance of Randomness}

\label{sec:randomness}

One of the most interesting aspects Finito and the other fast incremental
gradient methods is the random choice of index at each iteration.
We are not in a stochastic approximation setting, so there is no inherent
randomness in the problem. Yet it seems that randomisation is required
for Finito. It diverges in practice if a cyclic access order is used.
It is hard to emphasise enough the importance of randomness here.
The technique of pre-permuting the data, then doing in order passes
after that, is not enough. Even reducing the step size in SAG or Finito
by 1 or 2 orders of magnitude does not fix convergence.

\label{sec:permuted} The permuted ordering described in Section \ref{subsec:access-orders}
is particularly well suited to use with Finito. Recall that for the
permuted ordering, each step within an epoch the data is sampled without
replacement from the points not accessed yet in that epoch. In practice,
this approach does not give any speed-up with SAG, however it works
spectacularly well with Finito. We see speed-ups of up to a factor
of two using this approach. This is one of the major differences in
practice between SAG and Finito. We should note that we have no theory
to support this case however. 

The SDCA method is also sometimes used with a permuted ordering \citep{sdca},
our experiments in Section \ref{sec:finito-experiments} show that
this sometimes results in a speed-up over uniform random sampling,
although it does not appear to be as reliable as with Finito.

\section{Experiments}

\label{sec:finito-experiments}

In this section we compare Finito, SAG, SDCA and LBFGS. The SVRG method
was not published at the time these experiments where run. We only
consider problems where the regulariser is large enough so that the
big-data condition holds, as this is the case our theory supports.
However, in practice our method can be used with smaller step sizes
in the more general case, in much the same way as SAG. 

Since we do not know the Lipschitz smoothness constant for these problems
exactly, the SAG method was run for a variety of step sizes, with
the one that gave the fastest rate of convergence plotted. The best
step size for SAG is usually not what the theory suggests. \citet{sag}
suggest using $\frac{1}{L}$ instead of the theoretical rate $\frac{1}{16L}$.
For Finito, we find that using $\alpha=2$ is the fastest rate when
the big-data condition holds for any $\beta>1$. This is the step
suggested by our theory when $\beta=2$. Interestingly, reducing $\alpha$
to 1 does not improve the convergence rate. Instead, we see no further
improvement in our experiments.

For both SAG and Finito we used a different step size rule than suggested
by the theory for the first pass. For Finito, during the first pass,
since we do not have derivatives for each $\phi_{i}$ yet, we simply
sum over the $k$ terms seen so far 
\[
w^{k}=\frac{1}{k}\sum_{i}^{k}\phi_{i}^{k}-\frac{1}{\alpha\mu k}\sum_{i}^{k}f_{i}^{\prime}(\phi_{i}^{k}),
\]
 where we process data points in cyclic order for the first pass only.
A similar trick is suggested by \citet{sag} for SAG.

For our test problems we choose log loss for 3 binary classification
datasets, and quadratic loss for 2 regression tasks. For classification,
we tested on the ijcnn1 and covtype datasets \footnote{\url{http://www.csie.ntu.edu.tw/~cjlin/libsvmtools/datasets/binary.html}},
as well as MNIST\footnote{\url{http://yann.lecun.com/exdb/mnist/}}
classifying 0-4 against 5-9. For regression, we choose the two datasets
from the UCI repository: the million song year regression dataset,
and the slice-localisation dataset. The training portion of the datasets
are of size $5.3\times10^{5}$, $5.0\times10^{4}$, $\,6.0\times10^{4}$,
$4.7\times10^{5}$ and $5.3\times10^{4}$ respectively.

\begin{figure}
\begin{centering}
\subfloat[MNIST]{\protect\includegraphics[width=0.7\textwidth]{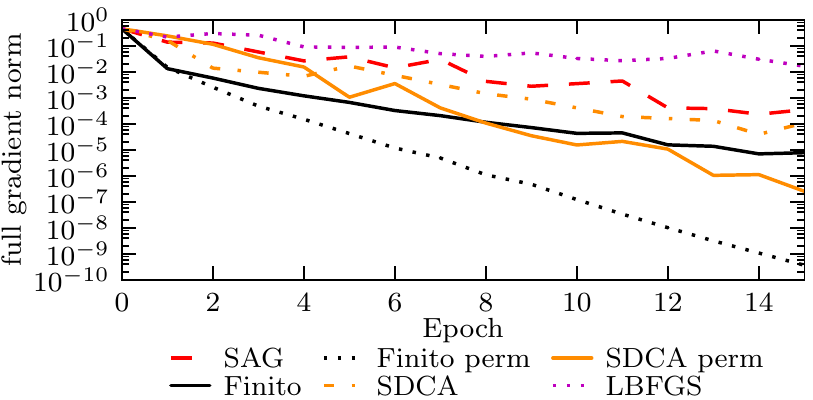}

}
\par\end{centering}

\begin{centering}
\subfloat[ijcnn1]{\protect\includegraphics[width=0.7\textwidth]{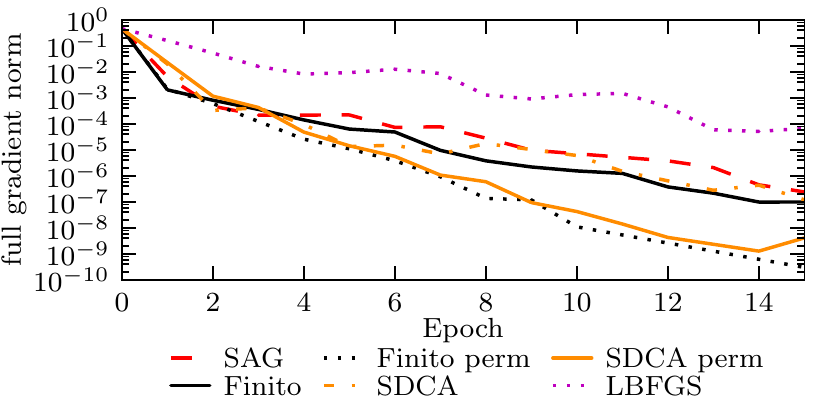}

}
\par\end{centering}

\begin{centering}
\subfloat[covtype]{\protect\includegraphics[width=0.7\textwidth]{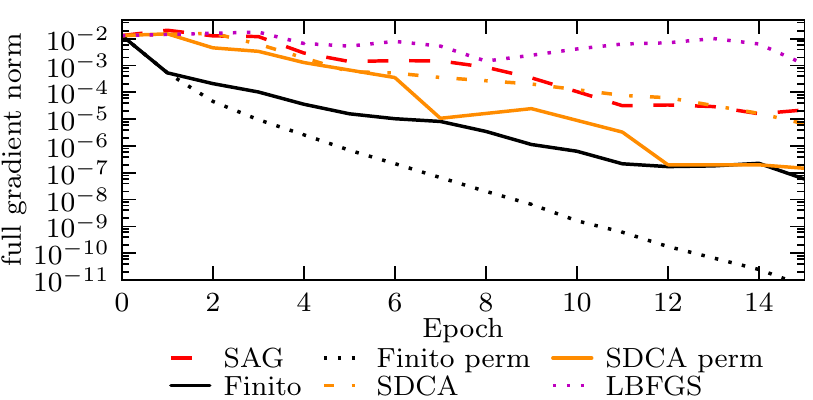}

}\protect\caption{\label{fig:finito-classification}Classification tasks}

\par\end{centering}

\end{figure}

\begin{figure}
\begin{centering}
\subfloat[million song]{\protect\includegraphics[width=0.7\textwidth]{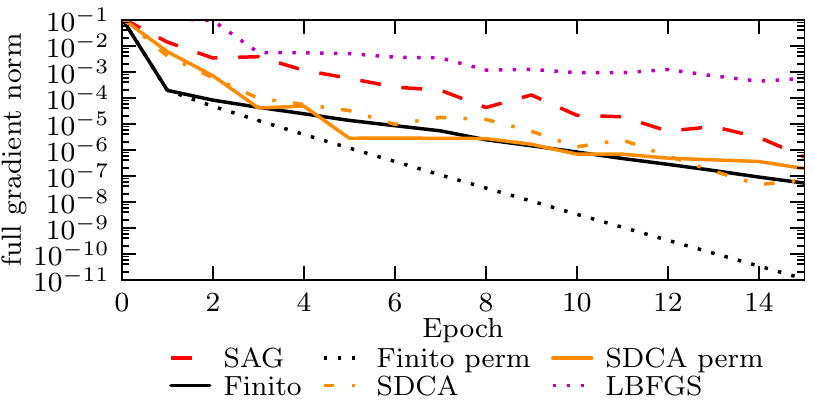}

}
\par\end{centering}

\begin{centering}
\subfloat[slice]{\protect\includegraphics[width=0.7\textwidth]{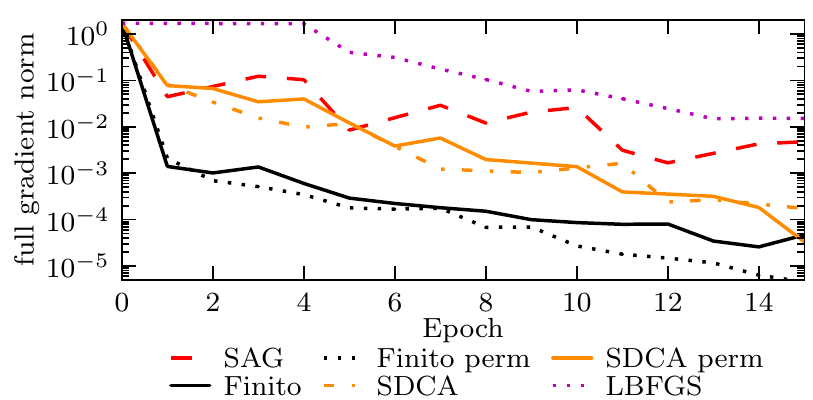}

}\protect\caption{\label{fig:finito-regression}Regression tasks}

\par\end{centering}

\end{figure}

Figures \ref{fig:finito-classification} \& \ref{fig:finito-regression}
shows the results of our experiments. Firstly we can see that LBFGS
is not competitive with any of the incremental gradient methods considered.
Secondly, the non-permuted SAG, Finito and SDCA often converge at
very similar rates. The observed differences are usually down to the
speed of the very first pass, where SAG and Finito are using the above
mentioned trick to speed their convergence. After the first pass,
the slopes of the line are usually comparable. When considering the
methods with permutation each pass, we see a clear advantage for Finito.
Interestingly, it gives very flat lines, indicating very stable convergence.
SAG with permutation is not shown as it is much slower and unstable
than randomised SAG.

\section{The MISO Method}

The MISO \citep{miso} method is an incremental gradient method applicable
when each of the terms $f_{i}$ can easily be majorised (that is upper
bounded everywhere) by known functions. When quadratic functions are
chosen for the majorisation step, the method has the same update as
Finito but with an alternative and much smaller step size. MISO maintains
the following upper bound on $f(x)$ at each step:

\[
B(x)=\frac{1}{n}\sum_{i}f_{i}(\phi_{i})+\frac{1}{n}\sum_{i}\left\langle f_{i}^{\prime}(\phi_{i}),x-\phi_{i}\right\rangle +\frac{L}{2n}\sum_{i}\left\Vert x-\phi_{i}\right\Vert ^{2}.
\]

This is just the sum of the Lipschitz smoothness upper bounds around
each $f_{i}$, at different points $\phi_{i}$. The method alternates
updating randomly chosen $\phi_{j}$ values, using the update:
\[
\phi_{j}^{k+1}=w^{k}\triangleq\arg\min_{x}B^{k}(x).
\]

This update is sensible. We are using the minimum of the upper bound
as our best guess at the solution, just like in gradient descent and
Newton's method. Taking the gradient of $B$ to zero gives an explicit
formula for $w^{k}$ of:
\begin{equation}
w^{k}=\bar{\phi}-\frac{1}{Ln}\sum_{i}f_{i}^{\prime}(\phi_{i}),\label{eq:miso-update}
\end{equation}
which is identical to Finito but with an inverse step size of $L$
instead of $2\mu$. Since the condition number of practical problems
is large, generally in the thousands to millions, this step size is
dramatically smaller. 

\citet{miso} establishes the following convergence rate for this
method in the strongly convex case:
\begin{equation}
\mathbb{E}\left[f(x^{k})-f(x^{*})\right]\leq\left(1-\frac{\mu}{(\mu+L)n}\right)^{k-1}\frac{L}{2}\left\Vert w^{0}-w^{*}\right\Vert ^{2}.\label{eq:miso-basic-rate}
\end{equation}
Note that on a per epoch basis, this rate essentially matches that
of gradient descent, as $\left(1-\frac{\mu}{(\mu+L)n}\right)^{n}\approx\left(1-\frac{\mu}{L}\right)$.
This approximation is described further in Lemma \ref{lem:bernoulli}.
Gradient descent is not considered a competitive method for optimisation,
and likewise, experiments suggest that this method is even slower
in practice than gradient descent with line searches. In Theorem \ref{thm:miso-fast-rate}
(Chapter \ref{chap:inc-discus}) we improve on this result, establishing
a convergence rate with a factor of $2$ better geometric constant:

\[
\mathbb{E}\left[f(x^{k})-f(x^{*})\right]\leq\left(1-\frac{\mu}{(\mu+L)n}\right)^{2k}\frac{2n}{\mu}\left\Vert f^{\prime}(\phi^{0})\right\Vert ^{2},
\]

although this result is for the specific form of MISO given in Equation
\ref{eq:miso-update} rather than the general upper bound minimisation
scheme they describe.

A recent technical report \citep{miso2} gives a proof of the MISO
method's convergence rate under the larger step sizes we consider
with Finito. Although this no longer fits into the upper bound minimisation
framework of MISO, they refer to this as the MISO$\mu$ method. Their
report was released while Finito was under peer review. They establish
for a step size $\alpha=1$ (Compared to our $\alpha=2$ rate used
in the Finito proof, see Equation \ref{eq:main-finito-update}) a
rate of 
\[
\mathbb{E}\left[f(x^{k})-f(x^{*})\right]\leq\left(1-\frac{1}{3n}\right)^{k}\frac{2n}{\mu}\left\Vert f^{\prime}(\phi^{0})\right\Vert ^{2}.
\]

This $1-\frac{1}{3n}$ constant is slightly worse than the $1-\frac{1}{2n}$
constant we get. They make the same big-data assumption $n\geq2L/\mu$
that we do.

\section{A Primal Form of SDCA}

At first glance, the SDCA method bears no obvious relation to Finito.
However, there is a simple transformation that can be applied to SDCA,
that makes the relation clear. This leads to a novel primal formulation
of SDCA. This also allows us to construct an algorithm that sits at
the midpoint between SDCA and Finito.

\begin{algorithm}
Initialise $\phi_{i}^{0}=\phi^{0}$ for some initial$\phi^{0}$ and
compute the table of gradients $f_{i}^{\prime}(\phi_{i}^{0})$ for
each $i$ .

At step $k+1$:
\begin{enumerate}
\item Pick an index $j$ uniformly at random.
\item \label{enu:psdca-phij}Compute $\phi_{j}^{k+1}=\text{prox}_{1/\lambda}^{f_{j}}(z)$,
where 
\begin{equation}
z=-\lambda\sum_{i\neq j}^{n}f_{i}^{\prime}(\phi_{i}^{k}),\label{eq:sdca-z}
\end{equation}
\[
\text{and \,\ }\lambda=\frac{1}{\mu n}.
\]

\item \label{enu:psdca-grad-step}Store the gradient $f_{j}^{\prime}(\phi_{j}^{k+1})=\frac{1}{\lambda}\left(z-\phi_{j}^{k+1}\right)$
in the table at location $j$. For $i\neq j$, the table entries are
unchanged ($\ensuremath{f_{i}^{\prime}(\phi_{i}^{k+1})=f_{i}^{\prime}(\phi_{i}^{k})}$).
\end{enumerate}
At completion, return $x^{k}=-\lambda\sum_{i}^{n}f_{i}^{\prime}(\phi_{i}^{k})$.

\protect\caption{\label{alg:primal-sdca}Primal SDCA}

\end{algorithm}

The primal version of SDCA is given as Algorithm \ref{alg:primal-sdca}.
At each step, this algorithm evaluates the proximal operator of a
term $f_{j}$ at a particular point $z$, yielding the new point $\phi_{j}^{k+1}$.
The gradient table is then updated with $f_{j}^{\prime}(\phi_{j}^{k+1})$.
The point $z$ is just a constant times the table's average gradient,
excluding location $j$. Compare this to the other fast incremental
gradient methods that use a table average, such as SDCA and Finito,
where in those cases the average includes the location $j$.

It's not immediately obvious that the $f_{j}^{\prime}(\phi_{j}^{k+1})$
update is actually computing $f_{j}^{\prime}(\phi_{j}^{k+1})$ in
Algorithm \ref{alg:primal-sdca}. To see why it works, consider the
optimality condition of the proximal operator from Step \ref{enu:psdca-phij},
found by taking the definition of the proximal operator and setting
its gradient to zero:
\[
f_{j}^{\prime}(\phi_{j}^{k+1})+\frac{1}{\lambda}(\phi_{j}^{k+1}-z)=0.
\]

Relating this to the update in Step \ref{enu:psdca-grad-step} shows
that the step does indeed compute the gradient.

We claim that this algorithm is exactly equivalent to SDCA when exact
coordinate minimisation is used (the standard variant). In fact, the
simple relation $\alpha_{i}^{k}=f_{i}^{\prime}(\phi_{i}^{k})$ holds.
Note that in SDCA the quantities $\phi_{i}$ are not used explicitly,
whereas in our notation here we have made them explicit.
\begin{thm}
Primal SDCA (Algorithm \ref{alg:primal-sdca}) is equivalent to SDCA
(Algorithm \ref{alg:sdca-exact}) when each $\alpha_{i}^{0}$ is initialised
to $f_{i}^{\prime}(\phi_{i}^{0})$.

In particular the iterate $x^{k}$ defined as $x^{k}=-\lambda\sum_{i}^{n}f_{i}^{\prime}(\phi_{i}^{k})$
is identical to the SDCA iterate $x^{k}=-\frac{1}{\mu n}\sum_{i}^{n}\alpha_{i}$,
due to the relation $\alpha_{i}^{k}=f_{i}^{\prime}(\phi_{i}^{k})$.\end{thm}
\begin{proof}
We will prove this by induction, by showing that $\alpha_{i}^{k}=f_{i}^{\prime}(\phi_{i}^{k})$
at each step. For the base case, we have assumed that $\alpha_{i}^{0}$
is initialised to $f_{i}^{\prime}(\phi_{i}^{0})$, so it holds by
construction. Next suppose that at step $k$, $\alpha_{i}^{k}=f_{i}^{\prime}(\phi_{i}^{k})$.
Then we just need to show that a single step in Algorithm \ref{alg:primal-sdca}
gives $\alpha_{j}^{k+1}=f_{j}^{\prime}(\phi_{j}^{k+1})$.

The relation between the SDCA and the primal variant we describe stems
from the fact that the standard dual version performs a minimisation
of the following form at each step:
\begin{equation}
\alpha_{j}^{k+1}=\arg\min_{\alpha}\left[f_{j}^{*}(\alpha)+\frac{\mu n}{2}\left\Vert x^{k}-\frac{1}{\mu n}\left(\alpha-\alpha_{j}^{k}\right)\right\Vert ^{2}\right].\label{eq:sdca-step}
\end{equation}
Here $\alpha$ are dual variables and $x^{k}$ the current primal
point. As noted by \citet{sdca-accel}, this is actually an instance
of the proximal operator of the convex conjugate of $f_{j}$. Recall
the definition of the proximal operator of a function $f$ with weight
$\gamma$:

\[
\text{prox}_{\gamma}^{f}(v)=\arg\min_{x}\left\{ f(x)+\frac{\gamma}{2}\left\Vert x-v\right\Vert ^{2}\right\} .
\]

Clearly the exact coordinate step (Equation \ref{eq:sdca-step}) is
the proximal operator of $f^{*}$ in the following form:
\begin{eqnarray}
\alpha_{j}^{k+1} & = & \mu n\cdot\text{prox}_{\mu n}^{f^{*}}(x^{k}+\frac{1}{\mu n}\alpha_{j}^{k}).\label{eq:dual-prox-sdca}\\
 & = & \mu n\cdot\text{prox}_{\mu n}^{f^{*}}(-\frac{1}{\mu n}\sum_{i\neq j}\alpha_{i}^{k}).\\
 & = & \frac{1}{\lambda}\text{prox}_{1/\lambda}^{f^{*}}(z^{k})\quad\text{(in notation from Eq \ref{eq:sdca-z})}\label{eq:aj-prox-def}
\end{eqnarray}
In the last step we have used our assumption that $\alpha_{i}^{k}=f_{i}^{\prime}(\phi_{i}^{k})$.
Our primal formulation exploits this definition of $\alpha_{j}^{k+1}$
using a relation between the proximal operator of a function and its
convex conjugate known as the Moreau decomposition:
\[
\text{prox}^{f^{*}}(v)=v-\text{prox}^{f}(v).
\]

See Section \ref{sec:convex-conj} for more details on the Moreau
decomposition. This decomposition allows us to compute the proximal
operator of conjugate via the primal proximal operator. As this is
the only use in the basic SDCA method of the conjugate function, applying
this decomposition allows us to completely eliminate the ``dual''
aspect of the algorithm. The same trick can be used to interpret Dykstra's
set intersection as a primal algorithm instead of a dual block coordinate
descent algorithm \citep{prox-dk}.

We now apply the Moreau decomposition to our primal SDCA (Algorithm
\ref{alg:primal-sdca}) formulation's update : $\phi_{j}^{k+1}=\text{prox}_{1/\lambda}^{f_{j}}(-\lambda\sum_{i\neq j}^{n}f_{i}^{\prime}(\phi_{i}^{k}))$,
giving:
\begin{eqnarray*}
\phi_{j}^{k+1} & = & \text{prox}_{1/\lambda}^{f_{j}}(z)\\
 & = & z-\text{prox}_{1/\lambda}^{f_{j}^{*}}(z).
\end{eqnarray*}
\begin{equation}
\therefore\text{prox}_{1/\lambda}^{f_{j}^{*}}(z)=z-\phi_{j}^{k+1}.\label{eq:z-phij-relation}
\end{equation}

Now consider Step \ref{enu:psdca-grad-step} of Algorithm \ref{alg:primal-sdca}.
Combining with Equation \ref{eq:z-phij-relation} and $\alpha_{j}^{k+1}=f_{j}^{\prime}(\phi_{j}^{k+1})$
we have: 
\begin{eqnarray*}
\frac{1}{\lambda}\left(z-\phi_{j}^{k+1}\right) & = & \frac{1}{\lambda}\cdot\text{prox}_{1/\lambda}^{f_{j}^{*}}(z).\\
 & = & \alpha_{j}^{k+1}\;\text{(Equation \ref{eq:aj-prox-def})}.
\end{eqnarray*}

This gives that the update for $\alpha_{j}^{k+1}$ in SDCA is identical
to the $f_{j}^{\prime}(\phi_{j}^{k+1})$ update in Algorithm \ref{alg:primal-sdca}. \end{proof}
\begin{rem}
The initialisation of SDCA typically takes each $\alpha_{i}$ as the
zero vector. This is a somewhat unnatural initialisation to use in
the primal variant, where it is more natural to take $\phi_{i}^{0}=0$
for all $i$. The SDCA style initialisation can be done using the
assignment $\phi_{i}^{0}=f_{i}^{*\prime}(0)$ if desired, but that
requires working with the convex conjugate function.
\end{rem}

\section{Prox-Finito: a Novel Midpoint Algorithm}

\label{sec:prox-finito}

The primal form of SDCA has some resemblance to Finito; the primary
difference being that it assumes strong convexity is induced by a
separate strongly convex regulariser, rather than each $f_{i}$ being
strongly convex. We now show how to modify SDCA so that it works without
a separate regulariser. 

Consider the dual problem maximised by SDCA, using our relation $\alpha_{i}=f_{i}^{\prime}(\phi_{i})$,
and $x^{k}=-\lambda\sum_{i}^{n}f_{i}^{\prime}(\phi_{i}^{k})$. Using
the Lagrangian from Equation \ref{eq:sdca-lagrangian}:

\begin{eqnarray*}
D & = & \min_{w,x_{1},\dots x_{n},}\left\{ \frac{1}{n}\sum_{i=1}^{n}f_{i}(x_{i})+\frac{\mu}{2}\left\Vert w\right\Vert ^{2}+\frac{1}{n}\sum_{i}\left\langle \alpha_{i},w-x_{i}\right\rangle \right\} \\
 & = & \min_{w}\left\{ \frac{1}{n}\sum_{i=1}^{n}\min_{x_{i}}\left[f_{i}(x_{i})+\left\langle f_{i}^{\prime}(\phi_{i}^{k}),w-x_{i}\right\rangle \right]+\frac{\mu}{2}\left\Vert w\right\Vert ^{2}\right\} .
\end{eqnarray*}

Notice now that a minimiser of each inner problem can be found by
setting the gradient to zero: $f^{\prime}(x_{i})-f^{\prime}(\phi_{i}^{k})=0$,
therefore we can take $x_{i}=\phi_{i}^{k}$. 
\begin{eqnarray}
D & = & \min_{w}\left\{ \frac{1}{n}\sum_{i=1}^{n}\left[f_{i}(\phi_{i}^{k})+\left\langle f_{i}^{\prime}(\phi_{i}^{k}),w-\phi_{i}^{k}\right\rangle \right]+\frac{\mu}{2}\left\Vert w\right\Vert ^{2}\right\} .\nonumber \\
 & = & \frac{1}{n}\sum_{i=1}^{n}\left[f_{i}(\phi_{i}^{k})+\left\langle f_{i}^{\prime}(\phi_{i}^{k}),x^{k}-\phi_{i}^{k}\right\rangle \right]+\frac{\mu}{2}\left\Vert x^{k}\right\Vert ^{2}.\label{eq:sdca-alt}
\end{eqnarray}
Using this form of the dual, the coordinate step for coordinate $j$
can be seen as the minimiser of the following function with respect
to $x$: 
\[
A_{j}^{k}(x)=\frac{1}{n}f_{j}(x)+\frac{1}{n}\sum_{i\neq j}^{n}\left[f_{i}(\phi_{i}^{k})+\left\langle f_{i}^{\prime}(\phi_{i}^{k}),x-\phi_{i}^{k}\right\rangle \right]+\frac{\mu}{2}\left\Vert x\right\Vert ^{2}.
\]

So each step of SDCA:
\begin{enumerate}
\item Replaces the linear lower bound around one of the functions $f_{j}$
with the explicit function $f_{j}$, yielding the global lower bound
$A_{j}^{k}$;
\item Sets $x^{k+1}$ to be the minimiser of $A_{j}^{k}$;
\item Then re-linearises $f_{j}$ around the point $x^{k+1}$ by taking
$\phi_{j}^{k+1}=x^{k+1}$.
\end{enumerate}
Now notice that Equation \ref{eq:sdca-alt} is the standard convexity
lower bound in terms of each $f_{i}$ summed together, added to the
explicit regulariser term. Since Finito works with strongly convex
functions without an explicit regulariser, it is natural to replace
this expression with the strong convexity lower bound:
\[
B=\frac{1}{n}\sum_{i=1}^{n}\left[f_{i}(\phi_{i}^{k})+\left\langle f_{i}^{\prime}(\phi_{i}^{k}),x^{k}-\phi_{i}^{k}\right\rangle +\frac{\mu}{2}\left\Vert x^{k}-\phi_{i}^{k}\right\Vert ^{2}\right].
\]

Now applying the same procedure as SDCA uses to this alternative expression
yields Algorithm \ref{alg:prox-finito}.

\begin{algorithm}
\begin{enumerate}
\item Pick an index $j$ uniformly at random.
\item Compute $\phi_{j}^{k+1}=\text{prox}_{\eta}^{f_{j}}(z)$, where: 
\begin{equation}
z=\frac{1}{n-1}\sum_{i\neq j}^{n}\phi_{i}^{k}-\frac{1}{\eta}\sum_{i\neq j}^{n}f_{i}^{\prime}(\phi_{i}^{k}),\label{eq:prox-finito-z}
\end{equation}
\[
\text{and \,\ }\eta=\mu(n-1).
\]

\item \label{enu:psdca-grad-step-1}Store the gradient $f_{j}^{\prime}(\phi_{j}^{k+1})=\eta\left(z-\phi_{j}^{k+1}\right)$
in the table at location $j$. For $i\neq j$, the table entries are
unchanged ($\ensuremath{f_{i}^{\prime}(\phi_{i}^{k+1})=f_{i}^{\prime}(\phi_{i}^{k})}$).
\end{enumerate}
At completion, return $x^{k}=\frac{1}{n}\sum_{i}\phi_{i}^{k}-\frac{1}{\mu n}\sum_{i}f_{i}^{\prime}(\phi_{i}^{k})$.

\protect\caption{\label{alg:prox-finito}Prox-Finito}
\end{algorithm}

Like SDCA, Prox-Finito has a linear convergence rate. The geometric
factor is $1-\frac{\mu}{\mu n+L-\mu}$ instead of $1-\frac{\mu}{\mu n+L}$,
as the Lipschitz smoothness constant now includes the strong convexity
constant due to the assumption that strong convexity is induced by
each $f_{i}$ instead of a separate regulariser. We establish this
convergence rate directly in Section \ref{sec:prox-finito-theory}.

\subsection{Prox-Finito relation to Finito}

We constructed Prox-Finito as a natural modification of the SDCA algorithm.
We now show how it is closely related to Finito. Consider again the
lower bound used in Prox-Finito:

\[
B=\frac{1}{n}\sum_{i=1}^{n}\left[f_{i}(\phi_{i}^{k})+\left\langle f_{i}^{\prime}(\phi_{i}^{k}),x^{k}-\phi_{i}^{k}\right\rangle +\frac{\mu}{2}\left\Vert x^{k}-\phi_{i}^{k}\right\Vert ^{2}\right].
\]

The update of Finito is also a minimisation involving this bound.
Instead of replacing $f_{j}$ with the true function then minimising,
the minimisation is done as-is on $B$ with respect to $x^{k}$. That
yields the equation:
\[
x^{k}=\frac{1}{n}\sum_{i}^{n}\phi_{i}^{k}-\frac{1}{\mu n}\sum_{i}^{n}f_{i}^{\prime}(\phi_{i}^{k}).
\]

Instead of using this update directly, Finito takes a more conservative
step, smaller by a factor of $\alpha=2:$ 
\[
x^{k}=\frac{1}{n}\sum_{i}^{n}\phi_{i}^{k}-\frac{1}{\alpha\mu n}\sum_{i}^{n}f_{i}^{\prime}(\phi_{i}^{k}).
\]

After each minimisation of $B$ with respect to $x^{k}$, we may then
minimise with respect to $\phi_{j}$, which yields $\phi_{j}=x^{k}$
as the update.

\subsection{Non-Uniform Lipschitz Constants}

The standard Prox-Finito method does not rely on the Lipschitz smoothness
constant of the gradient in the actual algorithm, just its convergence
rate is dependent on $L.$ So it potentially does not require knowledge
of $L$ to be used. Suppose instead that we know separate Lipschitz
smoothness constants $L_{i}$ for each $f_{i}$ function's gradient.
Let $\bar{L}$ be the average of these. This is also a bound on the
Lipschitz smoothness constant of the entire function's gradient $f^{\prime}$.
A straightforward modification of the Prox-Finito algorithm gives
a convergence rate dependent on $\bar{L}$ instead of $L$:
\begin{thm}
\label{thm:prox-finito-weighted} Define $c_{i}=\frac{\mu n+L_{i}-\mu}{\mu}$.
Modify the Prox-Finito algorithm so that sampling of $j$ is done
with probabilities
\[
p_{i}=\frac{1}{Z}\cdot c_{i},
\]

with $Z$ the normalisation constant so $\sum_{i}p_{i}=1$. Then the
expected convergence rate becomes:

\[
f(w^{*})-\mathbb{E}\left[B^{k}(w^{k})\right]\leq\left(1-\frac{\mu}{\left(\mu n+\bar{L}-\mu\right)}\right)^{k}\left[f(w^{*})-B^{0}(w^{0})\right].
\]
\end{thm}
\begin{proof}
See Section \ref{sub:prox-finito-weighted-proof}.\end{proof}
\begin{rem}
This modification to allow for non-uniform sampling strongly resembles
the non-uniform sampling scheme described by \citet{nes-coord}. Note
that for the other fast gradient methods, non-uniform sampling does
not have a proven convergence rate yet. For SAG, \citet{sag} consider
a simple $c_{i}=L_{i}$ weighting, which they say can diverge in practice.
\end{rem}

\section{Finito Theory}

\label{sec:finito-theory}

This section is highly technical. The main result is in Theorem \ref{thm:finito-main}.
We start by examining the expectation of a single Finito step. All
expectations in the following are over the choice of index $j$ at
step $k$, conditioned on all $\phi_{i}$ and $w$ from previous steps.
Quantities without superscripts are at their values at iteration $k$.
Note if we combine the $\bar{\phi}$ and $w$ updates together, we
get: 
\begin{equation}
w^{k+1}-w=\frac{1}{n}(w-\phi_{j})+\frac{1}{\alpha\mu n}\left[f_{j}^{\prime}(\phi_{j})-f_{j}^{\prime}(w)\right].\label{eq:w-unexp}
\end{equation}

\begin{lem}
The expectation of the Finito step is:
\[
\mathbb{E}[w^{k+1}]-w=-\frac{1}{\alpha\mu n}f^{\prime}(w).
\]

I.e. the $w$ step is a gradient descent step in expectation ($\frac{1}{\alpha\mu n}\propto\frac{1}{L}$).
A similar equality also holds for SGD and SVRG, but not for SAG or
SDCA.\end{lem}
\begin{proof}
\begin{eqnarray*}
\mathbb{E}[w{}^{k+1}]-w & = & \mathbb{E}\left[\frac{1}{n}(w-\phi_{j})-\frac{1}{\alpha\mu n}\left(f_{j}^{\prime}(w)-f_{j}^{\prime}(\phi_{j})\right)\right]\\
 & = & \frac{1}{n}(w-\bar{\phi})-\frac{1}{\alpha\mu n}f^{\prime}(w)+\frac{1}{\alpha\mu n^{2}}\sum_{i}f_{i}^{\prime}(\phi_{i}).
\end{eqnarray*}

Now simplify $\frac{1}{n}(w-\bar{\phi})$ as $-\frac{1}{\alpha\mu n^{2}}\sum_{i}f_{i}^{\prime}(\phi_{i})$,
so the only term that remains is $-\frac{1}{\alpha\mu n}f^{\prime}(w)$.
\end{proof}

\subsection{Main proof}

Our proof proceed by construction of a Lyapunov function $T$. That
is, a function that bounds a quantity of interest, and that decreases
each iteration in expectation. Our Lyapunov function $T=T_{1}+T_{2}+T_{3}+T_{4}$
is composed of the sum of the following terms:
\[
T_{1}=f(\bar{\phi}),
\]
\[
T_{2}=-\frac{1}{n}\sum_{i}f_{i}(\phi_{i})-\frac{1}{n}\sum_{i}\left\langle f_{i}^{\prime}(\phi_{i}),w-\phi_{i}\right\rangle ,
\]
\[
T_{3}=-\frac{\mu}{2n}\sum_{i}\left\Vert w-\phi_{i}\right\Vert ^{2},
\]
\[
T_{4}=\frac{\mu}{2n}\sum_{i}\left\Vert \bar{\phi}-\phi_{i}\right\Vert ^{2}.
\]

In the following set of Lemmas we determine how each term changes
between steps $k+1$ and $k$ in expectation. We use extensively that
$\phi_{j}^{k+1}=w$, and that $\phi_{i}^{k+1}=\phi_{i}$ for $i\neq j$. 
\begin{lem}
Between steps $k$ and $k+1$, the $T_{1}=f(\bar{\phi})$ term changes
as follows:
\[
\mathbb{E}[T_{1}^{k+1}]-T_{1}\leq\frac{1}{n}\left\langle f^{\prime}(\bar{\phi}),w-\bar{\phi}\right\rangle +\frac{L}{2n^{3}}\sum_{i}\left\Vert w-\phi_{i}\right\Vert ^{2}.
\]
\end{lem}
\begin{proof}
First we use the standard Lipschitz smoothness upper bound (Theorem
\ref{thm:lipschitz-ub}):
\[
f(y)\leq f(x)+\left\langle f^{\prime}(x),y-x\right\rangle +\frac{L}{2}\left\Vert x-y\right\Vert ^{2}.
\]

We can apply this using $y=\bar{\phi}^{k+1}=\bar{\phi}+\frac{1}{n}(w-\phi_{j})$
and $x=\bar{\phi}$:

\begin{eqnarray*}
f(\bar{\phi}^{k+1}) & \leq & f(\bar{\phi})+\frac{1}{n}\left\langle f^{\prime}(\bar{\phi}),w-\phi_{j}\right\rangle +\frac{L}{2n^{2}}\left\Vert w-\phi_{j}\right\Vert ^{2}.
\end{eqnarray*}

We now take expectations over $j$, giving:
\[
\mathbb{E}[f(\bar{\phi}^{k+1})]-f(\bar{\phi})\leq\frac{1}{n}\left\langle f^{\prime}(\bar{\phi}),w-\bar{\phi}\right\rangle +\frac{L}{2n^{3}}\sum_{i}\left\Vert w-\phi_{i}\right\Vert ^{2}.
\]
\end{proof}
\begin{lem}
Between steps $k$ and $k+1$, the $T_{2}=-\frac{1}{n}\sum_{i}f_{i}(\phi_{i})-\frac{1}{n}\sum_{i}\left\langle f_{i}^{\prime}(\phi_{i}),w-\phi_{i}\right\rangle $
term changes as follows:
\begin{eqnarray*}
\mathbb{E}[T_{2}^{k+1}]-T_{2} & \leq & -\frac{1}{n}T_{2}-\frac{1}{n}f(w)\\
 & + & (\frac{1}{\alpha}-\frac{\beta}{n})\frac{1}{sn^{3}}\sum_{i}\left\Vert f_{i}^{\prime}(w)-f_{i}^{\prime}(\phi_{i})\right\Vert ^{2}\\
 & + & \frac{1}{n}\left\langle \bar{\phi}-w,f^{\prime}(w)\right\rangle -\frac{1}{n^{3}}\sum_{i}\left\langle f_{i}^{\prime}(w)-f_{i}^{\prime}(\phi_{i}),w-\phi_{i}\right\rangle .
\end{eqnarray*}
\end{lem}
\begin{proof}
We introduce the notation $T_{21}=-\frac{1}{n}\sum_{i}f_{i}(\phi_{i})$
and $T_{22}=-\frac{1}{n}\sum_{i}\left\langle f_{i}^{\prime}(\phi_{i}),w-\phi_{i}\right\rangle $.
We simplify the change in $T_{21}$ first using $\phi_{j}^{k+1}=w$:
\begin{eqnarray*}
T_{21}^{k+1}-T_{21} & = & -\frac{1}{n}\sum_{i}f_{i}(\phi_{i}^{k+1})+\frac{1}{n}\sum_{i}f_{i}(\phi_{i})\\
 & = & -\frac{1}{n}\sum_{i}f_{i}(\phi_{i})+\frac{1}{n}f_{j}(\phi_{j})-\frac{1}{n}f_{j}(w)+\frac{1}{n}\sum_{i}f_{i}(\phi_{i})\\
 & = & \frac{1}{n}f_{j}(\phi_{j})-\frac{1}{n}f_{j}(w).
\end{eqnarray*}

Now we simplify the change in $T_{22}$:

\[
T_{22}^{k+1}-T_{22}=-\frac{1}{n}\sum_{i}\left\langle f_{i}^{\prime}(\phi_{i}^{k+1}),w^{k+1}-w+w-\phi_{i}^{k+1}\right\rangle -T_{22}
\]
\begin{equation}
\therefore T_{22}^{k+1}-T_{22}=-\frac{1}{n}\sum_{i}\left\langle f_{i}^{\prime}(\phi_{i}^{k+1}),w-\phi_{i}^{k+1}\right\rangle -T_{22}-\frac{1}{n}\sum_{i}\left\langle f_{i}^{\prime}(\phi_{i}^{k+1}),w^{k+1}-w\right\rangle .\label{eq:t2-start}
\end{equation}

We now simplify the first two terms using $\phi_{j}^{k+1}=w$:
\begin{eqnarray*}
-\frac{1}{n}\sum_{i}\left\langle f_{i}^{\prime}(\phi_{i}^{k+1}),w-\phi_{i}^{k+1}\right\rangle -T_{22} & = & T_{22}-\frac{1}{n}\left\langle f_{j}^{\prime}(\phi_{j}),w-\phi_{j}\right\rangle \\
 &  & +\frac{1}{n}\left\langle f_{j}^{\prime}(w),w-w\right\rangle -T_{22}\\
 & = & \frac{1}{n}\left\langle f_{j}^{\prime}(\phi_{j}),w-\phi_{j}\right\rangle .
\end{eqnarray*}
The last term of Equation \ref{eq:t2-start} expands further: 
\begin{eqnarray}
 &  & -\frac{1}{n}\sum_{i}\left\langle f_{i}^{\prime}(\phi_{i}^{k+1}),w^{k+1}-w\right\rangle \nonumber \\
 & = & -\frac{1}{n}\left\langle \sum_{i}f_{i}^{\prime}(\phi_{i})-f_{j}^{\prime}(\phi_{j})+f_{j}^{\prime}(w),w^{k+1}-w\right\rangle \\
 & = & -\frac{1}{n}\left\langle \sum_{i}f_{i}^{\prime}(\phi_{i}),w^{k+1}-w\right\rangle -\frac{1}{n}\left\langle f_{j}^{\prime}(w)-f_{j}^{\prime}(\phi_{j}),w^{k+1}-w\right\rangle .\label{eq:t2-ip-exp}
\end{eqnarray}

The second inner product term in \ref{eq:t2-ip-exp} simplifies further
using Equation \ref{eq:w-unexp}:
\begin{eqnarray*}
 &  & -\frac{1}{n}\left\langle f_{j}^{\prime}(w)-f_{j}^{\prime}(\phi_{j}),w^{k+1}-w\right\rangle \\
 & = & -\frac{1}{n}\left\langle f_{j}^{\prime}(w)-f_{j}^{\prime}(\phi_{j}),\frac{1}{n}(w-\phi_{j})+\frac{1}{\alpha\mu n}\left[f_{j}^{\prime}(\phi_{j})-f_{j}^{\prime}(w)\right]\right\rangle \\
 & = & -\frac{1}{n^{2}}\left\langle f_{j}^{\prime}(w)-f_{j}^{\prime}(\phi_{j}),w-\phi_{j}\right\rangle -\frac{1}{\alpha\mu n^{2}}\left\langle f_{j}^{\prime}(w)-f_{j}^{\prime}(\phi_{j}),f_{j}^{\prime}(\phi_{j})-f_{j}^{\prime}(w)\right\rangle .
\end{eqnarray*}

We simplify the second term: 
\begin{eqnarray*}
-\frac{1}{\alpha\mu n^{2}}\left\langle f_{j}^{\prime}(w)-f_{j}^{\prime}(\phi_{j}),f_{j}^{\prime}(\phi_{j})-f_{j}^{\prime}(w)\right\rangle  & = & \frac{1}{\alpha\mu n^{2}}\left\Vert f_{j}^{\prime}(w)-f_{j}^{\prime}(\phi_{j})\right\Vert ^{2}.
\end{eqnarray*}

Grouping all remaining terms gives:
\begin{eqnarray*}
T_{2}^{k+1}-T_{2} & \leq & \frac{1}{n}f_{j}(\phi_{j})+\frac{1}{n}\left\langle f_{j}^{\prime}(\phi_{j}),w-\phi_{j}\right\rangle -\frac{1}{n}f_{j}(w)\\
 & + & \frac{1}{\alpha\mu n^{2}}\left\Vert f_{j}^{\prime}(w)-f_{j}^{\prime}(\phi_{j})\right\Vert ^{2}-\frac{1}{n^{2}}\left\langle f_{j}^{\prime}(w)-f_{j}^{\prime}(\phi_{j}),w-\phi_{j}\right\rangle \\
 & - & \frac{1}{n}\left\langle \sum_{i}f_{i}^{\prime}(\phi_{i}),w^{k+1}-w\right\rangle .
\end{eqnarray*}

We now take expectations of each remaining term. For the bottom inner
product we use Lemma 1:
\begin{eqnarray*}
-\frac{1}{n}\left\langle \sum_{i}f_{i}^{\prime}(\phi_{i}),w^{k+1}-w\right\rangle  & = & \frac{1}{\alpha\mu n^{2}}\left\langle \sum_{i}f_{i}^{\prime}(\phi_{i}),f^{\prime}(w)\right\rangle \\
 & = & \frac{1}{n}\left\langle \bar{\phi}-w,f^{\prime}(w)\right\rangle .
\end{eqnarray*}

Taking expectations of the remaining terms is straight forward. We
get:
\begin{eqnarray*}
\mathbb{E}[T_{2}^{k+1}]-T_{2} & \leq & \frac{1}{n^{2}}\sum_{i}f_{i}(\phi_{i})-\frac{1}{n}f(w)+\frac{1}{n^{2}}\sum_{i}\left\langle f_{i}^{\prime}(\phi_{i}),w-\phi_{i}\right\rangle \\
 & + & \frac{1}{\alpha\mu n^{3}}\sum_{i}\left\Vert f_{i}^{\prime}(w)-f_{i}^{\prime}(\phi_{i})\right\Vert ^{2}-\frac{1}{n^{3}}\sum_{i}\left\langle f_{i}^{\prime}(w)-f_{i}^{\prime}(\phi_{i}),w-\phi_{i}\right\rangle \\
 & + & \frac{1}{n}\left\langle \bar{\phi}-w,f^{\prime}(w)\right\rangle .
\end{eqnarray*}
\end{proof}
\begin{lem}
Between steps $k$ and $k+1$, the $T_{3}=-\frac{\mu}{2n}\sum_{i}\left\Vert w-\phi_{i}\right\Vert ^{2}$
term changes as follows:
\begin{eqnarray*}
\mathbb{E}[T_{3}^{k+1}]-T_{3} & = & -(1+\frac{1}{n})\frac{1}{n}T_{3}\\
 & + & \frac{1}{\alpha n}\left\langle f^{\prime}(w),w-\bar{\phi}\right\rangle -\frac{1}{2\alpha^{2}\mu n^{3}}\sum_{i}\left\Vert f_{i}^{\prime}(\phi_{i})-f_{i}^{\prime}(w)\right\Vert ^{2}.
\end{eqnarray*}
\end{lem}
\begin{proof}
We expand as:
\begin{eqnarray}
T_{3}^{k+1} & = & -\frac{\mu}{2n}\sum_{i}\left\Vert w^{k+1}-\phi_{i}^{k+1}\right\Vert ^{2}\nonumber \\
 & = & -\frac{\mu}{2n}\sum_{i}\left\Vert w^{k+1}-w+w-\phi_{i}^{k+1}\right\Vert ^{2}\\
 & = & -\frac{\mu}{2}\left\Vert w^{k+1}-w\right\Vert ^{2}-\frac{\mu}{2n}\sum_{i}\left\Vert w-\phi_{i}^{k+1}\right\Vert ^{2}\label{eq:t3-initial}\\
 &  & -\frac{\mu}{n}\sum_{i}\left\langle w^{k+1}-w,w-\phi_{i}^{k+1}\right\rangle .
\end{eqnarray}

We expand the three terms on the right separately. For the first term:
\begin{eqnarray}
-\frac{\mu}{2}\left\Vert w^{k+1}-w\right\Vert ^{2} & = & -\frac{\mu}{2}\left\Vert \frac{1}{n}(w-\phi_{j})+\frac{1}{\alpha\mu n}\left(f_{j}(\phi_{j})-f_{j}(w)\right)\right\Vert ^{2}\nonumber \\
 & =- & \frac{\mu}{2n^{2}}\left\Vert w-\phi_{j}\right\Vert ^{2}-\frac{1}{2\alpha^{2}\mu n^{2}}\left\Vert f_{j}(\phi_{j})-f_{j}(w)\right\Vert ^{2}\nonumber \\
 & - & \frac{1}{\alpha n^{2}}\left\langle f_{j}(\phi_{j})-f_{j}(w),w-\phi_{j}\right\rangle .\label{eq:wkp-minus-w-expansion}
\end{eqnarray}

For the second term of Equation \ref{eq:t3-initial}, using $\phi_{j}^{k+1}=w$:
\begin{eqnarray*}
-\frac{\mu}{2n}\sum_{i}\left\Vert w-\phi_{i}^{k+1}\right\Vert ^{2} & = & -\frac{\mu}{2n}\sum_{i}\left\Vert w-\phi_{i}\right\Vert ^{2}+\frac{\mu}{2n}\left\Vert w-\phi_{j}\right\Vert ^{2}\\
 & = & T_{3}+\frac{\mu}{2n}\left\Vert w-\phi_{j}\right\Vert ^{2}.
\end{eqnarray*}

For the third term of Equation \ref{eq:t3-initial}:
\begin{eqnarray}
 &  & -\frac{\mu}{n}\sum_{i}\left\langle w^{k+1}-w,w-\phi_{i}^{k+1}\right\rangle \nonumber \\
 & = & -\frac{\mu}{n}\sum_{i}\left\langle w^{k+1}-w,w-\phi_{i}\right\rangle +\frac{\mu}{n}\left\langle w^{k+1}-w,w-\phi_{j}\right\rangle \\
 & = & -\mu\left\langle w^{k+1}-w,w-\frac{1}{n}\sum_{i}\phi_{i}\right\rangle +\frac{\mu}{n}\left\langle w^{k+1}-w,w-\phi_{j}\right\rangle .\label{eq:ip-terms-t3}
\end{eqnarray}

The second inner product term in Equation \ref{eq:ip-terms-t3} becomes
(using Equation \ref{eq:w-unexp}):
\begin{eqnarray*}
\frac{\mu}{n}\left\langle w^{k+1}-w,w-\phi_{j}\right\rangle  & = & \frac{\mu}{n}\left\langle \frac{1}{n}(w-\phi_{j})+\frac{1}{\alpha\mu n}\left[f_{j}^{\prime}(\phi_{j})-f_{j}^{\prime}(w)\right],w-\phi_{j}\right\rangle \\
 & = & \frac{\mu}{n^{2}}\left\Vert w-\phi_{j}\right\Vert ^{2}+\frac{1}{\alpha n^{2}}\left\langle f_{j}^{\prime}(\phi_{j})-f_{j}^{\prime}(w),w-\phi_{j}\right\rangle .
\end{eqnarray*}

Notice that the inner product term here cancels with the one in \ref{eq:wkp-minus-w-expansion}.

Now we can take expectations of each remaining term. Recall that $\mathbb{E}[w^{k+1}]-w=-\frac{1}{\alpha\mu n}f^{\prime}(w)$,
so the first inner product term in \ref{eq:ip-terms-t3} becomes:
\[
-\mu\mathbb{E}\left[\left\langle w^{k+1}-w,w-\frac{1}{n}\sum_{i}\phi_{i}\right\rangle \right]=\frac{1}{\alpha n}\left\langle f^{\prime}(w),w-\bar{\phi}\right\rangle .
\]

All other terms don't simplify under expectations. So the result is:
\begin{eqnarray*}
\mathbb{E}[T_{3}^{k+1}]-T_{3} & = & (\frac{1}{2}-\frac{1}{n})\frac{\mu}{n^{2}}\sum_{i}\left\Vert w-\phi_{i}\right\Vert ^{2}\\
 & + & \frac{1}{\alpha n}\left\langle f^{\prime}(w),w-\bar{\phi}\right\rangle -\frac{1}{2\alpha^{2}\mu n^{3}}\sum_{i}\left\Vert f_{i}(\phi_{i})-f_{i}(w)\right\Vert ^{2}.
\end{eqnarray*}
\end{proof}
\begin{lem}
Between steps $k$ and $k+1$, the $T_{4}=\frac{\mu}{2n}\sum_{i}\left\Vert \bar{\phi}-\phi_{i}\right\Vert ^{2}$
term changes as follows:

\[
\mathbb{E}[T_{4}^{k+1}]-T_{4}=-\frac{\mu}{2n^{2}}\sum_{i}\left\Vert \bar{\phi}-\phi_{i}\right\Vert ^{2}+\frac{\mu}{2n}\left\Vert \bar{\phi}-w\right\Vert ^{2}-\frac{\mu}{2n^{3}}\sum_{i}\left\Vert w-\phi_{i}\right\Vert ^{2}.
\]
\end{lem}
\begin{proof}
Note that $\bar{\phi}^{k+1}-\bar{\phi}=\frac{1}{n}(w-\phi_{j})$,
so:
\begin{eqnarray*}
T_{4}^{k+1} & = & \frac{\mu}{2n}\sum_{i}\left\Vert \bar{\phi}^{k+1}-\bar{\phi}+\bar{\phi}-\phi_{i}^{k+1}\right\Vert ^{2}\\
 & = & \frac{\mu}{2n}\sum_{i}\left(\left\Vert \bar{\phi}^{k+1}-\bar{\phi}\right\Vert ^{2}+\left\Vert \bar{\phi}-\phi_{i}^{k+1}\right\Vert ^{2}+2\left\langle \bar{\phi}^{k+1}-\bar{\phi},\bar{\phi}-\phi_{i}^{k+1}\right\rangle \right)\\
 & = & \frac{\mu}{2n}\sum_{i}\left(\left\Vert \frac{1}{n}(w-\phi_{j})\right\Vert ^{2}+\left\Vert \bar{\phi}-\phi_{i}^{k+1}\right\Vert ^{2}+\frac{2}{n}\left\langle w-\phi_{j},\bar{\phi}-\phi_{i}^{k+1}\right\rangle \right).
\end{eqnarray*}
Now using $\frac{1}{n}\sum_{i}\left(\bar{\phi}-\phi_{i}^{k+1}\right)=\bar{\phi}-\bar{\phi}^{k+1}=-\frac{1}{n}(w-\phi_{j})$
to simplify the inner product term:
\begin{eqnarray}
 & = & \frac{\mu}{2n^{2}}\left\Vert w-\phi_{j}\right\Vert ^{2}+\frac{\mu}{2n}\sum_{i}\left\Vert \bar{\phi}-\phi_{i}^{k+1}\right\Vert ^{2}+\frac{\mu}{n^{2}}\left\langle w-\phi_{j},\phi_{j}-w\right\rangle \nonumber \\
 & = & \frac{\mu}{2n^{2}}\left\Vert w-\phi_{j}\right\Vert ^{2}+\frac{\,u}{2n}\sum_{i}\left\Vert \bar{\phi}-\phi_{i}^{k+1}\right\Vert ^{2}-\frac{\mu}{n}\left\Vert w-\phi_{j}\right\Vert ^{2}\nonumber \\
 & = & \frac{\mu}{2n}\sum_{i}\left\Vert \bar{\phi}-\phi_{i}^{k+1}\right\Vert ^{2}-\frac{\mu}{2n}\left\Vert w-\phi_{j}\right\Vert ^{2}\nonumber \\
 & = & \frac{\mu}{2n}\sum_{i}\left\Vert \bar{\phi}-\phi_{i}\right\Vert ^{2}-\frac{\mu}{2n}\left\Vert \bar{\phi}-\phi_{j}\right\Vert ^{2}+\frac{\mu}{2n}\left\Vert \bar{\phi}-w\right\Vert ^{2}-\frac{\mu}{2n^{2}}\left\Vert w-\phi_{j}\right\Vert ^{2}.\label{eq:right-term}
\end{eqnarray}

Taking expectations gives the result.\end{proof}
\begin{lem}
\label{lem:scaled-lemma}Take $f(x)=\frac{1}{n}\sum_{i}f_{i}(x)$,
with the big-data condition holding with constant $\beta$. Then for
any $x$ and $\phi_{i}$ vectors: 
\begin{eqnarray*}
f(x) & \geq & \frac{1}{n}\sum_{i}f_{i}(\phi_{i})+\frac{1}{n}\sum_{i}\left\langle f_{i}^{\prime}(\phi_{i}),x-\phi_{i}\right\rangle +\frac{\beta}{2\mu n^{2}}\sum_{i}\left\Vert f_{i}^{\prime}(x)-f_{i}^{\prime}(\phi_{i})\right\Vert ^{2}\\
 &  & +\frac{\beta L}{2n^{2}}\sum_{i}\left\Vert x-\phi_{i}\right\Vert ^{2}+\frac{\beta}{n^{2}}\sum_{i}\left\langle f_{i}^{\prime}(x)-f_{i}^{\prime}(\phi_{i}),\phi_{i}-x\right\rangle .
\end{eqnarray*}
\end{lem}
\begin{proof}
We apply Lemma \ref{thm:full-strong-lb} to each $f_{i}$ , but instead
of using the actual constant $L$, we use $\frac{\mu n}{\beta}+\mu$,
which under the big-data assumption is larger than $L$:

\begin{eqnarray*}
f_{i}(x) & \geq & f_{i}(\phi_{i})+\left\langle f_{i}^{\prime}(\phi_{i}),x-\phi_{i}\right\rangle +\frac{\beta}{2\mu n}\left\Vert f_{i}^{\prime}(x)-f_{i}^{\prime}(\phi_{i})\right\Vert ^{2}\\
 &  & +\frac{\beta L}{2n}\left\Vert x-\phi_{i}\right\Vert ^{2}+\frac{\beta}{n}\left\langle f_{i}^{\prime}(x)-f_{i}^{\prime}(\phi_{i}),\phi_{i}-x\right\rangle .
\end{eqnarray*}

Averaging over $i$ gives the result.
\end{proof}
We are now ready to prove our main result:
\begin{thm}
\label{thm:finito-main}Between steps $k$ and $k+1$ of the Finito
algorithm, if $\frac{2}{\alpha}-\frac{1}{\alpha^{2}}-\beta+\frac{\beta}{\alpha}\leq0$,
$\alpha\geq2$ and $\beta\geq2$ then
\[
\mathbb{E}[T^{k+1}]-T\leq-\frac{1}{\alpha n}T.
\]
\end{thm}
\begin{proof}
First recall from the Lemmas above the change in each term:
\[
\mathbb{E}[T_{1}^{k+1}]-T_{1}\leq\frac{1}{n}\left\langle f^{\prime}(\bar{\phi}),w-\bar{\phi}\right\rangle +\frac{L}{2n^{3}}\sum_{i}\left\Vert w-\phi_{i}\right\Vert ^{2},
\]
\begin{eqnarray*}
\mathbb{E}[T_{2}^{k+1}]-T_{2} & \leq & -\frac{1}{n}T_{2}-\frac{1}{n}f(w)+(\frac{1}{\alpha}-\frac{\beta}{n})\frac{1}{\mu n^{3}}\sum_{i}\left\Vert f_{i}^{\prime}(w)-f_{i}^{\prime}(\phi_{i})\right\Vert ^{2}\\
 &  & +\frac{1}{n}\left\langle \bar{\phi}-w,f^{\prime}(w)\right\rangle -\frac{1}{n^{3}}\sum_{i}\left\langle f_{i}^{\prime}(w)-f_{i}^{\prime}(\phi_{i}),w-\phi_{i}\right\rangle ,
\end{eqnarray*}
\[
\mathbb{E}[T_{3}^{k+1}]-T_{3}=-(\frac{1}{n}+\frac{1}{n^{2}})T_{3}+\frac{1}{\alpha n}\left\langle f^{\prime}(w),w-\bar{\phi}\right\rangle -\frac{1}{2\alpha^{2}\mu n^{3}}\sum_{i}\left\Vert f_{i}^{\prime}(\phi_{i})-f_{i}^{\prime}(w)\right\Vert ^{2},
\]

\[
\mathbb{E}[T_{4}^{k+1}]-T_{4}=-\frac{\mu}{2n^{2}}\sum_{i}\left\Vert \bar{\phi}-\phi_{i}\right\Vert ^{2}+\frac{\mu}{2n}\left\Vert \bar{\phi}-w\right\Vert ^{2}-\frac{\mu}{2n^{3}}\sum_{i}\left\Vert w-\phi_{i}\right\Vert ^{2}.
\]

We now combine these and group like terms to get:
\begin{eqnarray}
\mathbb{E}[T^{k+1}]-T & \leq & \frac{1}{n}\left\langle f^{\prime}(\bar{\phi}),w-\bar{\phi}\right\rangle +\frac{1}{n^{2}}\sum_{i}f_{i}(\phi_{i})\nonumber \\
 &  & -\frac{1}{n}f(w)+\frac{1}{n^{2}}\sum_{i}\left\langle f_{i}^{\prime}(\phi_{i}),w-\phi_{i}\right\rangle \nonumber \\
 &  & +(1-\frac{1}{\alpha})\frac{1}{n}\left\langle f^{\prime}(w),\bar{\phi}-w\right\rangle \nonumber \\
 &  & +(\frac{L}{\mu n}+1)\frac{\mu}{2n^{2}}\sum_{i}\left\Vert w-\phi_{i}\right\Vert ^{2}\nonumber \\
 &  & -\frac{1}{n^{3}}\sum_{i}\left\langle f_{i}^{\prime}(w)-f_{i}^{\prime}(\phi_{i}),w-\phi_{i}\right\rangle \nonumber \\
 &  & +(1-\frac{1}{2\alpha})\frac{1}{\alpha\mu n^{3}}\sum_{i}\left\Vert f_{i}^{\prime}(\phi_{i})-f_{i}^{\prime}(w)\right\Vert ^{2}\nonumber \\
 &  & +\frac{\mu}{2n}\left\Vert w-\bar{\phi}\right\Vert ^{2}-\frac{\mu}{2n^{2}}\sum_{i}\left\Vert \bar{\phi}-\phi_{i}\right\Vert ^{2}.\label{eq:finito-main-s1}
\end{eqnarray}

Next we cancel part of the first line using the strong convexity lower
bound:
\[
\frac{1}{\alpha n}\left\langle f^{\prime}(\bar{\phi}),w-\bar{\phi}\right\rangle \leq\frac{1}{\alpha n}f(w)-\frac{1}{\alpha n}f(\bar{\phi})-\frac{\mu}{2\alpha n}\left\Vert w-\bar{\phi}\right\Vert ^{2},
\]

We then pull terms occurring in $-\frac{1}{\alpha n}T$ together,
giving 
\begin{eqnarray}
\mathbb{E}[T^{k+1}]-T & \leq & -\frac{1}{\alpha n}T+(1-\frac{1}{\alpha})\frac{1}{n}\left\langle f^{\prime}(\bar{\phi})-f^{\prime}(w),w-\bar{\phi}\right\rangle \nonumber \\
 &  & +(1-\frac{1}{\alpha})\left[-\frac{1}{n}f(w)-\frac{1}{n}T_{2}\right]\nonumber \\
 &  & +(\frac{L}{\mu n}+1-\frac{1}{\alpha})\frac{\mu}{2n^{2}}\sum_{i}\left\Vert w-\phi_{i}\right\Vert ^{2}\nonumber \\
 &  & -\frac{1}{n^{3}}\sum_{i}\left\langle f_{i}^{\prime}(w)-f_{i}^{\prime}(\phi_{i}),w-\phi_{i}\right\rangle \nonumber \\
 &  & +(1-\frac{1}{2\alpha})\frac{1}{\alpha\mu n^{3}}\sum_{i}\left\Vert f_{i}^{\prime}(\phi_{i})-f_{i}^{\prime}(w)\right\Vert ^{2}\nonumber \\
 &  & +(1-\frac{1}{\alpha})\frac{\mu}{2n}\left\Vert w-\bar{\phi}\right\Vert ^{2}-(1-\frac{1}{\alpha})\frac{\mu}{2n^{2}}\sum_{i}\left\Vert \bar{\phi}-\phi_{i}\right\Vert ^{2}.\label{eq:finito-main-s2}
\end{eqnarray}

Next we use the standard inequality $\left\langle f^{\prime}(x)-f^{\prime}(y),x-y\right\rangle \geq\mu\left\Vert x-y\right\Vert ^{2}$
in the form:
\[
(1-\frac{1}{\alpha})\frac{1}{n}\left\langle f^{\prime}(\bar{\phi})-f^{\prime}(w),w-\bar{\phi}\right\rangle \leq-(1-\frac{1}{\alpha})\frac{\mu}{n}\left\Vert w-\bar{\phi}\right\Vert ^{2},
\]

which changes the bottom row of Equation \ref{eq:finito-main-s2}
to
\[
-(1-\frac{1}{\alpha})\frac{\mu}{2n}\left\Vert w-\bar{\phi}\right\Vert ^{2}-(1-\frac{1}{\alpha})\frac{\mu}{2n^{2}}\sum_{i}\left\Vert \bar{\phi}-\phi_{i}\right\Vert ^{2}.
\]

These two terms can then be grouped using Lemma \ref{lem:decomposition-of-variance}
to give
\begin{eqnarray*}
\mathbb{E}[T^{k+1}]-T & \leq & -\frac{1}{\alpha n}T+\frac{L}{2n^{3}}\sum_{i}\left\Vert w-\phi_{i}\right\Vert ^{2}\\
 &  & +(1-\frac{1}{\alpha})\left[-\frac{1}{n}f(w)-\frac{1}{n}T_{2}\right]\\
 &  & -\frac{1}{n^{3}}\sum_{i}\left\langle f_{i}^{\prime}(w)-f_{i}^{\prime}(\phi_{i}),w-\phi_{i}\right\rangle \\
 &  & +(1-\frac{1}{2\alpha})\frac{1}{\alpha\mu n^{3}}\sum_{i}\left\Vert f_{i}^{\prime}(\phi_{i})-f_{i}^{\prime}(w)\right\Vert ^{2}.
\end{eqnarray*}

We now use Lemma \ref{lem:scaled-lemma} in the following scaled form
to cancel against the $\sum_{i}\left\Vert w-\phi_{i}\right\Vert ^{2}$
term:
\begin{eqnarray*}
\frac{1}{\beta}\left[-\frac{1}{n}f(w)-\frac{1}{n}T_{2}\right] & \leq & \frac{1}{n^{3}}\sum_{i}\left\langle f_{i}^{\prime}(w)-f_{i}^{\prime}(\phi_{i}),w-\phi_{i}\right\rangle \\
 &  & -\frac{L}{2n^{3}}\sum_{i}\left\Vert w-\phi_{i}\right\Vert ^{2}-\frac{1}{2\mu n^{3}}\sum_{i}\left\Vert f_{i}^{\prime}(w)-f_{i}^{\prime}(\phi_{i})\right\Vert ^{2},
\end{eqnarray*}
and then apply the following standard equality (Theorem \ref{thm:lipschitz-lb})
to partially cancel $\sum_{i}\left\Vert f_{i}(\phi_{i})-f_{i}(w)\right\Vert ^{2}$:
\[
\left(1-\frac{1}{\alpha}-\frac{1}{\beta}\right)\left[-\frac{1}{n}f(w)-\frac{1}{n}T_{2}\right]\leq-\left(1-\frac{1}{\alpha}-\frac{1}{\beta}\right)\frac{\beta}{2\mu n^{3}}\sum_{i}\left\Vert f_{i}^{\prime}(\phi_{i})-f_{i}^{\prime}(w)\right\Vert ^{2}.
\]

Leaving us with
\[
\mathbb{E}[T^{k+1}]-T\leq-\frac{1}{\alpha n}T+(\frac{2}{\alpha}-\frac{1}{\alpha^{2}}-\beta+\frac{\beta}{\alpha})\frac{1}{2\mu n^{3}}\sum_{i}\left\Vert f_{i}^{\prime}(\phi_{i})-f_{i}^{\prime}(w)\right\Vert ^{2}.
\]

The remaining gradient norm term is non-positive under the conditions
specified in our assumptions.\end{proof}
\begin{thm}
\label{thm:lyp-to-func}The Finito Lyapunov function bounds $f(\bar{\phi})-f(w^{*})$
as follows:
\[
f(\bar{\phi}^{k})-f(w^{*})\leq\alpha T^{k}.
\]
\end{thm}
\begin{proof}
Consider the following function, which we will call $R(x)$:
\[
R(x)=\frac{1}{n}\sum_{i}f_{i}(\phi_{i})+\frac{1}{n}\sum_{i}\left\langle f_{i}^{\prime}(\phi_{i}),x-\phi_{i}\right\rangle +\frac{\mu}{2n}\sum_{i}\left\Vert x-\phi_{i}\right\Vert ^{2}.
\]

When evaluated at its minimum with respect to $x$, which we denote
$w^{\prime}=\bar{\phi}-\frac{1}{\mu n}\sum_{i}f_{i}^{\prime}(\phi_{i})$,
it is a lower bound on $f(w^{*})$ by strong convexity. However, we
are evaluating at $w=\bar{\phi}-\frac{1}{\alpha\mu n}\sum_{i}f_{i}^{\prime}(\phi_{i})$
instead in the (negated) Lyapunov function. $R$ is convex with respect
to $x$, so by definition
\[
R(w)=R\left(\left(1-\frac{1}{\alpha}\right)\bar{\phi}+\frac{1}{\alpha}w^{\prime}\right)\leq\left(1-\frac{1}{\alpha}\right)R(\bar{\phi})+\frac{1}{\alpha}R(w^{\prime}).
\]

Therefore, by the lower bounding property
\begin{eqnarray*}
f(\bar{\phi})-R(w) & \geq & f(\bar{\phi})-\left(1-\frac{1}{\alpha}\right)R(\bar{\phi})-\frac{1}{\alpha}R(w^{\prime})\\
 & \geq & f(\bar{\phi})-\left(1-\frac{1}{\alpha}\right)f(\bar{\phi})-\frac{1}{\alpha}f(w^{*})\\
 & = & \frac{1}{\alpha}\left(f(\bar{\phi})-f(w^{*})\right).
\end{eqnarray*}

Now note that $T\geq f(\bar{\phi})-R(w)$. So
\[
f(\bar{\phi})-f(w^{*})\leq\alpha T.
\]
\end{proof}
\begin{thm}
If the Finito method is initialised with all $\phi_{i}^{0}$ the same,
$\phi_{i}^{0}=\phi^{0}$, and the assumptions of Theorem \ref{thm:finito-main}
hold, then the expected convergence rate is:
\[
\mathbb{E}\left[f(\bar{\phi}^{k})\right]-f(w^{*})\leq\frac{c}{\mu}\left(1-\frac{1}{\alpha n}\right)^{k}\left\Vert f^{\prime}(\phi^{0})\right\Vert ^{2},
\]

with $c=\left(1-\frac{1}{2\alpha}\right)$.\end{thm}
\begin{proof}
By unrolling Theorem \ref{thm:finito-main}, we get
\[
\mathbb{E}[T^{k}]\leq\left(1-\frac{1}{\alpha n}\right)^{k}T^{0}.
\]

Now using Theorem \ref{thm:lyp-to-func}:
\[
\mathbb{E}\left[f(\bar{\phi}^{k})\right]-f(w^{*})\leq\alpha\left(1-\frac{1}{\alpha n}\right)^{k}T^{0}.
\]
We need to control $T^{0}$ as well. Since we are assuming that all
$\phi_{i}^{0}$ start the same, we have that
\begin{eqnarray*}
T^{0} & = & f(\phi^{0})-\frac{1}{n}\sum_{i}f_{i}(\phi^{0})-\frac{1}{n}\sum_{i}\left\langle f_{i}^{\prime}(\phi^{0}),w^{0}-\phi^{0}\right\rangle -\frac{\mu}{2}\left\Vert w^{0}-\phi^{0}\right\Vert ^{2}\\
 & = & 0-\left\langle f^{\prime}(\phi^{0}),w^{0}-\phi^{0}\right\rangle -\frac{\mu}{2}\left\Vert -\frac{1}{\alpha\mu}f^{\prime}(\phi^{0})\right\Vert ^{2}\\
 & = & \frac{1}{\alpha\mu}\left\Vert f^{\prime}(\phi^{0})\right\Vert ^{2}-\frac{1}{2\alpha^{2}\mu}\left\Vert f^{\prime}(\phi^{0})\right\Vert ^{2}\\
 & = & \left(1-\frac{1}{2\alpha}\right)\frac{1}{\alpha\mu}\left\Vert f^{\prime}(\phi^{0})\right\Vert ^{2}.
\end{eqnarray*}

\end{proof}

\section{Prox-Finito Theory}

\label{sec:prox-finito-theory}

The Prox-Finito algorithm is very closely related to SDCA. Instead
of giving a proof in the spirit of SDCA, we now present a convergence
proof using a novel primal argument. This argument has a geometric
feel to it which we believe is clearer than the SDCA proof, although
it is a little longer. The main result is in Theorem \ref{thm:one-step}.
\begin{lem}
\label{lem:tight-bound}We consider a set of possible functions ($F)$.
$F$ is $S_{\mu,L}$, the set of $\mu$ strong convex and $L$ smooth
functions, with the additional restriction that each $f\in F$ is
lower bounded by a quadratic function $l$ with curvature $\mu$ that
we know ($l\in S_{\mu,\mu}$), and that for all $f\in F$, $f(w)-l(w)=\delta$,
for a single known $w$ and $\delta$. For each unit vector direction
$r$, we define the function: 
\[
b(x)=l(v)+\left\langle l^{\prime}(v),x-v\right\rangle +\frac{L}{2}\left\Vert x-v\right\Vert ^{2},
\]

where $v$ is the point in the direction $r$ such that $\frac{L-\mu}{2}\left\Vert w-v\right\Vert ^{2}=\delta$.
More precisely, $v=w+\alpha r$, for $\alpha=\sqrt{\frac{2\delta}{L-\mu}}$.
Then define the line segment $L$ as the segment between $w$ and
$v$, $\left(L=\left\{ w+\alpha r\,:\,0\leq\alpha\leq\sqrt{\frac{2\delta}{L-\mu}}\right\} \right)$,
Then for all $f\in F$ and points $x\in L$:

\[
f(x)\geq b(x)\geq l(x),
\]

i.e. $b(x)$ is at least as tight a lower bound on the interval $L$
as $l(x)$. Furthermore, we have:
\[
\forall x\in L,\;b(x)-l(x)=\frac{L-\mu}{2}\left\Vert x-v\right\Vert ^{2}.
\]

The above points also imply that:

\begin{equation}
f(w)-l(w)=\frac{L-\mu}{2}\left\Vert w-v\right\Vert ^{2}.\label{eq:lemma-1-at-w}
\end{equation}
\end{lem}
\begin{proof}
The situation is depicted in Figure \ref{fig:lem-sch}. Since $l$
is quadratic with curvature $\mu,$ it can be written as
\[
l(x)=l(v)+\left\langle l^{\prime}(v),x-v\right\rangle +\frac{\mu}{2}\left\Vert x-v\right\Vert ^{2}.
\]

By cancelling like terms, $b(x)-l(x)=\frac{L-\mu}{2}\left\Vert x-v\right\Vert ^{2}$
. Let $f$ be any function in $F$. Now suppose that there exists
a $y\in L$ such that $f(y)<b(y)$. Then we consider the possible
derivatives in the direction $r$. Let $f_{r}^{^{\prime}}(x)$ denote
the directional derivative $\left\langle f^{\prime}(x),r\right\rangle .$
\begin{enumerate}
\item Suppose that $f_{r}^{\prime}(w)<b_{r}^{\prime}(w)$ . Then taking
the Lipschitz smoothness upper bound about point $w$ of $f$, for
all $x\in L$: 
\begin{eqnarray*}
f(x) & \leq & f(w)+\left\langle f^{\prime}(w),x-w\right\rangle +\frac{L}{2}\left\Vert x-w\right\Vert ^{2}\\
 & < & f(w)+\left\langle b^{\prime}(w),x-w\right\rangle +\frac{L}{2}\left\Vert x-w\right\Vert ^{2}\\
 & = & b(x),
\end{eqnarray*}
which holds since $b(x)$ can be reparameterised as 
\[
b(w)+\left\langle b^{\prime}(w),x-w\right\rangle +\frac{L}{2}\left\Vert x-w\right\Vert ^{2},
\]
and $b(w)=f(w)$. Since $b(v)=l(v)$, this implies $f(v)<l(v)$, which
is a contradiction.
\item Suppose that $f_{r}^{\prime}(w)>b_{r}^{\prime}(w)$. Then $f$ lies
above $b$ at some point in the direction $r$ from $w$. In order
that our assumption $f(y)<b(y)$ holds, $f$ must equal $b$ somewhere
in the interval $L$, which we denote $u$, i.e. the curves cross.
Then at the crossing point $f$ is more negative than $b$ $\left(f_{r}^{\prime}(u)<b_{r}^{\prime}(u)\right)$,
and the argument from case (1) applies with $u$ instead of $w$.
\item Suppose that $f_{r}^{\prime}(w)=b_{r}^{\prime}(w)$. Then consider
the Lipschitz smoothness upper bound around $y$:
\begin{equation}
f(x)\leq f(y)+\left\langle f^{\prime}(y),x-y\right\rangle +\frac{L}{2}\left\Vert x-y\right\Vert ^{2}.\label{eq:ub-point3}
\end{equation}
Recall that we can reparameterise $b$ around $y$ as $b(x)=b(y)+\left\langle b^{\prime}(y),x-y\right\rangle +\frac{L}{2}\left\Vert x-y\right\Vert ^{2}.$
Since we require that $f(v)\geq b(v)$, the inner product term in
Equation \ref{eq:ub-point3} when evaluating at $v$ must be greater
than that in $b$ to offset the $b(y)-f(y)$ difference, i.e. $\left\langle f^{\prime}(y),v-y\right\rangle >\left\langle b^{\prime}(y),v-y\right\rangle $,
therefore $f_{r}^{\prime}(y)>b_{r}^{\prime}(y)$. But this violates
the Lipschitz smoothness condition $\left\Vert f^{\prime}(y)-f^{\prime}(w)\right\Vert \leq L\left\Vert y-w\right\Vert $
on the gradients of $f$, as $b^{\prime}(y)$ is as different as possible
for a gradient of $f$ at that distance from $w$ along $r$, by its
construction.
\end{enumerate}
\end{proof}
\begin{figure}
\begin{centering}
\includegraphics{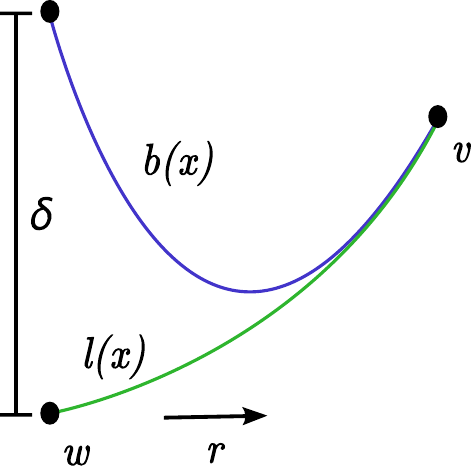}
\par\end{centering}

\protect\caption{\label{fig:lem-sch}Lemma \ref{lem:tight-bound} schematic. The vertical
axis is function value, and the horizontal axis is change along the
direction $r$.}
\end{figure}

\begin{lem}
\label{lem:constants}For any $v_{j}$ and $w^{k}$, the minimum of
\[
g_{j}(w_{j}^{k+1})=\frac{\mu}{2}\left\Vert w_{j}^{k+1}-w^{k}\right\Vert ^{2}+\frac{L-\mu}{2n}\left\Vert w_{j}^{k+1}-v_{j}\right\Vert ^{2},
\]

with respect to $w_{j}^{k+1}$ is:
\[
\frac{\mu\frac{L-\mu}{n}}{2\left(\mu+\frac{L-\mu}{n}\right)}\left\Vert w^{k}-v_{j}\right\Vert ^{2}.
\]
\end{lem}
\begin{proof}
Let the minimiser be $x$. We find the value of $x$ by setting the
gradient to zero of $g_{j}$:
\[
\mu(x-w^{k})+\frac{L-\mu}{n}(x-v_{j})=0,
\]
\[
\mu w^{k}+\frac{L-\mu}{n}v_{j}=(\mu+\frac{L-\mu}{n})x,
\]
\[
\therefore x=\frac{\mu w^{k}+\frac{L-\mu}{n}v_{j}}{\mu+\frac{L-\mu}{n}}.
\]

So the minimum for the first term is:
\begin{eqnarray*}
\frac{\mu}{2}\left\Vert \frac{\mu w^{k}+\frac{L-\mu}{n}v_{j}-\mu w^{k}-\frac{L-\mu}{n}w^{k}}{\mu+\frac{L-\mu}{n}}\right\Vert ^{2} & = & \frac{\mu}{2}\left\Vert (\frac{L-\mu}{n})\frac{v_{j}-w^{k}}{\mu+\frac{L-\mu}{n}}\right\Vert ^{2}\\
 & = & \frac{\mu}{2}\left(\frac{L-\mu}{n(\mu+\frac{L-\mu}{n})}\right)^{2}\left\Vert w^{k}-v_{j}\right\Vert ^{2},
\end{eqnarray*}

for the second term:
\begin{eqnarray*}
\frac{L-\mu}{2n}\left\Vert \frac{\mu w^{k}+\frac{L-\mu}{n}v_{j}-\mu v_{j}-\frac{L-\mu}{n}v_{j}}{\mu+\frac{L-\mu}{n}}\right\Vert ^{2} & = & \frac{L-\mu}{2n}\left\Vert \mu\frac{w^{k}-v_{j}}{\mu+\frac{L-\mu}{n}}\right\Vert ^{2}\\
 & = & \frac{L-\mu}{2n}\left(\frac{\mu}{\mu+\frac{L-\mu}{n}}\right)^{2}\left\Vert w^{k}-v_{j}\right\Vert ^{2},
\end{eqnarray*}

Combining the constants:
\begin{eqnarray*}
\frac{\mu}{2}\left(\frac{L-\mu}{n(\mu+\frac{L-\mu}{n})}\right)^{2}+\frac{L-\mu}{2n}\left(\frac{\mu}{\mu+\frac{L-\mu}{n}}\right)^{2} & = & \frac{\mu}{2\left(\mu+\frac{L-\mu}{n}\right)^{2}}\left(\left(\frac{L-\mu}{n}\right)^{2}+\mu\frac{L-\mu}{n}\right)\\
 & = & \frac{\mu\frac{L-\mu}{n}}{2\left(\mu+\frac{L-\mu}{n}\right)^{2}}\left(\left(\frac{L-\mu}{n}\right)+\mu\right)\\
 & = & \frac{\mu\frac{L-\mu}{n}}{2\left(\mu+\frac{L-\mu}{n}\right)}.
\end{eqnarray*}

\end{proof}

\subsection{Main result}
\begin{thm}
\label{thm:one-step}Define the Lyapunov function $T^{k}$ at step
$k$ as $T^{k}=f(w^{*})-B^{k}(w^{k})$. Then the expected change in
$T$ between steps $k$ and $k+1$ for the Prox-Finito algorithm is:
\[
\mathbb{E}[T^{k+1}]\leq\left(1-\frac{\mu}{\left(\mu n+L-\mu\right)}\right)T^{k}.
\]
\end{thm}
\begin{proof}
Let 
\[
l_{i}^{k}(x)=f_{i}(\phi_{i}^{k})+\left\langle f_{i}^{\prime}(\phi_{i}^{k}),x-\phi_{i}^{k}\right\rangle +\frac{\mu}{2}\left\Vert x-\phi_{i}^{k}\right\Vert ^{2},
\]

so that $B^{k}(x)=\frac{1}{n}\sum_{i}l_{i}(x)$. Let $w_{j}^{k+1}$
be the minimiser of $B^{k+1}$, which of course depends on the choice
of index $j$.

Now using the fact that $B^{k}$ is quadratic, we get that for all
$x$:
\[
B^{k}(x)=B^{k}(w^{k})+\left\langle \left(B^{k}\right)^{\prime}(w^{k}),x-w^{k}\right\rangle +\frac{\mu}{2}\left\Vert w^{k}-x\right\Vert ^{2},
\]

so using that $\left(B^{k}\right)^{\prime}(w^{k})=0$ by construction,
and evaluating at $w_{j}^{k+1}$, we have:
\begin{equation}
B^{k}(w_{j}^{k+1})-B^{k}(w^{k})=\frac{\mu}{2}\left\Vert w_{j}^{k+1}-w^{k}\right\Vert ^{2}.\label{eq:quad-term}
\end{equation}

We also need to bound the change between $B^{k+1}$ and $B^{k}$ at
the point $w_{j}^{k+1}$. Using the key property $\phi_{j}^{k+1}=w_{j}^{k+1}$,
the added lower bound $l_{j}^{k+1}(w_{j}^{k+1})$ is tight at $w_{j}^{k+1}$
by construction, so: 
\begin{eqnarray}
B^{k+1}(w_{j}^{k+1})-B^{k}(w_{j}^{k+1}) & = & \frac{1}{n}l_{j}^{k+1}(\phi_{j}^{k+1})-\frac{1}{n}l_{j}^{k}(w_{j}^{k+1})\nonumber \\
 & = & \frac{1}{n}f_{j}(w_{j}^{k+1})-\frac{1}{n}l_{j}^{k}(w_{j}^{k+1}).\label{eq:change-term}
\end{eqnarray}

The change in the $B^{k}(w^{k})$ term in the Lyapunov function between
steps is given by combining \ref{eq:quad-term} and \ref{eq:change-term}:
\begin{eqnarray*}
B^{k+1}(w_{j}^{k+1})-B^{k}(w^{k}) & = & \left[B^{k+1}(w_{j}^{k+1})-B^{k}(w_{j}^{k+1})\right]+\left[B^{k}(w_{j}^{k+1})-B^{k}(w^{k})\right]\\
 & = & \frac{\mu}{2}\left\Vert w_{j}^{k+1}-w^{k}\right\Vert ^{2}+\frac{1}{n}f_{j}(w_{j}^{k+1})-\frac{1}{n}l_{j}^{k}(w_{j}^{k+1}).
\end{eqnarray*}

Since our Lyapunov function is $f(w^{*})-B^{k}(w^{k})$ and $f(w^{*})$
is fixed, our goal is to lower bound the expectation of the change
in the $B^{k}$ function's minimum with respect to the choice of index
$j$:
\begin{equation}
\mathbb{E}\left[B^{k+1}(w_{j}^{k+1})-B^{k}(w^{k})\right]=\mathbb{E}\left[\frac{\mu}{2}\left\Vert w_{j}^{k+1}-w^{k}\right\Vert ^{2}+\frac{1}{n}f_{j}(w_{j}^{k+1})-\frac{1}{n}l_{j}^{k}(w_{j}^{k+1})\right].\label{eq:rem-terms}
\end{equation}

Let us fix the choice of $j$. We proceed by determining the minimum
over $w_{j}^{k+1}$ of the quantity inside the expectation. We will
also fix the unit vector direction $r_{j}$ from $w^{k}$ that $w_{j}^{k+1}$
is in. Then we can apply Lemma \ref{lem:tight-bound} with $\delta=f_{j}(w^{k})-l_{j}(w^{k})$,
to give: 
\[
\frac{1}{n}f_{j}(w_{j}^{k+1})-\frac{1}{n}l_{j}^{k}(w_{j}^{k+1})\geq\frac{1}{n}b(w_{j}^{k+1})-\frac{1}{n}l_{j}(w_{j}^{k+1})=\frac{L-\mu}{2n}\left\Vert w_{j}^{k+1}-v_{j}\right\Vert ^{2},
\]
where $v_{j}$ and $b$ are given in Lemma \ref{lem:tight-bound}.
Note that $v_{j}$ is coplanar with $w^{k}$ and $w_{j}^{k+1}$, namely
in direction $r_{j}$ from $w^{k}$. Applying this result: 
\begin{eqnarray*}
B^{k+1}(w_{j}^{k+1})-B^{k}(w^{k}) & = & \frac{\mu}{2}\left\Vert w_{j}^{k+1}-w^{k}\right\Vert ^{2}+\frac{1}{n}f_{j}(w_{j}^{k+1})-\frac{1}{n}l_{j}^{k}(w_{j}^{k+1})\\
 & \geq & \frac{\mu}{2}\left\Vert w_{j}^{k+1}-w^{k}\right\Vert ^{2}+\frac{L-\mu}{2n}\left\Vert w_{j}^{k+1}-v_{j}\right\Vert ^{2}\\
 & = & g_{j}(w_{j}^{k+1})\ \text{ (Lemma \ref{lem:constants})}\\
 & \geq & \frac{\mu\frac{L-\mu}{n}}{2\left(\mu+\frac{L-\mu}{n}\right)}\left\Vert w^{k}-v_{j}\right\Vert ^{2}\ \text{ (Lemma \ref{lem:constants})}\\
 & = & \frac{2}{L-\mu}\cdot\frac{\mu\frac{L-\mu}{n}}{2\left(\mu+\frac{L-\mu}{n}\right)}\left(f_{j}(w^{k})-l_{j}(w^{k})\right)\,\text{ (Lem. \ref{lem:tight-bound}, Eq \ref{eq:lemma-1-at-w})}\\
 & = & \frac{\mu}{\left(\mu n+L-\mu\right)}\left(f_{j}(w^{k})-l_{j}(w^{k})\right).
\end{eqnarray*}

We now take expectations of this quantity with respect to the choice
of index $j$, and note that the above argument holds for all possible
choices of the direction $r_{j}$. So:
\begin{eqnarray}
\mathbb{E}[B^{k+1}(w^{k+1})]-B^{k}(w^{k}) & \geq & \mathbb{E}\left[\frac{\mu}{\left(\mu n+L-\mu\right)}\left(f_{j}(w^{k})-l_{j}(w^{k})\right)\right]\label{eq:prox-finito-main-exp}\\
 & = & \frac{\mu}{\left(\mu n+L-\mu\right)}\left(f(w^{k})-B^{k}(w^{k})\right)\label{eq:prox-finito-inner-step}\\
 & \geq & \frac{\mu}{\left(\mu n+L-\mu\right)}\left(f(w^{*})-B^{k}(w^{k})\right).\nonumber 
\end{eqnarray}

So we have established a bound on Expression \ref{eq:rem-terms}.
Stating in terms of the Lyapunov function:

\[
\mathbb{E}[T^{k+1}]\leq\left(1-\frac{\mu}{\left(\mu n+L-\mu\right)}\right)T^{k}.
\]
\end{proof}
\begin{thm}
\label{thm:conv-main}The expected convergence rate of the Prox-Finito
algorithm for $n>1$ stopping at iteration $k$ is:
\end{thm}
\[
E\left[f(w^{k})-f(w^{*})\right]\leq\left(1-\frac{\mu}{\left(\mu n+L-\mu\right)}\right)^{k}\frac{\mu n+L-\mu}{\mu}\left[f(w^{*})-B^{0}(w^{0})\right].
\]

\begin{proof}
From here the proof follows closely the SDCA proof. Taking the expectation
of Theorem \ref{thm:one-step} with respect to all $k$, then summing
from $0$ to $k-1$ gives:
\end{proof}
\begin{equation}
f(w^{*})-\mathbb{E}\left[B^{k}(w^{k})\right]\leq\left(1-\frac{\mu}{\left(\mu n+L-\mu\right)}\right)^{k}\left[f(w^{*})-B^{0}(w^{0})\right].\label{eq:prox-finito-summed}
\end{equation}

This expectation is unconditional. We now need to bound $f(w^{k})-f(w^{*})$
by some constant times $f(w^{*})-B^{k}(w^{k})$. Recall Equation \ref{eq:prox-finito-inner-step}:
\[
\mathbb{E}[B^{k+1}(w^{k+1})]-B^{k}(w^{k})\geq\frac{\mu}{\left(\mu n+L-\mu\right)}\left(f(w^{k})-B^{k}(w^{k})\right).
\]

Rearranging gives:
\begin{eqnarray*}
f(w^{k})-B^{k}(w^{k}) & \leq & \frac{\mu n+L-\mu}{\mu}\left[\mathbb{E}[B^{k+1}(w^{k+1})]-B^{k}(w^{k})\right]\\
 & = & \frac{\mu n+L-\mu}{\mu}\left[\left(\mathbb{E}[B^{k+1}(w^{k+1})]-f(w^{*})\right)-\left(B^{k}(w^{k})-f(w^{*})\right)\right].\\
 & \leq & \frac{\mu n+L-\mu}{\mu}\left[-\left(B^{k}(w^{k})-f(w^{*})\right)\right].
\end{eqnarray*}

We have used the negativity of $B^{k}(w^{k})-f(w^{*})$ in the last
step. Since $B^{k}(w^{k})\leq f(w^{*})$, we have by taking expectations
over all previous steps that:
\[
\mathbb{E}\left[f(w^{k})-f(w^{*})\right]\leq\frac{\mu n+L-\mu}{\mu}\mathbb{E}\left[f(w^{*})-B^{k}(w^{k})\right].
\]

This expectation is also unconditional. Combining with Equation \ref{eq:prox-finito-summed}
gives the result.

\subsection{Proof of Theorem \ref{thm:prox-finito-weighted}}

\label{sub:prox-finito-weighted-proof}
\begin{proof}
The key change in the proof of Theorem \ref{thm:one-step} is in the
calculation of the following Expectation in Equation \ref{eq:prox-finito-main-exp}:
\[
\mathbb{E}\left[\frac{1}{c_{j}}\left(f_{j}(w^{k})-l_{j}(w^{k})\right)\right].
\]

Applying the weighted probabilities to this expectation yields:

\begin{eqnarray*}
\mathbb{E}\left[\frac{1}{c_{j}}\left(f_{j}(w^{k})-l_{j}(w^{k})\right)\right] & = & \frac{1}{Z}\sum_{i}c_{i}\frac{1}{c_{i}}\left(f_{i}(w^{k})-l_{i}(w^{k})\right)\\
 & = & \frac{1}{Z}\sum_{i}\left(f_{i}(w^{k})-l_{i}(w^{k})\right).
\end{eqnarray*}

The value of $Z$ is:
\begin{eqnarray*}
\sum_{i}\frac{\mu n+L_{i}-\mu}{\mu} & = & n\frac{\mu n+\bar{L}-\mu}{\mu},
\end{eqnarray*}

so we have
\begin{eqnarray*}
\mathbb{E}\left[\frac{1}{c_{j}}\left(f_{j}(w^{k})-l_{j}(w^{k})\right)\right] & = & \frac{\mu}{\left(\mu n+\bar{L}-\mu\right)}\cdot\frac{1}{n}\sum_{i}\left(f_{i}(w^{k})-l_{i}(w^{k})\right)\\
 & = & \frac{\mu}{\left(\mu n+\bar{L}-\mu\right)}\left(f(w^{k})-B^{k}(w^{k})\right).
\end{eqnarray*}

The remaining steps mirror the proof of Theorem \ref{thm:one-step}
and \ref{thm:conv-main}.\end{proof}

\chapter{New Primal Incremental Gradient Methods}

\label{chap:saga}The fast primal incremental gradient methods in
the existing literature have a number of deficiencies. The SAG method
has a complex theory, and no theoretical support for its use on composite
objectives (defined in Section \ref{sec:composite}). The SVRG method
avoids those problems, but it has an awkward theoretical convergence
rate in terms of the number of gradient evaluations, and it does not
natively support non-strongly convex problems. In this chapter we
present the SAGA method, which has none of these deficiencies. The
SAGA method has a simple theory, with a theoretical convergence rate
directly comparable to SDCA ($2\times$ worse) and SAG ($2\times$
better). It is also adaptive to strong convexity, as it may be used
without modification and with the same step size on both convex and
strongly convex problems. If the problem is strongly convex, it will
converge at a fast linear rate, falling back to an $O(1/k)$ rate
otherwise. This avoids one of the main deficiencies of SDCA.

An interesting property of SAGA is that the algorithm can be written
in a number of different ways, making the update similar to either
SVRG, SAG or Finito. This sheds light on the relationships between
the currently known fast incremental gradient methods.

An earlier version of the work in this chapter has been published
as \citet{adefazio-nips2014}.

\global\long\def\stepsize{\gamma}

\section{Composite Objectives}

\label{sec:composite}

Several of the fast incremental gradient methods can be extended to
support \emph{composite} objectives (Table \ref{tab:incremental-grad-properties}).
A composite problem is just the usual finite sum structure $f(x)=\frac{1}{n}\sum_{i=1}^{n}f_{i}(x)$,
together with an additional term $h$:

\[
F(x)=f(x)+h(x),
\]

where $h\colon\mathbb{R}^{d}\rightarrow\mathbb{R}^{d}$ is convex
but potentially non-differentiable. For the SAGA algorithm we will
solve this problem using a \emph{proximal} approach, making use of
the proximal operator of the function $h$. For most problems, known
convergence rates of composite objectives are as good as those for
the non-composite case \citep{prox-dk} when we have access to the
proximal operator of $h$. Although the problem is not necessarily
smooth, by accessing the non-smooth portion through the proximal operator
it is still possible to converge linearly, which is not the case when
we only have access to (sub-)gradients.

\section{SAGA Algorithm}

\selectlanguage{english}%
We start with some known initial vector $x^{0}\in\mathbb{R}^{d}$
and known derivatives $f_{i}^{\prime}(\phi_{i}^{0})\in\mathbb{R}^{d}$
with $\phi_{i}^{0}=x^{0}$ for each $i$. These derivatives are stored
in a table data-structure of length $n$, or alternatively a $n\times d$
matrix. For many problems of interest, such as binary classification
and least-squares, only a single floating point value instead of a
full gradient vector needs to be stored (see Section~\ref{sec:impl}).\foreignlanguage{australian}{
The update is given in Algorithm \ref{alg:SAGA}.}

\selectlanguage{australian}%
\begin{algorithm}[H]
Given the value of $x^{k}$ and of each $f_{i}^{\prime}(\phi_{i}^{k})$
at the end of iteration $k$, the updates for iteration $k+1$ is
as follows:
\begin{enumerate}
\item Pick a $j$ uniformly at random.
\item Take $\phi_{j}^{k+1}=x^{k}$, and store $f_{j}^{\prime}(\phi_{j}^{k+1})$
in the table. All other entries in the table remain unchanged. The
quantity $\phi_{j}^{k+1}$ is not explicitly stored.
\item Update $x$ using $f_{j}^{\prime}(\phi_{j}^{k+1})$, $f_{j}^{\prime}(\phi_{j}^{k})$
and the table average:
\begin{equation}
w^{k+1}=x^{k}-\stepsize\left[f_{j}^{\prime}(\phi_{j}^{k+1})-f_{j}^{\prime}(\phi_{j}^{k})+\frac{1}{n}\sum_{i=1}^{n}f_{i}^{\prime}(\phi_{i}^{k})\right],\label{eq:saga-main-update}
\end{equation}
\[
x^{k+1}=\text{prox}_{1/\stepsize}^{h}\left(w^{k+1}\right).
\]

\end{enumerate}
\protect\caption{\label{alg:SAGA}SAGA}

\end{algorithm}

We establish the following theoretical results in Section \ref{sec:saga-theory}:

\textit{}

\selectlanguage{english}%
In the strongly convex case, when a step size of $\stepsize=1/(2(\mu n+L))$
is chosen, we have the following convergence rate in the composite
and hence also the non-composite case: {\small{}
\[
\mathbb{E}\left\Vert x^{k}-x^{*}\right\Vert ^{2}\leq\left(1-\frac{\mu}{2(\mu n+L)}\right)^{k}\left[\left\Vert x^{0}-x^{*}\right\Vert ^{2}+\frac{n}{\mu n+L}\left[f(x^{0})-\left\langle f^{\prime}(x^{*}),x^{0}-x^{*}\right\rangle -f(x^{*})\right]\right].
\]
}We prove this result in Section \ref{sub:saga-theory}. The requirement
of strong convexity can be relaxed from needing to hold for each $f_{i}$
to just holding on average, but at the expense of a worse geometric
rate ($1-\frac{\mu}{6(\mu n+L)})$, requiring a step size of $\stepsize=1/(3(\mu n+L))$.

In the non-strongly convex case, we have established the convergence
rate in terms of the average iterate, excluding step 0: $\bar{x}^{k}=\frac{1}{k}\sum_{t=1}^{k}x^{t}$.
Using a step size of $\stepsize=1/(3L)$ we have 
\[
\mathbb{E}\left[F(\bar{x}^{k})\right]-F(x^{*})\leq\frac{10n}{k}\left[\frac{2L}{n}\left\Vert x^{0}-x^{*}\right\Vert ^{2}+f(x^{0})-\left\langle f^{\prime}(x^{*}),x^{0}-x^{*}\right\rangle -f(x^{*})\right].
\]
This result is proved in Section \foreignlanguage{australian}{\ref{sec:saga-theory}}.
Importantly, when this step size $\stepsize=1/(3L)$ is used, our
algorithm \emph{automatically adapts} to the level of strong convexity
$\mu>0$ naturally present, giving a convergence rate of: {\small{}
\[
\mathbb{E}\left\Vert x^{k}\!-\!x^{*}\right\Vert ^{2}\leq\left(1\!-\!\min\left\{ \frac{1}{4n}\,,\,\frac{\mu}{3L}\right\} \right)^{k}\left[\left\Vert x^{0}-x^{*}\right\Vert ^{2}\!+\!\frac{2n}{3L}\left[f(x^{0})\!-\!\left\langle f^{\prime}(x^{*}),x^{0}-x^{*}\right\rangle \!-\!=f(x^{*})\right]\right].
\]
}Although any incremental gradient method can be applied to non-strongly
convex problems via the addition of a small quadratic regularisation,
the amount of regularisation is an additional tunable parameter which
our method avoids.

\selectlanguage{australian}%

\section{Relation to Existing Methods}

\label{sec:saga-relations}

\noindent In the non-composite case, it is possible to write the SAGA
algorithm in terms of two quantities at each step instead of one:
$x^{k}$ and $u^{k}$. 

\noindent 
\begin{algorithm}[H]
Given the value of $u^{k}$ and of each $f_{i}^{\prime}(\phi_{i}^{k})$
at the end of iteration $k$ the updates for iteration $k+1$, is
as follows:
\begin{enumerate}
\item Calculate $x^{k}$:
\begin{equation}
x^{k}=u^{k}-\stepsize\sum_{i=1}^{n}f_{i}^{\prime}(\phi_{i}^{k}).\label{eq:saga-alt-form-update}
\end{equation}

\item Update $u$ with $u^{k+1}=u^{k}+\frac{1}{n}(x^{k}-u^{k})$.
\item Pick a $j$ uniformly at random.
\item Take $\phi_{j}^{k+1}=x^{k}$, and store $f_{j}^{\prime}(\phi_{j}^{k+1})$
in the table replacing $f_{j}^{\prime}(\phi_{j}^{k})$. All other
entries in the table remain unchanged. The quantity $\phi_{j}^{k+1}$
is not explicitly stored.
\end{enumerate}
\protect\caption{\label{alg:2-var-saga}2 variable SAGA}

\end{algorithm}

Writing the algorithm in this form makes the relationship with prior
methods more apparent. We explore the relationship between SAGA and
the other fast incremental gradient methods in this section. By using
SAGA as a midpoint, we are able to provide a more unified view than
is available in the existing literature. A brief summary of the properties
of each method considered in this section is given in Table \ref{tab:incremental-grad-properties}.

\subsection{SAG}

\label{sub:saga-sag-relation}

If we eliminate $x^{k}$ we get an update for $u$ in SAGA of:
\begin{equation}
u^{k+1}=u^{k}-\frac{\gamma}{n}\sum_{i=1}^{n}f_{i}^{\prime}(\phi_{i}^{k}).\label{eq:saga-u-only-update}
\end{equation}
After translating notation, this is identical to the Stochastic Average
Gradient \citep[SAG,][]{sag} update discussed in Section \ref{sec:sag-background},
except instead of setting $\phi_{j}^{k+1}=u^{k}$, we are using a
more aggressive update of $\phi_{j}^{k+1}=x^{k}=u^{k}-\gamma\sum_{i}f_{i}^{\prime}(\phi_{i}^{k})$.
The order of the steps is also changed, as in SAG the $j$th gradient
is updated before the $x$ step is taken, whereas above it is updated
after. The order of the steps for SAG doesn't affect the algorithm
as it only has two steps, so the ordering change is not significant.

The$\phi_{j}^{k+1}=x^{k}$ change has a large effect on the estimate
of the gradient. In SAG, the gradient approximation $\frac{1}{n}\sum_{i}f_{i}^{\prime}(\phi_{i}^{k})$
is biased away from the true gradient, whereas for all the other methods
considered here, including SAGA, the gradient approximation is unbiased.
The trade-off for the increased bias is a decreased variance in the
$f_{i}^{\prime}(\phi_{i}^{k})$ gradients, due to the less aggressive
$\phi_{i}^{k}$ updates used.

There is another way of interpreting the SAG algorithm in relation
to the SAGA algorithm. We can rewrite the SAG update to make it more
directly comparable:

\selectlanguage{english}%
\begin{align}
\textrm{(SAG)}\qquad x^{k+1} & =x^{k}-\stepsize\left[\frac{f_{j}^{\prime}(x^{k})-f_{j}^{\prime}(\phi_{j}^{k})}{n}+\frac{1}{n}\sum_{i=1}^{n}f_{i}^{\prime}(\phi_{i}^{k})\right],\label{SAG}\\
\textrm{(SAGA)}\qquad x^{k+1} & =x^{k}-\stepsize\left[f_{j}^{\prime}(x^{k})-f_{j}^{\prime}(\phi_{j}^{k})+\frac{1}{n}\sum_{i=1}^{n}f_{i}^{\prime}(\phi_{i}^{k})\right],\label{eq:w-update}
\end{align}

Notice that the central terms are weighted by $1/n$ for SAG, in comparison
to SAGA. The lower variance of the SAG update arises because of the
smaller magnitude of those two terms. 

\selectlanguage{australian}%
The per update computation and storage costs of SAGA and SAG is essentially
the same. They both require a gradient evaluation and a gradient storage
(although the storage cost can usually be ameliorated for both, see
Section \ref{sec:impl}). The advantage over SAG is that SAGA has
a much more complete, simple and tight theory. The theoretical and
practical convergence rate is better for SAGA, and no theory exists
for the use of proximal operators in SAG whereas we establish such
a theory for SAGA in this work. Based also on our superior experimental
results in Section \ref{sec:saga-experiments}, we consider SAGA as
a replacement for SAG in most situations.

\subsection{SVRG}

\label{subsec:saga-svrg-relation}

Recall the $x^{k+1}$ update for SAGA (Equation \ref{eq:saga-main-update})
in the non-composite case:
\[
x^{k+1}=x^{k}-\stepsize f_{j}^{\prime}(x^{k})+\stepsize\left[f_{j}^{\prime}(\phi_{j}^{k})-\frac{1}{n}\sum_{i=1}^{n}f_{i}^{\prime}(\phi_{i}^{k})\right].
\]
This can be directly compared against the SVRG (Stochastic Variance
Reduced Gradient) \citep{svrg} update from Section \ref{sec:svrg-background},
if we translate the notation:
\[
x^{k+1}=x^{k}-\stepsize f_{j}^{\prime}(x^{k})+\stepsize\left[f_{j}^{\prime}(\tilde{x})-\frac{1}{n}\sum_{i=1}^{n}f_{i}^{\prime}(\tilde{x})\right].
\]
The vector $\tilde{x}$ is not updated every step, but rather the
loop over $k$ appears inside an outer loop, where $\tilde{x}$ is
updated at the start of each outer iteration. Essentially SAGA is
at the midpoint between SVRG and SAG; it updates the $\phi_{j}$ value
each time index $j$ is picked, whereas SVRG updates all of $\phi$
function's as a batch.

The SVRG algorithm has been previously motivated as being a form of
variance reduction \citep{svrg}, similar to approaches used for Monte
Carlo estimation of integrals. The SAGA algorithm can also be interpreted
as a form of variance reduction, see \citet{adefazio-nips2014}.

SVRG makes a trade-off between time and space. For the equivalent
practical convergence rate it makes 2x-3x more gradient evaluations,
but in doing so it does not need to store a table of gradients, but
a single average gradient. The usage of SAG v.s. SVRG is problem dependent.
For example for linear predictors where gradients can be stored as
a reduced vector of dimension $p-1$ for $p$ classes, SAGA is preferred
over SVRG both theoretically and in practice. For neural networks,
where no theory is available for either method, the storage of gradients
is generally more expensive than the additional backwards propagations,
giving SVRG the advantage. 

SVRG also has an additional parameter besides step size that needs
to be set, namely the number of iterations per inner loop ($m$).
This parameter can be set via the theory, or conservatively as $m=n$,
however doing so does not give anywhere near the best practical performance.
Having to tune one parameter instead of two is a practical advantage
for SAGA.

\begin{table}
\center \small \begin{tabular*}{\textwidth}{@{\extracolsep{\fill}} |cccccc|}
\hline 
 & SAGA & SAG & SDCA & SVRG & FINITO \tabularnewline
Strongly Convex (SC) & \cmark & \cmark & \cmark & \cmark & \cmark \tabularnewline
Convex, Non-SC* & \cmark & \cmark & \xmark & ? & ? \tabularnewline
Prox Reg. & \cmark & ? & \cmark \cite{sdca-admm} & \cmark & \xmark \tabularnewline 
Non-smooth & \xmark & \xmark & \cmark & \xmark & \xmark \tabularnewline
Low Storage Cost & \xmark & \xmark & \xmark & \cmark & \xmark \tabularnewline
Simple(-ish) Proof & \cmark & \xmark & \cmark & \cmark & \cmark \tabularnewline 
Adaptive to SC & \cmark & \cmark & \xmark & ? & ? \tabularnewline
\hline
\end{tabular*}

\protect\caption{\label{tab:incremental-grad-properties} Basic summary of method properties.
Question marks denote unproven, but not experimentally ruled out cases.
A tick indicates that the method can be applied in that setting, or
that the property holds. ({*}) Note that any method can be applied
to non-strongly convex problems by adding a small amount of $L_{2}$
regularisation, this row describes methods that do not require this
trick.}

\end{table}

\subsection{Finito}

The Finito method (Section \ref{sec:finito}) is also closely related
to SAGA. Recall that Finito uses a step of the following form:
\[
x^{k+1}=\bar{\phi}^{k}-\stepsize\sum_{i=1}^{n}f_{i}^{\prime}(\phi_{i}^{k}).
\]

Note that the step sized used is of the order of $\stepsize=1/\mu n$,
roughly comparable to the $\stepsize$ in SAGA. 

Using the two variable SAGA formulation, SAGA can be interpreted as
Finito, but with the quantity $\bar{\phi}$ replaced with $u$, which
is updated in the same way as $\bar{\phi}$, but \emph{in expectation}.
To see this, consider how $\bar{\phi}$ changes in value and in expectation,
conditioned on the previous step:
\[
\mathbb{E}\left[\bar{\phi}^{k+1}\right]=\mathbb{E}\left[\bar{\phi}^{k}+\frac{1}{n}\left(x^{k}-\phi_{j}^{k}\right)\right]=\bar{\phi}^{k}+\frac{1}{n}\left(x^{k}-\bar{\phi}^{k}\right).
\]

The update is identical in expectation to the update for $u$, $u^{k+1}=u^{k}+\frac{1}{n}(x^{k}-u^{k})$.

There are three advantages of SAGA over Finito. SAGA doesn't require
strong convexity to work, it has support for proximal operators, and
it doesn't require storing the $\phi_{i}$ values. The big advantage
of Finito is that when it is applicable, it can be used with a permuted
access ordering, which can make it up to two times faster. Finito
is particularly useful when $f_{i}$ is computationally expensive
to compute compared to the extra storage costs required for it over
the other methods.

\section{Implementation}

\label{sec:impl}

We briefly discuss some implementation concerns:
\begin{enumerate}
\item For many problems each derivative $f_{i}^{\prime}$ is just a simple
weighting of the $i$th data vector. Logistic regression and least
squares have this property. In such cases, instead of storing the
full derivative $f_{i}^{\prime}$ for each $i$, we need only store
the weighting constants. This reduces the storage requirements to
be the same as the SDCA method in practice. A similar trick can be
applied to multi-class classifiers with $p$ classes by storing $p-1$
values for each $i$.
\selectlanguage{english}%
\item Our algorithm assumes that initial gradients are known for each $f_{i}$
at the starting point $x^{0}$. Instead, a heuristic may be used where
during the first pass, data-points are introduced one-by-one, in a
non-randomized order, with averages computed in terms of those data-points
processed so far. This procedure has been successfully used with SAG
\citep{sag}.
\selectlanguage{australian}%
\item The SAGA update as stated is slower than necessary when derivatives
are sparse. A just-in-time updating of $u$ or $x$ may be performed
just as is suggested for SAG \citep*{sag}, which ensures that only
sparse updates are done at each iteration.
\item We give the form of SAGA for the case where each $f_{i}$ are strongly
convex. In practice however, we usually have only convex $f_{i}$,
with strong convexity in $f$ induced by the addition of a quadratic
regulariser. This quadratic regulariser may be split amount the $f_{i}$
functions evenly, to satisfy our assumptions. It is perhaps easier
to use a variant of SAGA where the regulariser $\frac{\mu}{2}||x||^{2}$
is explicit, such as the following modification of Equation \ref{eq:saga-main-update}:
\[
x^{k+1}=\left(1-\mu\gamma\right)x^{k}-\stepsize f_{j}^{\prime}(x^{k})+\stepsize\left[f_{j}^{\prime}(\phi_{j}^{k})-\frac{1}{n}\sum_{i}f_{i}^{\prime}(\phi_{i}^{k})\right].
\]
For sparse implementations instead of scaling $x^{k}$ at each step,
a separate scaling constant $\beta^{k}$ may be scaled instead, with
$\ensuremath{\beta^{k}x^{k}}$ being used in place of $x^{k}$. This
is a standard trick used with stochastic gradient methods.
\end{enumerate}
\begin{figure}
\begin{minipage}[c]{0.09\columnwidth}%
\begin{sideways} \bfseries  Function sub-optimality \end{sideways}%
\end{minipage}%
\begin{minipage}[t]{0.9\columnwidth}%
\begin{center}
\includegraphics[width=0.5\columnwidth]{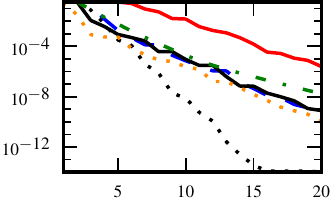}\includegraphics[width=0.49\columnwidth]{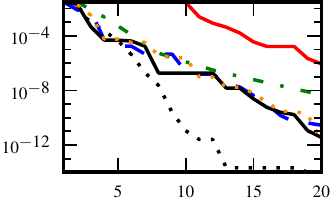}
\par\end{center}

\begin{center}
\medskip{}

\par\end{center}

\begin{center}
\includegraphics[width=0.5\columnwidth]{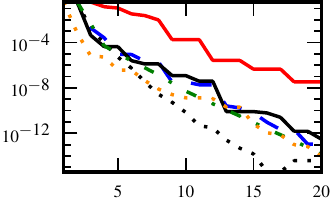}\includegraphics[width=0.49\columnwidth]{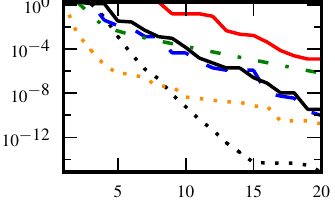}
\par\end{center}

\begin{center}
\textbf{Gradient Evaluations / $n$}
\par\end{center}%
\end{minipage}

\includegraphics[width=1\columnwidth]{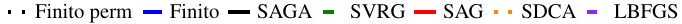}

\protect\caption{\label{fig:saga-l2}From left to right, top to bottom we have the
MNIST, COVTYPE, IJCNN1 and MILLIONSONG datasets with $L_{2}$ regularisation.}

\end{figure}

\begin{figure}
\begin{centering}
\begin{minipage}[c]{0.09\columnwidth}%
\begin{sideways} \bfseries  Function sub-optimality \end{sideways}%
\end{minipage}%
\begin{minipage}[t]{0.9\columnwidth}%
\begin{center}
\includegraphics[width=0.5\columnwidth]{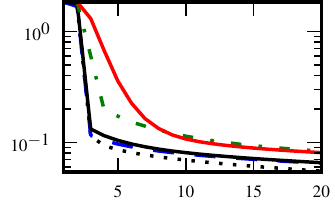}\includegraphics[width=0.49\columnwidth]{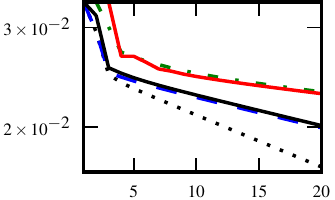}
\par\end{center}

\begin{center}
\medskip{}

\par\end{center}

\begin{center}
\includegraphics[width=0.5\columnwidth]{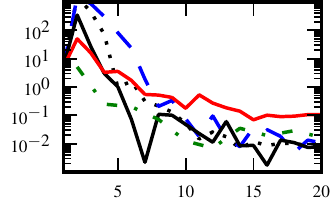}\includegraphics[width=0.49\columnwidth]{saga/millionsong-nl1-stats-runOn-2014-06-06--11-22-14}
\par\end{center}

\begin{center}
\textbf{Gradient Evaluations / $n$}
\par\end{center}%
\end{minipage}
\par\end{centering}

\includegraphics[width=1\columnwidth]{saga/caption}

\protect\caption{\label{fig:saga-l1}From left to right, top to bottom we have the
MNIST, COVTYPE, IJCNN1 and MILLIONSONG datasets with $L_{1}$ regularisation.}
\end{figure}

\section{Experiments}

\label{sec:saga-experiments}

We performed a series of experiments to validate the effectiveness
of SAGA. We tested a binary classifier on MNIST, COVTYPE, IJCNN1 and
a least squares predictor on MILLIONSONG. Details of these datasets
are in Section \ref{sec:finito-experiments}. We used the same code
base for each method, just changing the main update rule. SVRG was
tested with the recalibration pass used every $n$ iterations, as
suggested by \citet{semi}. Each method had its step size parameter
chosen to give the fastest convergence. In contrast to Section \ref{sec:finito-experiments},
we consider non-strongly convex problems, as well as the SVRG method,
which was not published at the time those earlier experiments were
run.

We tested with a $L_{2}$ regulariser, which all methods support,
and with a $L_{1}$ regulariser on a subset of the methods. The results
are shown in Figures \ref{fig:saga-l2} \& \ref{fig:saga-l1}. We
can see that Finito (perm) performs the best on a per epoch equivalent
basis, but it can be the most expensive method per step. We have observed
that Finito (perm) has the advantage in general on problems where
strong regularisation is used. SVRG is similarly fast on a per epoch
basis, but when considering the number of gradient evaluations per
epoch is double that of the other methods for this problem, it is
middle of the pack. SAGA can be seen to perform similar to the non-permuted
Finito case, and to SDCA. Note that SAG is slower than the other methods
at the beginning. To get the optimal results for SAG, an adaptive
step size rule needs to be used rather than the constant step size
we used.

In general, these tests confirm that the choice of methods should
be done based on their properties as discussed in Section \ref{sec:saga-relations},
rather than small constant factors in their convergence rates.

\section{SAGA Theory}

\label{sec:saga-theory}

In this section, all expectations are taken with respect to the choice
of $j$ at iteration $k+1$ and conditioned on $x^{k}$ and each $f_{i}^{\prime}(\phi_{i}^{k})$
unless stated otherwise. The main results are in Theorems \ref{thm:saga-main-lyp}
(strongly convex case) and \ref{thm:saga-non-sc-main}.

We start with Corollaries of standard convexity results.
\begin{cor}
\label{cor:saga-ip-bound} We can apply Theorem \ref{thm:full-strong-lb}
to our finite sum structure (using $y\leftarrow x$ and $x\leftarrow x^{*})$,
where each $f_{i}$ is $\mu$-strongly convex and is Lipschitz smooth
with constant $L$. We get that for all $x$ and $x^{*}$:
\begin{eqnarray*}
\left\langle f^{\prime}(x),x^{*}-x\right\rangle  & \leq & \frac{L-\mu}{L}\left[f(x^{*})-f(x)\right]-\frac{\mu}{2}\left\Vert x^{*}-x\right\Vert ^{2}\\
 &  & -\frac{1}{2Ln}\sum_{i}\left\Vert f_{i}^{\prime}(x^{*})-f_{i}^{\prime}(x)\right\Vert ^{2}-\frac{\mu}{L}\left\langle f^{\prime}(x^{*}),x-x^{*}\right\rangle .
\end{eqnarray*}

\end{cor}

\begin{cor}
\label{cor:grad-diff-phii} By applying Theorem \ref{thm:lipschitz-lb}:
\[
\ensuremath{f(y)\geq f(x)+\left\langle f^{\prime}(x),y-x\right\rangle +\frac{1}{2L}\left\Vert f^{\prime}(x)-f^{\prime}(y)\right\Vert ^{2}}
\]
 with $y\leftarrow\phi_{i}$ and $x\leftarrow x^{*}$, for each $f_{i}$
and summing, we have that for all $\phi_{i}$ and $x^{*}$:
\[
\frac{1}{n}\sum_{i}\left\Vert f_{i}^{\prime}(\phi_{i})-f_{i}^{\prime}(x^{*})\right\Vert ^{2}\leq2L\left[\frac{1}{n}\sum_{i}f_{i}(\phi_{i})-f(x^{*})-\frac{1}{n}\sum_{i}\left\langle f_{i}^{\prime}(x^{*}),\phi_{i}-x^{*}\right\rangle \right].
\]
\end{cor}
\begin{lem}
\label{lem:saga-error-bound}It holds that for any $\phi_{i}$,$x^{*}$
and $\ensuremath{\beta>0}$, with $w^{k+1}$ and $x^{k}$ as defined
in Equation \ref{eq:saga-main-update}:
\begin{eqnarray*}
\mathbb{E}\left\Vert w^{k+1}\!-\!x^{k}\!-\!\gamma f^{\prime}(x^{*})\right\Vert ^{2} & \!\leq\! & \!\gamma^{2}\left(1+\beta^{-1}\right)\mathbb{E}\left\Vert f_{j}^{\prime}(\phi_{j}^{k})-f_{j}^{\prime}(x^{*})\right\Vert ^{2}\\
 & \!+\! & \!\gamma^{2}\left(1+\beta\right)\mathbb{E}\left\Vert f_{j}^{\prime}(x^{k})-f_{j}^{\prime}(x^{*})\right\Vert ^{2}\!-\!\gamma^{2}\beta\left\Vert f^{\prime}(x^{k})-f^{\prime}(x^{*})\right\Vert ^{2}.
\end{eqnarray*}
\end{lem}
\begin{proof}
We follow a similar argument as occurs in the SVRG proof \citep{svrg}
for this term, but with a tighter argument. The tightening comes from
using $\left\Vert x+y\right\Vert ^{2}\leq(1+\beta^{-1})\left\Vert x\right\Vert ^{2}+(1+\beta)\left\Vert y\right\Vert ^{2}$
instead of the simpler $\beta=1$ case they use. The other key trick
is the use of the standard variance decomposition $\mathbb{E}[\left\Vert X-\mathbb{E}[X]\right\Vert ^{2}]=\mathbb{E}[\left\Vert X\right\Vert ^{2}]-\left\Vert \mathbb{E}[X]\right\Vert ^{2}$
three times.\foreignlanguage{english}{\texttt{
\begin{align*}
 & \mathbb{E}\left\Vert w^{k+1}-x^{k}+\stepsize f^{\prime}(x^{*})\right\Vert ^{2}\\
= & \mathbb{E}\bigg\Vert\underbrace{-\frac{\stepsize}{n}\sum_{i}f_{i}^{\prime}(\phi_{i}^{k})+\stepsize f^{\prime}(x^{*})+\stepsize\left[f_{j}^{\prime}(\phi_{j}^{k})-f_{j}^{\prime}(x^{k})\right]}_{:=\,\,\stepsize X}\bigg\Vert^{2}\\
= & \stepsize^{2}\mathbb{E}\Bigg\Vert\overbrace{\Bigg[f_{j}^{\prime}(\phi_{j}^{k})\!-\!f_{j}^{\prime}(x^{*})\!-\!\frac{1}{n}\sum_{i}f_{i}^{\prime}(\phi_{i}^{k})\!+\!f^{\prime}(x^{*})\Bigg]\!-\!\Bigg[f_{j}^{\prime}(x^{k})\!-\!f_{j}^{\prime}(x^{*})\!}^{X}-\!\overbrace{f^{\prime}(x^{k})\!+\!f^{\prime}(x^{*})\Bigg]}^{\mathbb{E}[X]}\Bigg\Vert^{2}\\
 & +\stepsize^{2}\bigg\Vert\overbrace{f^{\prime}(x^{k})-f^{\prime}(x^{*})}^{\mathbb{E}[X]}\bigg\Vert^{2}\\
\leq & \stepsize^{2}(1+\beta^{-1})\mathbb{E}\left\Vert f_{j}^{\prime}(\phi_{j}^{k})-f_{j}^{\prime}(x^{*})-\frac{1}{n}\sum_{i}f_{i}^{\prime}(\phi_{i})+f^{\prime}(x^{*})\right\Vert ^{2}\\
 & +\stepsize^{2}(1+\beta)\mathbb{E}\left\Vert f_{j}^{\prime}(x^{k})-f_{j}^{\prime}(x^{*})-f^{\prime}(x^{k})+f^{\prime}(x^{*})\right\Vert ^{2}+\stepsize^{2}\left\Vert f^{\prime}(x^{k})-f^{\prime}(x^{*})\right\Vert ^{2}\\
 & \textrm{\qquad\qquad(use variance decomposition twice more):}\\
\leq & \stepsize^{2}(1+\beta^{-1})\mathbb{E}\left\Vert f_{j}^{\prime}(\phi_{j}^{k})-f_{j}^{\prime}(x^{*})\right\Vert ^{2}+\stepsize^{2}(1+\beta)\mathbb{E}\left\Vert f_{j}^{\prime}(x^{k})-f_{j}^{\prime}(x^{*})\right\Vert ^{2}\\
 & -\stepsize^{2}\beta\left\Vert f^{\prime}(x^{k})-f^{\prime}(x^{*})\right\Vert ^{2}.
\end{align*}
}}\end{proof}
\begin{lem}
\label{lem:nonsc-bound}Define $\Delta=-\frac{1}{\stepsize}\left(w^{k+1}-x^{k}\right)-f^{\prime}(x^{k})$,
the difference between our approximation to the gradient at $x$ and
true gradient. Then using the SAGA step: 
\[
\mathbb{E}\left\Vert x^{k+1}-x^{*}\right\Vert ^{2}\leq\left\Vert x^{k}-x^{*}\right\Vert ^{2}-2\stepsize\mathbb{E}\left[F(x^{k+1})-F(x^{*})\right]+2\stepsize^{2}\mathbb{E}\left\Vert \Delta\right\Vert ^{2}.
\]
\end{lem}
\begin{proof}
This Lemma's proof is based on a similar Lemma 3 for SVRG \citep[2nd eq. on p.11][]{prox-svrg}.
The optimality condition of the proximal operator:
\[
z=\text{prox}_{1/\stepsize}^{h}\left(y\right):=\text{argmin}_{x\in\mathbb{R}^{d}}\,\left\{ h(x)+\frac{1}{2\stepsize}\Vert x-y\Vert^{2}\right\} 
\]
implies that there exists a subgradient $\xi$ of $h$ at $z$ such
that:
\[
\frac{1}{\stepsize}(z-y)+\xi=0.
\]
Applying this to our setting, where $x^{k+1}=\text{prox}_{1/\stepsize}^{h}\left(w^{k+1}\right)$,
and 
\[
w^{k+1}=x^{k}-\stepsize\left(f^{\prime}(x^{k})+\Delta\right),
\]
we have $\xi$ as a subgradient of $h$ at $x^{k+1}$. Now we proceed
to lower bound $F(x^{*})=f(x^{*})+h(x^{*})$, by treating the two
parts separately. Firstly the regulariser about $x^{k+1}$, using
the subgradient of $h$ defined above:
\begin{equation}
h(x^{*})\geq h(x^{k+1})+\left\langle \xi,x^{*}-x^{k+1}\right\rangle .\label{eq:saga-reg-lb}
\end{equation}
Next the loss about $x^{k}$:
\begin{equation}
f(x^{*})\geq f(x^{k})+\left\langle f^{\prime}(x^{k}),x^{*}-x^{k}\right\rangle .\label{eq:saga-nonsc-xstar-lb}
\end{equation}
We want $x^{k+1}$ on the right, so we further lower bound around
$x^{k+1}$ by using the negated Lipschitz smoothness upper bound (Theorem
\ref{thm:lipschitz-ub}):
\begin{equation}
f(x^{k})\geq f(x^{k+1})-\left\langle f^{\prime}(x^{k}),x^{k+1}-x^{k}\right\rangle -\frac{L}{2}\left\Vert x^{k+1}-x^{k}\right\Vert ^{2}.\label{eq:saga-nonsc-lip-lb}
\end{equation}
Combining Equation \ref{eq:saga-reg-lb} with \ref{eq:saga-nonsc-xstar-lb}
and \ref{eq:saga-nonsc-lip-lb} gives:
\begin{eqnarray*}
F(x^{*}) & \geq & f(x^{k+1})+h(x^{k+1})-\left\langle f^{\prime}(x^{k}),x^{k+1}-x\right\rangle -\frac{L}{2}\left\Vert x^{k+1}-x^{k}\right\Vert ^{2}\\
 &  & +\left\langle f^{\prime}(x^{k}),x^{*}-x^{k}\right\rangle +\left\langle \xi,x^{*}-x^{k+1}\right\rangle \\
 & = & F(x^{k+1})-\left\langle f^{\prime}(x^{k}),x^{k+1}-x^{k}-x^{*}+x^{k}\right\rangle -\frac{L}{2}\left\Vert x^{k+1}-x^{k}\right\Vert ^{2}-\left\langle \xi,x^{k+1}-x^{*}\right\rangle \\
 & = & F(x^{k+1})-\left\langle f^{\prime}(x^{k}),x^{k+1}-x^{*}\right\rangle -\frac{L}{2}\left\Vert x^{k+1}-x^{k}\right\Vert ^{2}-\left\langle \xi,x^{k+1}-x^{*}\right\rangle .
\end{eqnarray*}
Now we do some careful manipulation of the two inner product terms,
using the optimality condition
\begin{eqnarray*}
\xi & = & \frac{1}{\stepsize}(w^{k+1}-x^{k+1})\\
 & = & \frac{1}{\stepsize}(x^{k}-\stepsize\left(f^{\prime}(x^{k})+\Delta\right)-x^{k+1})\\
 & = & \frac{1}{\stepsize}(x^{k}-x^{k+1})-\left(f^{\prime}(x^{k})+\Delta\right),
\end{eqnarray*}
to simplify: 
\begin{eqnarray*}
 & = & -\left\langle f^{\prime}(x^{k}),x^{k+1}-x^{*}\right\rangle -\left\langle \xi,x^{k+1}-x^{*}\right\rangle \\
 &  & -\left\langle f^{\prime}(x^{k})+\xi,x^{k+1}-x^{*}\right\rangle \\
 & = & -\left\langle f^{\prime}(x^{k})+\frac{1}{\stepsize}(x^{k}-x^{k+1})-\left(f^{\prime}(x^{k})+\Delta\right),x^{k+1}-x^{*}\right\rangle 
\end{eqnarray*}
\begin{eqnarray*}
 & = & -\frac{1}{\stepsize}\left\langle x^{k}-x^{k+1},x^{k+1}-x^{*}\right\rangle +\left\langle \Delta,x^{k+1}-x^{*}\right\rangle \\
 & = & -\frac{1}{\stepsize}\left\langle x^{k}-x^{k+1},x^{k+1}-x^{k}+x^{k}-x^{*}\right\rangle +\left\langle \Delta,x^{k+1}-x^{*}\right\rangle \\
 & = & \frac{1}{\stepsize}\left\Vert x^{k+1}-x^{k}\right\Vert ^{2}-\frac{1}{\stepsize}\left\langle x^{k}-x^{k+1},x^{k}-x^{*}\right\rangle +\left\langle \Delta,x^{k+1}-x^{*}\right\rangle .
\end{eqnarray*}
So we have using $\frac{1}{\stepsize}\geq L$:
\begin{equation}
F(x^{*})\geq F(x^{k+1})+\frac{1}{\stepsize}\left\langle x^{k+1}-x^{k},x^{k}-x^{*}\right\rangle +\left\langle \Delta,x^{k+1}-x^{*}\right\rangle +\frac{1}{2\stepsize}\left\Vert x^{k+1}-x^{k}\right\Vert ^{2}.\label{eq:non-sc-midpoint}
\end{equation}
Now using: 
\begin{eqnarray*}
\left\Vert x^{k+1}-x^{*}\right\Vert ^{2} & = & \left\Vert x^{k+1}-x^{k}+x^{k}-x^{*}\right\Vert ^{2}\\
 & = & \left\Vert x^{k+1}-x^{k}\right\Vert ^{2}+2\left\langle x^{k+1}-x^{k},x^{k}-x^{*}\right\rangle +\left\Vert x^{k}-x^{*}\right\Vert ^{2},
\end{eqnarray*}
we have:
\[
F(x^{*})\geq F(x^{k+1})+\frac{1}{2\stepsize}\left\Vert x^{k+1}-x^{*}\right\Vert ^{2}-\frac{1}{2\stepsize}\left\Vert x^{k}-x^{*}\right\Vert ^{2}+\left\langle \Delta,x^{k+1}-x^{*}\right\rangle .
\]

Now we just need to bound the $\left\langle \Delta,x^{k+1}-x^{*}\right\rangle $
term. We do this by introducing a new point $y^{k+1}$ which is the
result of a non-stochastic proximal gradient step from $x^{k}$:
\[
y^{k+1}=\text{prox}_{1/\stepsize}^{h}\left(x^{k}-\stepsize f^{\prime}(x^{k})\right).
\]
So to lower bound $\left\langle \Delta,x^{k+1}-x^{*}\right\rangle $
we will look at upper bounding it's negation. We start by adding and
subtracting $y^{k+1}$, then we apply Cauchy-Schwarz:
\begin{eqnarray*}
-\left\langle \Delta,x^{k+1}-x^{*}\right\rangle  & = & -\left\langle \Delta,x^{k+1}-y^{k+1}\right\rangle -\left\langle \Delta,y^{k+1}-x^{*}\right\rangle \\
 & = & -\left\langle \Delta,x^{k+1}-y^{k+1}\right\rangle -\left\langle \Delta,y^{k+1}-x^{*}\right\rangle \\
 & \leq & \left\Vert \Delta\right\Vert \left\Vert x^{k+1}-y^{k+1}\right\Vert -\left\langle \Delta,y^{k+1}-x^{*}\right\rangle 
\end{eqnarray*}

Now note that the term $\left\langle \Delta,y^{k+1}-x^{*}\right\rangle $
has expectation zero as the right hand side is independent of the
$j$ choice, and $\mathbb{E}[\Delta]=0$. We next apply non-expansivity
of the proximal operator for $x^{k+1}-y^{k+1}$, and plug in the definitions
of $w^{k+1}$ in terms of $\Delta$:
\begin{eqnarray*}
\left\Vert \Delta\right\Vert \left\Vert x^{k+1}-y^{k+1}\right\Vert  & \leq & \left\Vert \Delta\right\Vert \left\Vert w^{k+1}-x^{k}+\stepsize f^{\prime}(x^{k})\right\Vert \\
 & = & \left\Vert \Delta\right\Vert \left\Vert x^{k}-\stepsize\left(f^{\prime}(x^{k})+\Delta\right)-x^{k}+\stepsize f^{\prime}(x^{k})\right\Vert \\
 & = & \left\Vert \Delta\right\Vert \left\Vert -\stepsize\left(\Delta\right)\right\Vert \\
 & = & \stepsize\left\Vert \Delta\right\Vert ^{2}
\end{eqnarray*}
Plugging this back into Equation \ref{eq:non-sc-midpoint} gives a
scaling of the result:
\[
F(x^{*})\geq\mathbb{E}\left[F(x^{k+1})\right]+\frac{1}{2\gamma}\mathbb{E}\left\Vert x^{k+1}-x^{*}\right\Vert ^{2}-\frac{1}{2\gamma}\left\Vert x^{k}-x^{*}\right\Vert ^{2}-\stepsize\mathbb{E}\left\Vert \Delta\right\Vert ^{2}.
\]

Multiplying through by $2\stepsize$ and rearranging gives the result.\end{proof}
\begin{lem}
\label{lem:delta-bound}We will also need to use a simple modification
of Lemma \ref{lem:saga-error-bound} to bound $\Delta$ as defined
in Lemma \ref{lem:nonsc-bound} directly above:
\[
\mathbb{E}\left\Vert \Delta\right\Vert ^{2}\leq\left(1+\beta^{-1}\right)\mathbb{E}\left\Vert f_{j}^{\prime}(\phi_{j}^{k})+f_{j}^{\prime}(x^{*})\right\Vert ^{2}+(1+\beta)\mathbb{E}\left\Vert f_{j}^{\prime}(x^{k})-f_{j}^{\prime}(x^{*})\right\Vert ^{2}.
\]
\end{lem}
\begin{proof}
We first expand and group terms:
\begin{eqnarray*}
 & = & \mathbb{E}\left\Vert \Delta\right\Vert ^{2}\\
 & = & \mathbb{E}\left\Vert \frac{1}{\stepsize}\left(w^{k+1}-x^{k}\right)+f^{\prime}(x^{k})\right\Vert ^{2}\\
 & = & \mathbb{E}\left\Vert -\left[f_{j}^{\prime}(\phi_{j}^{k+1})-f_{j}^{\prime}(\phi_{j}^{k})+\frac{1}{n}\sum_{i=1}^{n}f_{i}^{\prime}(\phi_{i}^{k})\right]+f^{\prime}(x^{k})\right\Vert ^{2}\\
 & = & \mathbb{E}\left\Vert -\left[f_{j}^{\prime}(x^{k})-f_{j}^{\prime}(\phi_{j}^{k})\right]+f^{\prime}(x^{k})-\frac{1}{n}\sum_{i=1}^{n}f_{i}^{\prime}(\phi_{i}^{k})\right\Vert ^{2}\\
 & = & \mathbb{E}\left\Vert \left[f_{j}^{\prime}(x^{k})-f_{j}^{\prime}(\phi_{j}^{k})\right]-\left[f^{\prime}(x^{k})-\frac{1}{n}\sum_{i=1}^{n}f_{i}^{\prime}(\phi_{i}^{k})\right]\right\Vert ^{2}.
\end{eqnarray*}

Now let $X=f_{j}^{\prime}(x^{k})-f_{j}^{\prime}(\phi_{j}^{k}).$ Note
that $E[X]=f^{\prime}(x^{k})-\frac{1}{n}\sum_{i=1}^{n}f_{i}^{\prime}(\phi_{i}^{k})$,
which is the other term, so we have $\mathbb{E}\left\Vert \Delta\right\Vert ^{2}=\mathbb{E}\left\Vert X-E[X]\right\Vert ^{2}$.
Applying decomposition of variance, $\mathbb{E}[\left\Vert X-\mathbb{E}[X]\right\Vert ^{2}]=\mathbb{E}[\left\Vert X\right\Vert ^{2}]-\left\Vert \mathbb{E}[X]\right\Vert ^{2}$.
We drop the 2nd term $-\left\Vert \mathbb{E}[X]\right\Vert ^{2}$
since we want an upper bound:
\[
\mathbb{E}\left\Vert \Delta\right\Vert ^{2}\leq\mathbb{E}\left\Vert f_{j}^{\prime}(x^{k})-f_{j}^{\prime}(\phi_{j}^{k})\right\Vert ^{2}.
\]

Now add and subtract $f_{j}^{\prime}(x^{*})$:
\begin{eqnarray*}
\mathbb{E}\left\Vert \Delta\right\Vert ^{2} & \leq & \mathbb{E}\left\Vert f_{j}^{\prime}(x^{k})-f_{j}^{\prime}(x^{*})-\left[f_{j}^{\prime}(\phi_{j}^{k})+f_{j}^{\prime}(x^{*})\right]\right\Vert ^{2}\\
 & \leq & \left(1+\beta^{-1}\right)\mathbb{E}\left\Vert f_{j}^{\prime}(\phi_{j}^{k})+f_{j}^{\prime}(x^{*})\right\Vert ^{2}+(1+\beta)\mathbb{E}\left\Vert f_{j}^{\prime}(x^{k})-f_{j}^{\prime}(x^{*})\right\Vert ^{2}.
\end{eqnarray*}

\end{proof}

\subsection{Linear convergence for strongly convex problems}

\label{sub:saga-theory}
\begin{thm}
\label{thm:saga-main-lyp} With $x^{*}$ the optimal solution, we
define the Lyapunov function: 
\[
T^{k}=\frac{1}{n}\sum_{i}f_{i}(\phi_{i}^{k})-f(x^{*})-\frac{1}{n}\sum_{i}\left\langle f_{i}^{\prime}(x^{*}),\phi_{i}^{k}-x^{*}\right\rangle +c\left\Vert x^{k}-x^{*}\right\Vert ^{2}.
\]

with $c=\frac{1}{2\stepsize(1-\stepsize\mu)n}$. It is implicitly
a function of $x^{k}$ and each $\phi_{i}^{k}$. \foreignlanguage{english}{Then
under the SAGA algorithm it holds that:}
\[
\mathbb{E}[T^{k+1}]\leq(1-\frac{1}{\kappa})T^{k},
\]

where \foreignlanguage{english}{$\frac{1}{\kappa}=\stepsize\mu=\frac{\mu}{2(\mu n+L)}$
for step size} \foreignlanguage{english}{$\stepsize=\frac{1}{2(\mu n+L)}$
and $\frac{1}{\kappa}=\min\left\{ \frac{1}{4n},\,\frac{\mu}{3L}\right\} $
for step size $\gamma=\frac{1}{3L}$.}\end{thm}
\begin{proof}
The first three terms in $T^{k+1}$ are straightforward to simplify:
\begin{eqnarray*}
\mathbb{E}\left[\frac{1}{n}\sum_{i}f_{i}(\phi_{i}^{k+1})\right] & = & \frac{1}{n}f(x^{k})+\left(1-\frac{1}{n}\right)\frac{1}{n}\sum_{i}f_{i}(\phi_{i}^{k}).\\
\mathbb{E}\left[-\frac{1}{n}\sum_{i}\left\langle f_{i}^{\prime}(x^{*}),\phi_{i}^{k+1}-x^{*}\right\rangle \right] & = & -\frac{1}{n}\left\langle f^{\prime}(x^{*}),x^{k}-x^{*}\right\rangle \\
 &  & -\left(1-\frac{1}{n}\right)\frac{1}{n}\sum_{i}\left\langle f_{i}^{\prime}(x^{*}),\phi_{i}^{k}-x^{*}\right\rangle .
\end{eqnarray*}

For the change in the last term of $T$ we apply the non-expansiveness
of the proximal operator. Note that this is the only place in the
proof where we use the fact that $x^{*}$ is an optimality point.
\begin{eqnarray*}
c\left\Vert x^{k+1}-x^{*}\right\Vert ^{2} & = & c\left\Vert \text{prox}_{1/\gamma}(w^{k+1})-\text{prox}_{1/\stepsize}(x^{*}-\stepsize f^{\prime}(x^{*}))\right\Vert ^{2}\\
 & \leq & c\left\Vert w^{k+1}-x^{*}+\stepsize f^{\prime}(x^{*})\right\Vert ^{2}.
\end{eqnarray*}
Then we expand the quadratic and apply $\mathbb{E}[w^{k+1}]=x^{k}-\stepsize f^{\prime}(x^{k})$
to simplify the inner product term:\foreignlanguage{english}{
\begin{align*}
 & c\mathbb{E}\left\Vert w^{k+1}-x^{*}+\stepsize f^{\prime}(x^{*})\right\Vert ^{2}\\
= & c\mathbb{E}\left\Vert x^{k}-x^{*}+w^{k+1}-x^{k}+\stepsize f^{\prime}(x^{*})\right\Vert ^{2}\\
= & c\left\Vert x^{k}-x^{*}\right\Vert ^{2}+2c\mathbb{E}\left[\left\langle w^{k+1}-x^{k}+\stepsize f^{\prime}(x^{*}),x^{k}-x^{*}\right\rangle \right]+c\mathbb{E}\left\Vert w^{k+1}-x^{k}+\stepsize f^{\prime}(x^{*})\right\Vert ^{2}\\
= & c\left\Vert x^{k}-x^{*}\right\Vert ^{2}-2c\stepsize\left\langle f^{\prime}(x^{k})-f^{\prime}(x^{*}),x^{k}-x^{*}\right\rangle +c\mathbb{E}\left\Vert w^{k+1}-x^{k}+\stepsize f^{\prime}(x^{*})\right\Vert ^{2}
\end{align*}
We now apply Lemma }\ref{lem:saga-error-bound} to bound the error
term $c\mathbb{E}\left\Vert w^{k+1}-x^{k}+\stepsize f^{\prime}(x^{*})\right\Vert ^{2}$,
giving:
\begin{eqnarray*}
 &  & c\mathbb{E}\left\Vert x^{k+1}-x^{*}\right\Vert ^{2}\\
 & \leq & c\left\Vert x^{k}-x^{*}\right\Vert ^{2}-c\stepsize^{2}\beta\left\Vert f^{\prime}(x^{k})-f^{\prime}(x^{*})\right\Vert ^{2}\\
 &  & -2c\stepsize\left\langle f^{\prime}(x^{k}),x^{k}-x^{*}\right\rangle +2c\stepsize\left\langle f^{\prime}(x^{*}),x^{k}-x^{*}\right\rangle \\
 &  & +\left(1+\beta^{-1}\right)c\stepsize^{2}\mathbb{E}\left\Vert f_{j}^{\prime}(\phi_{j}^{k})-f_{j}^{\prime}(x^{*})\right\Vert ^{2}+\left(1+\beta\right)c\stepsize^{2}\mathbb{E}\left\Vert f_{j}^{\prime}(x^{k})-f_{j}^{\prime}(x^{*})\right\Vert ^{2}.
\end{eqnarray*}
The value of $\beta$ shall be fixed later. Now we bound $-2c\gamma\left\langle f^{\prime}(x),x-x^{*}\right\rangle $
with Corollary \ref{cor:saga-ip-bound} and then apply Corollary \ref{cor:grad-diff-phii}
to bound $\mathbb{E}\left\Vert f_{j}^{\prime}(\phi_{j})-f_{j}^{\prime}(x^{*})\right\Vert ^{2}$:\foreignlanguage{english}{
\begin{align*}
c\mathbb{E}\left\Vert x^{k+1}-x^{*}\right\Vert ^{2}\leq & \left(c-c\stepsize\mu\right)\left\Vert x^{k}-x^{*}\right\Vert ^{2}\\
 & +\left((1\!+\!\beta)c\stepsize^{2}-\frac{c\stepsize}{L}\right)\mathbb{E}\left\Vert f_{j}^{\prime}(x^{k})-f_{j}^{\prime}(x^{*})\right\Vert ^{2}\!-\!c\stepsize^{2}\beta\left\Vert f^{\prime}(x^{k})-f^{\prime}(x^{*})\right\Vert ^{2}\\
 & -\frac{2c\stepsize(L-\mu)}{L}\left[f(x^{k})-f(x^{*})-\left\langle f^{\prime}(x^{*}),x^{k}-x^{*}\right\rangle \right]\\
 & +2\left(1+\beta^{-1}\right)c\stepsize^{2}L\left[\frac{1}{n}\sum_{i}f_{i}(\phi_{i})-f(x^{*})-\frac{1}{n}\sum_{i}\left\langle f_{i}^{\prime}(x^{*}),\phi_{i}-x^{*}\right\rangle \right].
\end{align*}
}

We can now combine the bounds we have derived for each term in $T$,
and pull out a fraction $\frac{1}{\kappa}$ of $T^{k}$ (for any $\kappa$
at this point). Together with the inequality from Theorem \ref{thm:strong-ub}:
\[
-\left\Vert f^{\prime}(x)-f^{\prime}(x^{*})\right\Vert ^{2}\leq-2\mu\left[f(x)-f(x^{*})-\left\langle f^{\prime}(x^{*}),x-x^{*}\right\rangle \right],
\]
 that yields:\foreignlanguage{english}{{\small{}
\begin{align}
\mathbb{E}[T^{k+1}]-T^{k}\leq & -\frac{1}{\kappa}T^{k}+\left(\frac{1}{n}-\frac{2c\stepsize(L-\mu)}{L}-2c\stepsize^{2}\mu\beta\right)\left[f(x^{k})-f(x^{*})-\left\langle f^{\prime}(x^{*}),x^{k}-x^{*}\right\rangle \right]\nonumber \\
 & +\left(\frac{1}{\kappa}+2(1\!+\!\beta^{-1})c\stepsize^{2}L-\frac{1}{n}\right)\left[\frac{1}{n}\sum_{i}f_{i}(\phi_{i}^{k})-f(x^{*})-\frac{1}{n}\sum_{i}\left\langle f_{i}^{\prime}(x^{*}),\phi_{i}^{k}-x^{*}\right\rangle \right]\nonumber \\
 & +\left(\frac{1}{\kappa}-\stepsize\mu\right)c\left\Vert x^{k}-x^{*}\right\Vert ^{2}+\left((1+\beta)\stepsize-\frac{1}{L}\right)c\stepsize\mathbb{E}\left\Vert f_{j}^{\prime}(x^{k})-f_{j}^{\prime}(x^{*})\right\Vert ^{2}.\label{eq:constants-strong}
\end{align}
}}{\small \par}

Note that each of the terms in square brackets are positive. In Section
\ref{sec:verf-constants} we verify that for the step size \foreignlanguage{english}{$\stepsize=\frac{1}{2(\mu n+L)}$,
the constants $c=\frac{1}{2\stepsize(1-\stepsize\mu)n}$, and $\kappa=\frac{1}{\stepsize\mu}$
together with $\beta=\frac{2\mu n+L}{L}$ ensure that each of the
quantities in the round brackets are non-positive. Likewise in Section
}\ref{sub:3L-step}\foreignlanguage{english}{ we show that for the
step size $\stepsize=\frac{1}{3L}$, the same $c=\frac{1}{2\stepsize(1-\stepsize\mu)n}$
as above can be used with $\beta=2$ and $\frac{1}{\kappa}=\min\left\{ \frac{1}{4n},\,\frac{\mu}{3L}\right\} $
to ensure non-positive terms.}\end{proof}
\begin{thm}
For step size $\stepsize=\frac{1}{2(\mu n+L)}$ it holds that: {\small{}
\[
\mathbb{E}\left[\left\Vert x^{k}\!-\!x^{*}\right\Vert ^{2}\right]\leq\left(1-\frac{\mu}{2(\mu n+L)}\right)^{k}\left[\left\Vert x^{0}-x^{*}\right\Vert ^{2}\!+\!\frac{n}{\mu n+L}\left[f(x^{0})\!-\!\left\langle f^{\prime}(x^{*}),x^{0}-x^{*}\right\rangle \!-\!f(x^{*})\right]\right],
\]
}and for step size $\stepsize=\frac{1}{3L}$ it holds that:

\selectlanguage{english}%
{\small{}
\[
\mathbb{E}\left\Vert x^{k}\!-\!x^{*}\right\Vert ^{2}\leq\left(1-\min\left\{ \frac{1}{4n}\,,\,\frac{\mu}{3L}\right\} \right)^{k}\left[\left\Vert x^{0}-x^{*}\right\Vert ^{2}\!+\!\frac{2n}{3L}\left[f(x^{0})\!-\!\left\langle f^{\prime}(x^{*}),x^{0}-x^{*}\right\rangle \!-\!f(x^{*})\right]\right].
\]
}{\small \par}

\selectlanguage{australian}%
Here the expectation is over all choices of index $j^{k}$ up to step
$k$. \end{thm}
\begin{proof}
In both cases we have from Theorem \ref{thm:saga-main-lyp} directly
above that:
\[
\mathbb{E}[T^{k+1}]\leq(1-\frac{1}{\kappa})T^{k},
\]

where the expectation is with respect to choices made at step $k+1$,
conditioning on the $x^{k}$ and $\phi^{k}$ values from step $k$.
We now take expectation with respect to those quantities also, giving:

\[
\mathbb{E}[T^{k+1}]\leq(1-\frac{1}{\kappa})\mathbb{E}[T^{k}].
\]

Now we may chain this inequality over $k$:
\begin{eqnarray*}
\mathbb{E}[T^{k}] & \leq & (1-\frac{1}{\kappa})\mathbb{E}[T^{k-1}]\\
 & \leq & (1-\frac{1}{\kappa})^{2}\mathbb{E}[T^{k-2}]\\
 &  & \dots\\
 & \leq & (1-\frac{1}{\kappa})^{k}T^{0}.
\end{eqnarray*}

Now looking at the definition of $T^{k}$:
\begin{eqnarray*}
T^{k} & = & \frac{1}{n}\sum_{i}f_{i}(\phi_{i}^{k})-f(x^{*})-\frac{1}{n}\sum_{i}\left\langle f_{i}^{\prime}(x^{k*}),\phi_{i}^{k}-x^{*}\right\rangle +c\left\Vert x^{k}-x^{*}\right\Vert ^{2}\\
 & \geq & c\left\Vert x^{k}-x^{*}\right\Vert ^{2}.
\end{eqnarray*}

So:
\[
c\left\Vert x^{k}-x^{*}\right\Vert ^{2}\leq(1-\frac{1}{\kappa})^{k}T^{0}.
\]

Now for the two step size cases, we simplify plug in the explicit
constant values for $c$ and $\kappa$, and divide the whole expression
through by $c$.
\end{proof}

\subsection{$1/k$ convergence for non-strongly convex problems}

\label{sub:saga-nonsc-theory}
\begin{thm}
\label{thm:saga-non-sc-main}When each $f_{i}$ is Lipschitz smooth
with constant $L$ and convex, using $\gamma=1/3L$, we have for $\bar{x}^{k}=\frac{1}{k}\sum_{t=1}^{k}x^{t}$
that:

\[
\mathbb{E}\left[F(\bar{x}^{k})\right]-F(x^{*})\leq\frac{10n}{k}\left[\frac{2L}{n}\left\Vert x^{0}-x^{*}\right\Vert ^{2}+f(x^{0})-\left\langle f^{\prime}(x^{*}),x^{0}-x^{*}\right\rangle -f(x^{*})\right].
\]

Here the expectation is over all choices of index $j^{k}$ up to step
$k$. \end{thm}
\begin{proof}
We proceed by using a similar argument as in Theorem \ref{thm:saga-main-lyp},
but with an additional $\alpha\left\Vert x-x^{*}\right\Vert ^{2}$
together with the existing $c\left\Vert x-x^{*}\right\Vert ^{2}$
term in the Lyapunov function. I.e.
\[
T=\frac{1}{n}\sum_{i}f_{i}(\phi_{i}^{k})-f(x^{*})-\frac{1}{n}\sum_{i}\left\langle f_{i}^{\prime}(x^{*}),\phi_{i}^{k}-x^{*}\right\rangle +\left(c+\alpha\right)\left\Vert x^{k}-x^{*}\right\Vert ^{2}.
\]

We will bound $\alpha\left\Vert x^{k}-x^{*}\right\Vert ^{2}$ in a
different manner to $c\left\Vert x^{k}-x^{*}\right\Vert ^{2}$. Instead
of use the non-expansiveness property, we apply Lemma \ref{lem:nonsc-bound}
scaled by $\alpha$, using $\Delta=-\frac{1}{\stepsize}\left(w^{k+1}-x^{k}\right)-f^{\prime}(x^{k})$
as defined in that lemma:

\[
\alpha\mathbb{E}\left\Vert x^{k+1}-x^{*}\right\Vert ^{2}\leq\alpha\left\Vert x^{k}-x^{*}\right\Vert ^{2}-2\alpha\gamma\mathbb{E}\left[F(x^{k+1})-F(x^{*})\right]+2\alpha\gamma^{2}\mathbb{E}\left\Vert \Delta\right\Vert ^{2}.
\]
Recall the bound on $\Delta$ from Lemma \ref{lem:delta-bound}:
\[
\mathbb{E}\left\Vert \Delta\right\Vert ^{2}\leq\left(1+\beta^{-1}\right)\mathbb{E}\left\Vert f_{j}^{\prime}(\phi_{j}^{k})-f_{j}^{\prime}(x^{*})\right\Vert ^{2}+\left(1+\beta\right)\mathbb{E}\left\Vert f_{j}^{\prime}(x^{k})-f_{j}^{\prime}(x^{*})\right\Vert ^{2},
\]
applying this gives
\begin{eqnarray*}
\alpha\mathbb{E}\left\Vert x^{k+1}-x^{*}\right\Vert ^{2} & \leq & \alpha\left\Vert x-x^{*}\right\Vert ^{2}-2\alpha\gamma\mathbb{E}\left[F(x^{k+1})-F(x^{*})\right]\\
 &  & +2(1+\beta^{-1})\alpha\gamma^{2}\mathbb{E}\left\Vert f_{j}^{\prime}(\phi_{j}^{k})-f_{j}^{\prime}(x^{*})\right\Vert ^{2}\\
 &  & +2\left(1+\beta\right)\alpha\gamma^{2}\mathbb{E}\left\Vert f_{j}^{\prime}(x^{k})-f_{j}^{\prime}(x^{*})\right\Vert ^{2}.
\end{eqnarray*}
now combining with the rest of the Lyapunov function change from Equation
\ref{eq:constants-strong}, using $\mu=0$ and bounding $\mathbb{E}\left\Vert f_{j}^{\prime}(\phi_{j}^{k})-f_{j}^{\prime}(x^{*})\right\Vert ^{2}$
with Corollary \ref{cor:grad-diff-phii}:

\selectlanguage{english}%
\begin{eqnarray*}
 &  & \mathbb{E}[T^{k+1}]-T^{k}\\
 & \leq & \left(\frac{1}{n}-2c\gamma\right)\left[f(x^{k})-f(x^{*})-\left\langle f^{\prime}(x^{*}),x^{k}-x^{*}\right\rangle \right]-2\alpha\gamma\mathbb{E}\left[F(x^{k+1})-F(x^{*})\right]\\
 &  & +\left(4(1+\beta^{-1})\alpha L\gamma^{2}+2(1+\beta^{-1})cL\gamma^{2}-\frac{1}{n}\right)\left[\frac{1}{n}\sum_{i}f_{i}(\phi_{i}^{k})-f(x^{*})\right.\\
 &  & \left.-\frac{1}{n}\sum_{i}\left\langle f_{i}^{\prime}(x^{*}),\phi_{i}^{k}-x^{*}\right\rangle \right]\\
 &  & +\left((1+\beta)c\gamma+2(1+\beta)\alpha\gamma-\frac{c}{L}\right)\gamma\mathbb{E}\left\Vert f_{j}^{\prime}(x^{k})-f_{j}^{\prime}(x^{*})\right\Vert ^{2}.
\end{eqnarray*}
\foreignlanguage{australian}{If we take $\gamma=1/3L$, $c=\frac{3L}{2n}$
and $\alpha=\frac{3L}{20n}$ and $\beta=\frac{3}{2}$ (See Section
\ref{sub:non-sc-constants-verf}), then we are left with}

\selectlanguage{australian}%
\begin{equation}
\mathbb{E}[T^{k+1}]-T^{k}\leq-\frac{1}{10n}\mathbb{E}\left[F(x^{k+1})-F(x^{*})\right].\label{eq:t-change-nonsc}
\end{equation}
These expectations are conditional on information from step $k$.
We now take the expectation with respect to all previous steps, yielding
\begin{eqnarray*}
\mathbb{E}[T^{k+1}]-\mathbb{E}[T^{k}] & \leq & -\frac{1}{10n}\mathbb{E}\left[F(x^{k+1})-F(x^{*})\right],
\end{eqnarray*}
where all expectations are unconditional. We want to sum Equation
\ref{eq:t-change-nonsc} from $0$ to $k-1$. For the left hand side
this results in telescoping of the $T$ terms:
\begin{eqnarray*}
\sum_{t=0}^{k-1}\left[\mathbb{E}[T^{t+1}]-\mathbb{E}[T^{t}]\right] & = & \mathbb{E}[T^{k}]-\mathbb{E}[T^{k-1}]+\mathbb{E}[T^{k-1}]-\mathbb{E}[T^{k-2}]+\dots-\mathbb{E}[T^{0}]\\
 & = & \mathbb{E}[T^{k}]-\mathbb{E}[T^{0}],
\end{eqnarray*}
For the right hand side we simply get $-\frac{1}{10n}\mathbb{E}\left[\sum_{t=0}^{k-1}\left[F(x^{t+1})-F(x^{*})\right]\right]$.
So it holds that:
\[
\frac{1}{10n}\mathbb{E}\left[\sum_{t=1}^{k}\left[F(x^{t})-F(x^{*})\right]\right]\leq T^{0}-\mathbb{E}[T^{k}].
\]
We can drop the $-\mathbb{E}\left[T^{k}\right]$ term since $T^{k}$
is always positive. For the right hand side we apply Jensen's inequality
to pull the summation inside of $F$: 
\[
\mathbb{E}\left[F(\frac{1}{k}\sum_{t=1}^{k}x^{t})-F(x^{*})\right]\leq\frac{1}{k}\mathbb{E}\left[\sum_{t=1}^{k}\left[F(x^{t})-F(x^{*})\right]\right].
\]
Further multiplying through by $10n/k$ gives:
\[
\mathbb{E}\left[F(\frac{1}{k}\sum_{t=1}^{k}x^{t})-F(x^{*})\right]\leq\frac{10n}{k}T^{0}.
\]

Plugging in the $T^{0}$ expression and using the rough bound $c+\alpha\leq\frac{2L}{n}$
gives the result.\end{proof}
\begin{rem}
The convergence rate in the non-strongly convex case is given in terms
of the average iterate instead of the last iterate. This is a fairly
common problem that arises in stochastic optimisation theory. In practice
the last iterate is still a good choice for the method to return,
despite the fact that it is not backed by the theory. Doing so allows
the method to be identical in the strongly convex and non-strongly
convex cases.

For a more theoretically sound approach, the SAGA algorithm can compute
the full function value at $F(x^{k})$ and at $F(\frac{1}{k}\sum_{t=1}^{k}x^{t})$,
and return the one that yields the lower function value.
\end{rem}

\section{Understanding the Convergence of the SVRG Method}

\label{sec:svrg-understanding}

Recall from Section \ref{sec:svrg-background} that SVRG has the following
convergence rate in terms of the number of recalibrations $r=k/m$:

\[
E[f(\tilde{x}^{k})-f(x^{*})]\leq\left(\frac{\eta}{\mu(1-4L/\eta)m}+\frac{4L(m+1)}{\eta(1-4L/\eta)m}\right)^{r}\left[f(\tilde{x}^{0})-f(x^{*})\right].
\]

Our analysis of the SAGA method uses some of the same techniques as
the SVRG analysis, so we would like to make a direct comparison of
the convergence rates of the two methods. However, the complexity
of the above expression makes the true convergence rate quite opaque.
In this section we provide a convergence rate proof of the SVRG method
that is tighter, and gives a rate directly relatable to the other
fast incremental gradient methods. Unfortunately, we have to impose
some assumptions on the behaviour of the iterates $x^{k}$ in order
to do so.

As discussed in Section \ref{sec:svrg-background}, in a practical
implementation of the SVRG method, it is assumed that
\begin{equation}
f(x^{k})-f(x^{*})\leq\frac{1}{m}\sum_{r=k-m}^{k}\left[f(x^{r})-f(x^{*})\right],\label{eq:svrg-even-weight}
\end{equation}

so that the last point $x^{k}$ may be returned by the method. This
is not directly supported by the theory, but empirically is overwhelmingly
likely. We will encode a similar assumption directly into our theory,
namely that
\begin{equation}
f(x^{k})-f(x^{*})\leq\frac{1}{\sum_{t=0}^{m-1}c_{1}^{t}}\sum_{t=0}^{m-1}c_{1}^{t}\left[f(x^{k-1-t})-f(x^{*})\right],\label{eq:svrg-geometric-weighting}
\end{equation}
\[
\text{where }\:c_{1}=\left(1-\frac{\mu}{\eta}\right).
\]

Instead of the uniform weighting of Equation \ref{eq:svrg-even-weight}
we use a geometric weighting. In practice the weight of $c_{1}^{0}$
v.s. $c_{1}^{m}$ will only a factor of 2 difference, so the practical
difference between this geometric weighting and a uniform weighting
is quite small. 

Under this assumption, we establish in Theorem \ref{thm:svrg-tight}
a convergence rate of the form:

\[
\mathbb{E}\left\Vert x^{k}-x^{*}\right\Vert ^{2}\leq\left(\max\left\{ \frac{1}{2},\left(1-\frac{\mu}{4L}\right)^{m}\right\} \right)^{k/m}\left[\left\Vert x^{0}-x^{*}\right\Vert ^{2}+\frac{c_{4}}{2L}\left[f(x^{0})-f(x^{*})\right]\right],
\]
for some constant $c^{4}$. Recall that the recalibration pass is
performed every $m$ iterations. This rate tells us that we should
set $m$ so that $\left(1-\frac{\mu}{4L}\right)^{m}=\frac{1}{2}$,
as using a larger $m$ will just lead to a slower convergence rate.
If we use that value, we get a convergence rate of simply:
\[
\mathbb{E}\left\Vert x^{k}-x^{*}\right\Vert ^{2}\leq\left(1-\frac{\mu}{4L}\right)^{k}\left[\left\Vert x^{0}-x^{*}\right\Vert ^{2}+\frac{c_{4}}{2L}\left[f(x^{0})-f(x^{*})\right]\right].
\]
 This is directly comparable to the slightly faster $1-\frac{\mu}{3L}$
rate for SAGA outside the big-data regime. Note that the effective
number of gradient calculations is larger for SVRG than SAGA as well,
as during each recalibration pass it does $n$ additional term-gradient
evaluations. In cases where $m\gg n$, the efficiency of the methods
is broadly comparable.
\begin{thm}
\label{thm:svrg-tight}We define the following Lyapunov function at
the $r$th recalibration: 
\[
T^{r}=\mathbb{E}\left\Vert \tilde{x}^{r}-x^{*}\right\Vert ^{2}+\frac{c_{4}}{2L}\left[f(\tilde{x}^{r})-f(x^{*})\right].
\]

Where $c_{4}=\sum_{t=0}^{m}c_{1}^{t}$ with $c_{1}=1-\frac{\mu}{\eta}$.
Then between recalibrations $r$ and $r+1$ when using step size $\eta=4L$
and Assumption \ref{eq:svrg-geometric-weighting}:
\[
\mathbb{E}[T^{r+1}]\leq\max\left\{ \frac{1}{2},\left(1-\frac{\mu}{4L}\right)^{m}\right\} T^{r}.
\]

This expectation is conditioned on $\tilde{x}^{r}$.\end{thm}
\begin{proof}
Recall the SVRG step:
\[
x^{k+1}=x^{k}-\frac{1}{\eta}f_{j}^{\prime}(x^{k})+\frac{1}{\eta}\left[f_{j}^{\prime}(\tilde{x}^{k})-g^{k}\right]
\]
Define:
\[
g_{j}(x)=f_{j}^{\prime}(x)-f_{j}^{\prime}(\tilde{x})+\mathbb{E}[f_{j}^{\prime}(\tilde{x})].
\]
Now like in the regular SVRG proof as well as the SAGA proof, we start
by expanding $\mathbb{E}\left\Vert x^{k+1}-x^{*}\right\Vert ^{2}$.
\begin{eqnarray*}
\mathbb{E}\left\Vert x^{k+1}-x^{*}\right\Vert ^{2} & = & \mathbb{E}\left\Vert x-\frac{1}{\eta}g_{j}(x)-x^{*}\right\Vert ^{2}\\
 & = & \left\Vert x-x^{*}\right\Vert ^{2}-\frac{2}{\eta}\mathbb{E}\left\langle g_{j}(x),x-x^{*}\right\rangle +\frac{1}{\eta^{2}}\left\Vert g_{j}(x)\right\Vert ^{2}\\
 & = & \left\Vert x-x^{*}\right\Vert ^{2}-\frac{2}{\eta}\left\langle f^{\prime}(x),x-x^{*}\right\rangle +\frac{1}{\eta^{2}}\left\Vert g_{j}(x)\right\Vert ^{2}
\end{eqnarray*}
Now we apply Corollary \ref{cor:saga-ip-bound} giving:
\begin{eqnarray*}
\mathbb{E}\left\Vert x^{k+1}-x^{*}\right\Vert ^{2} & \leq & \left(1-\frac{\mu}{\eta}\right)\left\Vert x-x^{*}\right\Vert ^{2}+\frac{1}{\eta^{2}}\mathbb{E}\left\Vert g_{j}(x)\right\Vert ^{2}\\
 &  & \frac{2\left(L-\mu\right)}{\eta L}\left[f(x^{*})-f(x)\right]-\frac{1}{L\eta}\mathbb{E}\left\Vert f_{j}^{\prime}(x^{*})-f_{j}^{\prime}(x)\right\Vert ^{2}.
\end{eqnarray*}
Now we bound the gradient norm term $\left\Vert g_{j}(x)\right\Vert ^{2}$,
using a straightforward modification of Lemma \ref{lem:saga-error-bound}
to apply to SVRG instead of SAGA:
\begin{eqnarray*}
\mathbb{E}\left\Vert g_{j}(x)\right\Vert ^{2} & \leq & \left(1+\beta^{-1}\right)\mathbb{E}\left\Vert f_{j}^{\prime}(\tilde{x})-f_{j}^{\prime}(x^{*})\right\Vert ^{2}+\left(1+\beta\right)\mathbb{E}\left\Vert f_{j}^{\prime}(x)-f_{j}^{\prime}(x^{*})\right\Vert ^{2}\\
 &  & -\beta\left\Vert f^{\prime}(x)\right\Vert ^{2}.
\end{eqnarray*}
Giving:
\begin{eqnarray*}
\mathbb{E}\left\Vert x^{k+1}-x^{*}\right\Vert ^{2} & \leq & \left(1-\frac{\mu}{\eta}\right)\left\Vert x-x^{*}\right\Vert ^{2}-\frac{\beta}{\eta^{2}}\left\Vert f^{\prime}(x)\right\Vert ^{2}\\
 &  & +\frac{1+\beta^{-1}}{\eta^{2}}\mathbb{E}\left\Vert f_{j}^{\prime}(\tilde{x})-f_{j}^{\prime}(x^{*})\right\Vert ^{2}+\frac{1+\beta}{\eta^{2}}\mathbb{E}\left\Vert f_{j}^{\prime}(x)-f_{j}^{\prime}(x^{*})\right\Vert ^{2}\\
 &  & +\frac{2\left(L-\mu\right)}{\eta L}\left[f(x^{*})-f(x)\right]-\frac{1}{L\eta}\mathbb{E}\left\Vert f_{j}^{\prime}(x^{*})-f_{j}^{\prime}(x)\right\Vert ^{2}.
\end{eqnarray*}
Next we apply the following three inequalities
\[
-\left\Vert f^{\prime}(x)\right\Vert ^{2}\leq-2\mu\left[f(x)-f(x^{*})\right]\ \text{(Theorem \ref{thm:strong-ub}),}
\]
\[
\mathbb{E}\left\Vert f_{j}^{\prime}(x)-f_{j}^{\prime}(x^{*})\right\Vert ^{2}\leq2L\left[f(x)-f(x^{*})\right]\;\text{(Theorem \ref{thm:lipschitz-lb}),}
\]
\[
\text{and }\mathbb{E}\left\Vert f_{j}^{\prime}(\tilde{x})-f_{j}^{\prime}(x^{*})\right\Vert ^{2}\leq2L\left[f(\tilde{x})-f(x^{*})\right]\:\text{(Corollary \ref{cor:grad-diff-phii}).}
\]
 Doing so gives:
\begin{eqnarray*}
\mathbb{E}\left\Vert x^{k+1}-x^{*}\right\Vert ^{2} & \leq & c_{1}\left\Vert x-x^{*}\right\Vert ^{2}+c_{2}\left[f(\tilde{x})-f(x^{*})\right]-c_{3}\left[f(x)-f(x^{*})\right].
\end{eqnarray*}
Where for convenience we have introduced the constants:
\[
c_{1}=1-\frac{\mu}{\eta},
\]
\[
c_{2}=\frac{2L\left(1+\beta^{-1}\right)}{\eta^{2}},
\]
\[
c_{3}=\frac{2\left(L-\mu\right)}{\eta L}+\frac{2\mu\beta}{\eta^{2}}+\frac{2}{\eta}-\frac{2L\left(1+\beta\right)}{\eta^{2}}.
\]
We want to use this equation recursively on the squared norm $\left\Vert x-x^{*}\right\Vert ^{2}$,
between recalibration steps $r+1$ and $r$. Note that $\tilde{x}^{r+1}=x^{k}$
and $\tilde{x}^{r}=x^{k-m}$. Doing the substitution yields:
\begin{eqnarray*}
\mathbb{E}\left\Vert \tilde{x}^{r+1}-x^{*}\right\Vert ^{2} & \leq & c_{1}^{m}\left\Vert \tilde{x}^{r}-x^{*}\right\Vert ^{2}+c_{2}\left(\sum_{t=0}^{m-1}c_{1}^{t}\right)\left[f(\tilde{x})-f(x^{*})\right]\\
 &  & -c_{3}\sum_{t=0}^{m-1}c_{1}^{t}\left[f(x^{k-1-t})-f(x^{*})\right].
\end{eqnarray*}
Now to simplify the notation further we define $c_{4}:=\sum_{t=0}^{m-1}c_{1}^{t}$.
and use Assumption \ref{eq:svrg-geometric-weighting} in the form
$c_{4}\mathbb{E}\left[f(\tilde{x}^{r+1})-f(x^{*})\right]\leq\sum_{t=0}^{m-1}c_{1}^{t}\left[f(x^{k-1-t})-f(x^{*})\right]$:
\[
\mathbb{E}\left\Vert \tilde{x}^{r+1}-x^{*}\right\Vert ^{2}+c_{3}c_{4}\mathbb{E}\left[f(\tilde{x}^{r+1})-f(x^{*})\right]\leq c_{1}^{m}\left\Vert \tilde{x}^{r}-x^{*}\right\Vert ^{2}+c_{2}c_{4}\left[f(\tilde{x})-f(x^{*})\right].
\]
Now we have two rates of descent to consider here. The squared norm
terms have a decrease of $\left(1-\frac{\mu}{\eta}\right)^{m}$ and
the function values descent by $c_{2}/c_{3}$. Our Lyapunov function
decrease is going to be the worst of the two. We can now plug in the
step size $\eta=4L$ and use $\beta=1$:
\[
c_{2}=\frac{2L\left(1+\beta^{-1}\right)}{\eta^{2}}=\frac{1}{4L}.
\]
Then:
\begin{eqnarray*}
c_{3} & = & \frac{2\left(L-\mu\right)}{\eta L}+\frac{2\mu\beta}{\eta^{2}}+\frac{2}{\eta}-\frac{2L\left(1+\beta\right)}{\eta^{2}}\\
 & = & \frac{2\left(L-\mu\right)}{4L^{2}}+\frac{2\mu}{16L^{2}}+\frac{2}{4L}-\frac{4}{16L}\\
 & = & \frac{\left(L-\mu\right)}{2L^{2}}+\frac{\mu}{8L^{2}}+\frac{1}{2L}-\frac{1}{4L}\\
 & = & \frac{1}{2L}-\frac{\mu}{2L^{2}}+\frac{\mu}{8L^{2}}+\frac{1}{4L}\\
 & \geq & \frac{1}{2L}.
\end{eqnarray*}
So we have $c_{2}/c_{3}\leq\frac{1}{2}$, giving a per recalibration
reduction of $\max\left\{ \frac{1}{2},\left(1-\frac{\mu}{4L}\right)^{m}\right\} $.
\end{proof}

\section{Verifying SAGA Constants}

\label{sec:verf-constants}This section covers the somewhat laborious
task of checking the constants work out in the SAGA proofs. For each
step size, constants need to be found so that the following constraints
hold:

\[
c_{1}=\frac{1}{n}-\frac{2c\stepsize(L-\mu)}{L}-2c\stepsize^{2}\mu\beta\leq0,
\]
\[
c_{2}=\frac{1}{\kappa}+2(1+\beta^{-1})c\stepsize^{2}L-\frac{1}{n}\leq0,
\]

\[
c_{3}=\frac{c}{\kappa}-\stepsize\mu c\leq0,
\]
\[
c_{4}=(1+\beta)c\gamma^{2}-\frac{c\gamma}{L}\leq0.
\]

\subsection{Strongly convex step size $\gamma=1/2(\mu n+L)$}

\label{sub:verf-strongly-convex}

We proceed by determining the values for the each of the constants
that ensure $c_{1},c_{3}$ and $c_{4}$ are exactly equal to 0. We
then verify that those constants leave $c_{2}$ as negative. Some
experimentation suggests this approach provides the best constants.
If we look at the constraints in the order $c_{4}$,$c_{1}$,$c_{3}$
then $c_{2}$ , each constrain restricts a single variable to a particular
value, assuming we want the first tree constraints to be equalities.
There is then very little looseness in the resulting $c_{2}$ constrain,
suggesting we are near optimal values for the constants.

First consider $c_{4}$. Under the assumed $\stepsize=1/2(\mu n+L)$
value:
\[
(1+\beta)\gamma^{2}-\frac{\gamma}{L}=0,
\]
\[
\therefore(1+\beta)=\frac{1}{\gamma L},
\]
\begin{eqnarray*}
\therefore\beta & = & \frac{1}{\gamma L}-1\\
 & = & \frac{2(\mu n+L)}{L}-\frac{L}{L}\\
 & = & \frac{2\mu n+L}{L}.
\end{eqnarray*}
So that locks in the value of $\beta$. Next consider $c_{1}$. Given
$\beta$ and $\stepsize$, we can determine the constant $c$ from
$c_{1}$:

\[
c_{1}=\frac{1}{n}-\frac{2c(L-\mu)\stepsize}{L}-2c\mu\beta\stepsize^{2}=0.
\]

\begin{eqnarray*}
\therefore\frac{L}{2\gamma n} & = & c\left((L-\mu)+\mu\gamma(2\mu n+L)\right)\\
 & = & c\left((L-\mu)+2\mu\gamma(\mu n+L)-\mu\gamma L\right)\\
 &  & c\left((L-\mu)+\mu-\mu\gamma L\right)\\
 & = & cL\left(1-\mu\gamma\right).
\end{eqnarray*}
Rearranging in terms of $c:$
\begin{eqnarray*}
c & = & \frac{1}{2n\gamma}/\left(1-\mu\gamma\right)\\
 & = & \frac{1}{2\gamma\left(1-\gamma\mu\right)n}.
\end{eqnarray*}
Now to determine $\kappa$ we take $c_{3}=0$:
\[
c_{3}=\frac{c}{\kappa}-\frac{\mu c}{n}=0,
\]
\[
\therefore\frac{1}{\kappa}=\mu\gamma.
\]
Now for $c_{2}$, we first calculate the value of $1+\beta^{-1}$:
\[
1+\beta^{-1}=1+\frac{L}{2\mu n+L}=\frac{2\mu n+2L}{2\mu n+L}=\frac{1}{\gamma\left(2\mu n+L\right)}.
\]
Plugging this into $c_{2}$:
\begin{eqnarray*}
c_{2} & = & \mu\gamma+2(1+\beta^{-1})cL\gamma^{2}-\frac{1}{n}\\
 & = & \mu\gamma+\frac{2(1+\beta^{-1})L\gamma^{2}}{2\gamma\left(1-\gamma\mu\right)n}-\frac{1}{n}\\
 & = & \mu\gamma+\frac{(1+\beta^{-1})L\gamma}{\left(1-\gamma\mu\right)n}-\frac{1}{n}\\
 & = & \mu\gamma+\frac{L\gamma}{\gamma\left(1-\gamma\mu\right)\left(2\mu n+L\right)n}-\frac{1}{n}\\
 & = & \mu\gamma+\frac{L}{n\left(1-\gamma\mu\right)\left(2\mu n+L\right)}-\frac{1}{n}.
\end{eqnarray*}

Now multiplying through by $n\left(1-\gamma\mu\right)\left(2\mu n+L\right)$
gives the right hand side:
\begin{eqnarray*}
 &  & L+(\mu\gamma-\frac{1}{n})n\left(1-\gamma\mu\right)\left(2\mu n+L\right)\\
 & = & L+(\mu\gamma n-1)\left(1-\gamma\mu\right)\left(2\mu n+L\right)\\
 & = & L+\mu\gamma n\left(1-\gamma\mu\right)\left(2\mu n+L\right)-\left(1-\gamma\mu\right)\left(2\mu n+L\right)\\
 & = & L+\mu\gamma n\left(2\mu n+L\right)-\mu^{2}\gamma^{2}n\left(2\mu n+L\right)-\left(1-\gamma\mu\right)\left(2\mu n+L\right)\\
 & \leq & L+\mu\gamma n\left(2\mu n+L\right)-\left(1-\gamma\mu\right)\left(2\mu n+L\right)\\
 & = & \mu\gamma n\left(2\mu n+L\right)+\gamma\mu\left(2\mu n+L\right)-2\mu n.
\end{eqnarray*}

Now note that:
\begin{eqnarray*}
\gamma\left(2\mu n+L\right) & = & \frac{\left(2\mu n+L\right)}{2\mu n+2L}\leq1.
\end{eqnarray*}

So plugging that into the $c_{2}$ bound:
\begin{eqnarray*}
c_{2} & \leq & \mu n+\mu-2\mu n\\
 & \leq & 0.
\end{eqnarray*}

So we have verified that each of $c_{1}$ to $c_{4}$ is non-positive
for the derived constants.

\subsection{Strongly convex step size $\gamma=1/3L$}

\label{sub:3L-step}

This step size is of interest because it does not involve $\mu$.
Using such a step size we can hope to achieve a level of \emph{adaptivity},
where the algorithm automatically speeds up based on the unknown level
of strong convexity.

First we look at $c_{4}$:
\begin{eqnarray*}
c_{4} & = & (1+\beta)c\gamma^{2}-\frac{c\gamma}{L}\\
 & \propto & (1+\beta)\gamma-\frac{1}{L}\\
 & \propto & \frac{(1+\beta)}{3}-1,
\end{eqnarray*}
\[
\therefore1+\beta=3,
\]
\[
\beta=2.
\]
So we can use the value $\beta=2$. Now to determine $c$ we look
at $c_{1}$:
\begin{eqnarray*}
c_{1} & = & \frac{1}{n}-\frac{2c\stepsize(L-\mu)}{L}-2c\stepsize^{2}\mu\beta\\
 & = & \frac{1}{n}-c\left(\frac{2(L-\mu)}{3L^{2}}+\frac{4\mu}{9L^{2}}\right)\\
 & = & \frac{1}{n}-c\left(\frac{2}{3L}+\left(\frac{4}{9}-\frac{2}{3}\right)\frac{\mu}{L^{2}}\right)\\
 & = & \frac{1}{n}-c\left(\frac{2}{3L}-\frac{2\mu}{9L^{2}}\right)\\
 & = & \frac{1}{n}-2c\left(\gamma-\mu\gamma^{2}\right)\\
 & = & \frac{1}{n}-2c\gamma\left(1-\mu\gamma\right),
\end{eqnarray*}
\[
\therefore c=\frac{1}{2\gamma\left(1-\gamma\mu\right)n}.
\]

Note that this is the same constant as we got for the $\gamma=1/2(\mu n+L)$
step size, when put in terms of $\gamma$. 

To determine the next constant, rather than fixing the descent rate
$\kappa$ using $c_{3}$ like in Section \ref{sub:verf-strongly-convex},
we will consider $c_{2}$ first:
\begin{eqnarray*}
c_{2} & = & \frac{1}{\kappa}+2(1+\beta^{-1})cL\gamma^{2}-\frac{1}{n}\\
 & = & \frac{1}{\kappa}+\frac{2(1+\beta^{-1})L\gamma^{2}}{2\gamma\left(1-\gamma\mu\right)n}-\frac{1}{n}\\
 & = & \frac{1}{\kappa}+\frac{(1+\frac{1}{2})L\gamma}{\left(1-\gamma\mu\right)n}-\frac{1}{n}\\
 & = & \frac{1}{\kappa}+\frac{3L\gamma}{2\left(1-\gamma\mu\right)n}-\frac{1}{n}\\
 & = & \frac{1}{\kappa}+\frac{1}{2\left(1-\mu/3L\right)n}-\frac{1}{n}
\end{eqnarray*}

\begin{eqnarray*}
 & = & \frac{1}{\kappa}+\frac{3L}{2\left(3L-\mu\right)n}-\frac{1}{n}\\
 & \leq & \frac{1}{\kappa}+\frac{3L}{2\left(3L-L\right)n}-\frac{1}{n}\\
 & = & \frac{1}{\kappa}+\frac{3}{4n}-\frac{1}{n}\\
 & = & \frac{1}{\kappa}-\frac{1}{4n}.
\end{eqnarray*}

So in order that $c_{2}$ be non-positive, we require that:
\[
\frac{1}{\kappa}\leq\frac{1}{4n}.
\]
Now we need to consider the set of possible $\kappa$ values carefully.
In addition to this bound on $1/\kappa$, have the restriction on
$1/\kappa$ from $c_{3}$ of:
\[
\frac{1}{\kappa}\leq\frac{\mu}{3L}.
\]
So we can take $\frac{1}{\kappa}$ as the minimum of the two quantities:
\[
\frac{1}{\kappa}=\min\left\{ \frac{1}{4n},\,\frac{\mu}{3L}\right\} .
\]
So our constants are $\gamma=1/3L$, $\beta=2$, $c=\frac{1}{2\gamma\left(1-\gamma\mu\right)n}$
and $\frac{1}{\kappa}=\min\left\{ \frac{1}{4n},\,\frac{\mu}{3L}\right\} $.

\subsection{Non-strongly convex step size $\gamma=1/3L$}

\label{sub:non-sc-constants-verf}

The constraints are a little different here than the strongly convex
case. We will name each as $\tau$:\foreignlanguage{english}{
\[
\tau_{1}=\frac{1}{n}-2c\gamma,
\]
\[
\tau_{2}=4(1+\beta^{-1})\alpha L\gamma^{2}+2(1+\beta^{-1})cL\gamma^{2}-\frac{1}{n},
\]
\[
\tau_{3}=(1+\beta)c\gamma+2(1+\beta)\alpha\gamma-\frac{c}{L}.
\]
}

\selectlanguage{english}%
A set of constants that work here are $\alpha=\frac{3L}{20n},$ $c=\frac{3L}{2n}$,
$\beta=\frac{3}{2}$ and $\gamma=\frac{1}{3L}$. These are not quite
as tight to the constraints as in the strongly convex case. We will
derive these by examining each of these $\tau$ expressions in turn.
We will start with $\beta=\frac{3}{2}$ as a fixed variable, since
$\beta=2$ (used above) doesn't quite work.

\selectlanguage{australian}%
For $\tau_{1}$ 

\begin{eqnarray*}
\frac{1}{n}-2c\gamma & = & \frac{1}{n}-\frac{2}{3L}c,
\end{eqnarray*}
So $\tau_{1}$ is equal to zero when: 
\begin{eqnarray*}
c & = & \frac{3L}{2n}.
\end{eqnarray*}
Now lets look at $\tau_{3}$:
\begin{eqnarray*}
(1+\beta)c\gamma+2(1+\beta)\alpha\gamma-\frac{c}{L} & = & (1+\frac{3}{2})c\gamma+2(1+\frac{3}{2})\alpha\gamma-\frac{c}{L}\\
 &  & \frac{5}{2}c\gamma+5\alpha\gamma-\frac{c}{L}\\
 & = & \frac{5}{6}\frac{c}{L}+\frac{5\alpha}{3L}-\frac{c}{L}\\
 & = & \frac{5\alpha}{3L}-\frac{c}{6L}\\
 & \propto & 5\alpha-\frac{c}{2}.
\end{eqnarray*}
\begin{eqnarray*}
\therefore\alpha & \leq & \frac{1}{10}c\\
 & \leq & \frac{3L}{20n}.
\end{eqnarray*}
So we can take $\alpha=\frac{3L}{20n}.$ We are now ready to look
at $\tau_{2}$:
\begin{eqnarray*}
4(1+\beta^{-1})\alpha L\gamma^{2}+2(1+\beta^{-1})cL\gamma^{2}-\frac{1}{n} & = & 4(1+\frac{2}{3})\alpha L\gamma^{2}+2(1+\frac{2}{3})cL\gamma^{2}-\frac{1}{n}\\
 & = & \frac{8L}{3}\alpha L\gamma^{2}+\frac{10}{3}cL\gamma^{2}-\frac{1}{n}\\
 & = & \frac{8L}{3}\alpha L\gamma^{2}+\frac{10}{3*9L}c-\frac{1}{n}\\
 & = & \frac{8L}{3}\alpha L\gamma^{2}+\frac{10}{3*9L}c-\frac{1}{n}\\
 & = & \frac{8L}{3}\alpha L\gamma^{2}+\frac{10*3}{3*9*2n}-\frac{1}{n}\\
 & = & \frac{8}{3*9L}\alpha+\frac{15}{27n}-\frac{1}{n}\\
 & = & \frac{8*3}{3*9*20n}+\frac{15}{27n}-\frac{1}{n}\\
 & \leq & (0.599...)\frac{1}{n}-\frac{1}{n}\leq0.
\end{eqnarray*}
So each of the constraints are satisfied.

\chapter{Access Orders and Complexity Bounds}

\label{chap:inc-discus}

In this chapter we take a high level view of the class of fast incremental
gradient methods. We start by formalising the class of lower complexity
bounds for finite sum problems. We discuss the best possible convergence
rates that can achieved by an incremental gradient method on such
problems, and we prove one such bound. In Section \ref{sub:lower-non-sc}
we show that incremental gradient methods can not be faster than batch
optimisation methods when dealing with non-strongly convex problems,
unless further assumptions are placed on the problem class.

The necessity of randomisation in incremental gradient methods is
also discussed in this chapter. We give an overview of which of the
commonly used optimisation methods have theoretical or practical speed-ups
under randomised access orders. For the SAG method, it is known that
using much smaller step sizes removes the need for randomisation.
We prove a similar result for Finito in Section \ref{sec:miso-cyclic},
showing that other access patterns are at least as fast as random
access when using small enough step sizes.

\section{Lower Complexity Bounds}

\label{sec:oracles}

The convergence rate bounds we have established so far for the fast
incremental gradient methods give us an idea of the worst case convergence
rate, for \emph{any problem} satisfying our assumptions, under a particular
algorithm. Lower complexity bounds are the opposite side of the coin.
They show the best possible convergence rate for \emph{any algorithm}
on the hardest problems in class considered. In equational form they
resemble convergence rate estimates but with the inequality $\geq$
instead of $\leq$. Convergence rates are a property of an algorithm,
applied to a class, whereas lower complexity bounds are a property
of the class of problems, for all algorithms\footnote{We usually need some assumptions on the algorithm's possible behaviour,
but these are typically very weak.}. 

If we are able to establish lower complexity bounds that match our
convergence rate bounds for a particular method within a constant,
then by convention we call that method \emph{optimal} for the class
of problems under consideration.

It is helpful to formalise what information about a function that
an optimisation method has access to at each step. This is normally
done using the notion of an \emph{oracle}. An oracle is simply an
ancillary function $\mathcal{O}$ that when queried returns some specific
information about the function $f$. The simplest case we consider
is a first order oracle, where $\mathcal{O}:\mathbb{R}^{d}\rightarrow(\mathbb{R},\mathbb{R}^{d})$
just takes some point $x$ in the domain of $f$ as input and returns
the function value and gradient at $x$. From a software development
point of view, the oracle is the \emph{interface} specification of
our problem. An optimisation algorithm can only interact with the
objective function through the interface defined by the oracle. We
state the lower complexity bounds on a class of problems in terms
of the accuracy that can be obtained after $k$ invocations to a particular
oracle.

The theory for the class of $L$-smooth $\mu$-strongly convex problems
under the first order oracle (known as $S_{s,L}^{1,1}$) is well developed\footnote{The $1,1$ in $S_{s,L}^{1,1}$ refers to the problem being at least
once differentiable, and the oracle being first order}. These results require the technical condition that the dimensionality
of the input space $\mathbb{R}^{d}$ is much larger than the number
of iterations we will take. For simplicity we will assume this is
the case in the following discussions. 

It has been proven \citep{nem-yudin} that problems exist in $S_{s,L}^{1,1}$
for which the lower complexity bound is:
\[
\left\Vert w^{k}-w^{*}\right\Vert ^{2}\geq\left(\frac{\sqrt{L/\mu}-1}{\sqrt{L/\mu}+1}\right)^{2k}\left\Vert w^{0}-w^{*}\right\Vert ^{2}.
\]

In fact, when $\mu$ and $L$ are known in advance, this rate is achieved
up to a small constant factor by several methods, most notably by
Nesterov's accelerated gradient descent method \citep{nes-opt}. So
in order to achieve convergence rates faster than this, additional
assumptions must be made on the class of functions considered. 

In this work we have considered problems in $S_{s,L}^{1,1}$ with
an additional finite sum structure. Under this structure, we have
established convergence rates that in expectation are can be substantially
better than the above lower complexity bound. Ideally we would like
to establish a tight lower complexity bound for our new class.

To be precise, we now define the (stochastic) oracle class ${FS}_{s,L,n}^{1,1}(\mathbb{R}^{d})$
for which SAG and Finito most naturally fit.

\fbox{\begin{minipage}[t]{1\columnwidth}%
\textbf{Function class:} $f(w)=\frac{1}{n}\sum_{i=1}^{n}f_{i}(w)$,
with $f_{i}\in S_{s,L}^{1,1}(\mathbb{R}^{d})$.

\textbf{Oracle:} Each query takes a point $x\in\mathbb{R}^{d}$, and
returns $j$, $f_{j}(w)$ and $f_{j}^{\prime}(w)$, with $j$ chosen
uniformly at random.

\textbf{Accuracy:} Find $w$ such that $\mathbb{E}[\left\Vert w^{k}-w^{*}\right\Vert ^{2}]\leq\epsilon$.%
\end{minipage}}

The main choice made in formulating this definition is putting the
random choice inside the oracle. This restricts the methods allowed
quite strongly. The alternative case, where the index $j$ is input
to the oracle in addition to $x$, is also interesting. Assuming that
the method has access to a source of true random indices, we call
that class ${DS}_{s,L,n}^{1,1}(\mathbb{R}^{d})$. In Section \ref{sec:finito-experiments}
we showed empirical evidence that suggests that faster rates are possible
in ${DS}_{s,L,n}^{1,1}(\mathbb{R}^{d})$ than for ${FS}_{s,L,n}^{1,1}(\mathbb{R}^{d})$.

It should first be noted that there is a simple lower bound rate for
$f\in{FS}_{s,L,n}^{1,1}(\mathbb{R}^{d})$ of $\left(1-\frac{1}{n}\right)$
reduction per step. We establish this in Section \ref{sub:simple-lower-bound}.
Finito is only a factor of $2$ off this rate, when the big-data condition
holds, namely $\left(1-\frac{1}{2n}\right)$. SDCA also achieves the
rate asymptotically as the amount of data is increased.

Another case to consider is the smooth convex but non-strongly convex
setting. We still assume Lipschitz smoothness. In this setting we
will show in Section \ref{sub:lower-non-sc} that for sufficiently
high dimensional input spaces, the (non-stochastic) lower complexity
bound is the same for the finite sum case and cannot be better than
that given by treating $f$ as a single black box function. Essentially,
strongly convexity is crucial to the construction of a fast incremental
gradient method.

\subsection{Technical assumptions}

In the remainder of this section we use the following technical assumption,
as used in \citet{nes-book}:

\textit{Assumption 1: An optimisation method at step $k$ may only
invoke the oracle with a point $x^{k}$ that is of the form:
\[
x^{k}=x^{0}+\sum_{i}a_{i}g^{(i)},
\]
}\emph{where $g^{(i)}$ is the derivative returned by the oracle at
step $i$, and $a_{i}\in\mathbb{R}$. }

This assumption prevents an optimisation method from just guessing
the correct solution without doing any work. Virtually all optimisation
methods satisfy this assumption.

\subsection{Simple $(1-\frac{1}{n})^{k}$ bound}

\label{sub:simple-lower-bound}

Any procedure that minimises a sum of the form $f(w)=\frac{1}{n}\sum_{i}f_{i}(w)$
by uniform random access of $f_{i}$ is restricted by the requirement
that it has to actually see each term at least once in order to find
the minimum. This leads to a $\left(1-\frac{1}{n}\right)^{k}$ rate
in expectation. We now formalise such an argument. We will work in
$\mathbb{R}^{n}$, matching the dimensionality of the problem to the
number of terms in the summation.
\begin{thm}
For any $f\in{FS}_{1,n,n}^{1,1}(\mathbb{R}^{d})$, and initial point
$w^{0}\in\mathbb{R}^{d}$ we have that a $k$ step optimisation procedure
gives an expected sub-optimality of at best:
\[
\mathbb{E}[f(w)]-f(w^{*})\geq\left(1-\frac{1}{n}\right)^{k}\left(f(w^{0})-f(w^{*})\right).
\]
\end{thm}
\begin{proof}
We will exhibit a simple worst-case problem. Without loss of generality
we assume that the first oracle access by the optimisation procedure
is at $w=0$. In any other case, we shift our space in the following
argument appropriately. 

Let $f(w)=\frac{1}{n}\sum_{i}\left[\frac{n}{2}\left(w_{i}-1\right)^{2}+\frac{1}{2}\left\Vert w\right\Vert ^{2}\right]$.
Then clearly the solution is $w_{i}=\frac{1}{2}$ for each $i$, with
minimum of $f(w^{*})=\frac{n}{4}$. For $w=0$ we have $f0=\frac{n}{2}$.
Since the derivative of each $f_{j}$ is $0$ on the $i$th component
if we have not yet seen $f_{i}$, the value of each $w_{i}$ remains
$0$ unless term $i$ has been seen. 

Let $v^{k}$ be the number of unique terms we have not seen up to
step $k$. Between steps $k$ and $k+1$, $v$ decreases by $1$ with
probably $\frac{v}{n}$ and stays the same otherwise. So 
\[
\mathbb{E}[v^{k+1}|v^{k}]=v^{k}-\frac{v^{k}}{n}=\left(1-\frac{1}{n}\right)v^{k}.
\]

So we may define the sequence $X^{k}=\left(1-\frac{1}{n}\right)^{-k}v^{k}$,
which is then martingale with respect to $v$, as
\begin{eqnarray*}
\mathbb{E}[X^{k+1}|v^{k}] & = & \left(1-\frac{1}{n}\right)^{-k-1}\mathbb{E}[v^{k+1}|v^{k}]\\
 & = & \left(1-\frac{1}{n}\right)^{-k}v^{k}\\
 & = & X^{k}.
\end{eqnarray*}

Now since $k$ is chosen in advance, stopping time theory gives that
$\mathbb{E}[X^{k}]=\mathbb{E}[X^{0}]$. So 
\[
\mathbb{E}[\left(1-\frac{1}{n}\right)^{-k}v^{k}]=n,
\]
\[
\therefore\mathbb{E}[v^{k}]=\left(1-\frac{1}{n}\right)^{k}n.
\]

By Assumption 1, the function can be at most minimised over the dimensions
seen up to step $k$. The seen dimensions contribute a value of $\frac{1}{4}$
and the unseen terms $\frac{1}{2}$ to the function. So we have that:
\begin{eqnarray*}
\mathbb{E}[f(w^{k})]-f(w^{*}) & \geq & \frac{1}{4}\left(n-\mathbb{E}[v^{k}]\right)+\frac{1}{2}\mathbb{E}[v^{k}]-\frac{n}{4}\\
 & = & \frac{1}{4}\mathbb{E}[v^{k}]\\
 & = & \left(1-\frac{1}{n}\right)^{k}\frac{n}{4}\\
 & = & \left(1-\frac{1}{n}\right)^{k}\left[f(w^{0})-f(w^{*})\right].
\end{eqnarray*}

\end{proof}

\subsection{Minimisation of non-strongly convex finite sums}

\label{sub:lower-non-sc}

Consider the class of convex \& differentiable problems, with $L$
Lipschitz smoothness $F_{L}^{1,1}(\mathbb{R}^{d})$ . For every $f$
in $F_{L}^{1,1}(\mathbb{R}^{d})$, the following lower complexity
bound holds when $k<d$, for any $x^{0}\in\mathbb{R}^{d}$ and for
every minimiser $x^{*}$ of $f$:
\[
f(x^{k})-f(x^{*})\geq\frac{L\left\Vert x^{0}-x^{*}\right\Vert ^{2}}{8\left(k+1\right)^{2}},
\]

which is proved via explicit construction of a worst-case function
$f$ where it holds with equality \citep{nes-book}. Let this worst
case function be denoted $h^{k}$ at step $k$.

We will show that the same bound applies for the finite-sum case,
on a per pass equivalent basis, by a simple construction. 
\begin{thm}
The following lower bound holds for $k$ a multiple of $n$, for any
$x^{0}\in\mathbb{R}^{d}$:
\end{thm}
\[
f(x^{k})-f(x^{*})\geq\frac{L\left\Vert x^{0}-x^{*}\right\Vert ^{2}}{8(\frac{k}{n}+1)^{2}},
\]

when $f$ is a finite sum of $n$ terms $f(x)=\frac{1}{n}\sum_{i}f_{i}(x)$,
with each $f_{i}\in F_{L}^{1,1}(\mathbb{R}^{d})$, and with $d>kn$,
under the oracle model where the optimisation method may choose the
index $i$ to access at each step.
\begin{proof}
Let $h_{i}$ be a copy of $h^{k}$ redefined to be acting on the subset
of dimensions $i+jn$, for $j=1\dots k$, or in other words, $h_{i}^{k}(x)=h^{k}([x_{i},x_{i+n},\dots x_{i+jn},\dots])$.
Then we will use:
\[
f^{k}(x)=\frac{1}{n}\sum_{i}h_{i}^{k}(x),
\]

as a worst case function for step $k$. 

Since the derivatives are orthogonal between $h_{i}$ and $h_{j}$
for $i\neq j$, by Assumption 1, the bound on $h_{i}^{k}(x^{k})-h_{i}^{k}(x^{*})$
depends only on the number of times the oracle has been invoked with
index $i$, for each $i$. Let this be denoted $c_{i}$. Then we have
that:
\[
f^{k}(x^{k})-f^{k}(x^{*})\geq\frac{L}{8n}\sum_{i}\frac{\left\Vert x^{0}-x^{*}\right\Vert _{(i)}^{2}}{(c_{i}+1)^{2}}.
\]

Where $\left\Vert \cdot\right\Vert _{(i)}^{2}$ is the norm on the
dimensions $i+jn$ for $j=1\dots k$. We can combine these norms into
a regular Euclidean norm:
\[
f^{k}(x^{k})-f^{k}(x^{*})\geq\frac{L\left\Vert x^{0}-x^{*}\right\Vert ^{2}}{8n}\sum_{i}\frac{1}{(c_{i}+1)^{2}}.
\]

Now notice that $\sum_{i}\frac{1}{(c_{i}+1)^{2}}$ under the constraint
$\sum c_{i}=k$ is minimised when each $c_{i}=\frac{k}{n}$. So we
have:

\begin{eqnarray*}
f^{k}(x^{k})-f^{k}(x^{*}) & \geq & \frac{L\left\Vert x^{0}-x^{*}\right\Vert ^{2}}{8n}\sum_{i}\frac{1}{(\frac{k}{n}+1)^{2}},\\
 &  & \frac{L\left\Vert x^{0}-x^{*}\right\Vert ^{2}}{8(\frac{k}{n}+1)^{2}},
\end{eqnarray*}

which is the same lower bound as for $k/n$ iterations of an optimisation
method on $f$ directly.
\end{proof}

\subsection{Open problems}

\label{sub:inc-open-prob}

The above lower complexity theory is unfortunately far from complete.
For example, we do not know a tight complexity estimate for the random
access case ${FS}_{s,L,n}^{1,1}(\mathbb{R}^{d})$. We conjecture the
lower complexity bound will have a geometric constant of
\begin{equation}
1-\sqrt{\frac{\mu}{n(\mu n+L)}}.\label{eq:conjecture}
\end{equation}
When the condition number is small compared to $n$, then $1-\sqrt{\frac{\mu}{n(\mu n+L)}}\approx1-\sqrt{\frac{\mu}{\mu n^{2}}}=1-\frac{1}{n}$,
which we know is the best possible rate, and is achieved up to a constant
for known methods when $n\geq2\frac{L}{\mu}$. When $n$ is small,
say $n=1$, then we get $1-\sqrt{\frac{\mu}{\mu+L}}\approx1-\sqrt{\frac{\mu}{L}}$,
which is the best possible rate in the black box setting, which is
identical to ${FS}_{s,L,n}^{1,1}(\mathbb{R}^{d})$ when $n=1$. We
also have known methods that achieve this rate. The in-between setting
is what is interesting here. There exists the ASDCA method (Section\ref{sub:asdca})
which is applicable in this setting, unfortunately its known convergence
rate is not in the simple geometric form of Equation \ref{eq:conjecture};
it includes additional constants and logarithmic factors.

There is also the non-randomised class ${DS}_{s,L,n}^{1,1}(\mathbb{R}^{d})$.
It's not clear if the lower complexity bound for this more general
class is actually any different. Perhaps the most important question
is if a deterministic method exists with a convergence rate matching
the expected rate for those two classes.

\section{Access Orderings}

As discussed in the previous chapters, the order of accessing the
$f_{i}$ terms in an incremental gradient method can have a dramatic
effect on the convergence rate. In this section we compare and contrast
the various access orders on the most commonly used randomised optimisation
methods. 

An access order is defined on a per\emph{-epoch} basis. An epoch is
just defined as $n$ steps for an incremental gradient algorithm,
or $d$ coordinate steps for a coordinate descent algorithm. Recall
the three most common access orderings (Section \ref{subsec:access-orders}):
\begin{description}
\item [{Cyclic}] Each step, $j=1+(k\mod n)$.
\item [{Permuted}] Each epoch, $j$ is sampled without replacement from
the set of indices not accessed yet in that epoch. 
\item [{Randomised}] The value of $j$ is sampled with replacement from
$1,\dots,n$.
\end{description}
We now summarise the behaviour of incremental and coordinate methods
under the random, permuted and cyclic orderings for strongly convex
problems:

\begin{sidewaystable}
\bgroup \def\arraystretch{1.5}

\begin{center}
\begin{tabular}{|c|c|c|c|c|c|}
\hline 
\label{tab:order-table}Method & %
\parbox[c][1\totalheight][b]{4cm}{%
Empirical convergence\\
 (Randomised)%
} & %
\parbox[c]{4cm}{%
Empirical convergence\\
 (Permuted)%
} & %
\parbox[c]{4cm}{%
Empirical convergence\\
 (Cyclic)%
} & %
\parbox[c]{4cm}{%
\emph{Fast} Rate on strong\\
convex problems%
} & %
\parbox[c]{4cm}{%
Theory covers\\
cyclic/permuted%
}\tabularnewline
\hline 
\hline 
Finito & \cmark & \cmark & \xmark & \cmark & \xmark\tabularnewline
\hline 
SAGA & \cmark & \cmark & \xmark & \cmark & \xmark\tabularnewline
\hline 
SVRG & \cmark & \cmark & \xmark/Slow & \cmark & \xmark\tabularnewline
\hline 
SAG & \cmark & \cmark/Slower & \xmark & \cmark & \xmark\tabularnewline
\hline 
SDCA & \cmark & \cmark & Slow Rate & \cmark & Slow Rate\tabularnewline
\hline 
MISO & \cmark & \cmark & \cmark & \xmark & \cmark(Sec \ref{sec:miso-cyclic})\tabularnewline
\hline 
SGD & \cmark & \cmark & \cmark/Slower & \xmark & \cmark/Slow\tabularnewline
\hline 
CD & \cmark & \cmark & \cmark & \xmark & \cmark/Slower\tabularnewline
\hline 
\end{tabular}
\par\end{center}

\caption{The effect of different access orders on a selection of optimisation methods} 
\egroup \end{sidewaystable}
\begin{description}
\item [{Finito}] Empirically converges rapidly under randomised ordering.
Is always faster under permuted ordering. Can diverge under cyclic
ordering. Theory covers randomised ordering only.
\item [{SAGA}] Empirically converges rapidly under randomised ordering.
Cyclic ordering converges but extremely slowly. The permuted ordering
converges at a similar rate to randomised ordering, but sometimes
50\% slower or faster. Theory covers randomised ordering only.
\item [{SVRG}] Situation is identical to SAGA. Interestingly, even when
the recalibration step of SVRG is run every pass, using a cyclic ordering
is still much slower, although less so than if the recalibration is
done less frequently. The ``memory'' of the previous ordering carries
through between epochs even though only a single vector is actually
passed.
\item [{SAG}] Empirically converges rapidly under randomised ordering.
Cyclic ordering can diverge. The permuted ordering is extremely slow,
sometimes diverging, particularly if aggressive step sizes are used.
Theory covers randomised ordering only.
\item [{SDCA}] Empirically converges rapidly under randomised ordering.
The permuted ordering is sometimes faster, sometimes the same as randomised
ordering. Cyclic ordering is extremely slow but convergent. Theory
for fast convergence covers randomised ordering only. Convergence
under cyclic and permuted orders is provable, but rates unknown.
\item [{MISO}] MISO is the small step size variant of Finito. Empirically
it converges under randomised, cyclic and permuted ordering. Previous
theory covers randomised ordering \citep{miso2}. We show in Section
\ref{sec:miso-cyclic} that cyclic and permuted orders have the same
theoretical rate as the randomised ordering.
\item [{SGD}] Empirically fast under randomised ordering, but generally
faster under permuted ordering. Cyclic order convergent with small
enough step sizes, but slower. Theoretical convergence under cyclic
ordering known, slower than randomised ordering by a factor $n$,
for both non-smooth and smooth strongly convex \citep{nedic-conv}.
\item [{CD}] Coordinate descent is applied in different circumstances than
the above methods, but it can also be compared based on its properties
under different access orders. Empirically, cyclic can be better than
randomised in some cases. Permuted at least as fast as randomised.
Classical convergence proof covers randomised ordering. Theory establishes
cyclic and permuted ordering convergence rates, with slower constants
dependent on how uniform the Lipschitz smoothness constants are between
the $f_{i}$ functions. This clashes somewhat with the practical results
\citep{beck-tet}. There is an additional access order, the Gauss-Soutwell
rule, that can be applied to CD problems. At each step the coordinate
with the largest gradient is chosen. \citet{icml2015_nutini15} show
quite extensive theretical and practical results supporting the GS
rule and its variants.
\end{description}
Unfortunately, there is no general pattern to be seen from these results.
It appears that the effect of the ordering between two methods can
be radically different, even when the methods have otherwise similar
performance using a different ordering. One thing is clear though,
for all fast incremental gradient methods, the cyclic ordering is
extremely inferior. 

It is interesting that for Finito and SAG, if we reduce the step size
by a factor of $n$, they then work with cyclic and permuted orderings.
SDCA doesn't have a step size, but using the cyclic ordering automatically
has a similar effect, lowering the convergence rate to that of small-step-size
Finito and SAG.

The permuted ordering is rarely covered explicitly in the existing
literature. When the rate is known for the permuted ordering, it is
because the cyclic ordering proof also applies to the permuted case.
That kind of rate is unsatisfactory, as for those same methods the
permuted ordering is faster than randomised in practice, whereas the
cyclic rate is slower. 

For the SGD method, the permuted case has been explicitly considered
by \citet{valley}. They show that a cyclic ordering can be exponentially
slower than the randomised ordering. They also give a simple conjecture
that would imply that permuted ordering is at least as fast as the
randomised ordering.

\section{MISO Robustness}

\label{sec:miso-cyclic}

As noted above, the MISO method (the small step size case of Finito)
is robust to the use of alternative data access patterns. Both cyclic
and permuted orderings work in practice as fast as the randomised
ordering. This behaviour is quite unlike the other incremental gradient
methods and coordinate descent methods discussed above.

In this section we provide an extension of the theory of MISO to cover
these two alternate access orderings. Normally in the incremental
gradient literature for cyclic orderings, the analysis is performed
on a whole epoch at a time. In our novel analysis, we provide a simple
single-step proof for random access ordering, then we show that the
alternative access patterns give as least as much descent each step.
In particular, the expected convergence rate under each of the three
orderings is 
\[
E\left[f(x^{k})-f(x^{*})\right]\leq\left(1-\frac{\mu}{(\mu+L)n}\right)^{2k}\frac{2n}{\mu}\left\Vert f^{\prime}(\phi^{0})\right\Vert ^{2},
\]
where the expectation is over the choice of index at each step for
that ordering.

The key trick is the introduction of a new metric, which replaces
the use of the Euclidean metric in our proof.
\begin{defn}
The path distance $d(\cdot,\cdot)$ is defined for a sequence of points
$w^{0},\dots,w^{k}$ as follows. for any pair $w^{i}$ and $w^{j}$,
we assume without loss of generality that $i\leq j$. Then:
\[
d(w^{i},w^{j})=\sum_{l=i}^{j-1}\left\Vert w^{l}-w^{l+1}\right\Vert .
\]

This function has the simple interpretation as measuring distance
of points in the sequence as the distance along the ordered path between
the two points. It is easily verified that all requirements of a distance
function are satisfied here. Most importantly the triangle inequality.
The other key property we use is that the path distance always dominates
the Euclidean distance. Note also that our distance measure is defined
on a set, not on the actual vector space. This matters if an identical
point appears twice in the sequence, as then it is disambiguated by
its index.

For the purposes of this work we will only use the path distance along
the sequence of points $\phi^{0},w^{0},w^{1},w^{2},\dots,w^{k}$,
which includes the initial point $\phi^{0}$ , as well as all iterates
occurring during the MISO algorithm. Note that each $\phi$ update
sets $\phi_{j}=w^{k}$, so all $\phi_{i}$ the algorithm uses are
included in this path as well. For reference, the MISO update is:
\[
w^{k}=\bar{\phi}^{k}-\frac{1}{Ln}\sum_{i}f_{i}^{\prime}(\phi_{i}^{k}).
\]

\end{defn}
Now we establish a lemma on the change in $w$ between steps:
\begin{lem}
\label{lem:miso-descent-lemma}For the MISO method the distance between
successive iterates is bounded by:
\[
\left\Vert w^{k+1}-w\right\Vert \leq\left(1-\frac{\mu}{(\mu+L)}\right)\frac{1}{n}\left\Vert w-\phi_{j}\right\Vert .
\]

Note that $d(w^{k+1},w)=\left\Vert w^{k+1}-w\right\Vert $, and $\left\Vert w-\phi_{j}\right\Vert \leq d(w,\phi_{j})$
by domination. So this theorem holds under the path distance as well
as Euclidean distance.\end{lem}
\begin{proof}
\begin{eqnarray*}
\left\Vert w^{k+1}-w\right\Vert ^{2} & = & \left\Vert \bar{\phi}^{k+1}-\frac{1}{Ln}\sum_{i}f_{i}^{\prime}(\phi_{i}^{k+1})-\bar{\phi}+\frac{1}{Ln}\sum_{i}f_{i}^{\prime}(\phi_{i})\right\Vert ^{2}\\
 & = & \left\Vert \frac{1}{n}(w-\phi_{j})-\frac{1}{Ln}f_{j}^{\prime}(w)+\frac{1}{Ln}f_{j}^{\prime}(\phi_{j})\right\Vert ^{2}\\
 & = & \frac{1}{n^{2}}\left\Vert w-\phi_{j}\right\Vert ^{2}+\frac{1}{L^{2}n^{2}}\left\Vert f_{j}^{\prime}(w)-f_{j}^{\prime}(\phi_{j})\right\Vert ^{2}\\
 &  & -\frac{2}{n^{2}L}\left\langle w-\phi_{j},f_{j}^{\prime}(w)-f_{j}^{\prime}(\phi_{j})\right\rangle .
\end{eqnarray*}

We now apply the inner product bound $\left\langle f^{\prime}(x)-f^{\prime}(y),x-y\right\rangle \geq\frac{\mu L}{\mu+L}\left\Vert x-y\right\Vert ^{2}+\frac{1}{\mu+L}\left\Vert f^{\prime}(x)-f^{\prime}(y)\right\Vert ^{2}$:
\begin{eqnarray*}
\left\Vert w^{k+1}-w\right\Vert ^{2} & \leq & \left(\frac{1}{n^{2}}-\frac{2\mu L}{n^{2}L(\mu+L)}\right)\left\Vert w-\phi_{j}\right\Vert ^{2}\\
 &  & +\left(\frac{1}{L^{2}n^{2}}-\frac{2}{n^{2}L(\mu+L)}\right)\left\Vert f_{j}^{\prime}(w)-f_{j}^{\prime}(\phi_{j})\right\Vert ^{2}\\
 & = & \left(1\!-\!\frac{2\mu}{(\mu+L)}\right)\frac{1}{n^{2}}\left\Vert w-\phi_{j}\right\Vert ^{2}\!+\!\left(\frac{1}{L}\!-\!\frac{2}{\mu+L}\right)\frac{1}{n^{2}L}\left\Vert f_{j}^{\prime}(w)-f_{j}^{\prime}(\phi_{j})\right\Vert ^{2}\\
 & \leq & \left(1-\frac{2\mu}{(\mu+L)}\right)\frac{1}{n^{2}}\left\Vert w-\phi_{j}\right\Vert ^{2}.
\end{eqnarray*}

Now we use that $\sqrt{1-\frac{1}{a}}\leq1-\frac{1}{2a}$ for $a\geq1$
(Lemma \ref{lem:root-bernoulli}) to get the result.
\end{proof}
We now establish the convergence rate of MISO under the path distance
for the standard random sampling ordering for MISO.
\begin{thm}
\label{thm:miso-one-step} For the MISO method, conditioning on information
at step $k$, the expected descent in the quantity $\frac{1}{n}\sum_{i}d(w,\phi_{i})$
after one step is:
\[
\mathbb{E}\left[\frac{1}{n}\sum_{i}d(w^{k+1},\phi_{i}^{k+1})\right]\leq\left(1-\frac{\mu}{(\mu+L)n}\right)\frac{1}{n}\sum_{i}d(w,\phi_{i}).
\]

This theorem holds under the Euclidean distance as well, with the
same proof technique.\end{thm}
\begin{proof}
We will start by an application of the triangle inequality:

\begin{eqnarray*}
\frac{1}{n}\sum_{i}d(w^{k+1},\phi_{i}^{k+1}) & \leq & \frac{1}{n}\sum_{i}d(w,\phi_{i}^{k+1})+d(w^{k+1},w)\\
 & = & \frac{1}{n}\sum_{i}d(w,\phi_{i})+\left[\frac{1}{n}d(w,w)-\frac{1}{n}d(w,\phi_{j})\right]+d(w^{k+1},w).\\
 & = & \frac{1}{n}\sum_{i}d(w,\phi_{i})-\frac{1}{n}d(w,\phi_{j})+d(w^{k+1},w).
\end{eqnarray*}
Now we apply Lemma \ref{lem:miso-descent-lemma}:
\begin{eqnarray*}
\frac{1}{n}\sum_{i}d(w^{k+1},\phi_{i}^{k+1}) & \leq & \frac{1}{n}\sum_{i}d(w,\phi_{i})+\left(1-\frac{\mu}{(\mu+L)}-1\right)\frac{1}{n}d(w,\phi_{j})\\
 & = & \frac{1}{n}\sum_{i}d(w,\phi_{i})-\frac{\mu}{(\mu+L)n}d(w,\phi_{j}).
\end{eqnarray*}
Taking expectations with respect to the choice of $j$ gives:
\[
\mathbb{E}\left[\frac{1}{n}\sum_{i}d(w^{k+1},\phi_{i}^{k+1})\right]\leq\left(1-\frac{\mu}{(\mu+L)n}\right)\frac{1}{n}\sum_{i}d(w,\phi_{i}).
\]

\end{proof}
We are now ready to extend our result to alternative access orderings
\begin{thm}
If the MISO algorithm is used with a cyclic access pattern then the
one-step convergence rate is no worse than the expected rate for a
random ordering:
\[
\frac{1}{n}\sum_{i}d(w^{k+1},\phi_{i}^{k+1})\leq\left(1-\frac{\mu}{(\mu+L)n}\right)\frac{1}{n}\sum_{i}d(w,\phi_{i}).
\]
\end{thm}
\begin{proof}
The only point in the proof of Theorem \ref{thm:miso-one-step} where
we use the access ordering is when we simplify with: 
\[
\mathbb{E}\left[d(w,\phi_{j})\right]=\frac{1}{n}\sum_{i}d(w,\phi_{i}).
\]
This appears negated and on the right in the proof, so we need to
show that when $j$ is chosen via the cyclic pattern that $\frac{1}{n}\sum_{i}d(w,\phi_{i})\leq d(w,\phi_{j})$.
Now under the cyclic access pattern, the point $\phi_{j}$ is the
furthest (or tied for furthest) of the $\phi_{i}^{k}$ points in the
path under which our path distance is defined. In other words, it
is always the least recently accessed index. So for all $\phi_{i}$,
$i\neq j$:
\[
d(w,\phi_{i})\leq d(w,\phi_{j}).
\]
\[
\therefore d(w,\phi_{j})\geq\frac{1}{n}\sum_{i}d(w,\phi_{i}).
\]
The rest of the proof follows through as for Theorem \ref{thm:miso-one-step}
.\end{proof}
\begin{thm}
If the MISO algorithm is used with the permuted ordering then the
one-step convergence rate is no worse than the expected rate for a
random ordering:
\[
\mathbb{E}_{\text{perm}}\left[\frac{1}{n}\sum_{i}d(w^{k+1},\phi_{i}^{k+1})\right]\leq\left(1-\frac{\mu}{(\mu+L)n}\right)\frac{1}{n}\sum_{i}d(w,\phi_{i}).
\]
\end{thm}
\begin{proof}
As in the previous theorem, we just need to prove that: 
\[
\frac{1}{n}\sum_{i}d(w,\phi_{i})\leq\mathbb{E}_{\text{perm}}\left[d(w,\phi_{j})\right].
\]
So suppose we are at the $(r+1)$th step within an epoch. Let $A$
be the set of $r$ indices updated so far in this epoch. Now note
that for each $i\in A$, $d(w,\phi_{i})$ is less than $d(w,\phi_{j})$
for any $j\notin A$, as they are further back in the path for which
the path norm is defined. This implies that $\frac{1}{r}\sum_{i\in A}d(w,\phi_{i})\leq\frac{1}{n-r}\sum_{i\notin A}d(w,\phi_{i})$.
So: 
\begin{eqnarray*}
\frac{1}{n}\sum_{i}d(w,\phi_{i}) & = & \frac{1}{n}\sum_{i\in A}d(w,\phi_{i})+\frac{1}{n}\sum_{i\notin A}d(w,\phi_{i})\\
 & \leq & \frac{1}{n}\cdot\frac{r}{n-r}\sum_{i\notin A}d(w,\phi_{i})+\frac{1}{n}\sum_{i\notin A}d(w,\phi_{i})\\
 & = & \frac{1}{n}\cdot\frac{r}{n-r}\sum_{i\notin A}d(w,\phi_{i})+\frac{1}{n}\cdot\frac{n-r}{n-r}\sum_{i\notin A}d(w,\phi_{i})\\
 & = & \frac{1}{n-r}\sum_{i\notin A}d(w,\phi_{i}).
\end{eqnarray*}
This is just the expectation over the choice of $j$ under the permuted
ordering: 
\[
\mathbb{E}_{\text{perm}}[d(w,\phi_{j})]=\frac{1}{n-r}\sum_{i\notin A}d(w,\phi_{i}),
\]
and so the result is proven.
\end{proof}
Theorem \ref{thm:miso-one-step} and the two variants of it above
gives us a convergence rate in terms of $\frac{1}{n}\sum_{i}d(w,\phi_{i})$.
It is not immediately obvious how this quantity relates to traditional
measures of convergence such as $f(w)-f(w^{*})$ or $\left\Vert w-w^{*}\right\Vert ^{2}$.
So we now relate $\frac{1}{n}\sum_{i}d(w,\phi_{i})$ at step $k$
to function suboptimality and at step $0$ to gradient norm, to give
a convergence rate comparable to traditional optimisation methods.
\begin{thm}
\label{thm:miso-fast-rate}Assuming that we initialise with $\phi_{i}^{0}=\phi^{0}$
for all $i$, the expected convergence rate of the MISO method is:

\textup{
\[
\mathbb{E}\left[f(x^{k})-f(x^{*})\right]\leq\left(1-\frac{\mu}{(\mu+L)n}\right)^{2k}\frac{2n}{\mu}\left\Vert f^{\prime}(\phi^{0})\right\Vert ^{2},
\]
}

under each of the cyclic, permuted and random orderings.\end{thm}
\begin{proof}
We start by bounding the gradient norm of $f$:
\begin{eqnarray*}
\left\Vert f^{\prime}(w)\right\Vert  & = & \left\Vert \frac{1}{n}\sum_{i}f_{i}^{\prime}(w)\right\Vert \\
 & = & \left\Vert \frac{1}{n}\sum_{i}\left[f_{i}^{\prime}(w)-f_{i}^{\prime}(\phi_{i})+f_{i}^{\prime}(\phi_{i})\right]\right\Vert \\
 & \leq & \left\Vert \frac{1}{n}\sum_{i}f_{i}^{\prime}(\phi_{i})\right\Vert +\left\Vert \frac{1}{n}\sum_{i}\left[f_{i}^{\prime}(w)-f_{i}^{\prime}(\phi_{i})\right]\right\Vert \\
 & \leq & L\left\Vert w-\bar{\phi}\right\Vert +\frac{1}{n}\sum_{i}\left\Vert f_{i}^{\prime}(w)-f_{i}^{\prime}(\phi_{i})\right\Vert \\
 & \leq & L\left\Vert w-\bar{\phi}\right\Vert +\frac{L}{n}\sum_{i}\left\Vert w-\phi_{i}\right\Vert .
\end{eqnarray*}

Now recall that $\left\Vert w-\bar{\phi}\right\Vert ^{2}=\frac{1}{n}\sum_{i}\left\Vert w-\phi_{i}\right\Vert ^{2}-\frac{1}{n}\sum_{i}\left\Vert \bar{\phi}-\phi_{i}\right\Vert ^{2}$
by the variance decomposition. Therefore $\left\Vert w-\bar{\phi}\right\Vert \leq\sqrt{\frac{1}{n}\sum_{i}\left\Vert w-\phi_{i}\right\Vert ^{2}}$.
The $L_{2}$ norm is always no greater than the $L_{1}$ norm, so
$\sqrt{\frac{1}{n}\sum_{i}\left\Vert w-\phi_{i}\right\Vert ^{2}}\leq\sqrt{\frac{1}{n}}\sum_{i}\left\Vert w-\phi_{i}\right\Vert $.
We have shown so far that:
\begin{eqnarray*}
\left\Vert f^{\prime}(w)\right\Vert  & \leq & \frac{L}{\sqrt{n}}\sum_{i}\left\Vert w-\phi_{i}\right\Vert +\frac{L}{n}\sum_{i}\left\Vert w-\phi_{i}\right\Vert ,\\
 & \leq & \frac{2L}{\sqrt{n}}\sum_{i}\left\Vert w-\phi_{i}\right\Vert ,
\end{eqnarray*}

and since $2\mu\left[f(x)-f(x^{*})\right]\leq\left\Vert f^{\prime}(w)\right\Vert ^{2},$
we get
\[
\sqrt{2\mu\left[f(x)-f(x^{*})\right]}\leq\frac{2L}{\sqrt{n}}\sum_{i}\left\Vert w-\phi_{i}\right\Vert ,
\]

\[
2\mu\left[f(x)-f(x^{*})\right]\leq\frac{4L^{2}}{n}\left(\sum_{i}\left\Vert w-\phi_{i}\right\Vert \right)^{2},
\]
\[
\therefore\frac{2\mu}{4L^{2}n}\left[f(x)-f(x^{*})\right]\leq\left(\frac{1}{n}\sum_{i}\left\Vert w-\phi_{i}\right\Vert \right)^{2}\leq\left(\frac{1}{n}\sum_{i}d(w,\phi_{i})\right)^{2}.
\]
Combining with one step convergence rate in Theorem \ref{thm:miso-one-step}:
\[
\mathbb{E}\left[\frac{1}{n}\sum_{i}d(w^{k+1},\phi_{i}^{k+1})\right]\leq\left(1-\frac{\mu}{(\mu+L)n}\right)\frac{1}{n}\sum_{i}d(w,\phi_{i}),
\]
 squared and chained between $0$ and $k-1$, we get:

\[
\mathbb{E}\left[f(x^{k})-f(x^{*})\right]\leq\left(1-\frac{\mu}{(\mu+L)n}\right)^{2k}\frac{2L^{2}n}{\mu}\left(\frac{1}{n}\sum_{i}d(w^{0},\phi_{i}^{0})\right)^{2}.
\]
Now we use Jensen's inequality to bound $\left(\frac{1}{n}\sum_{i}d(w^{0},\phi_{i}^{0})\right)^{2}\leq\frac{1}{n}\sum_{i}d(w^{0},\phi_{i}^{0})^{2}$:
\[
\mathbb{E}\left[f(x^{k})-f(x^{*})\right]\leq\left(1-\frac{\mu}{(\mu+L)n}\right)^{2k}\frac{2L^{2}n}{\mu}\left(\frac{1}{n}\sum_{i}d(w^{0},\phi_{i}^{0})^{2}\right).
\]
Now assuming that all the $\phi_{i}^{0}$ start at the same value
($\phi^{0}$) gives 
\[
\frac{1}{n}\sum_{i}d(w^{0},\phi_{i}^{0})^{2}=\left\Vert w^{0}-\phi^{0}\right\Vert ^{2}=\frac{1}{L^{2}}\left\Vert f^{\prime}(\phi^{0})\right\Vert ^{2}.
\]
The result follows:
\[
\mathbb{E}\left[f(x^{k})-f(x^{*})\right]\leq\left(1-\frac{\mu}{(\mu+L)n}\right)^{2k}\frac{2n}{\mu}\left\Vert f^{\prime}(\phi^{0})\right\Vert ^{2}.
\]
\end{proof}

\chapter{Beyond Finite Sums: Learning Graphical Models}

\label{chap:background-gaussian}

In the preceding chapters we considered incremental gradients methods
whose applicability is limited to objectives that contain a finite
sum structure. In the following chapters we consider problems with
the other major function structure, a \emph{graph} structure. We particularly
focus on the simplest case, that of the Gaussian graphical model.
We look at two sides of the problem, that of learning smaller complex
structured models, and that of learning very large models. For smaller
models, we introduce in Chapter \ref{chap:submodular} a method for
learning graphs with the \emph{scale-free} property, common in real
world graphs. For large scale applications, we consider in Chapter
\ref{chap:colab} the use of the Bethe approximation for learning
large Gaussian graphical models with known structure. When the structure
is not known, it can be approximated using a variety of techniques.
In Chapter \ref{chap:approx-covsel} we give a fast variant of an
existing technique which is highly scalable, and which often learns
better structures than slower techniques.

\section{Beyond the Finite Sum Structure}

\label{sec:log-linear-intro}

Not all loss minimisation problems in machine learning are suited
to the application of the incremental gradient methods considered
in the previous chapters. For some problems, there is shared structure
between the losses that can be exploited in batch optimisation methods.
Consider the probability density for a log linear model with $p$
variables and $d$ features:
\[
P(x:\theta)=\frac{1}{Z(\theta)}\exp\left\{ -\sum_{r=1}^{d}\theta_{r}f_{r}(x)\right\} ,
\]
where $\theta$ is a vector of feature weights and $Z(\theta)$ is
the partition function whose value ensures that $\int_{x}P(x:\theta)=1$.
Now suppose we have a dataset $X\colon n\times p$ of $n$ data-points
$x_{k}$. Using the negative log-likelihood gives the objective:
\begin{eqnarray}
NLL(\theta) & = & \frac{1}{n}\sum_{k=1}^{n}\left[-\log1/Z(\theta)+\sum_{r=1}^{d}\theta_{r}f_{r}(x_{k})\right]\nonumber \\
 & = & \frac{1}{n}\sum_{k=1}^{n}\sum_{r=1}^{d}\theta_{r}f_{r}(x_{k})+\log Z(\theta).\label{eq:ll-pure}
\end{eqnarray}
While the objective does decompose as a sum of functions over individual
data-points as we require for the application of incremental gradient
methods, the $\log Z(\theta)$ term is shared by all the data-points.
The cost of evaluating $NLL(\theta)$ is dominated by the evaluation
of $\log Z(\theta)$, so it essentially costs the same to compute
one full gradient as one incremental gradient in this framework.

Evaluating the NLL doesn't always require the full dataset. A set
of empirical statistics known as the \emph{sufficient statistics}
can be extracted from the data, which contain sufficient information
for evaluating the NLL and its gradient. These statistics can be substantially
more compact than the full dataset.

Most applications of log-linear models in modern machine learning
use \emph{conditional} models, where instead each partition function
$Z$ is a function of a set of conditioning variables $y_{i}$ as
well as the parameters $\theta$. Conditional models do not share
the same partition function for each data-point, so evaluating the
whole loss is much more expensive than the per-datum loss. The incremental
gradient methods discussed in previous chapters are directly applicable
to conditional models.

In the remainder of this work we focus on unconditional models. The
main application of unconditional models in machine learning is in
structure learning.

\section{The Structure Learning Problem}

\label{sec:structure-learning}

Given a dataset $X:n\times p$, where we have $n$ samples of $p$
variables, the structure learning problem is the estimation of a network
of $p$ nodes, where an edge $(i,j)$ between two nodes signifies
that there is a \emph{direct dependence} between the two variables
$i$ \& $j$ in the data. The most natural notion of direct dependence
is that of conditional independence. Under the assumption that our
data is Gaussian distributed, this task is precisely estimation of
the sparsity pattern of the precision matrix of the data distribution.
The assumption of Gaussianity is a standard approach to structure
estimation, primarily because it results in a tractable optimisation
problem. Even when the Gaussian assumption does not hold, it can be
used as a subtask in the estimation of more complex models such as
Gaussian Copulas \citep{copula-icassp}.

This estimation task is known under a variety of names; the earliest
references confusingly use the term ``Covariance Selection'' \citep{dempster},
whereas some machine learning literature uses the more accurate but
longer ``Sparse Inverse Covariance Selection''. We use the terms
interchangeably in this work. The inverse of the covariance matrix
is also known as the precision matrix or the concentration matrix,
and all three terms are used in the Covariance selection literature.
This task is complicated by the fact that the sample covariance matrix
$C$ of the data is noisy and potentially low rank, so the sample
precision matrix as a natural estimate of the data precision matrix
may not even be well defined.

\section{Covariance Selection}

\label{sec:covsel}

A Gaussian graphical model is an example of log-linear model of the
above form. The feature set consists of singleton feature functions
for each variable $f_{i}$, and feature functions $f_{ij}$ defined
over pairs of variables, with one feature per pair. The model has
the following general structure:

\begin{equation}
P(x;\Theta)=\frac{1}{Z(\Theta)}\exp\left(-\sum_{i}^{p}\Theta_{ii}f_{i}(x_{i})-\sum_{i}^{p}\sum_{j>i}^{p}\Theta_{ij}f_{ij}(x_{i},x_{j})\right).\label{eq:pairwise-log-linear}
\end{equation}
We are using $\Theta$ instead of $\theta$, as the weights are better
thought of as having a symmetric matrix structure $(\Theta_{ij}=\Theta_{ji})$
rather than a flat vector structure. Additionally, in order that $Z(\Theta)$
exist, we will require a positive definiteness constraint on $\Theta$,
leading to a conic structure. 

The unary features $f_{i}$ penalise deviations of $x_{i}$ from its
average value: 
\[
f_{i}(x_{i})=\frac{1}{2}(x_{i}-\mu_{i})^{2}.
\]
The pair-wise features penalise $x_{i}$ and $x_{j}$ if they vary
from their respective means in the same direction: 
\[
f_{ij}(x_{i},x_{j})=(x_{i}-\mu_{i})(x_{j}-\mu_{j}).
\]

Note that \emph{negative} values of $\Theta_{ij}$ encourage correlation,
which might be the opposite of what would be expected. The reason
for defining $f_{ij}$ this way, instead of say $f_{ij}(x_{i},x_{j})=-(x_{i}-\mu_{i})(x_{j}-\mu_{j})$,
is so that we are consistent with standard Gaussian distribution notation.
To see this, notice that if we equate terms in Equation \ref{eq:pairwise-log-linear}
with the definition of a multivariate Gaussian distribution in terms
of the precision matrix:
\begin{equation}
P(x;\Theta)=\frac{1}{\sqrt{\left(2\pi\right)^{n}\text{det}(\Theta^{-1})}}\exp\left(-\frac{1}{2}\left(x-\mu\right)^{T}\Theta\left(x-\mu\right)\right),\label{eq:prec-gaussian}
\end{equation}
The feature weights $\Theta$ we have defined exactly form the precision
matrix $\Theta$ of the Gaussian distribution. Note that the pair-wise
features now appear twice due to the symmetry of $\Theta$, in the
forms $\Theta_{ij}$ and $\Theta_{ji}$ , which is why the weighting
factor of half appears not just on the unary features.

In the following sections, we will discuss learning $\Theta$ given
a dataset of points $x_{k}$. To this end, we now derive the maximum
likelihood estimator of $\Theta$. First we take the negation of the
logarithm of Equation \ref{eq:prec-gaussian}:

\begin{eqnarray*}
-\log P & = & -\log1/\sqrt{\left(2\pi\right)^{n}\text{det}(\Theta^{-1})}+\frac{1}{2}\left(x-\mu\right)^{T}\Theta\left(x-\mu\right)\\
 & = & \frac{1}{2}\log\left[\left(2\pi\right)^{n}\text{det}(\Theta^{-1})\right]+\frac{1}{2}\left(x-\mu\right)^{T}\Theta\left(x-\mu\right)\\
 & = & \frac{n}{2}\log\left[2\pi\right]+\frac{1}{2}\log\left[\text{det}(\Theta^{-1})\right]+\frac{1}{2}\left(x-\mu\right)^{T}\Theta\left(x-\mu\right)\\
 & = & \frac{n}{2}\log\left[2\pi\right]-\frac{1}{2}\log\left[\text{det}(\Theta)\right]+\frac{1}{2}\left(x-\mu\right)^{T}\Theta\left(x-\mu\right).
\end{eqnarray*}

Now we can drop the constant factor $\frac{n}{2}\log\left[2\pi\right]$,
and the multiplicative factor $\frac{1}{2}$ to simplify the objective.
Giving the single datapoint negative log-likelihood (NLL):
\[
NNL(\Theta\colon x)=\left(x-\mu\right)^{T}\Theta\left(x-\mu\right)-\log\text{det}(\Theta).
\]
The full data NLL is the average of this over the data points: $\frac{1}{n}\sum_{k}^{n}NLL(x_{k}:\Theta)$.
We can write this is a simple form using the Frobenius inner product
notation for matrices $\left\langle \cdot,\cdot\right\rangle $, defined
as $\left\langle X,Y\right\rangle =\sum_{i}\sum_{j}X_{ij}Y_{ij}$,
together with the covariance matrix of the data $C$:
\begin{eqnarray*}
NLL(\Theta\colon X) & = & \frac{1}{n}\sum_{k}^{n}\left(x_{k}-\mu\right)^{T}\Theta\left(x_{k}-\mu\right)-\log\text{det}(\Theta)\\
 & = & \frac{1}{n}\sum_{k}^{n}\sum_{i}\sum_{j}\Theta_{ij}\left(x_{ki}-\mu_{i}\right)\left(x_{kj}-\mu_{j}\right)-\log\det\Theta\\
 & = & \sum_{i}\sum_{j}\Theta_{ij}\left[\frac{1}{n}\sum_{k}^{n}\left(x_{ki}-\mu_{i}\right)\left(x_{kj}-\mu_{j}\right)\right]-\log\det\Theta\\
 & = & \sum_{i}\sum_{j}\Theta_{ij}C_{ij}-\log\det\Theta\\
 & = & \left\langle C,\Theta\right\rangle -\log\det\Theta.
\end{eqnarray*}
The maximum likelihood problem is equivalent to minimisation of the
$NLL$ (the negative log-likelihood) over the set of matrices that
form valid probability distributions, which is $\Theta\in S_{++}^{N}$,
the cone of symmetric positive definite matrices. So we have: 
\begin{equation}
\min_{\Theta\in S_{++}^{N}}\left\langle C,\Theta\right\rangle -\log\det\Theta.\label{eq:nll}
\end{equation}
The gradient of this objective is straight-forward: 
\[
\frac{d}{d\Theta}\left(\left\langle C,\Theta\right\rangle -\log\det\Theta\right)=C-\Theta^{-1}.
\]
When $C$ is an observed full-rank covariance matrix, taking the gradient
to zero gives the feasable solution $\Theta=C^{-1}$. If $C$ is low
rank but PSD (positive semi-definite), then there will be a subspace
of possible solutions. Since we are concerned with covariance matrices,
we do not consider non-PSD $C$ matrices further in this work.

We are more interested in the support of the precision matrix (the
graphical model structure) than its entries, and using $C^{-1}$ as
an estimator of $\Theta$ does not lead to sparse graph structures.
In order to induce sparsity a $L_{1}$ regulariser is typically used:
\begin{equation}
\min_{\Theta\in S_{++}^{N}}\left\langle C,\Theta\right\rangle -\log\det\Theta+\lambda\left\Vert \Theta\right\Vert _{1}.\label{eq:covsel-l1}
\end{equation}
Where the $L_{1}$ norm is taken elementwise, and $\lambda$ is a
regularisation parameter. This formulation was first used for covariance
selection by \citet{banerjee-2006}. This objective is convex, and
therefore tractable, although specialised optimisation algorithms
needed to be developed for it as standard first or second order methods
do not work well. The $\Theta^{-1}$ term in the gradient is very
sensitive to changes in $\Theta$, and the inverse operation highly
couples the variables so standard optimisation methods make slow progress.
The cone constraint $S_{++}^{N}$ also complicates matters; it is
an open set and the objective is sharply curved near the boundary.

The dual of Equation \ref{eq:covsel-l1} is particularly simple. Let
$S$ be the matrix of dual variables. Then like other $L_{1}$ regularised
problems, the $L_{1}$ norm turns into a box constraint in the dual:
\begin{gather}
\max_{S\in S_{++}^{N}}\log\det S,\nonumber \\
\text{s.t.\ }\forall i,j\colon\left|S_{ij}-C_{ij}\right|\leq\lambda.\label{eq:covsel-dual}
\end{gather}

\subsection{Direct optimisation approaches}

\label{sub:covsel-opt}

By our count there are over 35 papers published on distinct optimisation
methods for the covariance selection objective. Early approaches applied
block coordinate descent to the dual problem, with a block being the
$i$th row and column \citep{glasso,banerjee-2006}. These approaches
are quadratic time per update, and can be very effective on small
problems, or those with high levels of sparsity. 

For medium sized problems, a proximal approach can be taken, where
the optimisation method alternates between evaluating the proximal
operator of the $\log\det$ function and the $L_{1}$ norm \citep{cov-admm}.
These \emph{alternating} approaches are easy to implement, converge
rapidly and have strong convergence guarantees, but they fail for
problems for which the eigenvalue decomposition at each step is too
expensive. 

For large scale problems, the QUIC framework is currently the state
of the art among direct optimisation methods \citep{quic,bigquic}.
It uses an approximate (proximal) Newton step with line search as
the core component. Finding the Newton step at each iteration is a
separate optimisation problem, requiring sophisticated methods. The
QUIC method is very complex compared to the other methods considered
here. 

A more complete review of optimisation approaches to covariance selection
appears in \citet{covsel-book-chapter}. It should be noted that some
approaches do not regularise the diagonal of $\Theta$, such as the
SPICE method of \citet{rothman-2008}, as there is no motivation for
pushing diagonal elements towards zero, and regularisation of the
diagonal seems to harm theoretical performance. Also, like all the
methods we will discuss, the regularised maximum likelihood approach
can be applied to the Pearson correlation matrix $\bar{C}$ instead
of the covariance matrix. \citet{rothman-2008} are able show better
theoretical performance for the correlation variant of their SPICE
method than the covariance version.

\subsection{Neighbourhood selection}

\label{sub:nei-sel}

Instead of joint maximum likelihood learning of all edges at once,
another approach is to determine the dependence relations for each
variable as a separate problem \citep{neigh-sel}. This involves performing
$p$ $L_{1}$-regularised linear regression operations, working with
the data matrix $X$ instead of its covariance $C$. To write the
method compactly, we introduce the subscript notation $(\neg i)$
to represent taking the subset of that the indices excluding the $i$th,
and the notation $i*$ to represent the $i$th row and $*i$ the $i$th
column. 

For each row $i$ of the precision matrix, the neighbourhood selection
algorithm performs a separate minimisation: 
\[
\Theta_{i(\neg i)}=\arg\min_{\theta}\left[\frac{1}{n}\ns{X_{*i}-X_{*(\neg i)}\theta}_{2}+\lambda\n{\theta}_{1}\right].
\]
The resulting matrix $\Theta$'s support will not necessarily be symmetric,
so further processing is required to determine a symmetric support.
Considering the support only, conflicting edges can be combined by
either a union (i.e. $\Theta_{ij}\neq0\,\cup\,\Theta_{ji}\neq0$)
or intersection operation; \citet{neigh-sel} show consistency results
for both. In the case where $n\approx\sqrt{p}$, neighbourhood selection
is among the fastest methods in theory. In practice, it requires a
very efficient lasso implementation to be practical for large problems.

\subsection{Thresholding approaches}

\label{sub:thresholding}

As the sample covariance matrix is readily available, the technique
of directly thresholding the absolute value of its entries can be
used \citep{bickel-lav}. This technique can give better estimates
of the true covariance matrix than directly using the empirical covariance
matrix for sparse Gaussians, however it doesn't necessarily reveal
the network structure. In the case where the covariance matrix is
full rank, the technique of inverse sample covariance matrix thresholding
may be used \citep{dempster}. This technique is not often used as
the inverse empirical covariance matrix as an estimate of the precision
matrix can be poor as the inversion is sensitive to small changes
in the covariance matrix.

A better approach is thresholding pairwise mutual information. For
Gaussian distributed variables $x_{i}$ and $x_{j}$, the mutual information
is given by:

\[
I(x_{i};x_{j})=\frac{1}{2}\log\left(\frac{C_{ii}C_{jj}}{C_{ii}C_{jj}-C_{ij}^{2}}\right).
\]

The well known Chow-Liu method \citep{chowliu} for learning the best-fitting
\emph{tree} structure relies on mutual information between variables,
and mutual information thresholding is also widely used for general
directed graphical model structure learning \citep{mutual-info-thesis}.
Mutual Information thresholding is a form of independence testing,
a standard technique in applied statistics. Indeed, the mutual information
score is proportional to the $\chi^{2}$ statistic up to first order
\citep{koller-friedman}.

\subsection{Conditional thresholding}

\label{sec:cond-thresh}

A newer class of techniques consider for each entry a set of conditional
covariance or conditional mutual information terms. Instead of thresholding
the mutual information term $I(x_{i};x_{j})$, \citet{anand-nips-2011}
consider thresholding the minimum of the set of conditional mutual
information's $I(x_{i};x_{j}|x_{S})$ where $S$ is a set of variables
not containing $x_{i}$ or $x_{j}$. For tractability the minimisation
is taken over subsets of size at most $\eta$: 
\[
w_{ij}=\min_{S\in V\backslash\{x_{i},x_{j}\},\,|S|\leq\eta}I(x_{i};x_{j}|x_{S}).
\]
Performing this minimisation takes time $O(p^{\eta})$ per edge, so
for thresholding all edges the running time is $O(p^{2+\eta})$. In
applications this is only practical for $\eta=1$ or possibly $\eta=2$.

For the $\eta=1$ case, we introduce some additional notation. We
will subscript matrices by a set of indices to denote selection of
the square submatrix containing the rows and columns in the set. E.g
$C_{\{i,j\}}$ is the submatrix of $C$ containing rows $i,j$ and
columns $i,j$. Using this notation, the conditional mutual information
for variables $x_{i},x_{j}$ given $x_{k}$ is given by:

\[
I(x_{i};x_{j}|x_{k})=\frac{1}{2}\log\left(\frac{\det C_{\{i,k\}}\det C_{\{j,k\}}}{C_{kk}\det C_{\{i,k,j\}}}\right).
\]
\citet{anand-jmlr-prepub} consider the case of conditional covariances,
which behave better theoretically, as the number of samples required
for correct reconstructions depends less on the strength of the data
distributions dependencies. The conditional covariance of $x_{i},x_{j}$
given a set of variable indices $S$ is: 
\[
Cov[x_{i},x_{j}|x_{S}]=C_{ij}-C_{iS}C_{SS}^{-1}C_{Sj}.
\]
As the conditional covariance can be negative, the minimisation is
taken of its absolute value. Like in the mutual information case,
the minimum over all subsets $S$ of a particular size is used. Besides
better theoretical bounds, the conditional covariance is faster to
evaluate for small $\eta$, and its values are easier to interpret.
The $C_{iS}C_{SS}^{-1}C_{Sj}$ term can be thought of as the transitive
covariance contributed by the path between $x_{i}$ and $x_{j}$ through
$x_{S}$. If $C_{ij}$ is significantly above the transitive covariance
over subset $S$, then it is likely due to a direct dependence between
variables $x_{i}$ and $x_{j}$.

\section{Alternative Regularisers}

\label{sec:alt-reg}

The $L_{1}$ norm used in the maximum likelihood approach is known
to encourage sparsity, but it also has a significant and potentially
undesirable effect on the learned weights. One way to overcome this
is to use a regulariser that is stronger in the vicinity of $0$ than
the $L_{1}$ norm, but weaker away from zero. Such a regulariser necessarily
be nonconvex. \citet{reweighted-l1-boyd} suggest the use of a logarithmic
regulariser, such as 

\begin{equation}
\lambda\sum_{i,j}\log\left(\left|\Theta_{ij}\right|+\epsilon\right),\label{eq:reweighted-l1}
\end{equation}

for some $\epsilon>0$. The constant $\epsilon$ just prevents the
log terms being infinity near $0$. In order to efficiently solve
this, the problem can be addressed using repeated linearisation of
the log term. At each stage, a separate weight $\lambda_{ij}$ is
found for each entry of $\Theta_{ij}$ so that the tangent of $\lambda_{ij}\left|\Theta_{ij}\right|$
matches $\lambda\log\left(\left|\Theta_{ij}\right|+\epsilon\right)$
at the current $\Theta_{ij}$ value. Then the resulting ``reweighted''
$L_{1}$ regularised problem is solved, and the reweighted is repeated.
This reweighted $L_{1}$ method falls within the class of majorisation-minimisation
methods \citep{mm}, as on both sides of $0$ the weighted $L_{1}$
regularised problem is an upper bound for the $\lambda\sum_{i,j}\log\left(\left|\Theta_{ij}\right|+\epsilon\right)$
regularised one (i.e. a majorisation), which we minimise at each stage.

This approach has been adapted by \citet{reweighted-l1} to encourage
a \emph{scale free} network structure (discussed in Chapter \ref{chap:submodular})
using the following objective:

\begin{equation}
f(\Theta)=\left\langle \Theta,C\right\rangle -\log\det\Theta+\lambda\sum_{i}\log\left(\left\Vert \Theta_{i*}\right\Vert _{1}-\left|\Theta_{ii}\right|+\epsilon\right)+\gamma\sum_{i}\left|\Theta_{ii}\right|.\label{eq:reweighted-sf}
\end{equation}

The regulariser is split into a row term which is designed to encourage
sparsity in the edge parameters, and a more traditional diagonal term.
The row term $\left(\left\Vert \Theta_{i*}\right\Vert _{1}-\left|\Theta_{ii}\right|\right)$
is in effect a continuous analogue to the degree of node $i$. The
same kind of reweighting optimisation scheme can be applied as for
Equation \ref{eq:reweighted-l1} to this objective.

The other major class of regularisers are those that encourage \emph{group
sparsity}. Group-sparse regularisers just encourage whole sets of
variables to be zero or non-zero together \citep{sparse-norms-chapter}.
For graph structures the most common application is to encourage variable
independence, by treating all edges incoming to a node as a group.
The objective is of the form:
\[
f(\Theta)=\left\langle \Theta,C\right\rangle -\log\det\theta+\lambda\sum_{i}\n{\Theta_{i*}}_{2}.
\]

This is effectively a sum of $L_{1}$ norms, one group per node $i$,
where the group consists of edges adjacent to a node $i$. The groups
overlap due to the symmetry constraint, as each edge is adjacent to
two nodes. 

\chapter{Learning Scale Free Networks}

\label{chap:submodular}

Structure learning for graphical models is a problem that arises in
many contexts. In applied statistics, graphical models with learned
structure can be used as a tool for understanding the underlying conditional
independence relations between variables in a dataset. For example,
in bioinformatics Gaussian graphical models are fitted to data resulting
from micro-array experiments, where the fitted graph can be interpreted
as a gene expression network \citep{gene-expression-2004}.

Recent research has seen the development of \emph{structured} sparsity,
where more complex prior knowledge about a sparsity pattern can be
encoded. Examples include group sparsity \citep{group-sparsity},
where parameters are linked so that they are regularised in groups.
More complex sparsity patterns, such as region shape constraints in
the case of pixels in an image \citep{sparse-pca}, or hierarchical
constraints \citep{hierarchical-sparse} have also been explored.

In this chapter we study the problem of recovering the structure of
a Gaussian graphical model under the assumption that the graph recovered
should be scale-free. Many real-world networks are known a priori
to be scale-free and therefore enforcing that knowledge through a
prior seems natural. Recent work has offered an approach to deal with
this problem which results in a non-convex formulation (Section \ref{sec:alt-reg}).
Here we present a convex formulation. We show that scale-free networks
can be induced by enforcing submodular priors on the network's degree
distribution, and then using their convex envelope (the Lovász extension)
as a convex relaxation. 

The resulting relaxed prior has an interesting non-differentiable
structure, which poses challenges to optimisation. We outline a few
options for solving the optimisation problem via proximal operators,
in particular an efficient dual decomposition method. Experiments
on both synthetic data produced by scale-free network models and a
real bioinformatics dataset suggest that the convex relaxation is
not weak: we can infer scale-free networks with similar or superior
accuracy than with the previous state-of-the-art.

An earlier version of the work in this chapter has been published
as \citet{adefazio-nips2012}.

\section{Combinatorial Objective}

\label{sec:scalefree-objective}

Consider an undirected graph with edge set $E$ and node set $V$,
where $n$ is the number of nodes. We denote the degree of node $v$
as $d_{E}(v)$, and the complete graph with $n$ nodes as $K_{n}$.
We are concerned with placing priors on the degree distributions of
graphs such as $(V,E)$. By degree distribution, we mean the bag of
degrees $\left\{ d_{E}(v)|v\in V\right\} $.

We would like to place higher prior probability mass on graphs with
scale-free structure. Scale-free graphs are simply graphs whose degree
distribution follows a power law. In other words, the random draws
from the bag of degrees of the graph has a distribution of the form:
\[
p(i)\propto i^{-\alpha},
\]
for some $\alpha>1$, typically in the range 2 to 3. It is too much
to ask that the empirical degree distribution of a graph exactly matches
a power-law distribution, so in practice a graph is considered scale
free if the tail of the degree distribution is significantly heavier,
in the sense of having more probability mass, than the exponential
distribution.

A natural prior on degree distributions can be formed from the family
of exponential random graphs \citep{exp-graphs}. Exponential random
graph (ERG) models assign a probability to each $n$ node graph using
an exponential family model. The probability of each graph depends
on a small set of sufficient statistics, in our case we only consider
the degree statistics. A ERG distribution with degree parameterisation
takes the form: 
\[
p(G=(V,E)\colon h)=\frac{1}{Z(h)}\exp\left\{ \left[\,-\sum_{v\in V}h(d_{E}(v))\,\right]\right\} ,
\]
The degree weighting function $h\colon\mathbb{Z}^{+}\rightarrow\mathbb{R}$
encodes the preference for each particular degree. The function $Z$
is chosen so that the distribution is correctly normalised over $n$
node graphs.

A number of choices for $h$ are reasonable; A geometric series $h(i)\propto1-\alpha^{i}$
with $\alpha\in(0,1)$ has been proposed by \citet{geo-erg} and has
been widely adopted. However, for encouraging scale free graphs we
require a more rapidly increasing sequence. Under the strong assumption
that each node's degree is independent of the rest, $h$ grows logarithmically
for scale free networks. To see this, take a scale free model $p(i)\propto i^{-\alpha},$
with scale $\alpha$; the joint distribution under the node degree
independence simplification takes the form:
\[
p(G=(V,E)\colon\epsilon,\alpha)=\frac{1}{Z(\epsilon,\alpha)}\prod_{v\in V}(d_{E}(v)+\epsilon)^{-\alpha},
\]
where $\epsilon>0$ is added to prevent infinite weights. Putting
this into ERG form gives the weight sequence $h(i)=\alpha\log(i+\epsilon)$.
We will consider this and other functions $h$ in Section\textbf{
}\ref{sec:degree-priors}. We intend to perform maximum a posteriori
(MAP) estimation of a graph structure using such a distribution as
a prior, so the object of our attention is the negative log-posterior,
which we denote $F$: 
\[
F(E)=\sum_{v\in V}h(d_{E}(v))+\text{const.}
\]
So far we have defined a function on edge sets only, however in practice
we want to optimise over a weighted graph, which is intractable when
using discontinuous functions such as $F$. We now consider the properties
of $h$ that lead to a convex relaxation of $F$.

\section{Submodularity}

\label{sec:submodularity}

A set function $F\colon2^{E}\rightarrow\mathbb{R}$ on $E$ is a non-decreasing
submodular function if for all $A\subset B\subset E$ and $x\in E\backslash B$
the following conditions hold:
\begin{align*}
F(A\cup\{x\})-F(A) & \ge F(B\cup\{x\})-F(B)\tag{submodularity}\\
\text{and }F(A) & \leq F(B).\tag{non-decreasing}
\end{align*}
The first condition can be interpreted as a diminishing returns condition;
adding $x$ to a set $A$ increases $F(A)$ by more than adding $x$
to a larger set $B$, if $B$ contains $A$. We now consider a set
of conditions that can be placed on $h$ so that $F$ is submodular. 
\begin{prop}
\label{prop:h} Denote $h$ as tractable if $h$ is non-decreasing,
concave and $h(0)=0$. For tractable $h$, $F$ is a non-decreasing
submodular function.\end{prop}
\begin{proof}
First note that the degree function is a set cardinality function,
and hence modular. A concave transformation of a modular function
is submodular \citep{bach-tutorial}, and the sum of submodular functions
is submodular.
\end{proof}
The concavity restriction we impose on $h$ is the key ingredient
that allows us to use submodularity to enforce a prior for scale-free
networks; any prior favouring long tailed degree distributions must
place a lower weight on new edges joining highly connected nodes than
on those joining other nodes. As far as we are aware, this is a novel
way of mathematically modelling the ``preferential attachment''
rule \citep{ba-model} that gives rise to scale-free networks: through
non-decreasing submodular functions on the degree distribution.

Let $X$ denote a symmetric matrix of edge weights. A natural convex
relaxation of $F$ would be the convex envelope of $F(\textrm{Supp}(X))$
under some restricted domain. For tractable $h$, we have by construction
that $F$ satisfies the conditions of Proposition 1 in \citet{subbach},
so that the convex envelope of $F(\textrm{Supp}(X))$ on the $L_{\infty}$
ball is precisely the Lovász extension evaluated on $|X|$. The Lovász
extension for our function is easy to determine as it is a sum of
``functions of cardinality'' which are considered in \citet{subbach}.
Below is the result from \citet{subbach} adapted to our problem. 
\begin{prop}
\label{prop:convex-envelope} Let $X_{i,(j)}$ be the weight of the
$j$th edge connected to $i$, under a decreasing ordering by absolute
value (i.e. $|X_{i,(0)}|\geq|X_{i,(1)}|\geq...\geq|X_{i,(n-1)}|$).
The notation $(i)$ maps from sorted order to the natural ordering,
with the diagonal not included. Then the convex envelope of $F$ for
tractable $h$ over the $L_{\infty}$ norm unit ball is:
\[
\Omega(X)=\sum_{i=0}^{n}\,\sum_{k=0}^{n-1}\left(h(k+1)-h(k)\right)\left|X_{i,(k)}\right|.
\]
This function is piece-wise linear and convex.
\end{prop}
The form of $\Omega$ is quite intuitive. It behaves like a $L_{1}$
norm with an additional weight on each edge that depends on how the
edge ranks with respect to the other edges of its neighbouring nodes.

\section{Optimisation}

\label{sec:scalefree-optimization}

We are interested in using $\Omega$ as a prior, for optimisations
of the form
\[
\textrm{minimize}_{X}\quad f(X)=g(X)+\alpha\Omega(X),
\]
for a convex negative log likelihood $g$ and prior strength parameter
$\alpha\geq0$, over symmetric $X$. We will focus on the simplest
structure learning problem that occurs in graphical model training,
that of Gaussian models. In which case we have
\[
g(X)=\left\langle X,C\right\rangle -\log\det X,
\]
where $C$ is the observed covariance matrix of our data (see Section
\ref{sec:covsel} for the derivation). The support of $X$ will then
be the set of edges in the undirected graphical model together with
the node precisions. This function is a rescaling of the maximum likelihood
objective. In order for the resulting $X$ to define a normalisable
distribution, $X$ must be restricted to the cone of symmetric positive
definite matrices. This is not a problem in practice as $g(X)$ is
infinite on the boundary of the PSD cone, and hence the constraint
can be handled by restricting optimisation steps to the interior of
the cone. In fact $X$ can be shown to be in a strictly smaller cone,
$X^{*}\succeq aI$, for a constant $a$ derivable from $C$ \citep{smooth-covsel}.
This restricted domain is useful as $g(X)$ is Lipschitz smooth over
$X\succeq aI$ but not over all positive definite matrices \citep{cov-admm}.

There are a number of possible algorithms that can be applied for
optimising a convex non-differentiable objective such as $f$. \citet{subbach}
suggests two approaches to optimising functions involving submodular
relaxation priors; a subgradient approach and a proximal approach.

Subgradient methods are the simplest class of methods for optimising
non-smooth convex functions. They provide a good baseline for comparison
with other methods. For our objective, a subgradient is simple to
evaluate at any point, due to the piecewise continuous nature of $\Omega(X)$.
Unfortunately (primal) subgradient methods for our problem will not
return sparse solutions except in the limit of convergence. They will
instead give intermediate values that oscillate around their limiting
values.

An alternative is the use of proximal methods. Proximal methods exhibit
superior convergence in comparison to subgradient methods, and produce
sparse solutions. Proximal methods rely on solving a simpler optimisation
problem, known as the \emph{proximal operator} at each iteration:
\[
\arg\min_{X}\,\left[\alpha\Omega(X)+\frac{1}{2}\left\Vert X-Z\right\Vert _{2}^{2}\right],
\]
where $Z$ is a variable that varies at each iteration. For many problems
of interest, the proximal operator can be evaluated using a closed
form solution. For non-decreasing submodular relaxations, the proximal
operator can be evaluated by solving a submodular minimisation on
a related (not necessarily non-decreasing) submodular function \citep{subbach}.

\citet{subbach} considers several example problems where the proximal
operator can be evaluated using fast graph cut methods. For the class
of functions we consider, graph-cut methods are not applicable. Generic
submodular minimisation algorithms can be as bad as $O(p^{6})$ for
$p$ variables \citep{submod-book}, giving $O(n^{12})$ for a $n$-vertex
graph (optimizing over $O(n^{2})$ edges), which is clearly impractical.
We will instead propose a dual decomposition method for solving this
proximal operator problem in Section \ref{sub:dd-proximal}. 

For solving our optimisation problem, instead of using the standard
proximal method (sometimes known as ISTA), which involves a gradient
step followed by the proximal operator, we propose to use the alternating
direction method of multipliers (ADMM), which has shown good results
when applied to the standard $L_{1}$ regularised covariance selection
problem \citep{cov-admm}. Next we show how to apply ADMM to our problem.

\subsection{Alternating direction method of multipliers}

\label{sec:admm}

The alternating direction method of multipliers (ADMM) is one approach
to optimising our objective. ADMM is a classical method (\citealt{admm-gabay,admm-glow}),
recently repopularised in \citealt{admm}.

ADMM has a number of advantages over the basic proximal method. Let
$U$ be the matrix of dual variables for the decoupled problem:
\begin{gather*}
\text{minimize}_{X}\quad g(X)+\alpha\Omega(Y),\\
s.t.\quad X=Y.
\end{gather*}

Following the presentation of the algorithm in \citet{admm}, given
the values $Y^{(l)}$ and $U^{(l)}$ from iteration $l$, with $U^{(0)}=0_{n}$
and $Y^{(0)}=I_{n}$ the ADMM updates for iteration $l+1$ are:
\begin{eqnarray*}
X^{(l+1)} & = & \arg\min_{X}\left[\left\langle X,C\right\rangle -\log\det X+\frac{\rho}{2}||X-Y^{(l)}+U^{(l)}||_{2}^{2}\right],\\
Y^{(l+1)} & = & \arg\min_{Y}\left[\alpha\Omega(Y)+\frac{\rho}{2}||X^{(l+1)}-Y+U^{(l)}||_{2}^{2}\right],\\
U^{(l+1)} & = & U^{(l)}+X^{(l+1)}-Y^{(l+1)},
\end{eqnarray*}
where $\rho>0$ is a fixed step-size parameter (we used $\rho=0.5$).
The advantage of this form is that both the $X$ and $Y$ updates
are a proximal operation. It turns out that the proximal operator
for $g$ (i.e. the $\ensuremath{X^{(l+1)}}$ update) actually has
a simple solution \citep{cov-admm} that can be computed by taking
an eigenvalue decomposition $Q^{T}\Lambda Q=\rho(Y-U)-C$, where $\Lambda=\textrm{diag}(\lambda_{1},\dots,\lambda_{n})$
and updating the eigenvalues using the formula
\[
\lambda_{i}^{\prime}:=\frac{\lambda_{i}+\sqrt{\lambda_{i}^{2}+4\rho}}{2\rho},
\]
to give $X=Q^{T}\Lambda^{\prime}Q$. The stopping criterion we used
was $\left\Vert X^{(l+1)}-Y^{(l+1)}\right\Vert <\epsilon$ and $\left\Vert Y^{(l+1)}-Y^{(l)}\right\Vert <\epsilon$.
In practice the ADMM method is one of the fastest methods for $L_{1}$
regularised covariance selection. \citet{cov-admm} show that convergence
is guaranteed if additional cone restrictions are placed on the minimisation
with respect to $X$, and small enough step sizes are used. For our
degree prior regulariser, the difficultly is in computing the proximal
operator for $\Omega$, as the rest of the algorithm is identical
to that presented in \citet{admm}. We now show how we solve the problem
of computing the proximal operator for $\Omega$.

\subsection{Proximal operator using dual decomposition}

\label{sub:dd-proximal}Here we describe the optimisation algorithm
that we effectively use for computing the proximal operator. The regulariser
$\Omega$ has a quite complicated structure due to the interplay between
the terms involving the two end points for each edge. We can decouple
these terms using the dual decomposition technique, by writing the
proximal operation for a given $Z=Y-U$ as:
\begin{gather*}
\textrm{minimize}_{X}=\frac{\alpha}{\rho}\sum_{i}^{n}\,\sum_{k}^{n-1}\left(h(k+1)-h(k)\right)\left|X_{i,(k)}\right|+\frac{1}{2}||X-Z||_{2}^{2}\\
s.t.\quad X=X^{T}.
\end{gather*}
The only difference so far is that we have made the symmetry constraint
explicit. Taking the dual gives a formulation where the upper and
lower triangle are treated as separate variables. The dual variable
matrix $V$ corresponds to the Lagrange multipliers of the symmetry
constraint, which for notational convenience we store in an anti-symmetric
matrix. The dual decomposition method is given in Algorithm \ref{alg:dd-main}.

\begin{algorithm}[t]
\begin{algorithmic}
	\STATE {\bfseries input:} matrix $Z$, constants $\alpha$, $\rho$
	\STATE {\bfseries input:} step-size $0<\eta<1$
	\STATE {\bfseries initialise:} $X=Z$
	\STATE {\bfseries initialise:} $V=0_n$
	\REPEAT
		\FOR{$l=0$ {\bfseries until} $n-1$}
			\STATE $X_{l*} = \text{solveSubproblem}(Z_{l*}, V_{l*})$ \emph{\# Algorithm  \ref{alg:dd-subproblem}}
		\ENDFOR
		\STATE $V = V + \eta(X - X^T)$
	\UNTIL {$||X-X^T|| < 10^{-6}$}
	\STATE $X = \frac{1}{2}(X+X^T)$ \emph{\# symmetrise}
	\STATE {\bfseries round:} any $|X_{ij}| < 10^{-15}$ to $0$
	\STATE {\bfseries return} X
\end{algorithmic}

\protect\caption{\label{alg:dd-main}Dual decomposition main}

\end{algorithm}

We use the notation $X_{i*}$ to denote the $i$th row of $X$. Since
this is a dual method, the primal variables $X$ are not feasible
(i.e. symmetric) until convergence. Essentially we have decomposed
the original problem, so that now we only need to solve the proximal
operation for each node in isolation, namely the sub-problems:
\begin{equation}
\forall i.\;X_{i*}^{(l+1)}=\arg\min_{x}\,\frac{\alpha}{\rho}\sum_{k}^{n-1}\left(h(k+1)-h(k)\right)\left|x_{(k)}\right|+||x-Z_{i*}+V_{i*}^{(l)}||_{2}^{2}.\label{eq:dd}
\end{equation}
Note that the dual variable has been integrated into the quadratic
term by completing the square. As the diagonal elements of $X$ are
not included in the sort ordering, they will be minimised by $X_{ii}=Z_{ii},$
for all $i$. Each sub-problem is strongly convex as they consist
of convex terms plus a positive quadratic term. This implies that
the dual problem is differentiable (as the subdifferential contains
only one subgradient), hence the $V$ update is actually gradient
ascent. Since a fixed step size is used, and the dual is Lipschitz
smooth, for sufficiently small step-size convergence is guaranteed.
In practice we used $\eta=0.9$ for all our tests.

This dual decomposition sub-problem can also be interpreted as just
a step within the ADMM framework. If applied in a standard way, only
one dual variable update would be performed before another expensive
eigenvalue decomposition step. Since each iteration of the dual decomposition
is much faster than the eigenvalue decomposition, it makes more sense
to treat it as a separate problem as we propose here. It also ensures
that the eigenvalue decomposition is only performed on symmetric matrices. 

Each sub-problem in our decomposition is still a non-trivial problem.
They do have a closed form solution, involving a sort and several
passes over the node's edges, as described in Algorithm \ref{alg:dd-subproblem}.

\begin{algorithm}[p]
\begin{algorithmic}
	\STATE {\bfseries input:} vectors $z$, $v$
	\STATE {\bfseries initialise:} Disjoint-set datastructure 
		with set membership function $\gamma$
	\STATE $w = z-v\quad$ \emph{\# $w$ gives the sort order}
	\STATE $u = 0_n$
	\STATE {\bfseries build:} sorted-to-original position function $\mu$
			under descending absolute value order of $w$, excluding the diagonal
	\FOR{$k=0$ {\bfseries until} $n-1$}
		\STATE $j = \mu(k)$
		\STATE $u_j = |w_j| -  \frac{\alpha}{\rho} \left(h(k+1)-h(k)\right)$
		\STATE $\gamma(j).\textrm{value} = u_j$
		\STATE r = k
		\WHILE{$r>1$ and $ \gamma(\mu(r)).\textrm{value}  \geq 
			\gamma(\mu(r-1)).\textrm{value} $ }
			\STATE {\bfseries join:} the sets containing $\mu(r)$ and $\mu(r-1)$
			\STATE $\gamma(\mu(r)).\textrm{value} = 
					\frac{1}{|\gamma(\mu(r))|} 
					\sum_{i \in \gamma(\mu(r))} u_i $
			\STATE {\bfseries set:} $r$ to the first element 
					of $\gamma(\mu(r))$ by the sort ordering
		\ENDWHILE
	\ENDFOR
	\FOR{$i=1$ {\bfseries to} $N$}
		\STATE $x_i = \gamma(i).\textrm{value}$
		\IF{ $x_i < 0$ }
			\STATE $x_i = 0\quad $ \emph{\# negative values imply shrinkage to 0}
		\ENDIF
		\IF{ $w_i < 0$ }
			\STATE $x_i = -x_i\quad$ \emph{\# Correct orthant}
		\ENDIF
	\ENDFOR
	\STATE {\bfseries return} $x$
\end{algorithmic}

\protect\caption{\label{alg:dd-subproblem}Dual decomposition sub-problem (\emph{solveSubproblem})}

\end{algorithm}

\begin{prop}
Algorithm \ref{alg:dd-subproblem} solves the sub-problem in Equation
\ref{eq:dd}.\end{prop}
\begin{proof}
See Section \ref{sec:scalefree-proof}. The main subtlety is the grouping
together of elements induced at the non-differentiable points. If
multiple edges connected to the same node have the same absolute value,
their subdifferential becomes the same, and they behave as a single
point whose weight is the average. To handle this grouping, we use
a disjoint-set data-structure, where each $x_{j}$ is either in a
singleton set, or grouped in a set with other elements, whose absolute
value is the same.
\end{proof}

\section{Alternative Degree Priors}

\label{sec:degree-priors}

Under the restrictions on $h$ detailed in Proposition \ref{prop:h},
several other choices seem reasonable. The scale free prior can be
smoothed somewhat, by the addition of a linear term, giving
\[
h_{\epsilon,\beta}(i)=\log(i+\epsilon)+\beta i,
\]
where $\beta$ controls the strength of the smoothing. A slower diminishing
choice would be a square-root function such as
\[
h_{\beta}(i)=(i+1)^{\frac{1}{2}}-1+\beta i.
\]
This requires the linear term in order to correspond to a normalisable
prior.

Ideally we would choose $h$ so that the expected degree distribution
under the ERG model matches the particular form we wish to encourage.
Finding such a $h$ for a particular graph size and degree distribution
amounts to maximum likelihood parameter learning, which for ERG models
is a hard learning problem. The most common approach is to use sampling
based inference. Approaches based on Markov chain Monte Carlo techniques
have been applied widely to ERG models \citep{erg-sampling} and are
therefore applicable to our model.

\section{Experiments}

\label{sec:scalefree-experiments}

\subsection{Reconstruction of synthetic networks}

We performed a comparison against the reweighted $L_{1}$ method of
\citet{reweighted-l1}, and a standard $L_{1}$ regularised method,
both implemented using ADMM for optimisation. Although \citet{reweighted-l1}
use the glasso \citep{glasso} method for the inner loop, ADMM will
give identical results, and is usually faster \citep{cov-admm}. Graphs
with $60$ nodes were generated using both the Barabasi-Albert model
\citep{ba-model} and a predefined degree distribution model sampled
using the method from \citet{random-degree-graph} implemented in
the NetworkX software package. Both methods generate scale-free graphs;
the BA model exhibits a scale parameter of $3.0$, whereas we fixed
the scale parameter at $2.0$ for the other model. 

To define a valid Gaussian model, edge weights of $X_{ij}=-0.2$ were
assigned, and the node weights were set at $X_{ii}=0.5-\sum_{i\neq j}X_{ij}$
so as to make the resulting precision matrix diagonally dominant.
The resulting Gaussian graphical model was sampled 500 times. The
covariance matrix of these samples was formed, then normalised to
have diagonal uniformly $1.0$. We tested with the two $h$ sequences
described in Section \ref{sec:degree-priors}. The parameters for
the degree weight sequences were chosen by grid search on random instances
separate from those we tested on. The resulting ROC curves for the
Hamming reconstruction loss are shown in Figure \ref{fig:synth-roc}.
Results were averaged over $30$ randomly generated graphs for each
each figure.

We can see from the plots that our method with the square-root weighting
presents results superior to those from \citet{reweighted-l1} for
these datasets. This is encouraging particularly since our formulation
is convex while the one from \citet{reweighted-l1} is not. Interestingly,
the log based weights give very similar but not identical results
to the reweighting scheme which also uses a log term. The only case
where it gives inferior reconstructions is when it is forced to give
a sparser reconstruction than the original graph.

\begin{figure}
\begin{centering}
\includegraphics[width=1\textwidth]{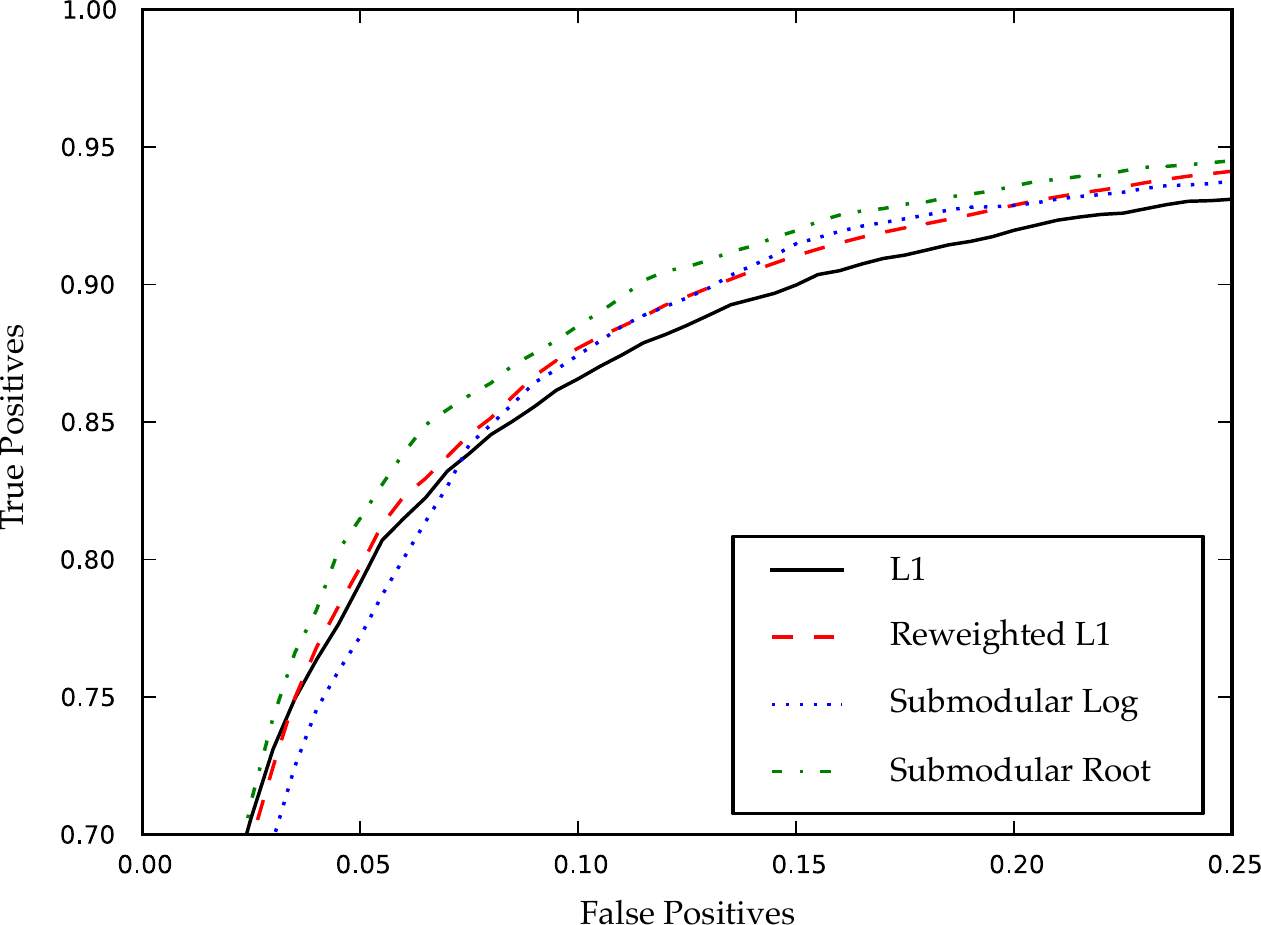}
\par\end{centering}

\begin{centering}
\includegraphics[width=1\textwidth]{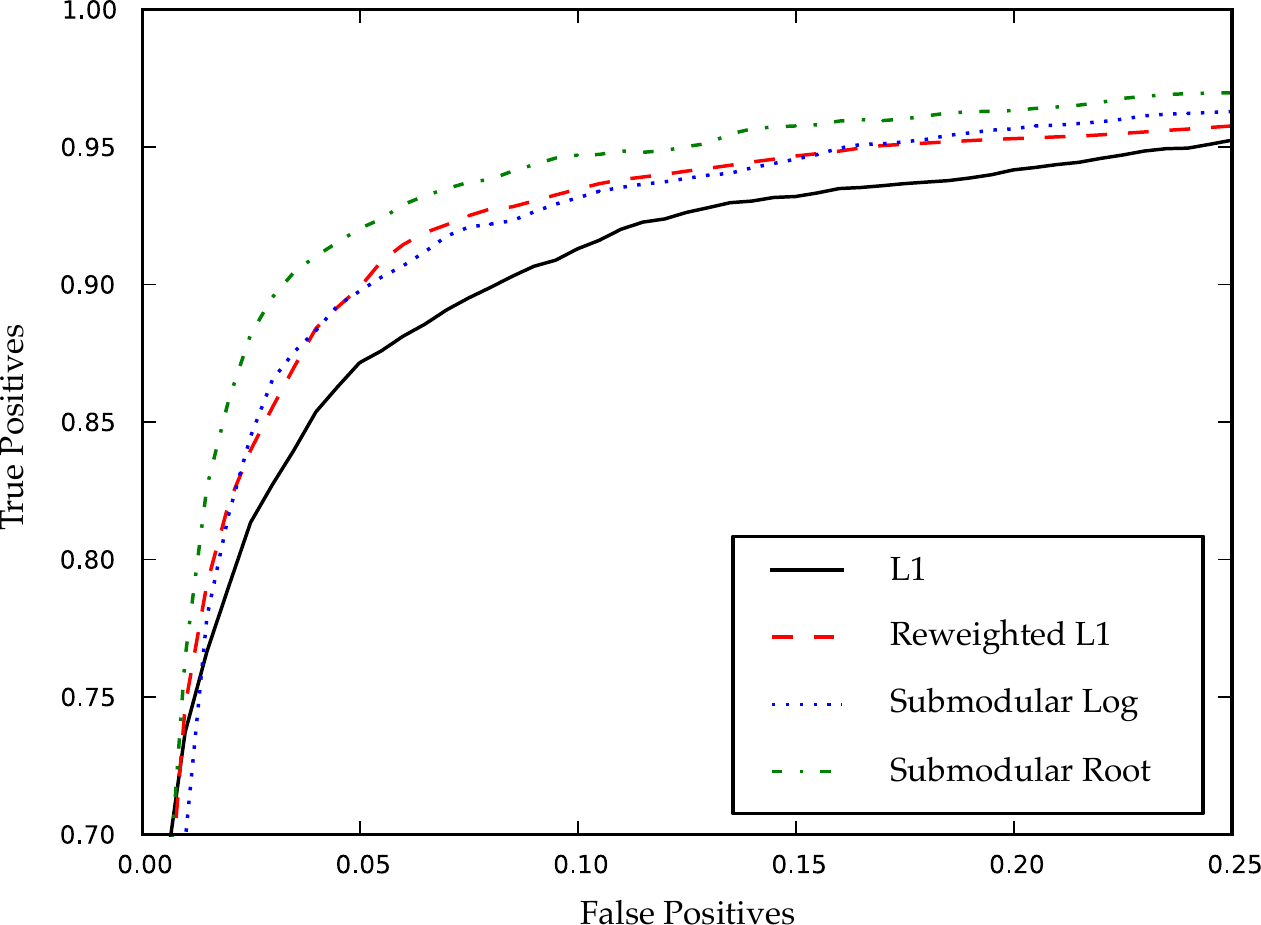}
\par\end{centering}

\protect\caption{\label{fig:synth-roc} ROC curves for BA model (top) and fixed degree
distribution model (bottom)}

\end{figure}

\subsection{Reconstruction of a gene activation network}

A common application of sparse covariance selection is the estimation
of gene association networks from experimental data. A covariance
matrix of gene co-activations from a number of independent micro-array
experiments is typically formed, on which a number of methods, including
sparse covariance selection, can be applied. Sparse estimation is
key for a consistent reconstruction due to the small number of experiments
performed. Many biological networks are conjectured to be scale-free,
and additionally ERG modelling techniques are known to produce good
results on biological networks \citep{erg-gene}. So we consider micro-array
datasets a natural test-bed for our method. We ran the three methods
considered on the first 500 genes from the GDS1429 dataset \footnote{http://www.ncbi.nlm.nih.gov/gds/1429},
which contains $69$ samples for $8565$ genes. The parameters for
both methods were tuned to produce a network with near to $50$ edges
for visualisation purposes. The major connected component for each
is shown in Figure \ref{fig:gene-recon}. 

While these networks are too small for valid statistical analysis
of the degree distribution, the submodular relaxation method produces
a network with structure that is commonly seen in scale free networks.
The star subgraph centered around gene $60$ is more clearly defined
in the submodular relaxation reconstruction, and the tight cluster
of genes in the right is less clustered in the $L_{1}$ reconstruction.
The reweighted $L_{1}$ method produced a quite different reconstruction,
with greater clustering.

\begin{figure}
\begin{centering}
\includegraphics[scale=0.7]{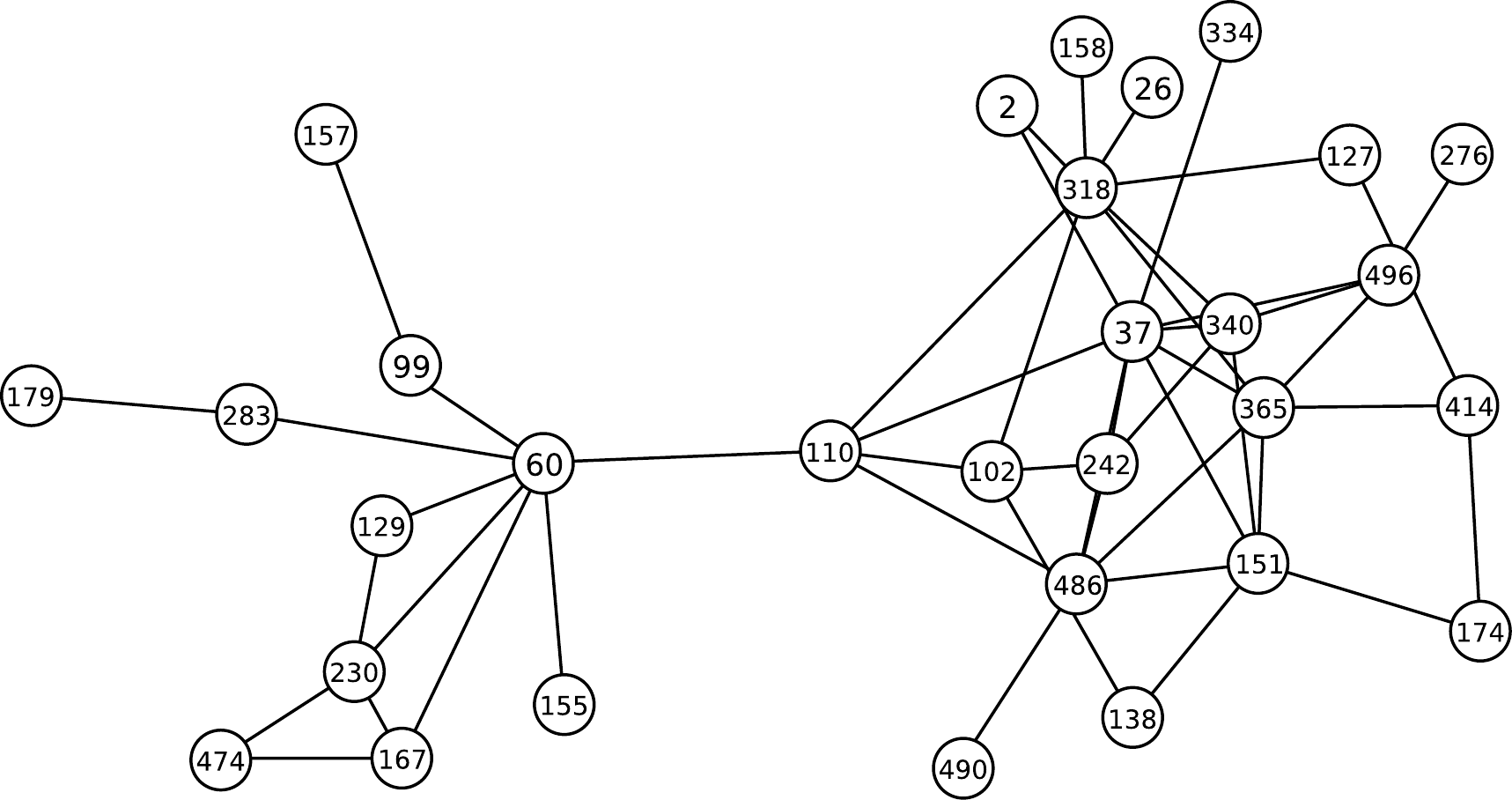}
\par\end{centering}

\begin{centering}
\includegraphics[scale=0.7]{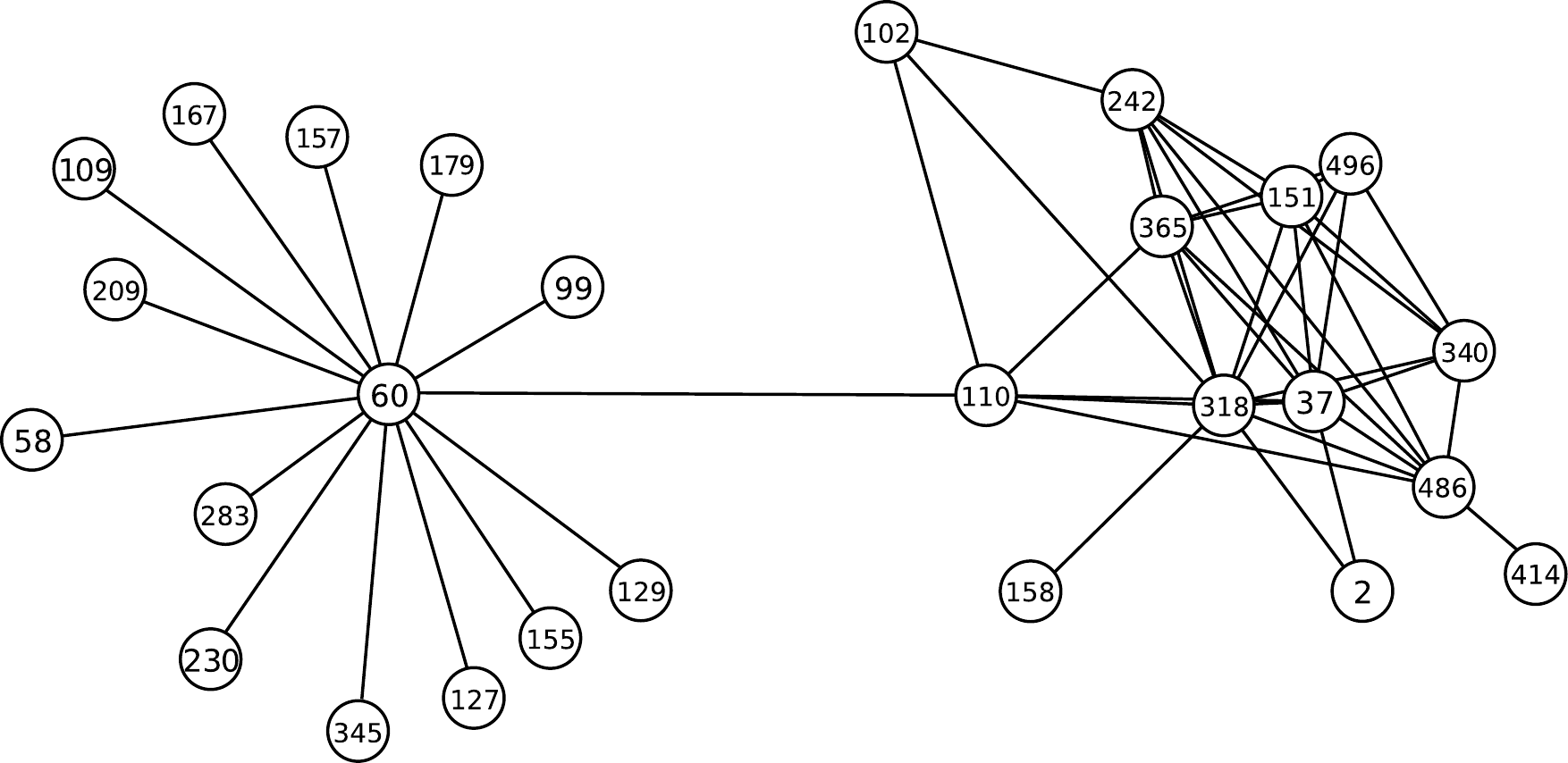}
\par\end{centering}

\begin{centering}
\-\-\-\includegraphics[scale=0.7]{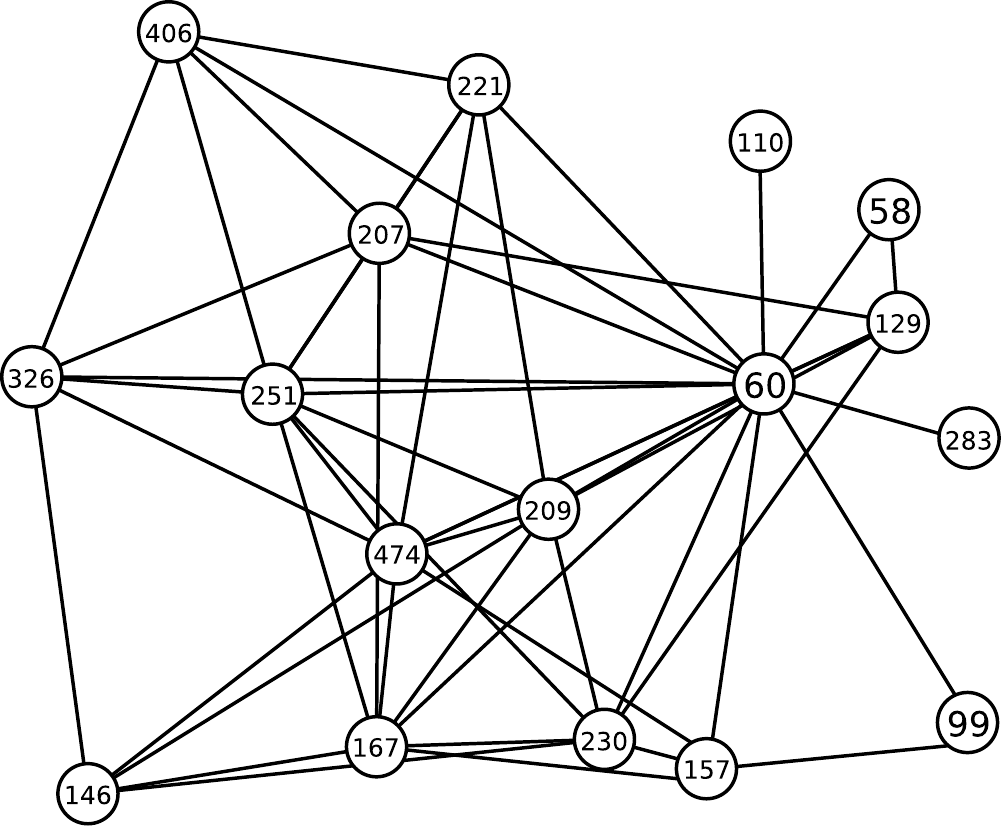}
\par\end{centering}

\protect\caption{\label{fig:gene-recon}Reconstruction of a gene association network
using $L_{1}$ (top), submodular relaxation (middle), and reweighted
$L_{1}$ (bottom) methods}
\end{figure}

\subsection{Runtime comparison: different proximal operator methods}

We performed a comparison against two other methods for computing
the proximal operator: subgradient descent and the minimum norm point
(MNP) algorithm. The MNP algorithm is a submodular minimisation method
that can be adapted for computing the proximal operator \citep{subbach}.
We took the input parameters from the last invocation of the proximal
operator in the BA test, at a prior strength of $0.7$. We then plotted
the convergence rate of each of the methods, shown in Figure \ref{fig:comparison-of-proximal}.
As the tests are on randomly generated graphs, we present only a representative
example.

It is clear from this and similar tests that we performed that the
subgradient descent method converges too slowly to be of practical
applicability for this problem. Subgradient methods can be a good
choice when only a low accuracy solution is required; for convergence
of ADMM the error in the proximal operator needs to be smaller than
what can be obtained by the subgradient method. The MNP method also
converges slowly for this problem, however it achieves a low but usable
accuracy quickly enough that it could be used in practice. The dual
decomposition method achieves a much better rate of convergence, converging
quickly enough to be of use even for strong accuracy requirements.

The time for individual iterations of each of the methods was $0.65$ms
for subgradient descent, $0.82$ms for dual decomposition and $15$ms
for the MNP method. The speed difference is small between a subgradient
iteration and a dual decomposition iteration as both are dominated
by the cost of a sort operation. The cost of a MNP iteration is dominated
by two least squares solves, whose running time in the worst case
is proportional to the square of the current iteration number. Overall,
it is clear that our dual decomposition method is significantly more
efficient.

\begin{figure}
\begin{centering}
\includegraphics[width=1\textwidth]{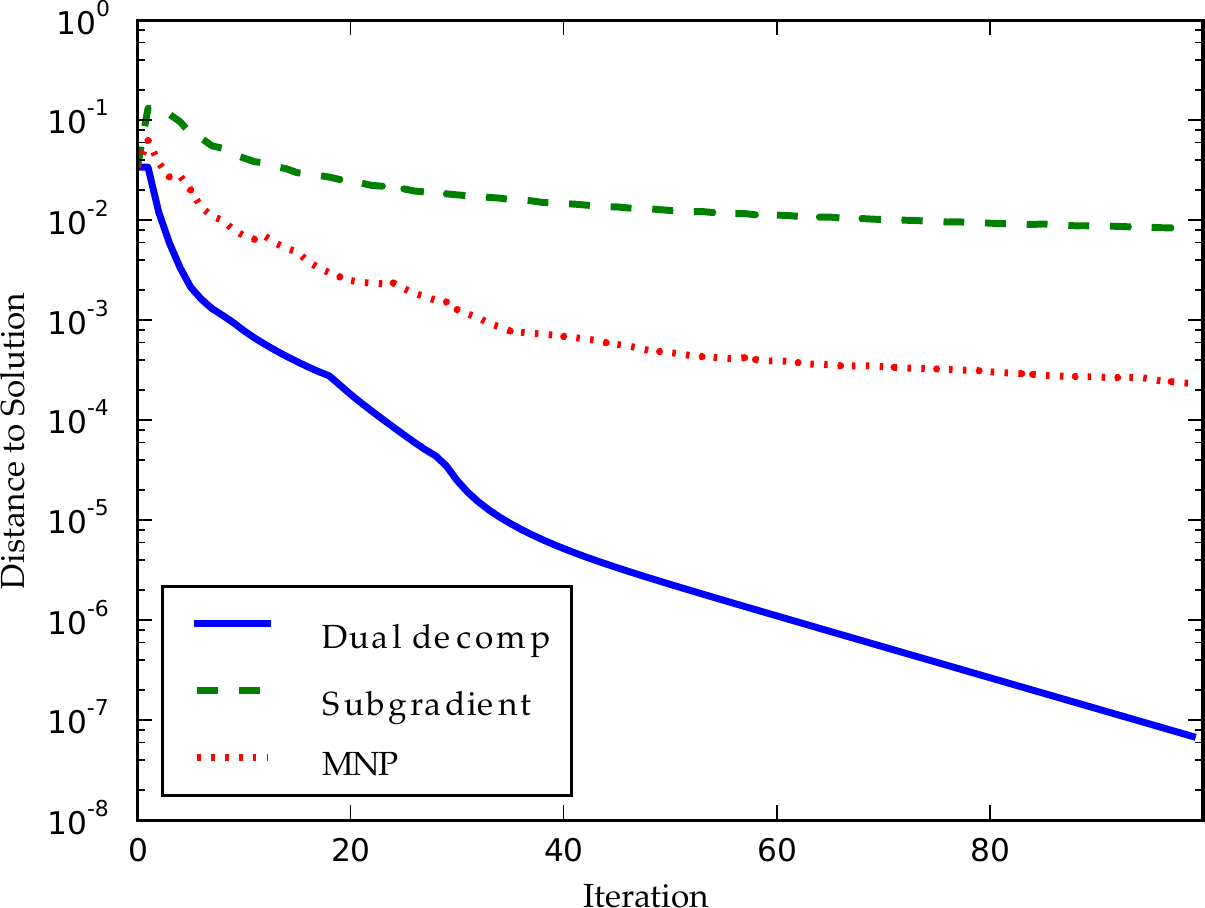}
\par\end{centering}

\protect\caption{\label{fig:comparison-of-proximal}Comparison of proximal operators}

\end{figure}

\subsection{Runtime comparison: submodular relaxation against other approaches}

The running time of the three methods we tested is highly dependent
on implementation details, so the following speed comparison should
be taken as a rough guide. For a sparse reconstruction of a BA model
graph with $100$ vertices and $200$ edges, the average running time
per $10^{-4}$ error reconstruction over $10$ random graphs was $16$
seconds for the reweighted $L_{1}$ method and $5.0$ seconds for
the submodular relaxation method. This accuracy level was chosen so
that the active edge set for both methods had stabilised between iterations.
For comparison, the standard $L_{1}$ method was significantly faster,
taking only $0.72$ seconds on average.

\section{Proof of Correctness}
\begin{proof}
\label{sec:scalefree-proof} In this section we establish that Algorithm
\ref{alg:dd-subproblem} solves Equation \ref{eq:dd}. First we note
that the pseudo-code uses $\mu(i)$ instead of $(i)$ to denote the
mapping from sort order to storage order, to remove ambiguity. We
also use object-oriented notation $\gamma(p)$ to represent the set
containing $p$, with $\gamma(p).\textrm{value}$ denoting the value
field of the set object. This value field is interpreted as the absolute
value of each of the elements of the set, with a negative value implying
zero. The value field is stored as an ancillary array in our implementation.
We also need an operation that returns the ``first'' element of
the set, which in our data structure is just the root of the tree
that represents the set.

We start by simplifying the notation to
\[
f(x)=\frac{\alpha}{\rho}\sum_{j=1}^{n}h^{\prime}(j)\left|x_{(j)}\right|+\frac{1}{2}||x-w_{j}||_{2}^{2},
\]
where $w_{j}=z_{(j)}-v_{(j)}$, and $h^{\prime}(j)=h(j)-h(j-1)$.
We want to find $\arg\min_{x}f(x)$. We consider the two main cases.

\textbf{Unique valued $x_{j}$}

First we take the gradient when $\ensuremath{x_{j}}$ is not at the
same value as another, and is non-zero:
\[
\frac{\partial f}{\partial x_{(j)}}=\frac{\alpha}{\rho}h^{\prime}(j)\textrm{sgn}(x_{(j)})+x_{(j)}-w_{j}.
\]
Equating to zero gives: 
\begin{equation}
|x_{(j)}|=|w_{j}|-\frac{\alpha}{\rho}h^{\prime}(j),\label{eq:unique-dd}
\end{equation}
when $|w_{j}|>\frac{\alpha}{\rho}h^{\prime}(j)$, with $\textrm{sgn}(x_{(j)})=\textrm{sgn}(w_{j})$.
If $|w_{j}|\leq\frac{\alpha}{\rho}h^{\prime}(j)$, then clearly the
gradient can not be equated to zero except at $x_{(j)}=0$, where
the subdifferential needs to be considered. Considering the subdifferential
at zero gives the requirement that 
\[
\frac{w_{j}}{\frac{\alpha}{\rho}h^{\prime}(j)}\in[-1,1],
\]
which follows from $|w_{j}|\leq\frac{\alpha}{\rho}h^{\prime}(j)$.
So essentially if $|w_{j}|-\frac{\alpha}{\rho}h^{\prime}(j)\leq0$
we set $x_{(j)}=0$. This is essentially the same shrinkage as performed
in the proximal operator of a $L_{1}$ regulariser.

\textbf{Non-unique $x_{j}$}

The above update is used for singleton sets in Algorithm \ref{alg:dd-subproblem}.
Now when several $x_{(j)}$ have the same value, say a set $Q$ of
them, and none of them are zero, then equating the gradients to zero
gives a set of equations, which can be solved to give 
\begin{equation}
\forall j\in Q,\quad\textrm{sgn}(w_{j})x_{(j)}=\frac{1}{|Q|}\sum_{i\in Q}\left(|w_{i}|-\frac{\alpha}{\rho}h^{\prime}(i)\right).\label{eq:non-unique-dd}
\end{equation}
One direct implication of this is that the sort ordering for equal
variables doesn't effect the solution, so a stable sort is not necessary
in Algorithm \ref{alg:dd-subproblem}. The result for the singleton
case where the value is negative follows also for the grouped together
variables.

It suffices to show that at termination of Algorithm \ref{alg:dd-subproblem},
these equations (\ref{eq:unique-dd} or \ref{eq:non-unique-dd} )
are satisfied by the returned $x^{*}$ for all $j$, as this implies
that $0$ is in the subdifferential at $x^{*}$. The assignments in
Algorithm \ref{alg:dd-subproblem} clearly ensure that this is the
case under the assumption that the sort ordering of $|w_{j}|$ is
the same as the sort order of $|x_{(j)}|$ after termination of the
algorithm, and likewise for the grouped variables. So we just need
to show that the assignments to the $x_{(j)}$ result in that ordering.
In particular since $x_{i}=\gamma(i).\text{value}$ for positive $x_{i}$and
$w_{i}$, we need that for all $p,q$: 
\[
|w_{p}|>|w_{q}|\implies\gamma(\mu(p)).\textrm{value}>\gamma(\mu(q)).\textrm{value}.
\]
We proceed by induction, with the inductive hypothesis at step $k$
being that the ordering is correct for each group containing $\mu(l)<k$,
i.e. 
\[
\forall\mu(p)<k,\mu(q)<k\colon\quad|w_{p}|>|w_{q}|\implies\gamma(\mu(p)).\textrm{value}>\gamma(\mu(q)).\textrm{value}.
\]
 The base case $k=0$ doesn't trigger the while loop, it just finds
$j=\mu(k)$, then sets $u_{j}$ using Equation \ref{eq:unique-dd}
and sets the singleton set containing $j$ with that value also ($\gamma(j).\text{value}=u_{j}$).
Our inductive hypothesis holds since it is a vacuous statement for
$k=0$.

Assume now that at iteration $k-1$ of the first loop in Algorithm
\ref{alg:dd-subproblem} the inductive hypothesis holds. Then at iteration
$k$, a new singleton set is created for $j=\mu(k)$, with value given
by $|w_{j}|-\frac{\alpha}{\rho}h^{\prime}(j)$. If $\gamma(j)$'s
value is less than the value of $\mu(k-1)$, then by induction the
ordering is correct. Otherwise, the algorithm proceeds by adding $j$
to set $\mu(k-1)$, and updating the value of $j$'s new set to the
average from Equation \ref{eq:non-unique-dd}. This will increase
the value of the set, which may cause its value to increase above
that of another set. That case is handled by the repeated merging
in the main while loop, in the same manner. It is clear that the while
loop must terminate after no more than $k$ steps, as it performs
at one merge per iteration. At termination of the while loop the ordering
is then correct up to $k$, as all changes in the set values that
would cause the ordering to change instead have caused set merges. \end{proof}

\chapter{Fast Approximate Structural Inference}

\label{chap:approx-covsel}

\selectlanguage{english}%
\global\long\def\n#1{\left\Vert #1\right\Vert }

\global\long\def\ns#1{\left\Vert #1\right\Vert ^{2}}

\global\long\def\r{\mathbb{R}}

\global\long\def\e{\mathbb{E}}

\global\long\def\ip#1#2{\left\langle #1,#2\right\rangle }
\selectlanguage{australian}%

\label{sec:intro} 

In Chapter \ref{chap:background-gaussian} we described how learning
the structure of a Gaussian graphical model can be phrased as a regularised
maximum likelihood learning problem. This formulation gives good results
but scales poorly as the number of variables is increased. Intuitively,
it seems overkill to find the values of all the edge weights when
we are only concerned with the existence or non-existence of edges.
In this chapter we consider the conditional covariance thresholding
method \citep[CCT, ][]{anand-nips-2011} which only looks at a small
neighbourhood of each edge when determining if it is part of the structure.
The effect of distant parts of a graph decays exponentially, so for
large graphs with high diameters, looking at just a neighbourhood
can be a substantial saving. We give a modification of the CCT method
that improves its running time by a factor of $O(\sqrt{p})$ for $p$
variables, reducing it to $O(p^{2.5})$ in the important $\eta=1$
case of the method. This allows the method to be applied to substantially
larger problems in practice, essentially any problem for which the
covariance matrix can be formed in memory.

Section \ref{sec:shortcut} describes our algorithm, and discusses
its running time in theory and practice. Our experiments in Section
\ref{sec:experiments} cover the reconstruction of synthetically generated
networks, as well as two practical problems: determining the conditional
independence relations in stock-market prices and weather station
data.

\section{SHORTCUT}

\label{sec:shortcut}

Recall the conditional covariance thresholding (CCT) method from Section
\ref{sub:thresholding}. It starts by forming the following matrix
$W$:
\begin{gather*}
W_{ij}=\min_{S\subset V\backslash\{i,j\},\,|S|\leq\eta}\left|Cov[x_{i},x_{j}|x_{S}]\right|,\\
\text{where }Cov[x_{i},x_{j}|x_{S}]=C_{ij}-C_{iS}C_{SS}^{-1}C_{Sj}.
\end{gather*}

The matrix $W$ is then thresholded by a constant $t$ to give the
edge structure. The CCT method is only practical for $\eta=1$ or
$\eta=2$ for moderately sized problems. In many applications the
number of variables can be as large as 10,000 to 100,000, and in those
cases even the $\eta=1$ variant is too slow. SHORTCUT is fundamentally
a variant of the conditional covariance thresholding method in the
$\eta=1$ regime. We propose to modify the minimisation: 
\[
\min_{k\in V\backslash\left\{ i,j\right\} }\left|C_{ij}-\frac{C_{ik}C_{kj}}{C_{kk}}\right|,
\]
to: 
\[
\max\left\{ 0\,,\,\min_{k\in V\backslash\left\{ i,j\right\} }\left|C_{ij}\right|-\text{sgn}(C_{ij})\frac{C_{ik}C_{kj}}{C_{kk}}\right\} .
\]
We make this approximation to allow for more efficient algorithms
for evaluating the inner minimisation. To understand this approximation,
without loss of generality, consider the case where $C_{ij}$ is positive.
For positive $C_{ij}$, our modification is identical except when
$\frac{C_{ik}C_{kj}}{C_{kk}}>C_{ij}$ for some $k$, in which case
the approximation will return 0 instead of a small positive value.
This special case only occurs if the transitive correlation over the
path $i\rightarrow k\rightarrow j$ is nearly as strong as the direct
correlation. Although this special case appears troubling at first
glance, we show in Section \ref{sec:theory-props} that the SHORTCUT
method still gives structurally consistent reconstructions, under
a similar set of problem instances as the CCT method.
\begin{algorithm}
\begin{algorithmic}
\STATE {\bfseries input:} $p \times p$ Correlation matrix $\bar{C}_{ij}$, threshold $t$
\STATE {\bfseries initialise:} Weight matrix $W$
\FOR{$i=1$ {\bfseries to} $p$}
	\STATE {\bfseries Build:} $P_i$ as the mapping that sorts the positive off-diagonal entries of row $i$ of $\bar{C}$ in descending order
	\STATE {\bfseries Build:} $N_i$ as the mapping that sorts the negative off-diagonal entries of row $i$ of $\bar{C}$ in ascending order
	\STATE {\bfseries Build:} $P_{i}^{-1}$ and $N_{i}^{-1}$ as the inverse maps of $P_i$ and $N_i$ respectively
\ENDFOR
\FOR{$i=1$ {\bfseries to} $p$, $j=1$ {\bfseries to} $i-1$}
	\IF{ $\bar{C}_{ij} >= 0$ }
		\STATE $t_{l}$ := maxProduct($\bar{C}_{i*}$, $\bar{C}_{j*}$, $P_i$, $P_j$, $P_{i}^{-1}$, $P_{j}^{-1}$ )
		\STATE $t_{r}$ := maxProduct($\bar{C}_{i*}$, $\bar{C}_{j*}$, $N_i$, $N_j$, $N_{i}^{-1}$, $N_{j}^{-1}$)
		\STATE $W_{ij}$ := $\max \left (0, \bar{C}_{ij} - max(t_{l}, t_{r}) \right)$
	\ELSE
		\STATE $t_{l}$ := maxProduct($\bar{C}_{i*}$, $|\bar{C}_{j*}|$, $P_i$, $N_j$, $P_{i}^{-1}$, $N_{j}^{-1}$)
		\STATE $t_{r}$ := maxProduct($|\bar{C}_{i*}|$, $\bar{C}_{j*}$, $N_i$, $P_j$,  $N_{i}^{-1}$, $P_{j}^{-1}$)
		\STATE $W_{ij}$ := $\max \left (0, |\bar{C}_{ij}| + max(t_{l}, t_{r}) \right)$
	\ENDIF
	\STATE $W_{ji} = W_{ij}$
\ENDFOR
\STATE {\bfseries Threshold} Set to zero any element $(i,j)$ where $|W_{ij}| < t$
\STATE {\bfseries return} $W$
\end{algorithmic}

\protect\caption{SHORTCUT algorithm\label{algo:main}}
\end{algorithm}
\begin{algorithm}
\begin{algorithmic}
\STATE {\bfseries input:} Sequences $C_{a}$, $C_b$, $P_a$, $P_b$, $P_{a}^{-1}$,  $P_{b}^{-1}$.
\STATE {\bfseries initialise:} $\text{start} := 0$, $\text{best} = P_{a}[0]$
\STATE {\bfseries initialise:} {$\text{max\_seen} := C_{a} \left[ \text{best} \right] \cdot C_{b} \left[ \text{best} \right]$}
\STATE {\bfseries initialise:} $\text{end}_{a} := P_{a}^{-1} \left[  P_{b} \left[ 0   \right]\right]$
	and $\text{end}_{b} := P_{b}^{-1} \left[  P_{a} \left[ 0   \right]\right]$

\STATE \emph{// Check if $P_{b} \left [ 0 \right]$ is a better starting point than $P_{a} \left [ 0 \right]$}
\IF{ $C_{a} \left[ P_{b} \left [ 0 \right] \right] \cdot  C_{b} \left[ P_{b} \left [ 0 \right] \right] > \text{max\_seen}$ }
\STATE $\text{best} = P_{b}[0]$
\STATE $\text{max\_seen} := C_{a} \left[ \text{best} \right] \cdot C_{b} \left[ \text{best} \right]$
\ENDIF

\WHILE { $\text{start} < \text{end}_{a}$  and $\text{start} < \text{end}_{b}$ }
	\IF{ $C_{a} \left[ P_{a} \left [ \text{start} \right] \right] \cdot  C_{b} \left[ P_{a} \left [ \text{start} \right] \right] > \text{max\_seen}$ }
		\STATE $\text{best} := P_{a} \left[ \text{start} \right] $
		\STATE {$\text{max\_seen} := C_{a} \left[ \text{best} \right] \cdot C_{b} \left[ \text{best} \right]$}
	\ENDIF
	\STATE { $ \text{end}_{b} := \min \left\{    
		\text{end}_{b},
		P_{b}^{-1} \left[ P_{a} \left[ \text{start} \right] \right] 
	   \right\} $ }

	   \STATE \emph{// The following is the same as the 5 lines above but with $a$ and $b$ interchanged}

	\IF{ $C_{a} \left[ P_{b} \left [ \text{start} \right] \right] \cdot  C_{b} \left[ P_{b} \left [ \text{start} \right] \right] > \text{max\_seen}$ }
		\STATE $\text{best} := P_{b} \left[ \text{start} \right] $
		\STATE {$\text{max\_seen} := C_{a} \left[ \text{best} \right] \cdot C_{b} \left[ \text{best} \right]$}
	\ENDIF
	\STATE { $ \text{end}_{a} := \min \left\{    
		\text{end}_{a},
		P_{a}^{-1} \left[ P_{b} \left[ \text{start} \right] \right] 
	   \right\} $ }
	   
	   \STATE { $\text{start} := \text{start} + 1$}
\ENDWHILE
\STATE {\bfseries return} max\_seen
\end{algorithmic}

\protect\caption{maxProduct algorithm\label{algo:maxproduct}}
\end{algorithm}

The reason we make the above approximation is so we can rearrange
the optimisation as: 
\[
\max\left\{ 0\,,\,\sqrt{C_{ii}C_{jj}}\left(\left|\bar{C}_{ij}\right|-\max_{k\in V\backslash\left\{ i,j\right\} }\text{sgn}(\bar{C}_{ij})\bar{C}_{ik}\bar{C}_{kj}\right)\right\} .
\]
Where $\bar{C}_{ik}=C_{ik}/\sqrt{C_{ii}C_{kk}}$ is the correlation
between $x_{i}$ and $x_{k}$. This form allows us to use a fast method
for computing the maximum of a set of pairwise products due to \citet{maxprod}.
This approach requires preprocessing, namely row-wise sorting of the
correlation matrix entries, sans-diagonal. Under the assumption that
the permutations that sorted row $i$ and row $j$ are independent,
McAuley \& Caetano's method performs the maximisation in expected
time $O(\sqrt{p})$, which results in a reduction of the running time
of the conditional covariance thresholding method to only expected
time $O(p^{2.5})$. Note that the worst case is still $O(p^{3})$.
McAuley \& Caetano's algorithm can only handle positive values, so
it is necessary to run it twice, split by case; The specific method
we use is shown in Algorithm \ref{algo:main}; it assumes a threshold
is given prescaled by $1/\sqrt{C_{ii}C_{kk}}$.

\section{Running Time}

\label{sec:shortcut-props}

\begin{figure}
\begin{centering}
\includegraphics[scale=1.2]{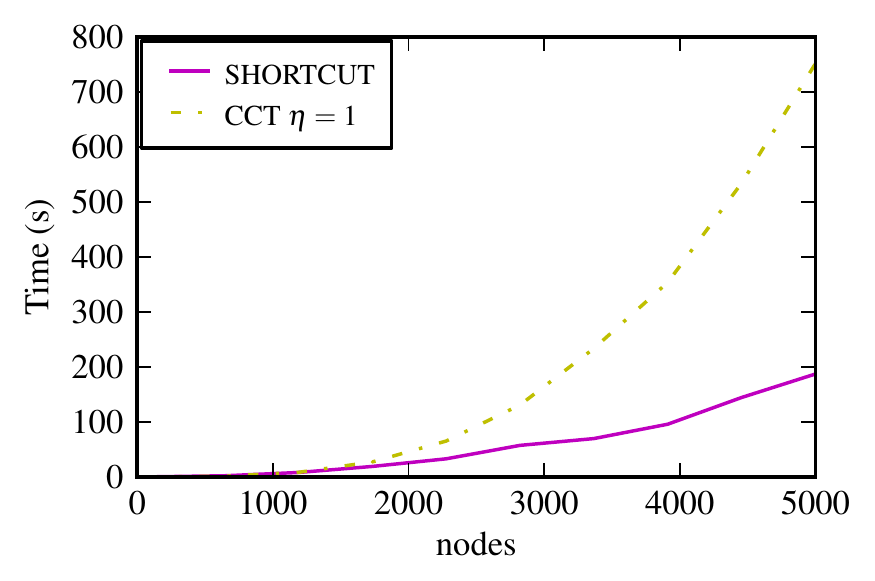} 
\par\end{centering}

\protect\caption{Running time of SHORTCUT on a representative problem\label{fig:fast-cond-cov}}
\end{figure}

The improvement in running time is not just theoretical. We find in
practice that, for randomly generated covariance selection problems,
less than $\sqrt{p}$ pairwise products per edge are required for
each maxProduct invocation; the runtime is not hindered by any hidden
multiplicative factors. The improvement in running time for building
the conditional covariance matrix over the naive approach is illustrated
in Figure \ref{fig:fast-cond-cov}, for implementations in compiled
python. It behaves empirically as predicted by the asymptotic bound.
The performance advantage is roughly a factor of 4 at 4000 nodes.
Our method is aimed at problems with mid tens of thousands of nodes,
where the advantage is much greater.

In general we have observed that the number of comparisons required
(and hence the running time) is lower when the matrix entries are
positively correlated within each row, although the speed-up is dependent
on the distribution of the entries. For example, there exists distributions
that have positive Pearson correlation coefficient, yet exhibit worst-case
run-time behaviour. A simple example is uniform distribution over
a diamond shape in the $[0-1]$ interval in 2 dimensions. Figure \ref{fig:unit-diamond}
shows the (worst-case) linear complexity of the fast max-product method
on it. We have never observed this worst-case complexity when the
max-product method is applied to correlation matrices, in fact they
seem to closer to the ideal case.

\begin{figure}[t]
\begin{centering}
\includegraphics[scale=1.3]{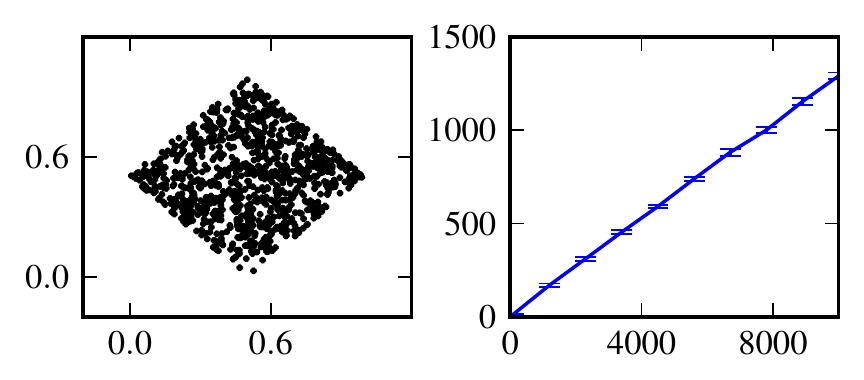} \includegraphics[scale=1.3]{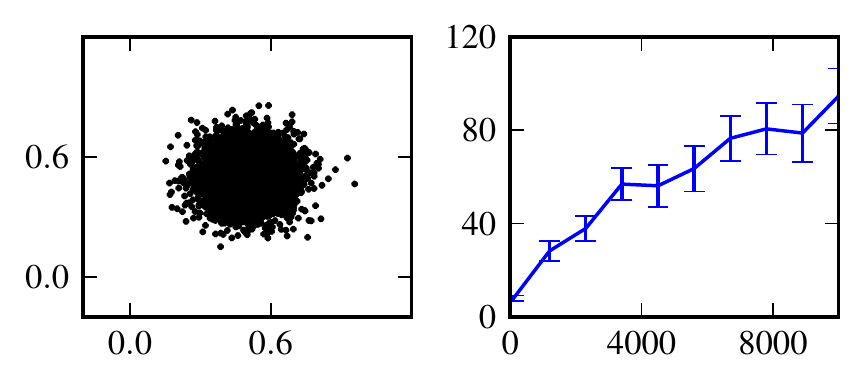} 
\par\end{centering}

\protect\caption{Empirical scaling of the max-product method on samples from the unit
diamond and isotropic Gaussian distributions. The left hand plots
are samples from the respective distributions, and the right hand
plots show scaling in terms of entries examined as the number of samples
increases. \label{fig:unit-diamond}}
\end{figure}

\section{Experiments}

\label{sec:experiments} We ran a suite of experiments testing the
SHORTCUT method against two CCT variants; CCT $\eta=1$ which based
on theoretical results should give similar results to SHORTCUT; and
CCT $\eta=2$ which should give better results at the expense of quartic
running time. We also compared against correlation thresholding, which
is often used by practitioners, $L_{1}$ regularised maximum likelihood
on the correlation matrix, and the neighbourhood selection method
discussed in Section \ref{sub:nei-sel} using both union and disjunction
variants. Covariance thresholding was significantly inferior on our
test problems, and covariance variants of the other methods performed
worse than the correlation versions shown in the plots that follow.

\subsection{Synthetic datasets}

There is a substantial difficulty in doing a comprehensive evaluation
of covariance selection methods due to the lack of ground truth structure
information. We begin with a series of tests on synthetically generated
data, using structures generated from two commonly used models: the
Barabasi-Albert model \citep{ba-model} and the classical Erdos-Renyi
model \citep{er}.

\subsubsection*{Generating synthetic data using walk summability}

\label{sec:walk-sum-gen} Walk summability is usually described in
terms of a spectral condition on the matrix of partial correlation
coefficients. For most practical purposes it is easier to work with
the notion of \emph{pairwise normalisability}. \citet{walksum} detail
the equivalence of the two definitions. Suppose we have a Gaussian
distribution with sparse precision matrix $\Theta$. It is pairwise
normalisable if $\Theta$ can be decomposed into a sum involving the
precision matrices $A^{(i,j)}$ of 2D Gaussians, one for each edge,
in particular: 
\[
\Theta_{ij}=\begin{cases}
A_{ij}^{(i,j)} & \text{if }i\neq j\\
\sum_{k\in\text{ne}(i)}A_{ii}^{(i,k)} & \text{if }i=j
\end{cases}.
\]
Each precision matrix must be positive semidefinite, which is where
the pairwise normalisable terminology comes from. This definition
gives substantial insight into the properties of walk summable models.
The condition is in a sense local, which is not at all clear from
the spectral definition.

We can use this definition for generating walk-summable but not necessarily
diagonally dominant test distributions. For our synthetic tests in
Section \ref{sec:experiments}, to choose the edge weights of our
test matrices, we generated 2x2 precision matrices and then used the
above summation. Each 2D precision matrix $A$ was formed by sampling
the off-diagonal element $A_{1,2}$ (chosen from a normal distribution
with mean $-0.2$ and variance $1$), then sampling a proportionality
constant $\beta$ and setting $A_{1,1}=\beta|A_{1,2}|$ and $A_{2,2}=\beta^{-1}|A_{1,2}|$.
This gives a lopsided distribution that is harder to optimise against
than typical diagonally dominant distributions.

\begin{figure}
\begin{centering}
\includegraphics[scale=1.5]{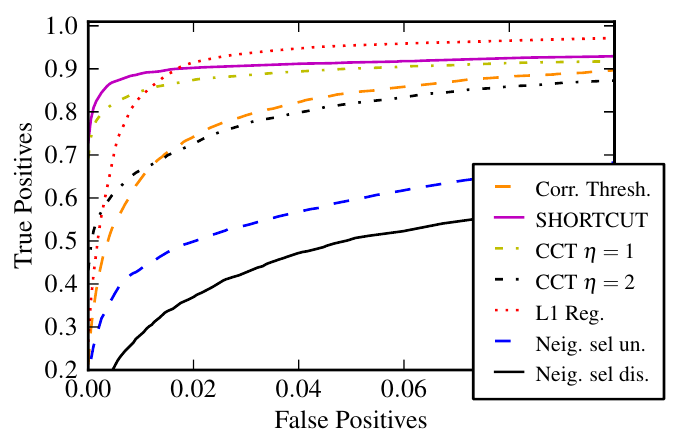} 
\par\end{centering}

\protect\caption{Synthetic data ROC curves for the BA model\label{fig:synth-walksum}}
\end{figure}

For each plot in Figure \ref{fig:synth-walksum}, 30 sparse graphs
were generated in each case with 60 nodes and 120 edges (in expectation).
The shown line is the average over these graphs. Due to the $O(n^{4})$
running time of the CCT $\eta=2$ method, we were not able to test
all methods on substantially larger graphs. For the cubic time methods,
the shown trends were consistent for larger graphs. Edge weights were
chosen so that the walk summable condition holds. The generation method
is detailed in Section \ref{sec:walk-sum-gen}.

We can see in Figure \ref{fig:synth-walksum} that the CCT $\eta=2$
method performs very similarly to the $\eta=1$ case. SHORTCUT performs
similarly on the easier ER model, and slightly better on the BA model.
As would be expected from the simplicity of the method, correlation
matrix thresholding performs poorly. Interestingly, the $L_{1}$ regularised
maximum likelihood method is out-performed by the thresholding methods
for sparse reconstructions, but is able to capture a larger proportion
of correct edges as the number of allowed false positives increases.
It's notable that neighbourhood selection performs particularly badly
on these synthetic problems. We were not able to determine a good
reason for this. Their performance is more in line with the other
methods on the real world datasets in the next section.

\begin{figure}
\begin{centering}
\includegraphics[scale=1.5]{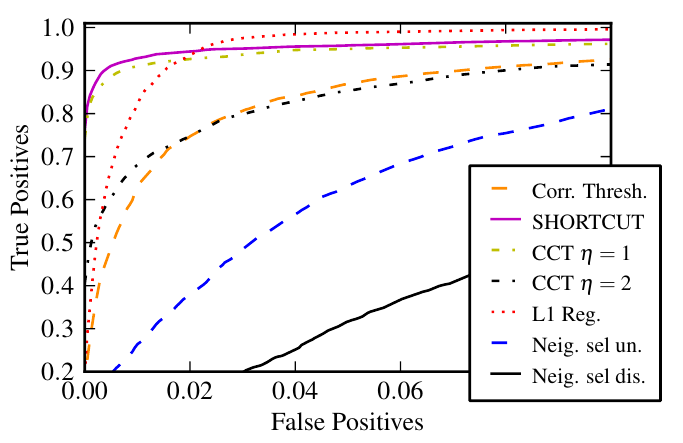} 
\par\end{centering}

\protect\caption{Synthetic data ROC curves for the ER model\label{fig:synth-walksum-er}}
\end{figure}

\subsection{Real world datasets}

In this set of experiments we used a different comparison methodology
than in the synthetic case. Instead of comparing the structure learned
we can compare the quality of the structure when used to fit a maximum
likelihood model. Essentially we use each method as described, then
we fit parameters to a model with each structure as fixed. We can
then compare the likelihood of held out data, which gives us an idea
of how well each structure captures the true dependencies in the data.
This avoids the difficulty of requiring the ground truth structure.
Instead of the likelihood, we plot the negation of the objective in
Equation \ref{eq:nll}, which is a shifted and scaled negative log
likelihood.

For our first real world dataset test we choose stock market data.
We gathered the daily fractional change in value of each share in
the ASX300 between Feb 1 2012 and 18 Dec 2012. We then removed any
shares with missing data, or those that were not part of the ASX300
for the whole period. This left us with 235 shares over 250 days.
We used a 125/125 train/test split. Figure \ref{fig:share-data} shows
the resulting test likelihood on model trained on structure from each
method, for varying number of edges.

We see that on this problem SHORTCUT performs better than $L_{1}$
regularised maximum likelihood, which is significantly better than
correlation thresholding. As in the other tests, SHORTCUT performs
similar to or better than the CCT $\eta=1$ method. The neighbourhood
selection methods give middle of the pack performance.

\begin{figure}
\begin{centering}
\includegraphics[scale=0.53]{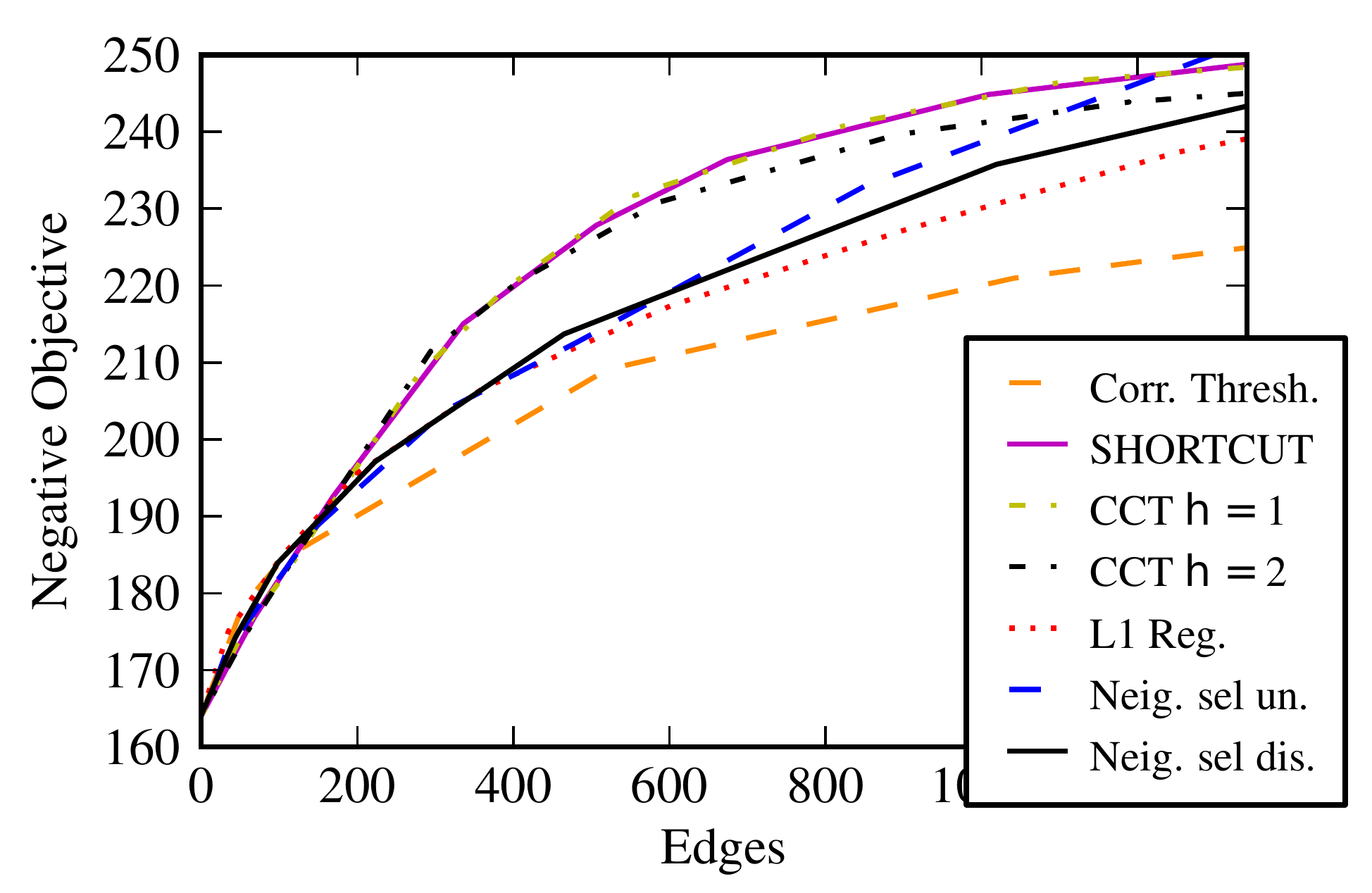} 
\par\end{centering}

\protect\caption{Comparison on stock market data \label{fig:share-data}}
\end{figure}

For our second test we used climate data. We extracted the maximum
daily temperatures from all of the weather stations in NSW, Australia
that reported temperatures for each day between Jan 1st 2012 and Dec
1st 2012. This gave us 171 stations with 336 temperature readings
for each. As before, we use even split between test and training data.

On this particular problem the advantages of CCT with $\eta=2$ is
much more apparent. Figure \ref{fig:climate-data} shows that it outperforms
the other methods, at the expense of quite prohibitive running time.
Like in the other cases, SHORTCUT performs similarly to CCT $\eta=1$,
but interestingly simple correlation thresholding is more effective
than either for this problem. Neighbourhood selection performs very
well using the union variant, and very poorly using the disjunction
variant.

\begin{figure}
\begin{centering}
\includegraphics[scale=0.53]{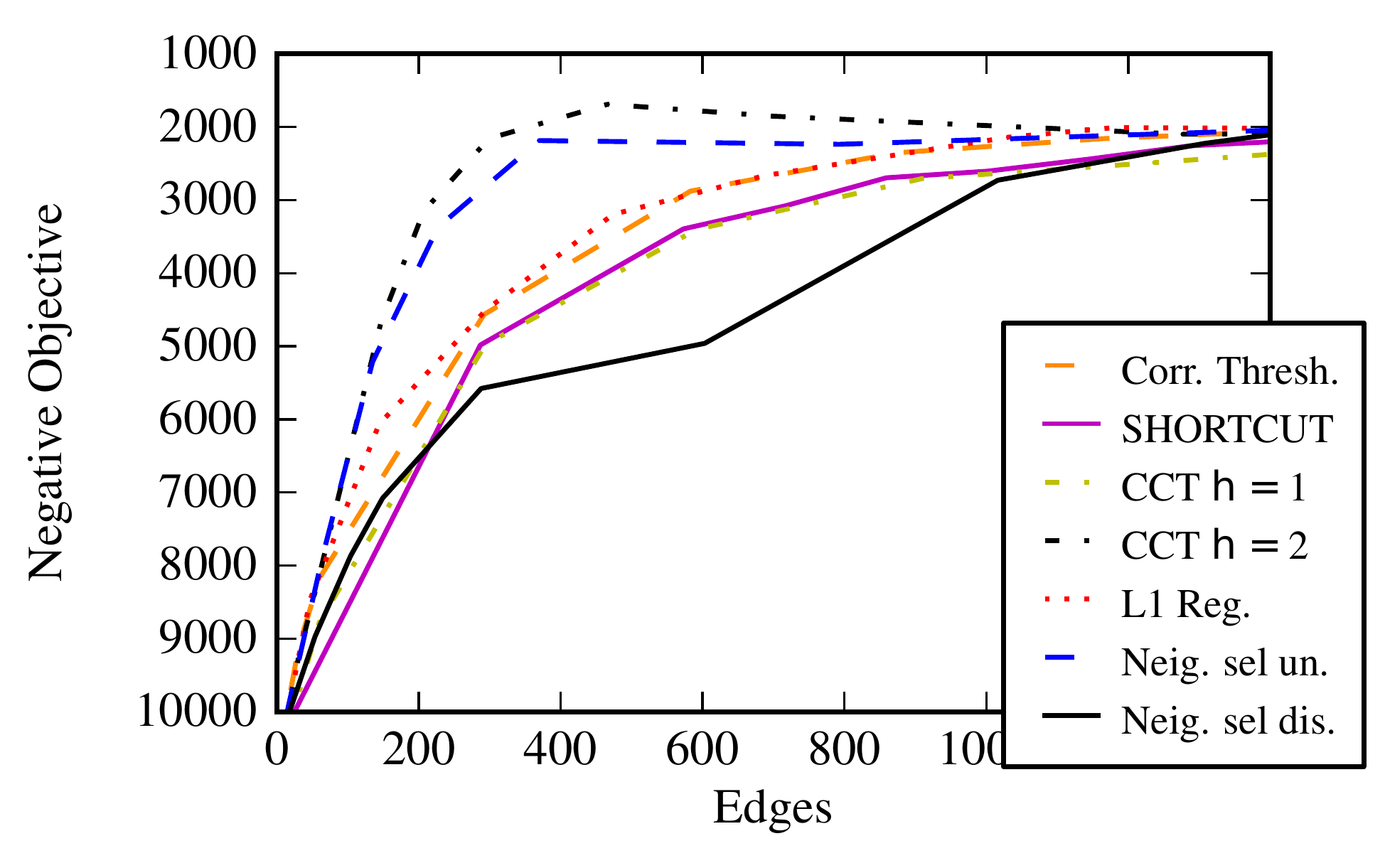} 
\par\end{centering}

\protect\caption{Comparison on climate data \label{fig:climate-data}}
\end{figure}

\section{Theoretical Properties}

\label{sec:theory-props} Although our algorithm differs from the
CCT method, we show in this section that our method has the same theoretical
guarantees, although with slightly stronger assumptions required.
\citet{anand-jmlr-prepub} are able to show that CCT provides reconstructions
structurally consistent with the true model, under a set of asymptotic
conditions. The key requirements are that the number of samples $n$
scales with the number of nodes $p$ as $n=\Omega\left(\Theta_{\min}^{-2}\log p\right)$,
that the model be walk summable, and that a form of local-separation
holds. See \citet[p10-11]{anand-jmlr-prepub} for details. Their proof
proceeds by considering the noise-free (exact covariance) case. They
show that if nodes $i$ and $j$ are not neighbours, then the maximum
absolute conditional covariance over nodes $k\in V\backslash\{i,j\}$
is bounded. They then show that if $i$ and $j$ are neighbours, then
the \emph{minimum} absolute conditional covariance over $k$ is bounded
away from zero. Thus the CCT method can separate the two cases successfully,
with an appropriate choice of threshold.

Pairs of non-neighbouring nodes will be thresholded under both SHORTCUT
and CCT, as the bounded distance from zero ensures that they would
be beneath of threshold of CCT if they changed sign (recall that SHORTCUT
only behaves differently for entries whose covariance's sign differs
from the CC matrices sign). So we only need to consider the case of
neighbouring nodes. The relevant lemma from the CCT proof is as follows:
\begin{lem}
\textbf{{[}Lemma 14 in \citet{anand-jmlr-prepub}, p27{]}} For an
$\alpha$-walk summable graphical model satisfying 
\[
D_{\min}(1-\alpha)\min_{(i,j)\in G}\frac{|\Theta_{ij}|}{K}>1+\delta,
\]
for some $\delta>0$ (not depending on p), where $D_{\min}:=\min_{i}\Theta_{ii}$
and $K=\ns{\Theta_{\{i,j\}(V\backslash\{i,j\})}}$, and where $K>2\Theta_{ii}\Theta_{jj}$.
Then we have 
\[
\left|\Sigma(i,j|k)\right|=\Omega(\Theta_{\min}),
\]
for any neighbouring $(i,j)$ and any $k\in V\backslash\{i,j\}$,
where $\Theta_{\min}$ is the minimum element of $\Theta$. 
\end{lem}
The assumptions in this lemma are the key technical assumptions about
the data. The bound on $K$ is the additional assumption on top of
those made by \citet{anand-jmlr-prepub}. Note that the absolute value
around $\Theta_{ij}$ is a correction of an apparent typo. We will
modify their proof of this lemma to additionally show that the sign
of $\Sigma(i,j|k)$ is the same as $\Sigma(i,j)$, so that the SHORTCUT
method will not incorrectly threshold in the neighbouring case.
\begin{proof}
We use the notation $A=\{i,j\}$ and $B=V\backslash\{i,j,k\}$. Let
$\tilde{\Sigma}$ denote the conditional covariance matrix of $A\cup B$
conditioned on $k$, it is formed by removing row and column $k$
from $\Theta$, then inverting. We are interested in the 2x2 sub-matrix
$\tilde{\Sigma}_{AA}$, as it contains the conditional covariance
value we are interested in. It can be expressed as part of block matrix
inversion: 
\[
\tilde{\Sigma}=\left[\begin{array}{cc}
\Theta_{AA} & \Theta_{AB}\\
\Theta_{BA} & \Theta_{BB}
\end{array}\right]^{-1}=\left[\begin{array}{cc}
\left(\Theta_{AA}-\Theta_{AB}\Theta_{BB}^{-1}\Theta_{BA}\right)^{-1} & \cdots\\
\cdots & \cdots
\end{array}\right].
\]
So in order to bound $\Sigma(i,j|k)=\tilde{\Sigma}_{i,j}$ away from
zero, we need to upper bound the absolute value of $\rho:=\Theta_{AB}\Theta_{BB}^{-1}\Theta_{BA}$,
so that $\tilde{\Theta}_{AA}:=\Theta_{AA}-\rho$ is bounded away from
zero.

\citet{anand-jmlr-prepub} establish that the maximum element of $\rho$
is bounded by 
\[
\frac{\ns{\Theta_{AB}}}{D_{\min}(1-\alpha)},
\]
a constant, and thus since $\Theta_{i,j}>\Theta_{\min}$, we have
that $\tilde{\Theta}_{ij}=\Omega\left(\Theta_{\min}\right)$. It follows
from 2x2 matrix inversion that 
\[
\Sigma(i,j|k)=\frac{-\tilde{\Theta}_{ij}}{\tilde{\Theta}_{ii}\tilde{\Theta}_{jj}-\tilde{\Theta}_{ij}^{2}}>\frac{-\tilde{\Theta}_{ij}}{\tilde{\Theta}_{ii}\tilde{\Theta}_{jj}}=\Omega\left(\Theta_{\min}\right).
\]
We note that the above bound establishes that the sign of $\tilde{\Theta}_{ij}$
is the same as $\Theta_{ij}$, and that $\Sigma(i,j|k)$ has the opposite
sign to $\tilde{\Theta}_{ij}$. Therefore, in order to show that the
sign of the non-conditional covariance coincides, we need to establish
the sign of $\Theta_{ij}$.

From Lemma 13 of \citet{anand-jmlr-prepub}, we have that 
\[
\left|\Sigma_{ij}+\frac{\Theta_{ij}}{\Theta_{ii}\Theta_{jj}-\Theta_{ij}^{2}}\right|\leq\frac{2\alpha^{2}}{D_{\min}(1-\alpha)}.
\]
Using this lemma, it is sufficient to establish that the right hand
side here is bounded by $\frac{\left|\Theta_{ij}\right|}{\Theta_{ii}\Theta_{jj}-\Theta_{ij}^{2}}$.
It follows from our initial assumption that: 
\[
\frac{1}{D_{\min}(1-\alpha)}\leq\frac{\left|\Theta_{ij}\right|}{K(1+\delta)},
\]
therefore 
\[
\left|\Sigma_{ij}+\frac{\Theta_{ij}}{\Theta_{ii}\Theta_{jj}-\Theta_{ij}^{2}}\right|\leq2\alpha^{2}\frac{\left|\Theta_{ij}\right|}{K(1+\delta)}.
\]
We can use our assumed bound on $K$ and the fact that $\alpha<1$
and $\delta>0$ to simplify as: 
\begin{eqnarray*}
\left|\Sigma_{ij}+\frac{\Theta_{ij}}{\Theta_{ii}\Theta_{jj}-\Theta_{ij}^{2}}\right| & \leq & \frac{\alpha^{2}}{1+\delta}\cdot\frac{\left|\Theta_{ij}\right|}{\Theta_{ii}\Theta_{jj}}\\
 & \leq & \frac{\left|\Theta_{ij}\right|}{\Theta_{ii}\Theta_{jj}-\Theta_{ij}^{2}}.
\end{eqnarray*}

The last line follows from $\Theta_{ii}\Theta_{jj}-\Theta_{ij}^{2}$
being a minor of $\Theta$.\end{proof}

\chapter{Fast Approximate Parameter Inference}

\label{chap:colab}In this chapter we consider the problem of large
scale learning of the parameters of a Gaussian graphical model. We
assume the structure is known, perhaps using the techniques from Chapter
\ref{chap:approx-covsel}. Our goal in machine learning is not learning
for its own sake, but to learn predictive models. To this end, we
consider a reasonable restriction of the class of all Gaussian models
which ensures that prediction using all models in the class is tractable.
In Section \ref{sec:training} we give a novel dual-decomposition
learning procedure for this restricted class. We then introduce the
collaborative filtering application domain, where a Gaussian graph
structure naturally arises. In Section \ref{sec:colab-experiments}
we show promising empirical performance of our method on this domain.

An earlier version of the work in this chapter has been published
as \citet{adefazio-icml2012}.

\section{Model Class}

\label{sec:cf-model-class}

In this section we impose a number of restrictions of the class of
precision matrices $\Theta$ that we will consider. These restrictions
have multiple purposes. Primarily, we are looking to ensure that inference
on the resulting model is tractable. A second reason for our structural
assumptions is that they form a hypothesis class restriction. The
approximation we use in Section \ref{sec:training} learns poor models
without the additional restrictions used here.

\subsection{Improper models}

A proper Gaussian graphical model is one in which the function $Z(\Theta)\propto\log\det\Theta$
(the log determinant of the precision matrix) which plays the part
of the normalisation constant is finite. In a sense proper models
are those that can be normalised. The requirement that a Gaussian
graphical be proper is actually stronger than we require in our applications,
and potentially difficult to enforce in practice. For Gaussian models,
if $\Theta$ is positive definite than it is proper. An improper model
is still required to have a symmetric positive semi-definite (PSD)
$\Theta$.

To see why we don't need a proper model, consider that our ultimate
goal is to perform prediction using Bayesian inference. Inference
in this case means conditioning on a subset of variables, then determining
the expected values and variances of the remaining variables. For
this purpose we just need that all of the principal minors of $\Theta$
are positive definite as they form the conditional sub-models which
we need to be proper. We will discuss in the next section how to ensure
this is the case.

\subsection{Precision matrix restrictions}

The first restriction we apply to $\Theta$ is to force it to be a
member of the class of $Z$-matrices.
\begin{defn}
A $Z$-matrix is a matrix for which all off-diagonal entries are non-positive.
I.e
\[
\forall i\neq j\quad\Theta_{ij}\leq0.
\]

For a Gaussian graphical model, a negative $\Theta_{ij}$ value implies
that (conditioned on all other variables) the variables $i$ and $j$
are positively correlated. So this restriction means that we can not
model negative direct correlations in our model. Although this is
a strong restriction, it is a natural one. When one thinks of a network
of dependencies, edges are naturally assumed to imply correlation
between the adjoined nodes rather than anti-correlation. In the collaborative
filtering application area we consider, this assumption is reasonable.
We discuss this further in Section \ref{sec:item-field}. Note that
this restriction to modelling negative correlations manifests in the
covariance matrix $\Theta^{-1}$; it can only contain non-negative
entries for PSD $Z-$matrices.

The $Z$-matrix property together with the SPSD requirement implies
quite a lot about $\Theta$. The most important property for our purpose
is the following:\end{defn}
\begin{thm}
\textbf{\citep[Theorem 4.16 p156]{nonnegative-matrix-book}} Suppose
$\Theta$ is a SPSD $Z$-matrix. Then if $\Theta$ is an irreducible
matrix, i.e. the support of $\Theta$ is a connected graph, then all
principal minors of $\Theta$ are positive definite.
\end{thm}
So our requirement that conditional sub-models are proper is satisfied
as long as we ensure that the graph structure of $\Theta$ is connected.
This requirement is very weak in practice. In fact, non-connected
models can still be handled by treating the connected components separately.
When no values in a component are conditioned on then the predictions
for that component are just the (unconditional) expectations of those
variables.

The final restriction we place on $\Theta$ is the requirement that
it is actively diagonally dominant:
\begin{defn}
A symmetric matrix $\Theta$ is \emph{actively diagonally dominant}
if:
\[
\forall i\;\left|A_{ii}\right|=\sum_{j\neq i}\left|A_{ij}\right|.
\]

This is a sub-class of the diagonally dominant matrices, where the
above expression holds with inequality instead ($\left|A_{ii}\right|\geq\sum_{j\neq i}\left|A_{ij}\right|$).
\end{defn}
In our case the $Z$-matrix restriction simplifies this condition
to $A_{ii}=-\sum_{j\neq i}A_{ij}.$ For general matrices diagonal
dominance is a very strong restriction. In our setting, it is not
as strong a requirement. Informally, in order that $A$ be PSD, there
needs to be sufficient weight on the diagonal to counter-act the off-diagonal
entries. If the diagonal entries are of roughly similar magnitude,
then diagonal dominance will hold. The purpose of the \emph{active}
constraint rather than the usual inequality is primarily a practical
one. When optimising over the set of diagonally dominant matrices,
the constraints are normally active during optimisation, so it simplifies
the algorithm to enforce that they are always active. The active constraint
also prevents the model from shrinking its predictions towards the
mean, which is undesirable in some applications. Interestingly, this
diagonal dominance condition is also implicitly assumed in other network
models in the collaborative filtering domain. This is discussed further
in Section \ref{sec:prediction}.

One immediate consequence of the active diagonal dominance is the
singularity of the matrix $\Theta$. So in fact we are certain to
not have a proper Gaussian graphical model.

\section{An Approximate Constrained Maximum Entropy Learning Algorithm}

\label{sec:training} \label{sec:max-ent} 

For Gaussian models, the mean and covariance matrix form the sufficient
statistics, so there is no need to work directly with the data points
during learning. Let $\Sigma$ be the empirical covariance matrix
and $\mu$ the mean vector. We will apply an approximate maximum entropy
approach to learning the parameters of a model under our class restriction.

\subsection{Maximum Entropy Learning}

In the previous chapter we discussed the method of maximum likelihood
learning in the context of Gaussian models. When a L1 regulariser
is used, the dual problem took the form:

\begin{gather}
\max_{C\in S_{++}^{N}}\log\det C,\nonumber \\
\text{s.t.\ }\forall i,j\colon\left|C_{ij}-\Sigma_{ij}\right|\leq\lambda.\label{eq:covsel-dual-1}
\end{gather}

This is actually an instance of Maximum Entropy learning (ME). In
general maximum likelihood problems have maximum entropy duals. The
entropy nomenclature refers to the fact that the objective is a scaled
and shifted form of the entropy of a Gaussian distribution:
\[
H(\mathcal{N}(.,C))=\frac{N}{2}\left(1+\log(2\pi)\right)+\frac{1}{2}\log\det C.
\]

In ME learning we need to place some sort of constraint on the allowed
distributions so that the solution is well defined. For the L1 case
that was a box constraint $\left|C_{ij}-\Sigma_{ij}\right|\leq\lambda$.
For the unregularised case it would simply be $C=\Sigma$, which obviously
gives the trivial solution for $C$. Once we have the ME solution
we can recover a ML solution using the gradients of the Lagrangian.
We give an example of this in Section \ref{sub:learning-restricted}. 

For Gaussian models exact maximum entropy learning can be tractable.
However, for the size of models we are considering it quickly becomes
impractical. There is also a lack of robustness in the solutions.
Instead, we will rely on the use of the Bethe approximation.

\subsection{The Bethe Approximation}

The Bethe approximation is better known as the implicit entropy approximation
made by the standard belief propagation method. For parameter learning,
we will work with the approximation more directly. Suppose we have
a set of pairwise beliefs $b_{ij}(x_{i},x_{j})$ and unary beliefs
$b_{i}(x_{i})$. Beliefs in this context refers to probability distributions
which we intend to (variationally) optimise over so that they become
better approximations to the probabilities $p(x_{i},x_{j})$. Beliefs
that are consistent on their overlap (i.e. $b_{i}(x_{i})=\sum_{y}b_{ij}(x_{i},y)$
for all $i$) encode the local properties of a probability distribution.
The Bethe approximation is an entropy like function that is defined
on such local information of a distribution. It has the following
form for pairwise models \citep[Section 11.3.7]{koller-friedman},
defined using the true entropy $H$ as:
\begin{eqnarray*}
H_{\text{Bethe}}(b) & = & \sum_{i}^{N}\,\sum_{j=i+1}^{N}H(b_{ij})\\
 & - & (N-1)\sum_{i}^{N}H(b_{i}).
\end{eqnarray*}

The goal of this approximation is to approximate the full entropy
using the entropy over the local beliefs $b_{ij}$. We can't just
use the approximation $H_{\text{Bethe}}(b)=\sum_{i}^{N}\sum_{j=i+1}^{N}H(b_{ij})$
as each $x_{i}$ appears $N$ times as part of an entropy, essentially
being over-counted $N-1$ times. The Bethe entropy deals with this
over-counting by subtracting off $N-1$ copies of each unary entropy,
resulting in each $x_{i}$ being counted only once.

In a pairwise model over a graph structure, we only define beliefs
that have matching factors (i.e. edges). So the Bethe entropy takes
the form:
\[
H_{\text{Bethe}}(b)=\sum_{(i,j)\in E}H(b_{ij})-\sum_{i\in V}(\text{deg}(i)-1)H(b_{i}).
\]

\subsection{Maximum entropy learning of unconstrained Gaussians distributions}

In a Gaussian model, it is convenient to encode the beliefs in matrix
form. Each belief $b_{ij}$ is a 2D Gaussian (As marginals of Gaussian
distributions must be Gaussian), and so its covariance matrix has
parameters $C_{ii},C_{ij},C_{jj}$. Since we are only interested in
consistent beliefs (those that agree on the overlap), two beliefs
that share a variable (say $x_{1}$) will agree on the variance of
it ($C_{11}$). So in fact we can write the full set of beliefs compactly
as a sparse matrix $C$, which contains the covariance $C_{ij}$ according
to belief $b_{ij}$ at location $(i,j)$. The diagonal contains all
the unary variances $C_{ii}$ for each $i$.

The Bethe entropy approximation in this notation is: 
\begin{eqnarray*}
H_{\textrm{Bethe}}(C) & = & \sum_{(i,j)\in E}\log\left(C_{ii}C_{jj}-C_{ij}^{2}\right)\\
 &  & +\sum_{i\in V}(1-\text{deg}(i))\log C_{ii}.
\end{eqnarray*}
Using the Bethe approximation, we can now perform approximate maximum
entropy learning. Recall that the maximum entropy objective has additional
constraints requiring that the marginals of the learned distribution
match the observed distribution. When using an approximation entropy,
we just require that the beliefs match. We start by considering a
model without the restrictions from Section \ref{sec:cf-model-class}.
In our matrix notation, the maximum entropy objective is: 
\begin{eqnarray*}
 & \underset{C}{\textrm{maximize}}\quad H_{\textrm{Bethe}}(C)\\
 & \textrm{s.t. }C=\Sigma.
\end{eqnarray*}
Stated this way, the solution is trivial as the constraints directly
specify $C$. However, we are interested in learning the weights $\Theta$,
which are the Lagrange multipliers of the equality constraints. The
Lagrangian is 
\begin{equation}
L_{\Sigma}(C,\Theta)=H_{\textrm{Bethe}}(C)+\left\langle \Theta,\Sigma-C\right\rangle ,\label{eqn:lagrangian}
\end{equation}
where $\left\langle \cdot,\cdot\right\rangle $ is the standard inner
product on matrices. The Lagrangian has derivatives: 
\begin{gather*}
\frac{\partial L_{\Sigma}(C,\Theta)}{\partial C}=\frac{\partial H_{\textrm{Bethe}}(C)}{\partial C}-\Theta,\\
\frac{\partial L_{\Sigma}(C,\Theta)}{\partial\Theta}=\Sigma-C.
\end{gather*}

Equating the gradients to zero, gives the following equation for $\Theta$:
\[
\Theta=\frac{\partial H_{\textrm{Bethe}}(C)}{\partial C}|_{C=\Sigma},
\]
which has the closed form solution: 
\begin{equation}
\begin{split}\Theta_{ij} & =\frac{-\Sigma_{ij}}{\Sigma_{ii}\Sigma_{jj}-\Sigma_{ij}^{2}}\\
\Theta_{ii} & =\frac{1}{\Sigma_{ii}}+\sum_{j\in ne(i)}\left(\frac{\Sigma_{jj}}{\Sigma_{ii}\Sigma_{jj}-\Sigma_{ij}^{2}}-\frac{1}{\Sigma_{ii}}\right).
\end{split}
\label{eqn:closed-form}
\end{equation}
This kind of solution is also known as pseudo-moment matching \citep[20.5.1.1 p963]{koller-friedman}.

\subsection{Restricted Gaussian distributions}

\label{sub:learning-restricted}

If we were applying a vanilla Gaussian model, we could use Equation
\ref{eqn:closed-form} directly. However, for our novel restricted
class we have active diagonal dominance constraints as well. To handle
these constraints, we use variable substitution, replacing $\Theta_{ii}$
with $\sum_{j\in ne(i)}\Theta_{ij}$. In particular:
\[
\Theta_{ii}\left(\Sigma_{ii}-C_{ii}\right)\rightarrow\sum_{j\in ne(i)}\Theta_{ij}\left(\Sigma_{ii}-C_{ii}\right).
\]
Grouping in terms of $\Theta$, we have the following Lagrangian:
\[
H_{\textrm{Bethe}}(C)+\sum_{(i,j)\in E}\Theta_{ij}\left(\Sigma_{ij}-\Sigma_{ii}-\Sigma_{jj}\right)-\sum_{(i,j)\in E}\Theta_{ij}\left(C_{ij}-C_{ii}-C_{jj}\right)
\]
which we denote $L_{\Sigma}^{\prime}(C,\Theta)$. It has gradients:
\begin{align*}
\frac{\partial L_{\Sigma}^{\prime}(C,\Theta)}{\partial C_{ij}}= & \frac{\partial H_{\textrm{Bethe}}(C)}{\partial C_{ij}}-\Theta_{ij},\\
\frac{\partial L_{\Sigma}^{\prime}(C,\Theta)}{\partial C_{ii}}= & \frac{\partial H_{\textrm{Bethe}}(C)}{\partial C_{ii}}+\sum_{j\in\textrm{ne}(i)}\Theta_{ij},\\
\frac{\partial L_{\Sigma}^{\prime}(C,\Theta)}{\partial\Theta_{ij}}= & \Sigma_{ij}-\Sigma_{ii}-\Sigma_{jj}-C_{ij}+C_{ii}+C_{jj}.
\end{align*}
In order to optimise this constrained objective, we can take advantage
of the closed form solution for the simpler unconstrained Gaussian.
The procedure we use is given in Algorithm \ref{alg:primal}. It gives
a quality solution in a small number of iterations (see Figure \ref{fig:early-stopping}).
The core idea is that if we fix the diagonal of $C$, the rest of
the elements are determined. So we take a block coordinate ascent
approach. One block is all off-diagonal entries. We do an exact solve
for that block. The other block is the diagonal elements, for which
we do a gradient step. The algorithm simply alternates between the
two block updates.

The gradient of the diagonal elements can be computed easily from
the dual variable matrix $\Theta$, so we recompute that at each step.
We use a $\alpha/\sqrt{k}$ step size regime for the gradient step,
where $k$ is the step number.

\begin{algorithm}[t]
\begin{algorithmic}
	\STATE {\bfseries input:} covariance $\Sigma$, size $N$, step size $\alpha$
	\STATE $C=\Sigma$ and $k=1$
	\REPEAT
	\STATE \emph{\# Compute $\Theta$, needed for each $C_{ii}$ gradient}
    \FOR{$i=1$ {\bfseries to} $N$}
    	\STATE $\Theta_{ii} = \frac{1}{C_{ii}}$
    \ENDFOR
    \FOR{$(i,j) \in E$}
    	\IF{ $C_{ii}C_{jj}-C_{ij}^{2} > 0$ }
			\STATE $\Theta_{ij} = \frac{-C_{ij}}{C_{ii}C_{jj}-C_{ij}^{2}}$
			\STATE $\Theta_{ii} $ += $\frac{C_{jj}}{C_{ii}C_{jj}-C_{ij}^{2}}-\frac{1}{C_{ii}}$
		\ELSE
			\STATE $\Theta_{ij} = 0$
		\ENDIF
	\ENDFOR
	\STATE \emph{\# Take a gradient step on each $C_{ii}$}
	\FOR{$i=1$ {\bfseries to} $N$}
	\STATE $C_{ii} $ += $\frac{\alpha}{\sqrt{k}}( \Theta_{ii} + \sum_{j\in\textrm{ne}(i)}\Theta_{ij}  ) $
	\ENDFOR
	\STATE \emph{\# Update the off-diagonal elements}
	\FOR{$(i,j) \in E$}
	\STATE $C_{ij} = \Sigma_{ij} - \Sigma_{ii} - \Sigma_{jj} + C_{ii} + C_{jj}$
	\STATE $C_{ij} = \max(C_{ij}, 0)$
	\ENDFOR
	\STATE $k = k + 1$
	\UNTIL{sufficient convergence}
	\STATE {\bfseries return} $\Theta$
\end{algorithmic}\protect\caption{\label{alg:primal} Diagonal ascent algorithm for approximate maximum
entropy learning}
\end{algorithm}

\begin{figure}[t]
\centering{}\includegraphics[width=1\columnwidth]{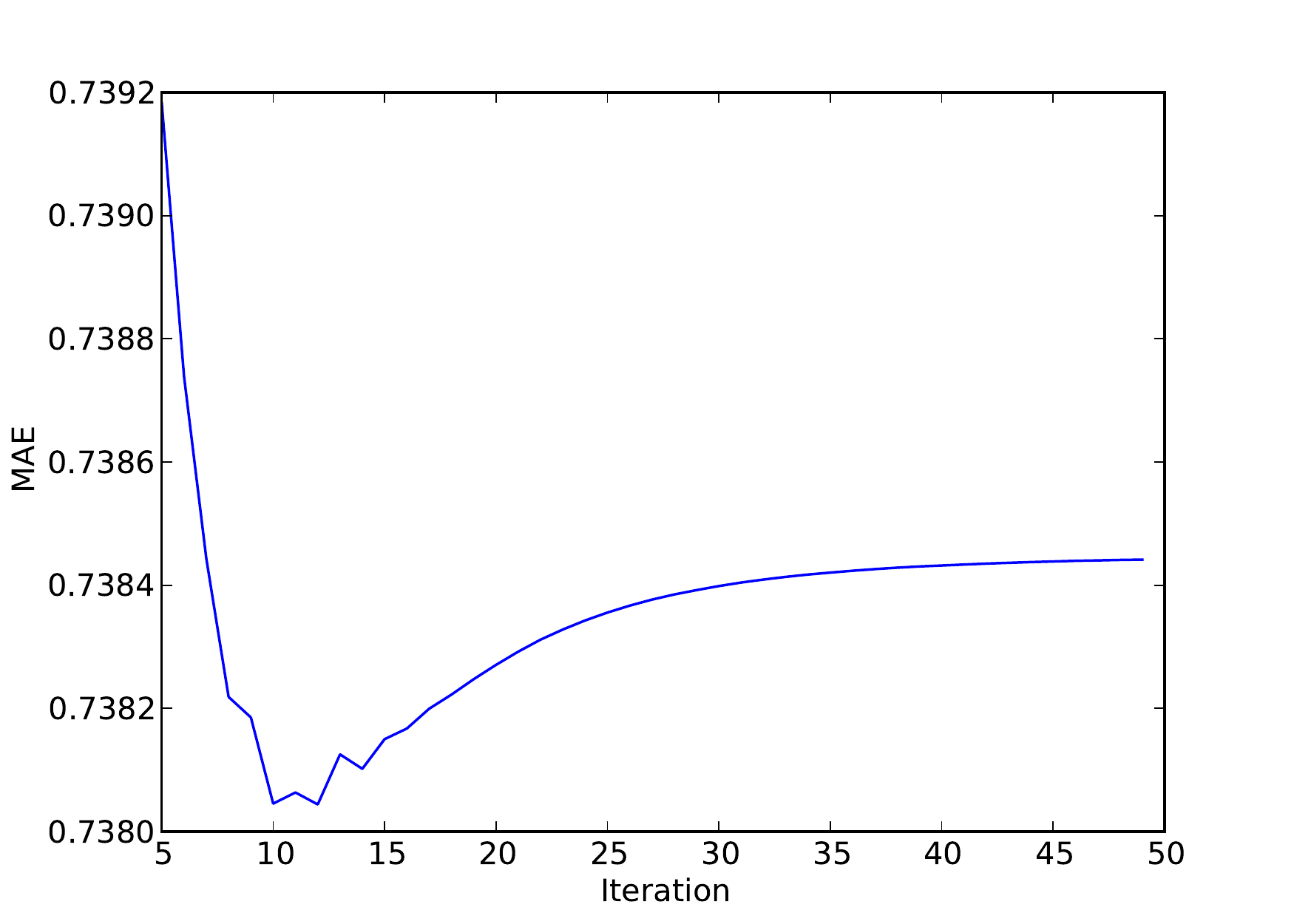}
\protect\caption{\label{fig:early-stopping}Test error as a function of the number
of training iterations on the 100K MovieLens dataset. A mild over-fitting
effect is visible.}
\end{figure}

We have not discussed yet the non-positive restriction on the off
diagonal elements (Section \ref{sec:cf-model-class}). When using
the Bethe approximation, we can see from Equation \ref{eqn:closed-form}
that $\Theta_{ij}$ is negative if $C_{ij}$ is positive. Our restricted
model class can only model covariance matrices with $C_{ij}$ positive,
so we do not need to consider negative $C_{ij}$ at all. Algorithm
\ref{alg:primal} just clamps $C_{ij}=0$ to be positive, which is
a well founded way of enforcing a positivity constraint in a coordinate
descent algorithm.

\section{Maximum Likelihood Learning with Belief Propagation}

\label{sec:max-likelihood}

In Section \ref{sec:training} we derived an efficient learning procedure
for our restricted class of models. It is also possible to apply general
black box optimisation methods to this problem, and we use such as
method as a baseline for comparison. 

Instead of working with a maximum entropy viewpoint, we will use a
more standard maximum likelihood approach. Using the Bethe approximation
to the entropy, we can form an approximate maximum likelihood objective
as follows. First note that the Lagrangian (Equation \ref{eqn:lagrangian})
can be split as follows: 
\begin{align*}
L_{\Sigma}(\Theta,C)= & H_{\textrm{Bethe}}(C)+\left\langle \Theta,\Sigma-C\right\rangle \\
= & \left\langle \Theta,\Sigma\right\rangle -\left(\left\langle \Theta,C\right\rangle -H_{\textrm{Bethe}}(C)\right);
\end{align*}
The dual is then formed by maximising in terms of $C$:
\begin{eqnarray*}
-\log p(\Sigma;\Theta) & \propto & \max_{C}L_{\Sigma}(\Theta,C)\\
 & = & \left\langle \Theta,\Sigma\right\rangle -\min_{C}\left(\left\langle \Theta,C\right\rangle -H_{\textrm{Bethe}}(C)\right);
\end{eqnarray*}
 The term inside of the minimisation on the right is the Bethe free
energy \citep{bethe}. By equating with the non-approximate likelihood,
it can be seen that the log partition function is being approximated
as: 
\[
\log Z(\Theta)=-\min_{C}\left(\left\langle \Theta,C\right\rangle -H_{\textrm{Bethe}}(C)\right).
\]
The value of $\log Z(\Theta)$ and a (locally) minimising $C$ can
be found efficiently using belief propagation \citep{cseke}. The
$C$ found by belief propagation is the gradient of $\log Z(\Theta)$,
so we can use a gradient based method on this approximate maximum
likelihood objective.

For diagonally dominant $\Theta$ belief propagation can be shown
to always converge \citep{WeissBP}. The diagonal constraints as well
as the non-positivity constraints on the off diagonal elements of
$\Theta$ ensure diagonal dominance in this case.

Maximum likelihood objectives for undirected graphical models are
typically optimised using quasi-Newton methods, and that is the approach
we took here. The diagonal constraints are easily handled by variable
substitution, and the non-positivity constraints are simple box constraints.
We used the L-BFGS-B algorithm \citep{lbfgsb} -- a quasi-Newton method
that supports such constraints. The log-partition function $\log Z(\Theta)$
is convex if we are able to exactly solve the inner minimisation over
$C$, which is not the case in general.

\section{Collaborative Filtering}

We will illustrate the application of our learning algorithm on the
collaborative filtering (CF) problem domain. CF is the sub-field of
recommendation systems that makes use of information from all users
that interact with a system in order to make personalised recommendations
for each individual user. Early recommendation systems treated each
user individually, but virtually all modern recommendation systems
fall within the collaborative filtering category. We will use the
standard setup, where a history of past \emph{ratings} are known for
each user. Each user $u$ is associated with a set $R_{K}$ of items,
for which the user's rating $r_{ui}$, $i\in R_{K}$ is known. Ratings
in our test problems are ordinal values in the range 1 to 5.

\section{The Item Graph}

\label{sec:item-graph}

We begin by introducing the foundations of graph based collaborative
filtering. We are given a set of users and items, along with a set
of real valued ratings of items by users. Classical item neighbourhood
methods \citep{item-methods} learn a graph structure over items $i=0\dots N-1$,
along with a set of edge weights $s_{ij}\in\mathbb{R}$, so that if
a query user $u$ is presented, along with his ratings for the neighbours
of item $i$ ($r_{uj},$ $j\in ne(i)$), the predicted rating of item
$i$ is 
\begin{equation}
r_{ui}=\mu_{i}+\frac{\sum_{j\in ne(i)}s_{ji}\left(r_{uj}-\mu_{j}\right)}{\sum_{j\in ne(i)}s_{ji}},\label{eq:cf-classical-rule}
\end{equation}
where $\mu_{i}\in\mathbb{R}$ represent average ratings for that item
over all users. When the rating of all neighbouring items is not known,
then the known subset is used in the prediction rule. See Figure \ref{fig:blade-runner}
for an example of an actual neighbourhood for a movie recommendation
system.

\begin{figure}[t]
\centering{}\includegraphics[width=1\columnwidth]{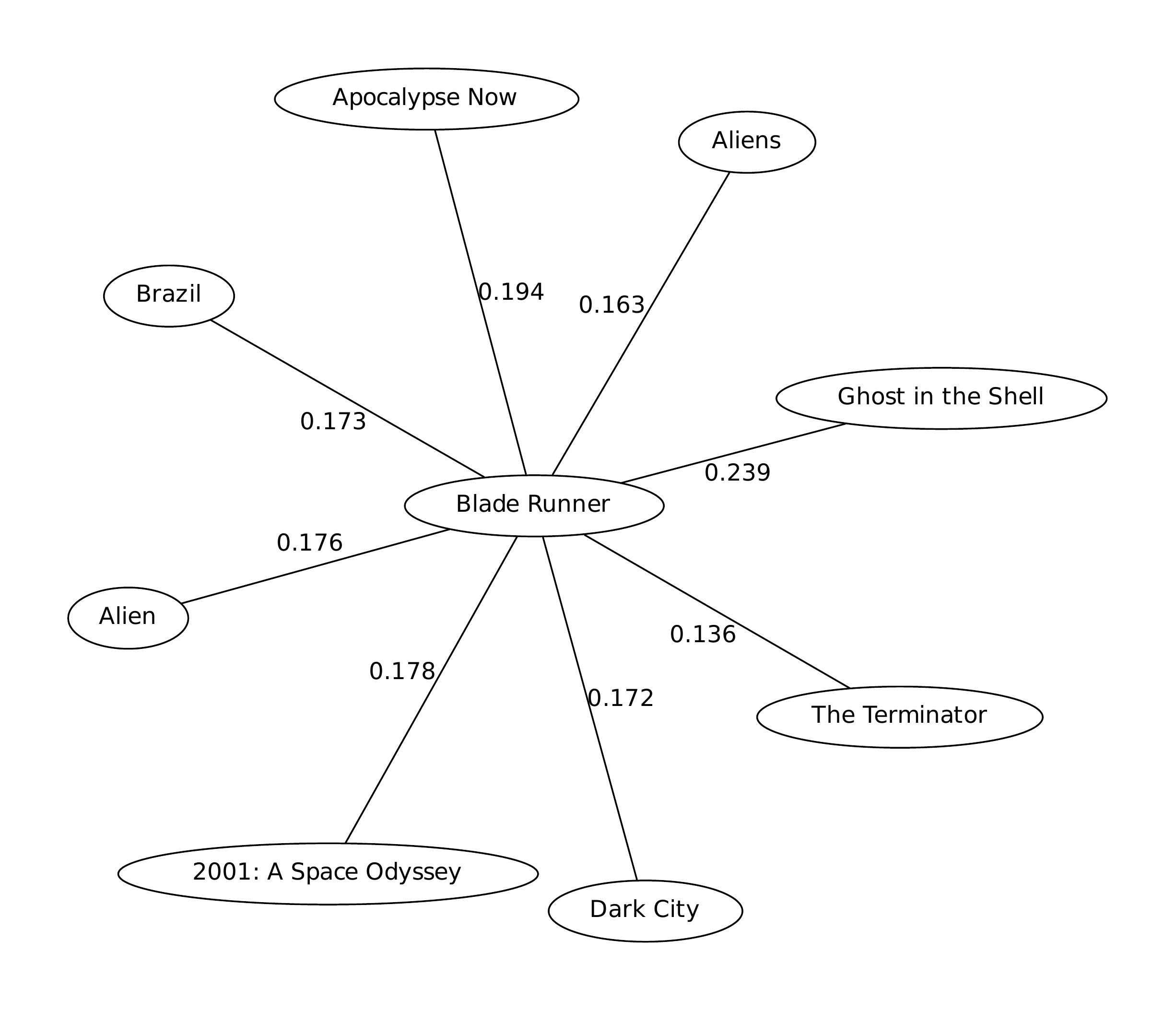}
\protect\caption{\label{fig:blade-runner}Neighbourhood of ``Blade Runner'' found
using the item field method (Section \ref{sec:item-field}) on the
MovieLens 1M dataset}
\end{figure}

In order to use the above method, some principle or learning algorithm
is needed to choose the neighbour weights. The earliest methods use
the Pearson correlation between the items as the weights. In our notation,
the Pearson correlation between two items is defined as 
\[
s_{ij}=\frac{\sum_{u}(r_{ui}-\mu_{i})(r_{uj}-\mu_{j})}{\sqrt{\sum_{u}(r_{ui}-\mu_{i})^{2}}\sqrt{\sum_{u}(r_{uj}-\mu_{j})^{2}}}.
\]
The set of neighbours of each item is chosen as the $k$ most similar,
under the same similarity measure used for prediction. More sophisticated
methods were developed for the NetFlix competition, including the
work of \citet{bellkor}, which identified the following problems
with the above: 
\begin{itemize}
\item Bounded similarity scores can not handle deterministic relations; 
\item Interactions among neighbours are not accounted for, which greatly
skews the results; 
\item The weights $s_{ij}$ cause over-fitting when none of the neighbours
provide useful information. 
\end{itemize}
\emph{Learning} the weights $s_{ij}$ under an appropriate model can
alleviate all of these problems, and provide superior predictions
\citep{bellkor}, with the only disadvantage being the computational
time required for training the model. Learning the neighbourhood structure
for such a model is not straightforward due to the potentially quadratic
number of edges. In this work we take the approach used by other neighbourhood
methods, and assume that the neighbourhood structure is chosen by
connecting each item to its $k$ most similar neighbours, using Pearson
correlation as the similarity measure. We denote this undirected edge
set $E$. The approach detailed in Chapter \ref{chap:approx-covsel}
would give a more principled way of choosing the edge set, but would
make comparison against results in the existing literature harder.

\subsection{Limitations of previous approaches}

The classical prediction rule in Equation \ref{eq:cf-classical-rule}
has a number of disadvantages. Suppose we wish to predict a rating
for a user $u$ on item $i$. If user $u$ has only rated a small
number (possibly 0) of the items neighbouring item $i$, then the
classical prediction rule gives poor results. Ideally in such cases
we want to use indirect information, say from distance 2 or 3 in the
graph, in order to make the predictions. Another limitation of the
classical approach is that it gives points estimates of the predicted
value. Ideally we would like to take a probabilistic approach, and
output a distribution over the possible values.

We show in the next section that the restricted Gaussian class of
models yields a model that fixes both these deficiencies. In fact
it gives a prediction rule that generalises the classical one. When
ratings for all items in a neighbourhood are known, it uses the classical
prediction rule with weights $s_{ji}=\Theta_{ji}$:

\begin{equation}
r_{ui}=\mu_{i}+\frac{\sum_{j\in ne(i)}s_{ji}\left(r_{uj}-\mu_{j}\right)}{\sum_{j\in ne(i)}s_{ji}}.\label{eq:local-pred}
\end{equation}

However, when fewer items are known the unknown items are marginalised
out in the probability model, yielding a well-founded prediction rule
that captures the additional uncertainty in such cases.

\section{The Item Field Model}

\label{sec:item-field}

We will now define a natural pairwise log-linear model over the item
graph structure. Recall that a log-linear model is defined by a set
of feature functions over subsets of variables, where the functions
return a local measure of compatibility. For pairwise models, these
subsets are just pairs of neighbouring variables in the graph structure.
In our case the set of variables is simply the set of items, whose
values we treat as continuous variables in the range $1$ to $5$.
For any particular user, given the set of their items ratings ($R_{K}$),
we will predict their ratings on the remaining items ($R_{U}$) by
conditioning on this distribution, namely computing expectations over
$P(R_{U}|R_{K})$.

Recall the form of a pairwise log-linear model (Section \ref{sec:log-linear-intro})
: 
\[
P(r;\Theta)=\frac{1}{Z(\Theta)}\exp\left(-\sum_{i\in V}\Theta_{ii}f_{i}(r_{i})-\sum_{(i,j)\in E}\Theta_{ij}f_{ij}(r_{i},r_{j})\right).
\]
Here $\Theta$ is a matrix of features weights, and $Z$ is the partition
function, whose value depends on the parameters $\Theta$, and whose
purpose is to ensure the distribution is correctly normalised. We
want functions that encourage smoothness, so that a discrepancy between
the ratings of similar items is penalised. We propose the use of (negated)
squared difference features: 
\begin{eqnarray*}
f_{ij}(r_{i},r_{j}) & = & -\frac{1}{2}((r_{i}-\mu_{i})-(r_{j}-\mu_{j}))^{2}.
\end{eqnarray*}
These features, besides being intuitive, have the advantage that for
any choice of parameters $\Theta$ we can form a Gaussian log-linear
model that defines that same distribution.

In a Gaussian log-linear model, pairwise features are defined for
each edge as $f_{ij}(r_{i},r_{j})=(r_{i}-\mu_{i})(r_{j}-\mu_{j})$,
and unary features as $f_{i}(r_{i})=\frac{1}{2}(r_{i}-\mu_{i})^{2}$.
The pairwise feature weights $\Theta_{ij}$ correspond precisely with
the off diagonal entries of the precision matrix (that is, the inverse
of the covariance matrix). The unary feature weights $\theta_{i}$
correspond then to the diagonal elements. Now notice that we can write
the squared difference features as:
\[
f_{ij}(r_{i},r_{j})=-\frac{1}{2}(r_{i}-\mu_{i})^{2}+(r_{i}-\mu_{i})(r_{j}-\mu_{j})-\frac{1}{2}(r_{j}-\mu_{j})^{2}.
\]
Thus, we can map the squared difference features to a constrained
Gaussian log-linear model, where the diagonal elements are constrained
so that for all $i$, 
\[
\Theta_{ii}=-\sum_{j\in ne(i)}\Theta_{ji}.
\]
As discussed in Section \ref{sec:cf-model-class}, we impose an additional
constraint on the allowable parameters, that each off-diagonal element
$\Theta_{ij}$ is non-positive. Our choice of squared-difference features
was also motivated by our model class restriction; As we just showed,
they yield an actively diagonally dominant $\Theta$ matrix.

\section{Prediction Rule}

\label{sec:prediction}

\begin{figure}[t]
\centering{}\includegraphics[width=1\columnwidth]{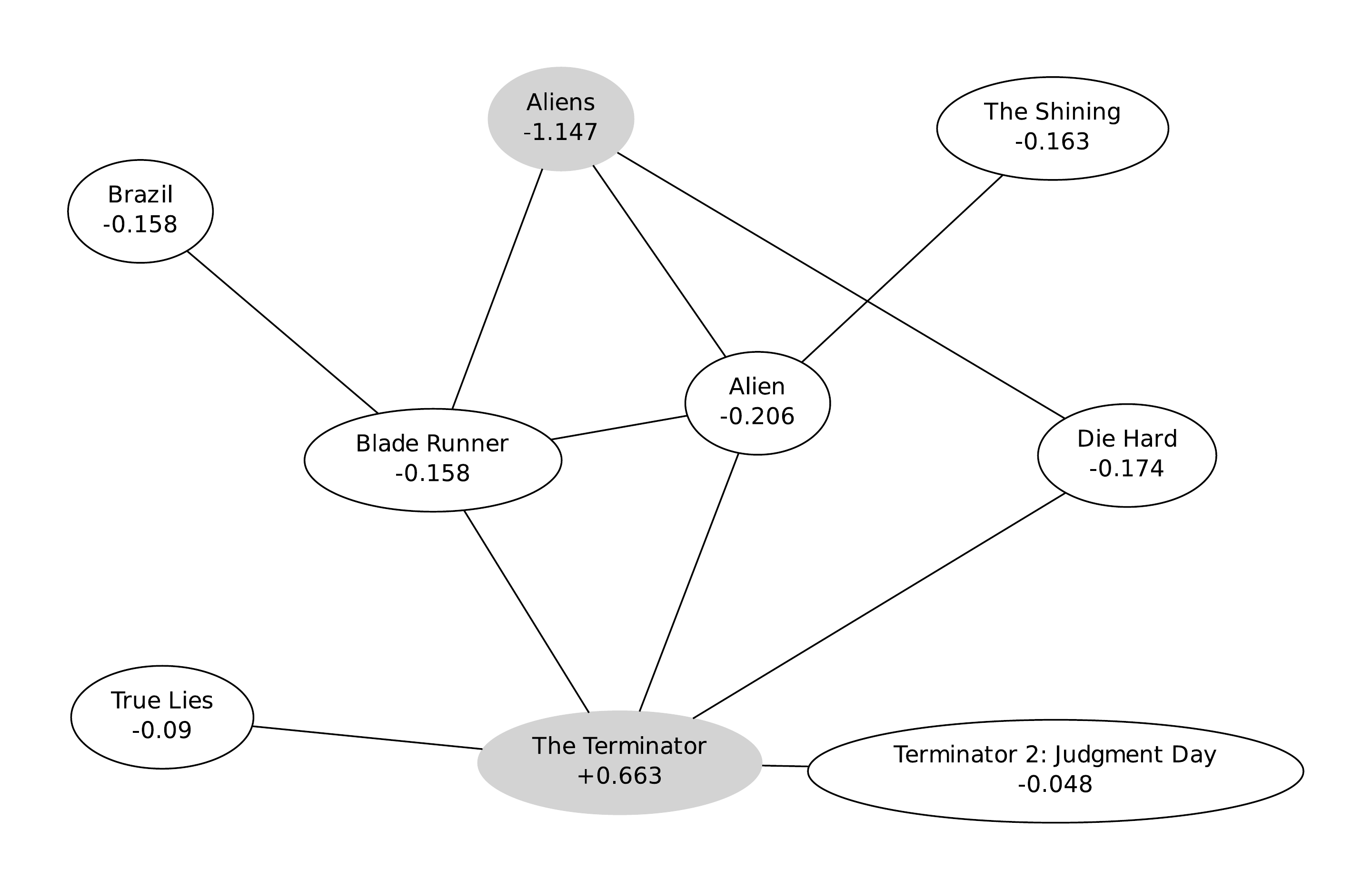}
\protect\caption{\label{fig:diffusion}Diffusion of ratings, shown as deltas from the
item means, for a user that has rated only the shaded items. A hand
chosen subgraph of the item graph generated for the 100K MovieLens
dataset is shown.}
\end{figure}

The feature functions defined in the previous section seem the natural
choice under our computational considerations. We now derive the associated
predictive distributions under that model. Consider the case of predicting
a rating $r_{ui}$, where for user $u$ all ratings $r_{uj}$, $j\in\textrm{ne}(i)$
are known. These neighbours form the Markov blanket of node $i$.
Using standard formulas for manipulating Gaussian distributions \citep[p86 Eq 2.73/2.75]{bishop},
we find that the conditional distribution under the item field model
is: 
\[
\mathcal{N}\left(r_{ui}\,;\,\mu_{i|-i},\frac{1}{\sigma^{2}}=\sum_{j\in ne(i)}\Theta_{ji}\right),\textrm{ where}
\]
\[
\mu_{i|-i}=\mu_{i}-\frac{\sum_{j\in ne(i)}\Theta_{ji}\left(r_{uj}-\mu_{j}\right)}{\sum_{j\in ne(i)}\Theta_{ji}}.
\]
This is a univariate Gaussian distribution, whose mean matches Equation
\ref{eq:local-pred}, the traditional neighbourhood method prediction
rule. In practice we rarely have ratings information for each item's
complete neighbourhood, so this special case is just for illustrating
the link with existing approaches.

In the general case, conditioning on a set of items $K$ with known
ratings $r_{K}$, with the remaining items denoted $U$, we have:
\[
\mathcal{N}\left(r_{U}\,;\,\mu_{U|K},\Theta_{UU}\right),\textrm{ where}
\]
\[
\mu_{U|K}=\mu_{U}-\left[\Theta_{UU}\right]^{-1}\Theta_{UK}\left(r_{K}-\mu_{K}\right).
\]
Thus computing the expected ratings requires nothing more than a few
fast sparse matrix operations, including one sparse solve. If the
prediction variances are required, both the variances and the expected
ratings can be computed using belief propagation, which often requires
fewer iterations than the sparse solve operation \citep{ISIT1}.

The linear solve in this prediction rule has the effect of diffusing
the known ratings information over the graph structure, in a transitive
manner. Figure \ref{fig:diffusion} shows this effect which is sometimes
known as spreading activation. Such transitive diffusion has been
explored previously for collaborative filtering, in a more ad-hoc
fashion \citep{associative-retrieval}.

Note that this prediction rule is a bulk method, in that for a particular
user, it predicts their ratings for all unrated items at once. This
is the most common case in real world systems, as all item ratings
are required in order to form user interface elements such as top
100 recommendation lists.

\section{Experiments}

\label{sec:colab-experiments}

For our comparison we tested on 2 representative datasets. The 1M
ratings MovieLens dataset\footnote{\url{http://grouplens.org/datasets/movielens/}}
consists of 3952 items and 6040 users. As there is no standard test/training
data split for this dataset, we took the approach from \citet{matchbox},
where all data for 90\% of the users is used for training, and the
remaining users have their ratings split into a 75\% training set
and 25\% test set. The 100K ratings MovieLens dataset involves 1682
items and 943 users. This dataset is distributed with five test/train
partitions for cross validation purposes which we made use of.

All reported errors use the mean absolute error (MAE) measure ($\frac{1}{N}\sum_{i}^{N}|\mu_{i}-r_{i}|$).
All 2 datasets consist of ratings on a discrete 1 to 5 star scale.
Each method we tested produced real valued predictions, and so some
scheme was needed to reduce the predictions to real values in the
interval $1$ to $5$. For our tests the values were simply clamped.
Methods that learn a user dependent mapping from the real numbers
into this interval have been explored in the literature \citep{matchbox}.

Table \ref{tab:results} shows the results for the two MovieLens datasets.
Comparisons are against our own implementation of a classical cosine
neighbourhood method \citep{item-methods}; a typical latent factor
model (similar to \citet{sfunk} but with simultaneous stochastic
gradient descent for all factors) and the neighbourhood method from
\citet{koren2010} (the version without latent factors), which uses
a non-linear least squares objective. All implementations were in
python (using cython compilation), so timings are not comparable to
fast implementations. We also show results for various methods of
training the item field model besides the maximum entropy approach,
including exact maximum likelihood training.

The item field model outperforms the other neighbourhood methods when
sparse (10 neighbour) models are used. Increasing the neighbourhood
size past roughly 10 actually starts to degrade the performance of
the item field model: at 50 neighbours using maximum entropy training
on the 1M dataset the MAE is 0.6866 v.s. 0.6772 at 10 neighbours.
We found this occurred with the other training methods as well. This
may be caused by an over-fitting effect as restricting the number
of neighbours is a form of regularisation.

The latent factor model and the least squares neighbourhood model
both use stochastic gradient descent for training. They required looping
over the full set of training data each iteration. The maximum entropy
method only loops over a sparse item graph each iteration which is
why it is roughly two thousand times faster to train. Note that the
dataset still has to be processed once to extract the neighbourhood
structure and covariance values, the timing of which is indicated
in the precomputation column. This is essentially the same for all
the neighbourhood methods we compared. In an on-line recommendation
system the covariance values can be updated as new ratings stream
in, so the precomputation time is amortised. Training time is more
crucial as multiple runs from a cold-start with varying regularisation
are needed to get the best performance (due to local minima).

\begin{table*}[t]
\centering{}%
\begin{tabular}{lccr}
 Method  & MAE  & Precomputation (s)  & Training (s) \tabularnewline
\hline 
100K MovieLens  &  &  & \tabularnewline
\cline{2-4} 
 \hspace*{2em} Maximum Entropy (k=10)  & 0.7384  & 18.9  & 0.12 \tabularnewline
\hspace*{2em} Cosine Neigh. (k=10)  & 0.8107  & 18.7  & 0 \tabularnewline
\hspace*{2em} Bethe Maximum Likelihood (k=10)  & 0.7390  & 18.9  & 28 \tabularnewline
\hspace*{2em} Exact Maximum Likelihood (k=10)  & 0.7398  & 18.9  & 99 \tabularnewline
\hspace*{2em} Least Squares Neighbours (k=10)  & 0.7510  & 18.9  & 90 \tabularnewline
\hspace*{2em} Maximum Entropy (k=50)  & 0.7439  & 19  & 1.7 \tabularnewline
\hspace*{2em} Least Squares Neighbours (k=50)  & 0.7340  & 19  & 288 \tabularnewline
\hspace*{2em} Latent Factor Model (50 factors)  & 0.7321  & 0  & 215 \tabularnewline
 1M MovieLens  &  &  & \tabularnewline
\cline{2-4} 
 \hspace*{2em} Maximum Entropy (k=10)  & 0.6772  & 514  & 0.64 \tabularnewline
\hspace*{2em} Cosine Neigh. (k=10)  & 0.7421  & 513  & 0 \tabularnewline
\hspace*{2em} Bethe Maximum Likelihood (k=10)  & 0.6767  & 514  & 75 \tabularnewline
\hspace*{2em} Exact Maximum Likelihood (k=10)  & 0.6755  & 514  & 2795 \tabularnewline
\hspace*{2em} Least Squares Neighbours (k=10)  & 0.6826  & 514  & 1369 \tabularnewline
\hspace*{2em} Maximum Entropy (k=50)  & 0.6866  & 517  & 7.17 \tabularnewline
\hspace*{2em} Least Squares Neighbours (k=50)  & 0.6756  & 517  & 4551 \tabularnewline
\hspace*{2em} Latent Factor Model (50 factors)  & 0.6683  & 0  & 4566 \tabularnewline
\hline 
\end{tabular}\protect\caption{\label{tab:results}Comparison of a selection of models against the
item field model with approximate maximum likelihood, exact maximum
likelihood and maximum entropy approaches}
\end{table*}

\section{Related Work}

\label{sec:colab-related-work}

There has been previous work that applies undirected graphical models
in recommendation systems. \citet{rbm} used a bipartite graphical
model, with binary hidden variables forming one part. This is essentially
a latent factor model, and due to the hidden variables, requires different
and less efficient training methods than those we apply in the present
paper. They apply a fully connected bipartite graph, in contrast to
the sparse, non-bipartite model we use. Multi-scale conditional random
fields models have also been applied to the more general social recommendation
task with some success \citep{multi-scale}. Directed graphical models
are commonly used as a modelling tool, such as in \citet{prob-factor}.
While undirected models can be used in a similar way, the graph structures
we apply in this work are far less rigidly structured.

Several papers propose methods of learning weights of a neighbourhood
graph \citep{koren2010} \citep{bellkor}, however our model is the
first neighbourhood method we are aware of which gives distributions
over its predictions. Our model uses non-local, transitive information
in the item graph for prediction. Non-local neighbourhood methods
have been explored using the concept of spreading activation, typically
on the user graph \citep{constrained-spreading} or on a bipartite
user-item graph \citep{bipartite-spreading}.

The closest probabilistic approach to ours is the work of \citealt{pref-networks}.
They form an undirected network with one node per user/item pair,
instead of the one node per item approach that characterises neighbourhood
methods. Instead of the squared difference features we use, they apply
absolute deviation features. For learning they apply a pseudo-likelihood
learning approach together with stochastic gradient descent.

\section{Extensions}

There are a few clear avenues for extension of the collaborative filtering
model we have described.

\subsection{Missing Data \& Kernel Functions}

\label{sec:kernel}

The training methods proposed take as input a sparse subset of a covariance
matrix $\Sigma$, which contains the sufficient statistics required
for training. It should be emphasised that we do not assume that the
covariance matrix is sparse, rather our training procedure only needs
to query the entries at the subset of locations where the precision
matrix is assumed to be non-zero.

As our samples are incomplete (we do not know all item ratings for
all users), the true covariance matrix is unknown. For our purposes,
we form a covariance matrix by assuming the unrated items are rated
at their item mean. More sophisticated methods of imputation are possible;
we explored an Expectation-Maximisation (EM) approach, which did not
result in a significant improvement in the predictions made. It did
however give better prediction covariances.

In general a kernel matrix can be used in place of the covariance
matrix, which would allow the introduction of item meta-data through
the kernel function. We left this avenue for future work.

\subsection{Conditional Random Field Variants}

\label{sec:crf} Much recent work in Collaborative filtering has concerned
the handling of additional user meta-data, such as age and gender
information usually collected by on-line systems \citep{matchbox}.
These attributes are naturally discrete, and so integrating them as
part of the MRF model results in mixed discrete/continuous model.
Approximate inference in such as model is no longer a simple linear
algebra problem, and convergence becomes an issue. User attributes
are better handled in a conditional random field (CRF) model, where
the conditional distributions involve the continuous item variables
only.

Unfortunately the optimisation technique described above does not
extend readily to CRF models. Approximate maximum entropy training
using Difference-of-convex methods has been applied to CRF training
successfully \citep{camel}, although such methods are slower than
maximum likelihood. We explored CRF extensions using maximum likelihood
learning, and while they did give better ratings predictions, training
was slow due to the large number of belief propagation calls. While
practical if the computation is distributed, the training time was
still several hundred times slower than any of the other methods we
tested.

\chapter{Conclusion and Discussion}

\label{chap:conclusion}

We have covered a number of new algorithms in this work that address
open problems in the intersection of machine learning and numerical
optimisation. Although these algorithms are of wide applicability,
we have focused on a limited set of experiments to validate each method,
as this is primarily a theoretical thesis. In this chapter we will
attempt to put our contributions in a wider context, discussing related
problems and remaining open problems.

\section{Incremental Gradient Methods}

\begin{figure}
\includegraphics[width=1\textwidth]{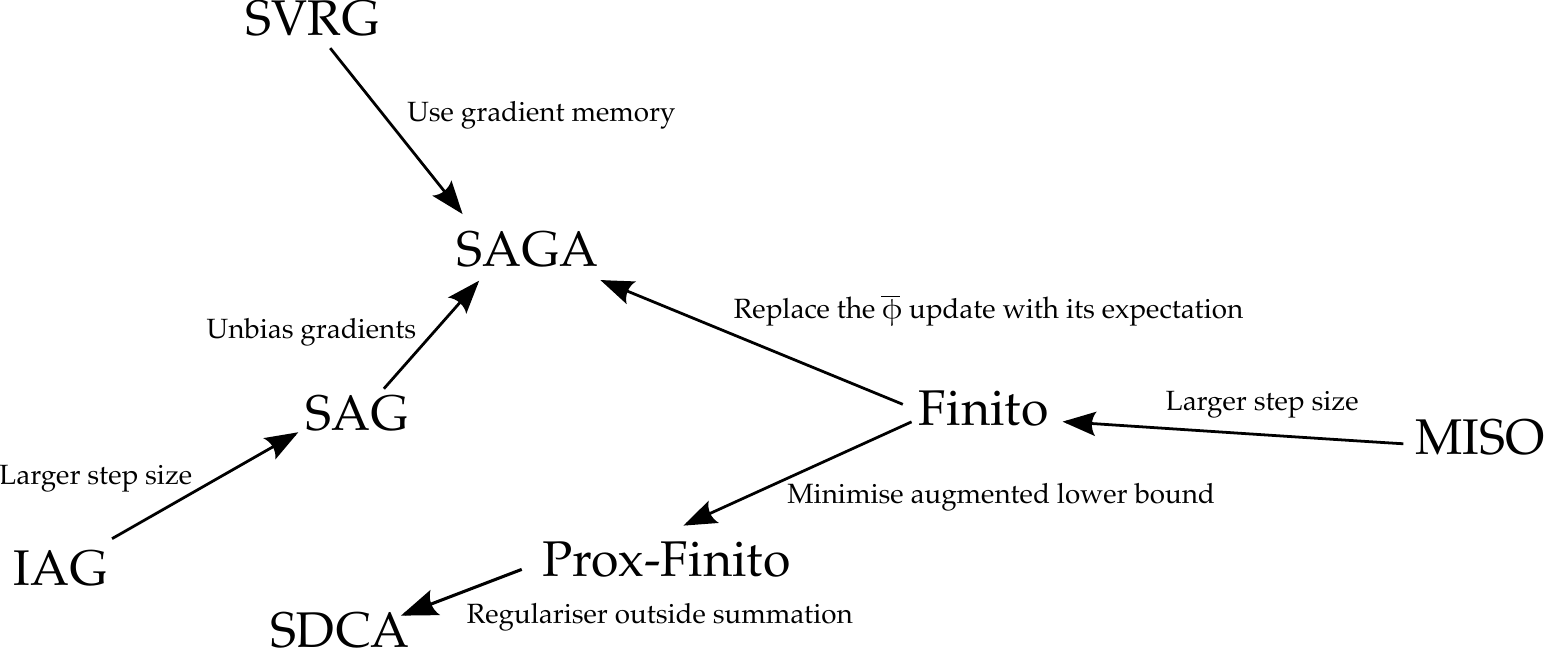}

\protect\caption{Incremental gradient methods}

\end{figure}

\subsection{Summary of contributions}

Our work on fast incremental gradient methods has tied together a
number of disparate methods. Our primary method, SAGA, has an interesting
role as a midpoint between two of the three major fast incremental
gradient methods. In relation to SAG, we showed that it is the natural
extension of SAG to having unbiased gradient estimates. The primary
problem with SAG is the difficulty of theoretical analysis of variants
of it. This has prevented the development of proximal and accelerated
versions of SAG. SAGA resolves this issue, having a much simpler and
more flexible analysis. In relation to SVRG, SAGA has been shown to
be a natural extension from its memory-less formulation to a version
with memory. The analysis techniques are similar but allow for improved
constants and a simplified convergence rate statement.

The SDCA method is not directly related to SAGA, but we have shown
an indirect relation, going from SAGA to Finito, then to Prox-Finito.
The later two methods were introduced in this work. This ties together
all the known fast incremental gradient methods as natural transformations
of each other.

The SAGA and Finito methods have a number of properties that make
them well suited to practical application. We established in our experiments
that none of the discussed methods has a clear advantage in all cases.
The choice in practice of what method to use depends more on the properties
of the methods than the actual speed of convergence. The Finito method
we introduced has a substantial advantage for dense valued problems
over the other methods. Unfortunately it is missing theory covering
the use of $L_{1}$ and other non-differentiable regularisers. The
existing SDCA method is favoured in strongly convex problems with
simple loss functions, but it is problematic (but not impossible)
to apply it on complex losses or non-strongly convex problems. The
SAGA method appears to be better than SAG in all situations covered
by the theory, but SAG can be sped-up by aggressive heuristics, whereas
such approaches have not yet be applied to SAGA. SAGA is the clear
method of choice for non-strongly convex problems. 

Besides the introduction of new incremental gradient methods, we have
also made contributions to their theoretical foundations. We discussed
the importance of strong convexity, and how it is the key that allows
faster convergence than traditional black-box methods. In finite sums,
for the purposes of optimisation the amount of strong convexity is
in a sense amplified by the number of terms in the sum. In convergence
rates $\mu n$ replaces $\mu$. In contrast, we established that without
additional restrictions non-strongly convex problems with summation
structures are no easier to optimise than the non-summation (black-box)
case.

We discussed in detail the importance of randomisation in fast incremental
gradient methods. The effect of randomisation appears to be a key
differentiator between fast and slow incremental gradient methods.
All known fast methods require randomisation to work, whereas slow
methods do not necessarily. For a number of methods, both fast and
slow methods are known that only differ by the step size used. Even
in those cases, the method will converge empirically and theoretically
without randomisation for small steps sizes, whereas only the randomised
variant will converge with large step sizes. We have provided a further
example of this, by establishing theoretically for the small-step
size case of Finito (MISO) a rate of a convergence for randomised
and non-randomised versions.

\subsection{Applications}

The fast incremental gradient methods we have described are primarily
of interest in statistics and machine learning problems, for a number
of reasons. Recall that the Lipschitz constant in their proven convergence
rates depends on the properties of each individual term $f_{i}$,
and so is different from the whole loss Lipschitz smoothness constant
$L_{f}$, which applies to the average $\frac{1}{n}\sum f_{i}$. The
applicability of fast incremental gradient methods depends on how
similar the full gradient Lipschitz constant is to the term one $L$.
It is when they are of comparable magnitude that methods like SAGA
are most effective. Problems in engineering and physics such as discretisations
of partial differential equations unfortunately do not usually have
the required uniformity. 

There is an interesting middle ground of problems for which both coordinate
descent methods and incremental gradient methods could be applicable.
These are the problems that have a summation structure where each
term in the summation only involves a small number of variables. An
example is MAP inference in graphical models, where terms correspond
to factors in the graphical model \citep[Sec. 13.5]{koller-friedman}.
We are not aware of any work applying fast incremental gradient methods
to such problems so far.

In machine learning most state-of-the-art methods use non-convex loss
minimisation objectives. For the primal incremental gradient methods
there is no principled reason why they can not be applied to such
problems. Any theoretical guarantees are likely to be weak though.
Some theoretical results appear in the literature for the MISO method
already \citep{miso2}. There is also some very preliminary experiments
apply the SVRG method to neural networks \citep{svrg}. The primary
difficulty with applying most fast incremental gradient methods is
the storage requirements. Other than SVRG, all the known methods require
storing a set of gradients, one per datapoint, which for neural networks
and other popular models can be on the order of thousands of times
larger than the input datapoints. Storing such gradients slows the
optimisation down by as much as 50 times\footnote{The slowdown is roughly proportional to the ratio of memory bandwidth
to computational throughput on the hardware used, sometimes called
the\emph{ machine balance\citep{balance}. }For reference, an Intel
i7 processor has roughly 40-100 gigaflops computational throughput
versus roughly 10 gigabytes per second memory bandwidth.} compared to SGD, making application fairly impractical. For SVRG,
the computational overhead is only between 2-3 times, making it an
attractive option that deserves further experimental validation.

We predict a large volume of literature in the future on the application
of various fast incremental gradient methods in areas outside of the
loss minimisation framework. The ease of implementation of the SAG/SAGA/SVRG/Finito
methods is likely to spur their future use, just as it has for the
classical stochastic gradient descent method. The SDCA method is not
well presented in the current literature, which may limit its future
application, although it is actually the most suited to be embedded
within software libraries due to the minimal tuning it requires. It
is already in use within the popular \emph{liblinear} library \citep{liblinear}
for instance.

\subsection{Open problems}

The fast incremental gradient framework suggests a number of interesting
open problems, a few of which have already been discussed in Section
\ref{sub:inc-open-prob}. The understanding of randomness is really
the key problem. All the currently know fast incremental gradient
methods require that data-points be accessed randomly both for their
theoretical and practical convergence rates. Is this a fundamental
limitation or can it be overcome? The current lower complexity bounds
are deterministic, but perhaps more sophisticated lower complexity
bounds are possible. The understanding of permuted (without-replacement)
random access rates is of equal interest. Even for classical SGD there
is no theory supporting the use of without-replacement sampling. For
some methods it gives faster convergence, and for some methods it
causes slower convergence. What property of these methods induces
such behaviour? It is likely new analytical tools will be needed to
fully understand such permuted orders, as the standard expected case
analysis techniques fail.

The randomness in each fast incremental gradient method is not necessarily
a problem in practice. When we have a large number of data-points,
the random effects are sufficiently averaged over that convergence
is quite predictable. Perhaps the biggest problem is an implementation
one. Modern hardware is much slower under random access patterns than
in-order patterns. Fortunately, for many large dense problems, the
latency induced by random access is of the order of 100 nanoseconds
\footnote{See \url{https://software.intel.com/en-us/articles/intelr-memory-latency-checker}},
and is small compared to the cost of the gradient evaluation at each
step, which can be >1000 nanoseconds. For sparse problems, there may
only be 50 or less active elements in a data-point vector, and so
that 100 nanosecond delay will result in the method being up-to 5
times slower than in-order access. The development of non-randomised
incremental gradient methods may fix this problem, but at the moment
the best solution is to use hardware data pre-caching using a separate
thread. As far as we are aware, for Intel processors this is only
possible on processors with hyper-threading, and is very difficult
to implement correctly.

The acceleration of incremental gradient methods is a major open problem.
Most optimisation methods can be accelerated using a double-loop procedure,
and this is approach taken in the ASDCA method. We believe the same
techniques are applicable to the other fast incremental gradient methods.
Such double-loop constructions are unsatisfying though. They are particularly
fragile to the values of the Lipschitz and strong convexity parameters,
and they introduce a $\log$ term to the convergence rate which is
almost certainly improvable. The development of simple single loop
primal incremental gradient methods is the next logical step, and
we expect so see them appear over the next few years.

The requirement of the knowledge of the strong convexity or Lipschitz
constants is an issue with incremental gradient methods. This is not
an issue with regular gradient descent, where a line-search can be
used instead of knowledge of the Lipschitz smoothness constant. Gradient
descent is also automatically adaptive to the level of strong convexity
of the problem, just like SAG and SAGA. The accelerated gradient descent
method can also be made to work without prior knowledge of the Lipschitz
and strong convexity constants, however a rather unsatisfying procedure
involving keeping estimates of the strong convexity constant during
the course of optimisation is required \citep{doublelinesearch}.
While such an approach does work, an approach that doesn't explicitly
estimate $\mu$ would be a huge improvement.

It would be a remarkable development to see a primal fast incremental
gradient method that didn't require knowledge of the Lipschitz constant
or estimation of it. Likewise, for the dual case a method that doesn't
require explicit knowledge of the amount of regularisation (or running
estimates of it) would be an improvement.

We have largely ignored the case of non-differentiable loss functions
in this work. The case of non-differentiable regularisers is easy
to cover using proximal methods, such as we do for the SAGA method,
however for primal methods at least, it is not clear how to handle
non-differentiable losses. For dual methods, the non-differentiable
case is quite straight-forward, and a $O(1/k)$ rate is established
for the convergence of SDCA on $L_{2}$ regularised problems. This
is the same rate known for black-box methods. We would like to see
a primal method that makes use of the proximal operator of the loss
in order to achieve a similar rate, likely using techniques from the
theory of mirror-descent and primal-dual optimisation.

Distributed optimisation has been a hot topic in optimisation research
in the last few years. The theory of distributed SGD has greatly expanded,
both with its use in gigantic applications and with theoretical advances.
There has been some preliminary results on the application of SDCA
in a distributed setting \citep{dist-SDCA}. It is not clear in what
settings an improvement in convergence rate is possible for incremental
gradient methods, especially primal ones. The theoretical analysis
is complicated by the need for a model of distributed machines, which
tends to require specific assumptions about hardware, latency and
network topology. A simple map-reduce style framework is provably
not sufficient.

\section{Learning Graph Models}

\subsection{Summary of contributions}

This thesis introduced new methods for structure and parameter learning
for graphical models. We primarily focused on the Gaussian graphical
model case, where the goal is to learn the direct dependencies between
variables in a dataset in the form of a graph structure over the variables.
For the task of learning the structure, we proposed the SHORTCUT method,
a modification of the existing Conditional Covariance Thresholding
(CCT) method which has an improved expected case running time, reduced
from $O(n^{3})$ to $O(n^{2.5})$.

Given a fixed known structure, we also proposed a new method for learning
the parameters (edge weights) in a model. It uses the Bethe approximation,
together with dual decomposition followed by coordinate descent. Our
formulation is careful to restrict the class of models that can be
learned to those for which later inference tasks using the model are
tractable.

For more advanced structure learning, we proposed a method that can
make use of prior knowledge that the true graph structure is scale-free.
This method makes use of sophisticated submodular regularisation techniques
to ensure the resulting optimisation problem is convex. 

In combination, the methods we proposed comprise are a powerful tool-set
for Gaussian graphical model learning. We have addressed both very
large scale learning via approximations as well as improvements to
the reconstruction quality for smaller scale models.

\subsection{Applications}

Learning graph structures is common outside of machine learning for
the purpose of data analysis. Graph structures are one of the primary
ways to visualise the dependencies between variables in a dataset.
A common application of Gaussian graphical model learning is in bioinfomatics,
where the interactions between gene expressions are modelled as a
graph structure (Davidson and Levin, 2005) \nocite{genenetworks}.
We included such a dataset in our experiments. Another potential application
is in modelling social network structures which are not observed directly.
For instance, the propagation of reposts of news articles between
sites can be modelled in such as way \citep{info-diffusion}. 

Graph structure learning can sometimes be used in place of clustering.
When a strong $L_{1}$ regulariser is used, the resulting graph can
be unconnected, with the components playing the roles of clusters.
As an example, in political science the voting decisions of politicians
are sometimes modelled as a unconnected graph structure \citep{ussenate},
and our methods could be applied to such networks as well.

We tested our parameter learning method on a recommendation system
problem, where we modelled a set of movies with a graph structure.
The edge weight between two movies $(i,j)$ indicated how predictive
the rating for movie $i$ by a user is on for their rating of movie
$j$, and vice-versa. Another approach to recommendation is to instead
form a graph over users instead of items. Such user graph approaches
were used in some of the first recommendation systems, but when large
amounts of item ratings information is available, they proved less
effective than item graph approaches. User graph approaches are making
a resurgence lately \citep{graph-social-rec}, as we now have known
graph structures such as the user graph in FaceBook that previously
we had to estimate. Additionally, the user graph structure has greater
utility when we don't have large numbers of numerical ratings for
items. For instance, the small numbers of Boolean ratings that result
from ``likes'' on social network posts \citep{sanner-colab}. Our
parameter learning algorithm could potentially be used for large-scale
learning on such networks.

\subsection{Open Problems}

Despite their basis in well understood Gaussian distribution theory,
Their are still many open problems relating to Gaussian graphical
models. Scaling the $L_{1}$ regularised maximum likelihood learning
problem up to larger problem instances is still an area of active
research. In theory fast matrix multiplication methods can be used
to also speed-up the inversion problem, but it is not clear if such
approaches will ever lead to practical algorithms. This would probably
require further advances towards the conjectured $O(n^{2+\epsilon})$
rate for fast matrix multiplication to occur.

It is well known that the Bethe approximation we apply for parameter
learning can be extended to an approximation over higher-order cliques.
It would be interesting to see if such an extension resulted in better
results in applications such as the recommender system example we
use. The dual decomposition approach may not be a practical approach
with higher-order cliques. 

We briefly mentioned extensions to CRF variants of the recommender
system ``item field'' model. The traditional \emph{collaborative
filtering} datasets we had access to did not have a lot of additional
information that could have been conditioned on, but in the \emph{social}
recommendation space large amounts of biographic information are often
available for each user. Extending our dual decomposition approach
to conditional models is an interesting possibility for future extension.

For our scale free network learning method, we make use of exponential
random graph models. These are still not well understood. Indeed,
just determining the expected degree distribution under an ERG model
currently requires the use of a slow MCMC sampling procedure. ERG
models are still a subject of active research, and we expect that
the classes of ERG models that induce scale free structures will be
better understood in the near future.

The ADMM learning procedure we use in our scale-free formulation is
not known to handle large problems well. In theory it should be possible
to adapt the current large-scale state-of-the-art Newton methods to
use my scale-free inducing regulariser instead of an $L_{1}$ regulariser.
It is not clear if this can be done without losing efficiently though.

\appendix

\chapter{Basic Convexity Theorems}

\label{chap:appendix-convexity}

The theory of convex analysis is large and we make no attempt to cover
it in depth in this appendix. We can get away with this sin as our
theory in the preceding chapters only concerns convex functions defined
on the whole of $\mathbb{R}^{d}$, mapping to the (non-extended) real
numbers $\mathbb{R}$. Such functions are particularly well behaved
as most of the complexity of analysis of convex functions occurs on
the boundary of their domain, or in regions where their values are
$\pm\infty$. For example, all such functions are continuous \citep[Corollary 10.1.1, ][]{rockafellar}.
Most of the theorems in this section are standard; we provide references
to textbooks for these results.

\section{Definitions}

\label{sec:appendix-defs}
\begin{itemize}
\item We use the angle bracket notation $\left\langle x,y\right\rangle $
to denote the inner product between two vectors $x$ and $y$.
\item The \textbf{proximal operator} $\text{prox}_{\gamma}^{f}(v)\colon\mathbb{R}^{d}\times\mathbb{R}^{d}$
is defined as: 
\[
\text{prox}_{\gamma}^{f}(v)=\min_{x}\left\{ \gamma f(x)+\frac{1}{2}\left\Vert x-v\right\Vert ^{2}\right\} .
\]

\item A function $f$ is \textbf{convex} if for all $x,y\in\mathbb{R}^{d}$
and $\alpha\in[0,1]$:
\[
f\left(\alpha x+(1-\alpha)y\right)\leq\alpha f(x)+(1-\alpha)f(y).
\]
Additionally, $f$ is \textbf{strictly convex} if this inequality
is strict.
\item A function $f$ is \textbf{strongly convex} with constant $\mu\geq0$
if for all $x,y\in\mathbb{R}^{d}$ and $\alpha\in[0,1]$:
\[
f\left(\alpha x+(1-\alpha)y\right)\leq\alpha f(x)+(1-\alpha)f(y)-\alpha\left(1-\alpha\right)\frac{\mu}{2}\left\Vert x-y\right\Vert ^{2}.
\]
when we say a function is strongly convex colloquially, we usually
mean that $\mu>0$. Typically, theorems about strongly convex functions
also hold for $\mu=0$, but sometimes become vacuous statements. 
\item A function $f$ has \textbf{Lipschitz continuous gradients} with constant
$L$, or equivalently is \textbf{$L$-smooth} if it is differentiable
and for all $x,y\in\mathbb{R}^{d}$:
\[
\left\Vert f^{\prime}(x)-f^{\prime}(y)\right\Vert \leq L\left\Vert x-y\right\Vert .
\]
We will virtually always make this assumption when working with differentiable
convex functions. 
\item The \textbf{condition number} of $f$ is the ratio $L/\mu$ of its
smoothness and strong convexity constants.
\item The \textbf{directional derivative} of a function at $x\in\mathbb{R}^{d}$
in direction $v\in\mathbb{R}^{d}$ is the quantity:
\[
\Delta_{v}f(x)=\lim_{h\rightarrow0}\frac{f(x+hv)-f(x)}{h}.
\]
This is related to the gradient through the relation $\Delta_{v}f(x)=\left\langle f^{\prime}(x),v\right\rangle $.
\end{itemize}

\section{Useful Properties of Convex Conjugates}

\label{sec:convex-conj}

The convex conjugate $f^{*}\colon\mathbb{R}^{d}\rightarrow\mathbb{R}^{+}$
of a convex function $f\colon\mathbb{R}^{d}\rightarrow\mathbb{R}$
is defined as:
\[
f^{*}(u)=\text{max}_{x}\left\{ \left\langle u,x\right\rangle -f(x)\right\} .
\]

We will call the set of points $u$ such that $f^{*}(u)<\infty$ the
dual space $X^{*}$ of $f$. The derivatives of a function are also
in $X^{*}$, so we can often think of dual points $u$ with finite
$f^{*}(u)$ as being derivatives of $f$ at some (unknown) location.
We can avoid the use of the extended real numbers if we restrict ourselves
to $X^{*}$ when evaluating $f^{*}$. Note also in our definition
of the convex conjugate we have used $\max$ instead of $sup$, as
they are equivalent for convex $f$ with the domain $\mathbb{R}^{d}$.

The following properties of convex conjugates are particularly useful:
\begin{enumerate}
\item $f^{*}$ is convex.
\item \textbf{Fenchel\textendash Moreau theorem} If $f$ is convex and continuous,
then the conjugate-of-the-conjugate (known as the biconjugate) is
the original function: 
\[
f=f^{**}.
\]

\item \label{enu:conj-grad-mapping}\textbf{The Legendre transform property}
For strictly convex differentiable functions, the gradient of the
convex conjugate maps a point in the dual space into the point at
which it is the gradient of. I.e.
\begin{equation}
f^{*\prime}(\left(f^{\prime}(x)\right)=x.\label{eq:Legendre}
\end{equation}
This result is quite surprising at first glance. It often lets one
transform proofs involving a function into a result about it's conjugate.
Note that when strict convexity does not hold, the conjugate may not
be differentiable at the given dual point $f^{\prime}(x)$. The value
$x$ is still in the sub-differential though. For general differentiable
functions $f$, a function $g$ whose gradient is the inverse map
of $f^{\prime}$ is known as a Legendre transform of $f$. For strictly
convex problems, Equation \ref{eq:Legendre} implies that $g(x)=f^{*}(x)$
for all $x\in X^{*}$, i.e. The Legendre transform is just the convex
conjugate restricted to $X^{*}$ where it is finite.\\
Another way of interpreting this result is to see the gradient and
the conjugate's gradient as the mappings to and from the dual space
respectively for a given function.
\item \textbf{Fenchel-Young Inequality} Given a dual point $u$ and a primal
point $x$, it holds that:
\[
f(x)+f^{*}(u)\geq\left\langle x,u\right\rangle .
\]
This inequality becomes an equality when $u=f^{\prime}(x)$, or for
non-differentiable functions, if $u$ is any subgradient at $x$:
\[
f(x)+f^{*}(f^{\prime}(x))=\left\langle f^{\prime}(x),x\right\rangle .
\]
This inequality is useful for computing the conjugate function value.
Together with the previous inequality, it is the main tool used to
convert proofs to their conjugate version.
\item \textbf{Moreau decomposition} For convex $f$, the proximal operator
of the convex conjugate function $f^{*}$ can be computed trivially
from the function's proximal operator using the following equality:
\[
\text{prox}_{\gamma}^{f^{*}}(x)=x-\gamma\text{prox}_{1/\gamma}^{f}(x/\gamma).
\]
This result is surprising, as the conjugate function value is not
necessarily this easy to compute.
\item \label{enu:conjugate-condition}If $f$ has Lipschitz continuous gradients
with constant $L$, then $f^{*}$ is strongly convex with constant
$1/L$. Likewise if $f$ is strongly convex with constant $\mu$,
then $f^{*}$ is Lipschitz smooth with constant $1/\mu$. It follows
that the condition number of $f^{*}$ is the same as $f$, namely
$L/\mu$.
\end{enumerate}

\section{Types of Duality}

\label{sec:duality}

The main use of convex conjugates is in forming the \emph{dual} \emph{problem}
of a given convex optimisation problem. This terminology can be ambiguous,
as there are several types of duals can potentially be formed. Here
we review three types:
\begin{description}
\item [{Fenchel duality}] Given a minimisation problem of the form $f(x)-g(x)$,
where $f$ is convex and $g$ is concave, the Fenchel dual problem
is:
\[
\max_{u}\left\{ g^{*}(u)-f^{*}(u)\right\} .
\]

\item [{Wolfe duality}] Given a constrained minimisation problem for the
function $f$ with constraints $g_{i}(x)\leq0$ for each $i=1\dots m$,
the Wolfe dual is the following:
\[
\max_{x,u_{i}}f(x)+\sum_{j}^{m}u_{j}g_{j}(x),
\]
\[
\text{s.t.}\;f^{\prime}(x)+\sum_{j}^{m}u_{j}g_{j}(x)=0,
\]
\[
u_{i}\geq0,\quad i=1,\dots,m.
\]
This is the least practical of the discussed duals to work with. The
objective is biconvex, and the constraints potentially non-convex.
\item [{Lagrangian duality}] Given a constrained minimisation problem
for the function $f$ with constraints $g_{i}(x)\leq0$ for each $i=1\dots m$,
the Lagrange dual is the following:
\[
\max_{u}\;\min_{x}\left(f(x)+\sum_{j}^{m}u_{j}g_{j}(x)\right),
\]
\[
\text{s.t. }u_{i}\geq0,\quad i=1,\dots,m.
\]
When applying this dual, the inner minimisation is usually solved
analytically.
\end{description}

\section{Properties of Differentiable Functions}

When proving properties about convex functions, it is cumbersome to
work directly with the definitions of differentiability in terms of
limits. Instead, we can often use one of the following higher-level
theorems.
\begin{thm}
\textbf{(Second fundamental theorem of Calculus}). Let $f\colon\mathbb{R}\rightarrow\mathbb{R}$
be a differentiable function. Then:
\[
f(y)=f(x)+\int_{x}^{y}f^{\prime}(z)dz.
\]
\end{thm}
\begin{cor}
\label{cor:calculus-corollary} Since we are almost exclusively concerned
with vector valued functions, we tend to use this theorem along intervals
in a vector space. i.e. let $x,y\in\mathbb{R}^{d}$ and $f\colon\mathbb{R}^{d}\rightarrow\mathbb{R}$,
then
\begin{equation}
f(y)=f(x)+\int_{0}^{1}\left\langle f^{\prime}(x+\tau(y-x))\,,\,y-x\right\rangle d\tau.\label{eq:ftc-interval}
\end{equation}

The theorem is also useful when used for functions of the form $f\colon\mathbb{R}^{d}\rightarrow\mathbb{R}^{d}$,
such as derivatives, where it is applied over an interval, but separately
for each coordinate. For example, assuming $f$ is twice differentiable:
\[
f^{\prime}(y)=f^{\prime}(x)+\int_{0}^{1}f^{\prime\prime}\left(x+\tau(y-x)\right)\left(y-x\right)d\tau.
\]
 This appears in proofs of local convergence of Newton's method, where
$x=x^{*}$, and $y$ is the current iterate $x_{k}$. \end{cor}
\begin{thm}
\textbf{(Mean value theorem)}. Let $f\colon\mathbb{R}\rightarrow\mathbb{R}$
be a continuously differentiable function and $a,b\in\mathbb{R}$
with $a<b$. Then there exists a $c\in(a,b)$ such that:
\[
f^{\prime}(c)=\frac{f(b)-f(a)}{b-a}.
\]

\end{thm}

\section{Convexity Bounds}

This section gives full proofs of a number of basic results in convexity
theory. These proofs often appear in the literature in condensed form,
or are omitted as exercises. Most theorems in this section can also
be proved in the twice differentiable case by arguments involving
the Hessian.

\subsection{Taylor like bounds}
\begin{thm}
\textup{\label{thm:convexity-lb} }\citep[Thm 2.1.8]{nes-book} \textup{Let
$f$ be $\mu\geq0$ strongly convex and differentiable. Then for all
$x,y\in\mathbb{R^{d}}$:}
\[
f(y)\geq f(x)+\left\langle f^{\prime}(x),y-x\right\rangle +\frac{\mu}{2}\left\Vert x-y\right\Vert ^{2}.
\]

When $\mu=0$, this is the familiar convexity lower bound $f(y)\geq f(x)+\left\langle f^{\prime}(x),y-x\right\rangle $,
which applies for general convex functions. \end{thm}
\begin{proof}
From the definition of strong convexity $f\left(\alpha x+(1-\alpha)y\right)\leq\alpha f(x)+(1-\alpha)f(y)-\alpha\left(1-\alpha\right)\frac{\mu}{2}\left\Vert x-y\right\Vert ^{2}$
we have:
\[
f\left(\alpha x+(1-\alpha)y\right)-\alpha f(x)+\alpha\left(1-\alpha\right)\frac{\mu}{2}\left\Vert x-y\right\Vert ^{2}\leq(1-\alpha)f(y),
\]
\begin{eqnarray}
\therefore f(y) & \geq & \frac{1}{1-\alpha}\left[f\left(\alpha x+(1-\alpha)y\right)-\alpha f(x)\right]+\alpha\frac{\mu}{2}\left\Vert x-y\right\Vert ^{2}\nonumber \\
 & = & f(x)+\frac{1}{1-\alpha}\left[f\left(\alpha x+(1-\alpha)y\right)-f(x)\right]+\alpha\frac{\mu}{2}\left\Vert x-y\right\Vert ^{2}\nonumber \\
 & = & f(x)+\frac{1}{1-\alpha}\left[f\left(x+(1-\alpha)(y-x)\right)-f(x)\right]+\alpha\frac{\mu}{2}\left\Vert x-y\right\Vert ^{2}.\label{eq:convex-dd}
\end{eqnarray}

Now recall the definition of the directional derivative as a limit:
\[
\Delta_{v}f(x)=\lim_{h\rightarrow0}\frac{f(x+hv)-f(x)}{h}.
\]
The second term of our expression \ref{eq:convex-dd} has this form
for $h=1-\alpha$ and $v=y-x$, so we take the limit as $\alpha\rightarrow1$,
and use the fact that directional derivatives obey $\Delta_{v}f(x)=\left\langle f^{\prime}(x),v\right\rangle $:
\[
\lim_{\alpha\rightarrow1}\frac{1}{1-\alpha}\left[f\left(x+(1-\alpha)(y-x)\right)-f(x)\right]=\Delta_{y-x}f(x)=\left\langle f^{\prime}(x),y-x\right\rangle ,
\]

We also need to take the limit of the term $\alpha\frac{\mu}{2}\left\Vert x-y\right\Vert ^{2}$
as $\alpha\rightarrow1$, which is trivially $\frac{\mu}{2}\left\Vert x-y\right\Vert ^{2}$.
Leaving us with:
\[
f(y)\geq f(x)+\left\langle f^{\prime}(x),y-x\right\rangle +\frac{\mu}{2}\left\Vert x-y\right\Vert ^{2}.
\]
\end{proof}
\begin{thm}
\textbf{\label{thm:lipschitz-ub}(Lipschitz upper bound) }\citep[Thm 2.1.5]{nes-book}
If $f$ is convex and $L$-smooth, then for all $x,y\in\mathbb{R}^{d}$:
\[
f(y)\leq f(x)+\left\langle f^{\prime}(x),y-x\right\rangle +\frac{L}{2}\left\Vert x-y\right\Vert ^{2}.
\]
\end{thm}
\begin{proof}
Using the second fundamental theorem of calculus, in the interval
form of Equation \ref{eq:ftc-interval}, we have:
\begin{eqnarray*}
f(y) & = & f(x)+\int_{0}^{1}\left\langle f^{\prime}(x+\tau(y-x))\,,\,y-x\right\rangle d\tau\\
 & = & f(x)+\left\langle f^{\prime}(x),y-x\right\rangle +\int_{0}^{1}\left\langle f^{\prime}(x+\tau(y-x))-f^{\prime}(x)\,,\,y-x\right\rangle d\tau\\
 & \leq & f(x)+\left\langle f^{\prime}(x),y-x\right\rangle +\int_{0}^{1}\left\Vert f^{\prime}(x+\tau(y-x))-f^{\prime}(x)\right\Vert \left\Vert x-y\right\Vert d\tau\;\text{(Cauchy-Schwarz)}\\
 & \leq & f(x)+\left\langle f^{\prime}(x),y-x\right\rangle +\int_{0}^{1}\tau L\left\Vert x-y\right\Vert ^{2}d\tau\quad\text{(Lipschitz condition)}\\
 & = & f(x)+\left\langle f^{\prime}(x),y-x\right\rangle +\frac{L}{2}\left\Vert x-y\right\Vert ^{2}.
\end{eqnarray*}

\end{proof}

\subsection{Gradient difference bounds}
\begin{thm}
\label{thm:strong-ub}If $f$ is strongly convex with constant $\mu$,
then for all $x,y\in\mathbb{R}^{d}$:
\[
f(y)\leq f(x)+\left\langle f^{\prime}(x),y-x\right\rangle +\frac{1}{2\mu}\left\Vert f^{\prime}(x)-f^{\prime}(y)\right\Vert ^{2}.
\]
\end{thm}
\begin{proof}
Let $u$ and $v$ be points in the dual space of $f$. As noted in
point \ref{enu:conjugate-condition} of the properties of convex conjugates
in Section \ref{sec:convex-conj}, $f^{*}$ has Lipschitz continuous
gradients with constant $1/\mu$. So: 
\[
f^{*}(v)-f^{*}(u)-\left\langle f^{*\prime}(u),u-v\right\rangle \leq\frac{1}{2\mu}\left\Vert u-v\right\Vert ^{2}.
\]
Now take $v=f^{\prime}(x)$ and $u=f^{\prime}(y)$. We replace each
$f^{*}$ with $f$ using The Fenchel-Young equality $f(x)+f^{*}(f^{\prime}(x))=\left\langle f^{\prime}(x),x\right\rangle $:
\[
\left\langle f^{\prime}(x),x\right\rangle -f(x)-\left\langle f^{\prime}(y),y\right\rangle +f(y)-\left\langle f^{*\prime}(f^{\prime}(y)),f^{\prime}(x)-f^{\prime}(y)\right\rangle \leq\frac{1}{2\mu}\left\Vert f^{\prime}(x)-f^{\prime}(y)\right\Vert ^{2}.
\]
The remaining dual quantity obeys $f^{*\prime}(f^{\prime}(y))=y$
(See point \ref{enu:conj-grad-mapping} in Section \ref{sec:convex-conj}),
so we can simplify further to get the result:

\[
f(y)-f(x)+\left\langle f^{\prime}(x),x-y\right\rangle \leq\frac{1}{2\mu}\left\Vert f^{\prime}(x)-f^{\prime}(y)\right\Vert ^{2}.
\]
\end{proof}
\begin{rem}
This bound in full generality rarely appears in the literature. It
is more often seen in the special case of $x=x^{*}$ and $y=x$ in
the form:
\[
\left\Vert f^{\prime}(x)\right\Vert ^{2}\geq2\mu\left[f(x)-f(x^{*})\right],
\]
which is used in proving convergence of gradient descent for strongly
convex functions in terms of function value.\end{rem}
\begin{thm}
\label{thm:lipschitz-lb}\citep[Thm 2.1.5]{nes-book} If $f$ is convex
and $L$-smooth, then for all $x,y\in\mathbb{R}^{d}$:
\[
f(y)\geq f(x)+\left\langle f^{\prime}(x),y-x\right\rangle +\frac{1}{2L}\left\Vert f^{\prime}(x)-f^{\prime}(y)\right\Vert ^{2}.
\]
\end{thm}
\begin{proof}
Define $g(y)=f(y)-\left\langle f^{\prime}(x),y\right\rangle .$ Then
$g$ is convex and has gradient $g^{\prime}(y)=f^{\prime}(y)-f^{\prime}(x)$.
Clearly the gradient is zero at $x$, so its minimiser is $y^{*}=x$.
Now consider the gradient step $y-\frac{1}{L}g^{\prime}(y)$. We apply
the Lipschitz upper bound (Theorem \ref{thm:lipschitz-ub}) to $g$
at $y-\frac{1}{L}f^{\prime}(y)$ around $y$.
\begin{eqnarray*}
g\left(y-\frac{1}{L}g^{\prime}(y)\right) & \leq & g(y)+\left\langle g^{\prime}(y),y-\frac{1}{L}g^{\prime}(y)\right\rangle +\frac{L}{2}\left\Vert -\frac{1}{L}g^{\prime}(y)\right\Vert ^{2}\\
 & = & g(y)-\frac{1}{2L}\left\Vert g^{\prime}(y)\right\Vert ^{2}.
\end{eqnarray*}
We know that $y^{*}$ gives as least as small a function value as
any other point, i.e. $g(y^{*})\leq g(y-\frac{1}{L}g^{\prime}(y))$.
Given that $y^{*}=x$, we thus have:
\[
g(x)\leq g(y)-\frac{1}{2L}\left\Vert g^{\prime}(y)\right\Vert ^{2}.
\]
Therefore, by plugging in $f$:
\[
f(x)-\left\langle f^{\prime}(x),x\right\rangle \leq f(y)-\left\langle f^{\prime}(x),y\right\rangle -\frac{1}{2L}\left\Vert f^{\prime}(y)-f^{\prime}(x)\right\Vert ^{2},
\]
\[
\therefore-f(y)\leq-f(x)+\left\langle f^{\prime}(x),x-y\right\rangle -\frac{1}{2L}\left\Vert f^{\prime}(y)-f^{\prime}(x)\right\Vert ^{2}.
\]
Negating gives the result.\end{proof}
\begin{rem}
This proof follows \citet{nes-book}. It can also be proved using
the techniques used in Theorem \ref{thm:strong-ub}, which gives a
less mysterious but more straightforward proof.
\end{rem}

\subsection{Inner product bounds}

\label{sec:ip-inequalities}

All (continuously-differentiable) convex functions satisfy the basic
inequality $\left\langle f^{\prime}(x)-f^{\prime}(y),x-y\right\rangle \geq0$
for any $x,y$ (See \citealp[Thm 2.1.3]{nes-book}). This inequality
can be strengthened for the classes of functions we consider, using
the following theorems.
\begin{thm}
\label{thm:ip-l}\citep[Thm 2.1.5]{nes-book} If $f$ is $L$-smooth
and convex, then for all $x,y\in\mathbb{R}^{d}$:
\[
\left\langle f^{\prime}(x)-f^{\prime}(y),x-y\right\rangle \geq\frac{1}{L}\left\Vert f^{\prime}(x)-f^{\prime}(y)\right\Vert ^{2}.
\]
\end{thm}
\begin{proof}
From Theorem \ref{thm:full-strong-lb} we have: $f(y)\geq f(x)+\left\langle f^{\prime}(x),y-x\right\rangle +\frac{1}{2L}\left\Vert f^{\prime}(x)-f^{\prime}(y)\right\Vert ^{2}$.
Add this to the same inequality but with $x$ and $y$ reversed and
we get:
\[
f(y)+f(x)\geq f(x)+f(y)+\left\langle f^{\prime}(x),y-x\right\rangle +\left\langle f^{\prime}(y),x-y\right\rangle +\frac{1}{L}\left\Vert f^{\prime}(x)-f^{\prime}(y)\right\Vert ^{2},
\]
\[
\therefore\left\langle f^{\prime}(x)-f^{\prime}(y),x-y\right\rangle \geq\frac{1}{L}\left\Vert f^{\prime}(x)-f^{\prime}(y)\right\Vert ^{2}.
\]
\end{proof}
\begin{thm}
\citep[Thm 2.1.9]{nes-book} If $f$ is strongly convex with constant
$\mu$, then for all $x,y\in\mathbb{R}^{d}$:
\[
\left\langle f^{\prime}(x)-f^{\prime}(y),x-y\right\rangle \geq\mu\left\Vert x-y\right\Vert ^{2}.
\]
\end{thm}
\begin{proof}
Same as previous theorem, using the strong convexity lower bound.
\end{proof}

\subsection{Strengthened bounds using both Lipschitz and strong convexity}
\begin{thm}
\label{thm:tight-ip}\citep[Thm 2.1.11]{nes-book} If $f$ is strongly
convex with constant $\mu$ and $L$-smooth, then for all $x,y\in\mathbb{R}^{d}$:
\[
\left\langle f^{\prime}(x)-f^{\prime}(y),x-y\right\rangle \geq\frac{\mu L}{\mu+L}\left\Vert x-y\right\Vert ^{2}+\frac{1}{\mu+L}\left\Vert f^{\prime}(x)-f^{\prime}(y)\right\Vert ^{2}.
\]
\end{thm}
\begin{proof}
Define the function $g$ as $g(x)=f(x)-\frac{\mu}{2}\left\Vert x\right\Vert ^{2}$.
Then the gradient is $g^{\prime}(x)=f^{\prime}(x)-\mu x$, and $g$
is $(L-\mu)$-smooth. From Theorem \ref{thm:ip-l} we have:
\[
\left\langle g^{\prime}(x)-g^{\prime}(y),x-y\right\rangle \geq\frac{1}{L-\mu}\left\Vert g^{\prime}(x)-g^{\prime}(y)\right\Vert ^{2}.
\]
Now replacing $g$ with $f$:

\[
\left\langle f^{\prime}(x)-f^{\prime}(y)-\mu\left(x-y\right),x-y\right\rangle \geq\frac{1}{L-\mu}\left\Vert f^{\prime}(x)-f^{\prime}(y)-\mu\left(x-y\right)\right\Vert ^{2}.
\]
So:

\begin{eqnarray*}
\left\langle f^{\prime}(x)\!-\!f^{\prime}(y),x-y\right\rangle  & \!\!\geq\!\! & \mu\left\Vert x-y\right\Vert ^{2}\!+\!\frac{1}{L\!-\!\mu}\left\Vert f^{\prime}(x)\!-\!f^{\prime}(y)\!-\!\mu\left(x\!-\!y\right)\right\Vert ^{2}\\
 & \!\!=\!\! & \mu\left\Vert x-y\right\Vert ^{2}\!+\!\frac{1}{L\!-\!\mu}\left\Vert f^{\prime}(x)\!-\!f^{\prime}(y)\right\Vert ^{2}\\
\!\! & \!\!\! & \!\!-\frac{2\mu}{L\!-\!\mu}\left\langle f^{\prime}(x)-f^{\prime}(y),x-y\right\rangle \!+\!\frac{\mu^{2}}{L\!-\!\mu}\left\Vert x-y\right\Vert ^{2}\\
\left(1+\frac{2\mu}{L-\mu}\right)\left\langle f^{\prime}(x)\!-\!f^{\prime}(y),x-y\right\rangle  & \!\!\geq\!\! & \left(\mu\!+\!\frac{\mu^{2}}{L\!-\!\mu}\right)\left\Vert x-y\right\Vert ^{2}\!+\!\frac{1}{L\!-\!\mu}\left\Vert f^{\prime}(x)\!-\!f^{\prime}(y)\right\Vert ^{2}\\
\frac{L\!-\!\mu\!+\!2\mu}{L\!-\!\mu}\left\langle f^{\prime}(x)\!-\!f^{\prime}(y),x-y\right\rangle  & \!\!\geq\!\! & \frac{(L\!-\!\mu)\mu\!+\!\mu^{2}}{L\!-\!\mu}\left\Vert x-y\right\Vert ^{2}\!+\!\frac{1}{L-\mu}\left\Vert f^{\prime}(x)\!-\!f^{\prime}(y)\right\Vert ^{2}\\
\frac{L+\mu}{L-\mu}\left\langle f^{\prime}(x)\!-\!f^{\prime}(y),x-y\right\rangle  & \!\!\geq\!\! & \frac{\mu L}{L-\mu}\left\Vert x-y\right\Vert ^{2}\!+\!\frac{1}{L-\mu}\left\Vert f^{\prime}(x)-f^{\prime}(y)\right\Vert ^{2}.
\end{eqnarray*}
Multiplying through by $\frac{L-\mu}{L+\mu}$ gives the result.\end{proof}
\begin{rem}
This inequality allows us to get the RHS from both inner product inequalities
from Section \ref{sec:ip-inequalities} added together, while only
losing about a multiplicative $\frac{\mu}{L}$ factor on each. We
can also achieve the same result by interpolating between the two
inequalities, but that loses a factor of $\frac{1}{2}$ for each.
Since the constants are more complex in this combined inequality,
it leads to more technical but stronger proofs.\end{rem}
\begin{thm}
\label{thm:full-strong-lb}If $f$ is strongly convex with constant
$\mu$ and $L$-smooth, then for all $x,y\in\mathbb{R}^{d}$:
\begin{eqnarray*}
f(x) & \geq & f(y)+\left\langle f^{\prime}(y),x-y\right\rangle +\frac{1}{2\left(L-\mu\right)}\left\Vert f^{\prime}(x)-f^{\prime}(y)\right\Vert ^{2}+\frac{\mu L}{2\left(L-\mu\right)}\left\Vert x-y\right\Vert ^{2}\\
 &  & +\frac{\mu}{\left(L-\mu\right)}\left\langle f^{\prime}(x)-f^{\prime}(y),y-x\right\rangle .
\end{eqnarray*}
\end{thm}
\begin{proof}
Define the function $g$ as $g(x)=f(x)-\frac{\mu}{2}\left\Vert x\right\Vert ^{2}$.
Then the gradient is $g^{\prime}(x)=f^{\prime}(x)-\mu x$. $g$ has
a Lipschitz continuous gradient with constant $L-\mu$. By convexity
we have:
\[
g(x)\geq g(y)+\left\langle g^{\prime}(y),x-y\right\rangle +\frac{1}{2L}\left\Vert g^{\prime}(x)-g^{\prime}(y)\right\Vert ^{2}.
\]

Now replacing $g$ with $f$:

\begin{eqnarray*}
f(x)-\frac{\mu}{2}\left\Vert x\right\Vert ^{2} & \geq & f(y)-\frac{\mu}{2}\left\Vert y\right\Vert ^{2}+\left\langle f^{\prime}(y)-\mu y,x-y\right\rangle \\
 &  & +\frac{1}{2\left(L-\mu\right)}\left\Vert f^{\prime}(x)-\mu x-f^{\prime}(y)+\mu y\right\Vert ^{2}.
\end{eqnarray*}
Note that

\begin{eqnarray*}
\frac{1}{2\left(L-\mu\right)}\left\Vert f^{\prime}(x)-\mu x-f^{\prime}(y)+\mu y\right\Vert ^{2} & = & \frac{1}{2\left(L-\mu\right)}\left\Vert f^{\prime}(x)-f^{\prime}(y)\right\Vert ^{2}\\
 &  & +\frac{\mu^{2}}{2\left(L-\mu\right)}\left\Vert y-x\right\Vert ^{2}\\
 &  & +\frac{\mu}{\left(L-\mu\right)}\left\langle f^{\prime}(x)-f^{\prime}(y),y-x\right\rangle ,
\end{eqnarray*}
so:

\begin{eqnarray*}
f(x) & \geq & f(y)+\left\langle f^{\prime}(y),x-y\right\rangle +\frac{1}{2\left(L-\mu\right)}\left\Vert f^{\prime}(x)-f^{\prime}(y)\right\Vert ^{2}+\frac{\mu^{2}}{2\left(L-\mu\right)}\left\Vert y-x\right\Vert ^{2}\\
 &  & +\frac{\mu}{2}\left\Vert x\right\Vert ^{2}-\frac{\mu}{2}\left\Vert y\right\Vert ^{2}+\frac{\mu}{\left(L-\mu\right)}\left\langle f^{\prime}(x)-f^{\prime}(y),y-x\right\rangle -s\left\langle y,x-y\right\rangle .
\end{eqnarray*}
Now using:

\[
\frac{\mu}{2}\left\Vert x\right\Vert ^{2}-\mu\left\langle y,x\right\rangle =-\frac{\mu}{2}\left\Vert y\right\Vert ^{2}+\frac{\mu}{2}\left\Vert x-y\right\Vert ^{2},
\]
we get:

\begin{eqnarray*}
f(x) & \geq & f(y)+\left\langle f^{\prime}(y),x-y\right\rangle +\frac{1}{2\left(L-\mu\right)}\left\Vert f^{\prime}(x)-f^{\prime}(y)\right\Vert ^{2}+\frac{\mu^{2}}{2\left(L-\mu\right)}\left\Vert x-y\right\Vert ^{2}\\
 &  & -\mu\left\Vert y\right\Vert ^{2}+\frac{\mu}{2}\left\Vert x-y\right\Vert ^{2}+\frac{\mu}{\left(L-\mu\right)}\left\langle f^{\prime}(x)-f^{\prime}(y),y-x\right\rangle +\mu\left\langle y,y\right\rangle .
\end{eqnarray*}
Note the norm $y$ terms cancel, and:
\begin{eqnarray*}
\frac{\mu}{2}\left\Vert x-y\right\Vert ^{2}+\frac{\mu^{2}}{2\left(L-\mu\right)}\left\Vert x-y\right\Vert ^{2} & = & \frac{(L-\mu)\mu+\mu^{2}}{2\left(L-\mu\right)}\left\Vert x-y\right\Vert ^{2}\\
 & = & \frac{\mu L}{2\left(L-\mu\right)}\left\Vert x-y\right\Vert ^{2}.
\end{eqnarray*}
So:

\begin{eqnarray*}
f(x) & \geq & f(y)+\left\langle f^{\prime}(y),x-y\right\rangle +\frac{1}{2\left(L-\mu\right)}\left\Vert f^{\prime}(x)-f^{\prime}(y)\right\Vert ^{2}+\frac{\mu L}{2\left(L-\mu\right)}\left\Vert y-x\right\Vert ^{2}\\
 &  & +\frac{\mu}{\left(L-\mu\right)}\left\langle f^{\prime}(x)-f^{\prime}(y),y-x\right\rangle .
\end{eqnarray*}
\end{proof}
\begin{rem}
This inequality uses the same proof technique as Theorem \ref{thm:tight-ip},
and indeed that theorem can be proved using this one, using a similar
proof technique as Theorem \ref{thm:ip-l}. I'm not aware of this
theorem appearing in the literature previously, although it is too
simple to be novel.\end{rem}
\begin{thm}
If $f$ is twice differentiable, strongly convex with constant $\mu$
and $L$-smooth, then for all $x,y\in\mathbb{R}^{d}$.
\end{thm}
\[
\left\Vert x-y+t\left(f^{\prime}(y)-f^{\prime}(x)\right)\right\Vert \leq\max\left\{ \left|1-tL\right|,\left|1-t\mu\right|\right\} \left\Vert x-y\right\Vert .
\]

\begin{proof}
We start by applying Corollary \ref{cor:calculus-corollary}: 
\[
f^{\prime}(y)-f^{\prime}(x)=\int_{0}^{1}f^{\prime\prime}\left(x+\tau(y-x)\right)\left(y-x\right)d\tau.
\]
Therefore:
\begin{eqnarray*}
\left\Vert x-y+t\left(f^{\prime}(y)-f^{\prime}(x)\right)\right\Vert  & = & \left\Vert x-y+t\int_{0}^{1}f^{\prime\prime}\left(x+\tau(y-x)\right)\left(y-x\right)d\tau\right\Vert \\
 & = & \left\Vert \int_{0}^{1}\left(tf^{\prime\prime}\left(x+\tau(y-x)\right)-I\right)\left(y-x\right)d\tau\right\Vert \\
 & \leq & \int_{0}^{1}\left\Vert \left(tf^{\prime\prime}\left(x+\tau(y-x)\right)-I\right)\left(y-x\right)\right\Vert d\tau\\
 & \leq & \int_{0}^{1}\left\Vert tf^{\prime\prime}\left(x+\tau(y-x)\right)-I\right\Vert \left\Vert x-y\right\Vert d\tau\\
 & \leq & \max_{z}\left\Vert tf^{\prime\prime}(z)-I\right\Vert \left\Vert x-y\right\Vert .
\end{eqnarray*}

Consider the eigenvalues of $f^{\prime\prime}(z)$. The minimum one
is at least $\mu$ and the maximum at least $L$. An examination of
the possible eigenvalues of $(tf^{\prime\prime}(z)-I)$ then gives
the result.\end{proof}
\begin{rem}
This lemma gives a straightforward proof of the convergence of gradient
descent for strongly convex problems. Suppose we use $x^{k+1}=x^{k}-\frac{2}{L+\mu}f^{\prime}(x^{k}).$
Then
\begin{eqnarray*}
\left\Vert x^{k+1}-x^{*}\right\Vert  & = & \left\Vert x^{k}-x^{*}+\frac{2}{L+\mu}\left(f^{\prime}(x^{*})-f^{\prime}(x^{k})\right)\right\Vert \\
 & \leq & \max\left\{ \left|1-\frac{2}{L+\mu}L\right|\,,\,\left|1-\frac{2}{L+\mu}\mu\right|\right\} \left\Vert x^{k}-x^{*}\right\Vert 
\end{eqnarray*}

Now note that 
\[
\left|1-\frac{2}{L+\mu}L\right|=\left|1-\frac{2L+2\mu-2\mu}{L+\mu}\right|=\left|1-2+\frac{2}{L+\mu}\mu\right|=\left|1-\frac{2}{L+\mu}\mu\right|
\]
 so the two parts of the max operation are balanced. The rate can
also be written as:
\[
1-\frac{2}{L+\mu}\mu=\frac{L+\mu}{L+\mu}-\frac{2\mu}{L+\mu}=\frac{L-\mu}{L+\mu}.
\]

So we have:
\[
\left\Vert x^{k}-x^{*}\right\Vert \leq\left(\frac{L-\mu}{L+\mu}\right)^{k}\left\Vert x^{0}-x^{*}\right\Vert .
\]
\end{rem}

\chapter{Miscellaneous Lemmas}

\label{chap:Misc}

In this Appendix we summarise a number of standard lemmas. These are
well known in some research circles and obscure in others. We omit
proofs when they are overly technical.
\begin{lem}
\label{lem:squared-triangle-inequality}For any vectors $a,b\in\mathbb{R}^{d}$,
and constant $\beta>0$:
\[
\left\Vert a+b\right\Vert ^{2}\leq(1+\beta)\left\Vert a\right\Vert ^{2}+\left(1+\frac{1}{\beta}\right)\left\Vert b\right\Vert ^{2}.
\]
\end{lem}
\begin{proof}
We start by expanding the quadratic:
\begin{eqnarray*}
\left\Vert a+b\right\Vert ^{2} & = & 2\left\langle a,b\right\rangle +\left\Vert a\right\Vert ^{2}+\left\Vert b\right\Vert ^{2}\\
 & = & 2\left\langle \sqrt{\beta}a,\frac{1}{\sqrt{\beta}}b\right\rangle +\left\Vert a\right\Vert ^{2}+\left\Vert b\right\Vert ^{2}.
\end{eqnarray*}
Now by expanding the quadratic $\left\Vert \sqrt{\beta}a+\frac{1}{\sqrt{\beta}}b\right\Vert ^{2}$
we get: $2\left\langle \sqrt{\beta}a,\frac{1}{\sqrt{\beta}}b\right\rangle \leq\beta\left\Vert a\right\Vert ^{2}+\frac{1}{\beta}\left\Vert b\right\Vert ^{2}$.
So combining we have:
\[
\left\Vert a+b\right\Vert ^{2}\leq(1+\beta)\left\Vert a\right\Vert ^{2}+\left(1+\frac{1}{\beta}\right)\left\Vert b\right\Vert ^{2}.
\]
\end{proof}
\begin{rem}
This lemma is a rather interesting strengthened version of the more
commonly used lemma $\left\Vert a+b\right\Vert ^{2}\leq2\left\Vert a\right\Vert ^{2}+2\left\Vert b\right\Vert ^{2}$.
It can be seen as the most natural generalisation of the triangle
inequality to squared norms. Since in optimisation theory we virtually
always work with squared norms, this kind of result is quite useful.\end{rem}
\begin{lem}
\label{lem:exp-log-bounds} For any $x\in\mathbb{R}$:
\[
\exp(1+x)\geq1+x,
\]
\[
\text{and for \ensuremath{x>-1}, }\log(1+x)\leq1+x.
\]
\end{lem}
\begin{proof}
We apply the convexity lower bound around $x=-1$ to $\exp(1+x)$:
\begin{eqnarray*}
\exp(1+x) & \geq & e^{0}+\left\langle e^{0},x+1\right\rangle \\
 & \geq & 1+x.
\end{eqnarray*}

Similarly, for $\log(1+x)$ we apply the concavity upper bound at
$x=0$:
\begin{eqnarray*}
\log(1+x) & \leq & \log(1)+\left\langle \frac{1}{1},x-0\right\rangle \\
 & = & x\\
 & \leq & 1+x.
\end{eqnarray*}
\end{proof}
\begin{lem}
\label{lem:bernoulli}(\textbf{Bernoulli's Inequality}) For any $\alpha\in[0,1)$
and integer $k\geq0$:
\[
(1-\alpha)^{k}\geq1-k\alpha.
\]
\end{lem}
\begin{proof}
This can easily be proved by induction. The base case $k=0$ is trivial.
Suppose it holds for $k$. Then:
\begin{eqnarray*}
(1-\alpha)^{k+1}=(1-\alpha)^{k}(1-\alpha) & \geq & (1-k\alpha)(1-\alpha)\\
 & = & 1-\alpha-k\alpha+k\alpha^{2}\\
 & = & 1-(k+1)\alpha+k\alpha^{2}\\
 & \geq & 1-(k+1)\alpha.
\end{eqnarray*}
\end{proof}
\begin{rem}
Notice that the terms that prevent this from being an equality are
of the form $\sum_{i}^{k}i\alpha^{2}\leq k^{2}\alpha^{2}$. If $\alpha\ll\frac{1}{k}$,
then it holds approximately as an equality.\end{rem}
\begin{lem}
\textbf{\label{lem:root-bernoulli}(Bernoulli's Inequality for roots)}
For any real numbers$\alpha\leq1$ and $r\in(0,1)$:

\[
(1-\alpha)^{r}\leq1-r\alpha.
\]

I.e. The direction of inequality for Bernoulli's inequality flips
when using roots instead of powers.\end{lem}
\begin{proof}
The functions $f(\alpha)=(1-\alpha)^{r}$ and $g(\alpha)=1-r\alpha$
have the same derivative at $0$. Since $g(\alpha)$ is linear and
$g(0)=f(0)$, it is the tangent line of $f(\alpha)$ at $\alpha=0$.
The $f(\alpha)$ function is concave for $r\in(0,1)$, so it is upper
bounded by its tangent, so $f(\alpha)\leq g(\alpha)$.\end{proof}
\begin{lem}
\label{lem:upper-bernoulli-bound}For any $\alpha$ and $k\geq0$:
\[
(1-\alpha)^{k}\leq\exp\left(-k\alpha\right).
\]
\end{lem}
\begin{proof}
In Lemma \ref{lem:exp-log-bounds}, we established that $\exp(1+x)\geq1+x$.
Raising both sides to the power $k$ will give the result. \end{proof}
\begin{rem}
This inequality bounds the same quantity as Bernoulli's inequality,
but with an upper bound instead of a lower bound. Like Bernoulli's
inequality this can hold with close to equality, but instead in the
setting where $\alpha$ is similar in magnitude to $\frac{1}{k}$;
indeed the error in such an approximation goes to zero asymptotically.
This is essentially the well known limiting definition of $e$:
\[
\left(1-\frac{1}{n}\right)^{n}\rightarrow\frac{1}{e},\,\text{and}\left(1+\frac{1}{n}\right)^{n}\rightarrow e.
\]
 This converges fairly quickly for our purposes. For example it is
only 2\% off at $n=30$, and for $n=1000$, we get $0.367\mathbf{7}\dots$
compare to $1/e=0.367\mathbf{8}\dots$. In optimisation proofs $n$
is approximately the amount of data, so $n\gg1000$ is typical. 

This lemma is convenient for converting bounds on the error that a
$k$ step procedure gives, to a bound on the number of steps required
to ensure at most some error level. For example, suppose we have a
bound on the error after $k$ steps of the form:
\[
\epsilon\leq\left(1-\alpha\right)^{k},
\]

Then using Lemma \ref{lem:upper-bernoulli-bound} we get $\epsilon\leq\exp\left(-\alpha k\right)$.
Now suppose we want to ensure our error is less than $\xi$ ($\epsilon\leq\xi$).
If we take $k\geq-\alpha^{-1}\log\xi$, then
\begin{eqnarray*}
\epsilon & \leq & \exp\left(-\alpha k\right)\\
 & \leq & \exp\left(\log\xi\right)\\
 & \leq & \xi,
\end{eqnarray*}
as required. These two ways of stating the convergence of an iterative
algorithm are essentially equivalent, and different literature uses
different styles. The lower bound $k$ style is more common when constants
are not considered important, as it is easily used in conjunction
with Big-O notation, like $k=O\left(-\alpha^{-1}\log\xi\right)$.
We use the explicit error bound style in this work as it is better
suited to the careful style of our analysis. We do not use Big-O notation
in our proofs, instead we explicitly compute the values of the extra
constants.
\end{rem}

\begin{lem}
\label{lem:decomposition-of-variance}(Decomposition of variance)

We can decompose $\mathbb{E}_{x}\left\Vert y-x\right\Vert ^{2}$ as:
\[
\mathbb{E}_{x}\left\Vert y-x\right\Vert ^{2}=\left\Vert y-\mathbb{E}_{x}[x]\right\Vert ^{2}+\mathbb{E}_{x}\left\Vert \mathbb{E}_{x}[x]-x\right\Vert ^{2}.
\]

Taking $y=0$ gives the useful equality:\textup{ $\mathbb{E}[\left\Vert x-\mathbb{E}[x]\right\Vert ^{2}]=\mathbb{E}[\left\Vert x\right\Vert ^{2}]-\left\Vert \mathbb{E}[x]\right\Vert ^{2}$.}\end{lem}
\begin{proof}
\begin{eqnarray}
\mathbb{E}_{x}\left\Vert y-x\right\Vert ^{2} & = & \mathbb{E}_{x}\left\Vert y-\mathbb{E}_{x}[x]+\mathbb{E}_{x}[x]-x\right\Vert ^{2}\nonumber \\
 & = & \left\Vert y-\mathbb{E}_{x}[x]\right\Vert ^{2}+\mathbb{E}_{x}\left\Vert \mathbb{E}_{x}[x]-x\right\Vert ^{2}+2\mathbb{E}_{x}\left\langle y-\mathbb{E}_{x}[x],\mathbb{E}_{x}[x]-x\right\rangle \\
 & = & \left\Vert y-\mathbb{E}_{x}[x]\right\Vert ^{2}+\mathbb{E}_{x}\left\Vert \mathbb{E}_{x}[x]-x\right\Vert ^{2}+2\left\langle y-\mathbb{E}_{x}[x],\mathbb{E}_{x}[x]-\mathbb{E}_{x}[x]\right\rangle \nonumber \\
 & = & \left\Vert y-\mathbb{E}_{x}[x]\right\Vert ^{2}+\mathbb{E}_{x}\left\Vert \mathbb{E}_{x}[x]-x\right\Vert ^{2}.\label{eq:main-cancelation}
\end{eqnarray}
\end{proof}

\bibliographystyle{anuthesis}
\addcontentsline{toc}{chapter}{\bibname}\bibliography{thesis,bibs/all}

\printindex{}
\end{document}